\documentclass[dvipsnames]{article} %
\usepackage{colm2024_conference}

\usepackage{booktabs}
\usepackage{enumitem}
\usepackage{wrapfig}
\usepackage{algorithm}
\usepackage{algpseudocode}
\usepackage{graphicx}
\usepackage[misc]{ifsym}
\usepackage{multicol} 
\usepackage{microtype}
\usepackage{colortbl}
\usepackage[utf8]{inputenc}
\usepackage[T1]{fontenc}
\definecolor{lightgray}{rgb}{0.9,0.9,0.9}
\usepackage{caption}
\usepackage{subcaption}
\usepackage{graphicx}
\usepackage{subcaption}
\usepackage{dsfont}
\usepackage{setspace}
\usepackage{url}
\usepackage{multirow}
\usepackage{tabularx}
\usepackage{blindtext}
\usepackage{pgfplots}
\pgfplotsset{compat=1.18} 
\usepackage{tikz}
\usetikzlibrary{er,positioning,bayesnet}
\usepackage{makecell}
\usepackage{tipa}
\usepackage{siunitx}
\usepackage{nicefrac}
\usepackage{listings}
\usepackage[raster,skins, most]{tcolorbox} %
\usepackage{xltabular}
\usepackage{adjustbox}
\usepackage{xurl}
\usepackage{rotating}
\usepackage[normalem]{ulem}

\usepackage{amsthm}

\theoremstyle{plain}
\newtheorem{theorem}{Theorem}[section]   
\newtheorem{lemma}[theorem]{Lemma}       
\newtheorem{proposition}[theorem]{Proposition}

\theoremstyle{definition}
\newtheorem{definition}[theorem]{Definition}

\theoremstyle{remark}
\newtheorem{remark}[theorem]{Remark}

\usepackage{placeins}

\usepackage{fontawesome}

\usepackage[table]{xcolor}

\usetikzlibrary{calc}

\useunder{\uline}{\ul}{}

\usepackage{amsmath,amsfonts,bm}

\def\eqref#1{equation~\ref{#1}}

\def\1{\bm{1}}

\DeclareMathAlphabet{\mathsfit}{\encodingdefault}{\sfdefault}{m}{sl}
\SetMathAlphabet{\mathsfit}{bold}{\encodingdefault}{\sfdefault}{bx}{n}

\newcommand{\KL}{D_{\mathrm{KL}}}

\newcommand*\justify{%
  \fontdimen2\font=0.4em
  \fontdimen3\font=0.2em
  \fontdimen4\font=0.1em
  \fontdimen7\font=0.1em
  \hyphenchar\font=`\-
}

\renewcommand{\texttt}[1]{%
  \begingroup
  \ttfamily
  \begingroup\lccode`~=`/\lowercase{\endgroup\def~}{/\discretionary{}{}{}}%
  \begingroup\lccode`~=`[\lowercase{\endgroup\def~}{[\discretionary{}{}{}}%
  \begingroup\lccode`~=`.\lowercase{\endgroup\def~}{.\discretionary{}{}{}}%
  \catcode`/=\active\catcode`[=\active\catcode`.=\active
  \justify\scantokens{#1\noexpand}%
  \endgroup
}

\usepackage{makecell}
\usetikzlibrary{tikzmark}
\makeatletter
\newcommand*\myfontsize{%
  \@setfontsize\myfontsize{7}{8}%
}
\makeatother

\definecolor{uclablue}{RGB}{159, 195, 224}

\definecolor{uclagold}{RGB}{255, 240, 180}

\definecolor{aliceblue}{RGB}{255, 238, 241}

\definecolor{cadmiumgreen}{rgb}{0.0, 0.42, 0.24}

\definecolor{myred}{rgb}{0.7, 0.3, 0.0}
\definecolor{myblue}{rgb}{0.2, 0.3, 0.6}
\definecolor{babygreen}{rgb}{0.85, 0.97, 0.85}

\definecolor{purple1}{RGB}{126, 107, 196}
\definecolor{purple2}{RGB}{199, 158, 207}
\definecolor{purple3}{RGB}{214, 200, 255}
\definecolor{purple4}{RGB}{254, 240, 255}

\definecolor{deepblue}{RGB}{48, 58, 82}

\definecolor{SoftLavender}{HTML}{F3EEFF}

\newcommand{\symboletongyi}{\raisebox{0pt}{~\includegraphics[scale=0.012]{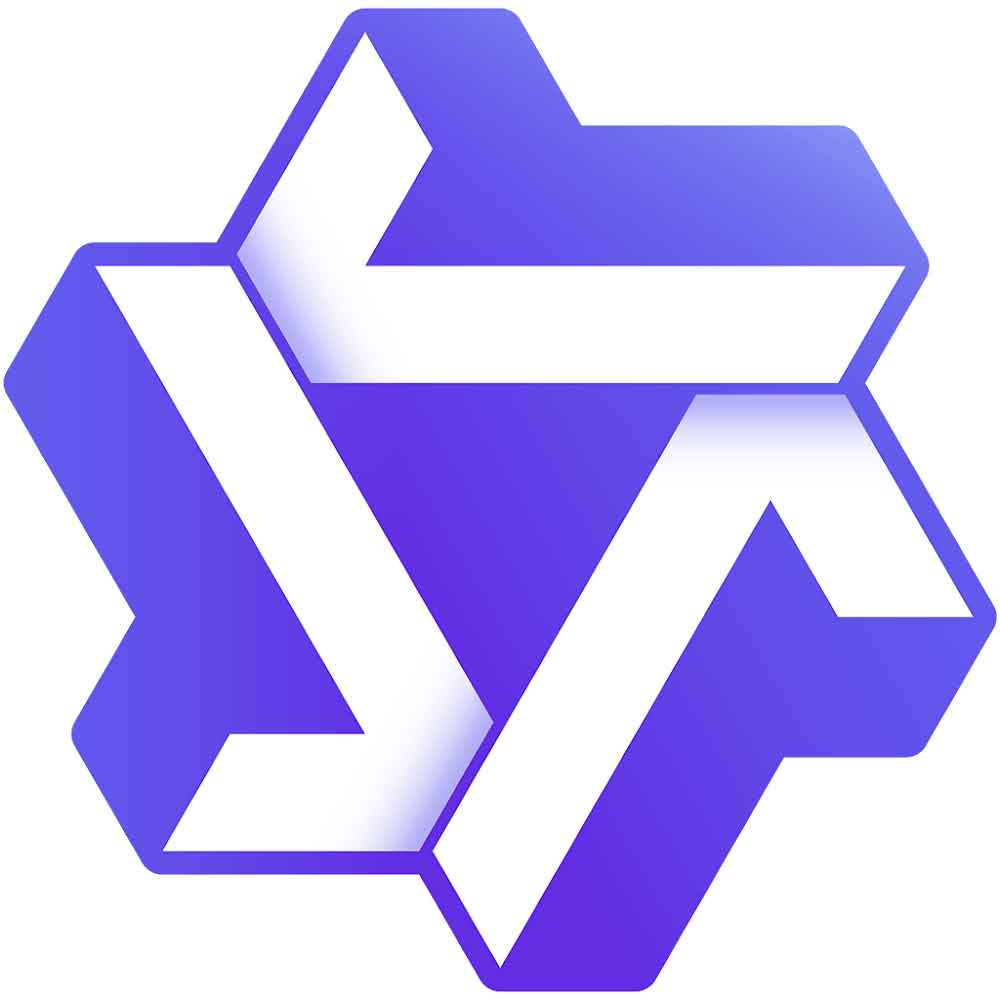}}~}

\definecolor{deepPurple}{HTML}{330066}

\definecolor{uclablue_old}{rgb}{0.15, 0.45, 0.68}
\hypersetup{
    breaklinks,
    citecolor=uclablue_old,
    colorlinks=true,
}
\usepackage[para]{footmisc}

\definecolor{basegray}{RGB}{160,165,175}
\definecolor{basegraydark}{RGB}{110,115,125}
\definecolor{rlblue}{RGB}{47,128,237}
\definecolor{rlbluefill}{RGB}{220,235,255}

\newtcolorbox{mybox}[2][]
  {colback = black!5!white, colframe = black!75!black, fonttitle = \bfseries,
    colbacktitle = black!100!black, enhanced, before upper={\fontsize{8}{11}\obeyspaces\obeylines\selectfont}, fontupper=\selectfont,
    attach boxed title to top left={yshift=-2.2mm,xshift=4mm},
    title=#2,#1}
\usetikzlibrary{arrows.meta,calc,positioning,backgrounds}
\newcommand{\js}{\mathrm{JS}}

\newcommand{\EE}{\mathbb{E}}
\newcommand{\JS}{D_{\mathrm{JS}}}

\newcounter{daggerfootnote}
\newcommand*{\daggerfootnote}[1]{%
    \setcounter{daggerfootnote}{\value{footnote}}%
    \renewcommand*{\thefootnote}{\fnsymbol{footnote}}%
    \footnote[2]{#1}%
    \setcounter{footnote}{\value{daggerfootnote}}%
    \renewcommand*{\thefootnote}{\arabic{footnote}}%
    }

\title{%
\begin{tabular}[t]{l} 
  \parbox[t]{0.8\textwidth}{\centering 
    Sparse but Critical: A Token-Level Analysis of Distributional Shifts in RLVR Fine-Tuning of LLMs
  }
\end{tabular}
}

\begingroup
\renewcommand{\thefootnote}{}
\footnotetext{Published at ICLR 2026.}
\endgroup
\author{%
\large \symboletongyi Qwen Pilot Team, Alibaba Group \thanks{Full author list available in the \hyperref[sec:authors]{Authors} section.}%
  \\[1em] 
}

\begin{document}

\maketitle

\begin{abstract}
    Reinforcement learning with verifiable rewards (RLVR) has significantly improved reasoning in large language models (LLMs), yet the token-level mechanisms underlying these improvements remain unclear. We present a systematic empirical study of RLVR’s distributional effects organized around three main analyses: (1) token-level characterization of distributional shifts between base and RL models, (2) the impact of token-level distributional shifts on sequence-level reasoning performance through cross-sampling interventions, and (3) fine-grained mechanics of these shifts at the token level. We find that RL fine-tuning induces highly sparse and targeted changes, with only a small fraction of token distributions exhibiting meaningful divergence between the base and RL policies. We further characterize the structure and evolution of these shifts through analyses of token entropy, positional concentration, and reallocation of probability mass. To assess the functional importance of these sparse changes, we conduct cross-sampling experiments that selectively swap token choices between the base and RL models with varying intervention budgets. We show that inserting only a small fraction of RL-sampled tokens into base generations progressively recovers RL performance gains, while injecting a similarly small number of base token choices into otherwise RL-generated sequences collapses performance to base levels, isolating a small set of token-level decisions directly responsible for RLVR's performance gains. Finally, we explore divergence-weighted variants of the advantage signal as a diagnostic intervention, finding that they can yield improvements over baselines. Together, our results shed light on the distributional changes induced by RLVR and provide a fine-grained, token-level lens for understanding RLVR fine-tuning as a targeted refinement process.
\end{abstract}

\begin{figure*}[t]
\centering
\begin{subfigure}[t]{0.48\textwidth}
    \centering
        \includegraphics[width=0.98\linewidth]{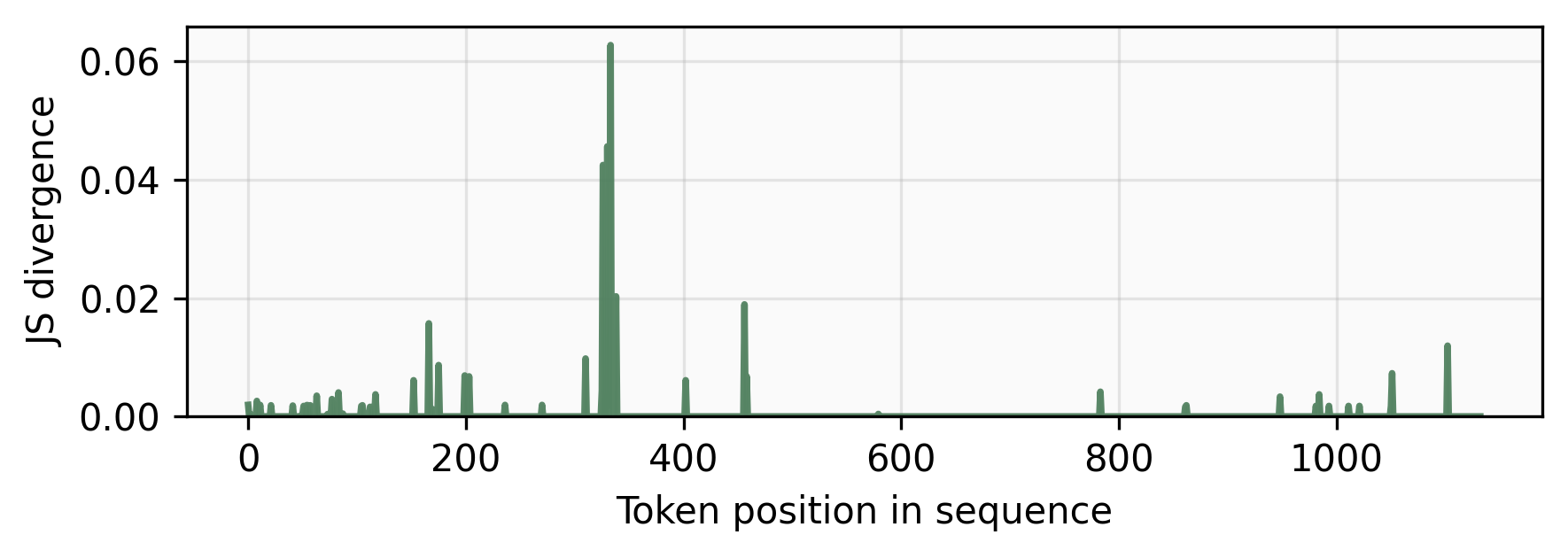}
    \caption{Divergence across a sequence (SimpleRL)}
    \label{fig:overview_divergence_simplerl}
\end{subfigure}
\hfill
\begin{subfigure}[t]{0.48\textwidth}
    \centering
        \includegraphics[width=0.98\linewidth]{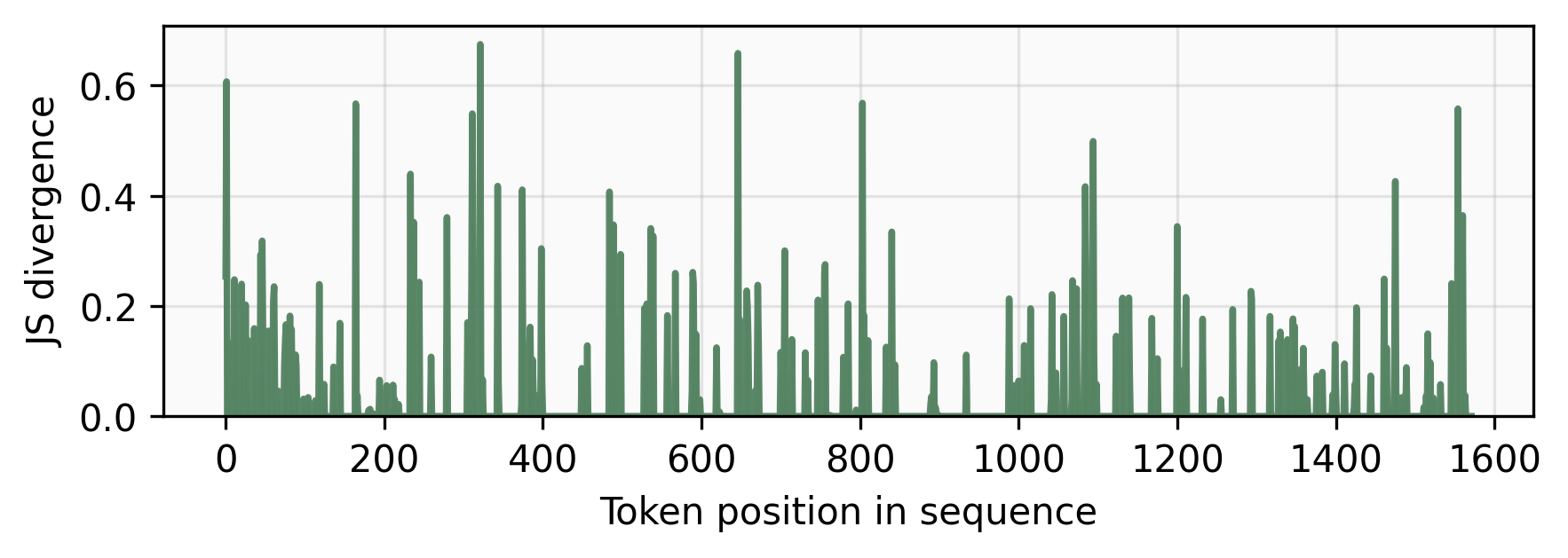}
    \caption{Divergence across a sequence (DAPO)}
    \label{fig:overview_divergence_dapo}
\end{subfigure}

\begin{subfigure}[t]{0.7\textwidth}
    \centering
    \vspace{1em}
        \resizebox{\linewidth}{!}{
    \begin{tikzpicture}[
    x=1cm,
    y=1cm,
    >={Latex[length=2.8mm]},
    line cap=round,
    line join=round,
    basepath/.style={draw=basegray, line width=1.25pt},
    rlpath/.style={draw=rlblue, line width=1.55pt},
    basenode/.style={
        circle,
        draw=basegray,
        line width=1pt,
        fill=white,
        minimum size=8.2mm,
        inner sep=0pt
    },
    rlnode/.style={
        circle,
        draw=rlblue,
        line width=1.35pt,
        fill=rlbluefill,
        minimum size=8.8mm,
        inner sep=0pt
    },
    every node/.style={font=\small}
]

\newcommand{\tinybars}[3]{%
    \begin{scope}[shift={({#1},{#2})}]
        \fill[#3] (-0.16,-0.16) rectangle (-0.10, 0.00);
        \fill[#3] (-0.07,-0.16) rectangle (-0.01, 0.10);
        \fill[#3] ( 0.02,-0.16) rectangle ( 0.08, 0.20);
        \fill[#3] ( 0.11,-0.16) rectangle ( 0.17, 0.05);
    \end{scope}
}

\coordinate (p0) at (0,0);
\coordinate (p1) at (2.0,0);
\coordinate (p2) at (4.0,0);
\coordinate (p3) at (6.0,0);
\coordinate (p4) at (8.0,0);
\coordinate (p5) at (10.0,0);
\coordinate (p6) at (12.0,0);

\coordinate (r1)  at (4.8,-1.7);
\coordinate (r1a) at (7.0,-1.7);
\coordinate (r1b) at (9.2,-1.7);

\coordinate (r2)  at (9.9,-3.4);
\coordinate (r2a) at (12.1,-3.4);
\coordinate (r2b) at (14.3,-3.4);

\node[align=center, anchor=east] at ($(p0)+(-0.55,0)$) {\textbf{Prompt}};
\node[align=center, anchor=west] at ($(p6)+(0.55,0)$) {\textbf{Answer}};
\node[align=center, anchor=west] at ($(r2b)+(0.55,0)$) {\textbf{Answer}};

\draw[basepath, ->] (p0) -- (p1);
\draw[basepath, ->] (p1) -- (p2);
\draw[basepath, ->] (p2) -- (p3);
\draw[basepath, ->] (p3) -- (p4);
\draw[basepath, ->] (p4) -- (p5);
\draw[basepath, ->] (p5) -- (p6);

\node[basenode] (n0) at (p0) {};
\node[basenode] (n1) at (p1) {};
\node[basenode] (n2) at (p2) {};
\node[basenode] (n3) at (p3) {};
\node[basenode] (n4) at (p4) {};
\node[basenode] (n5) at (p5) {};
\node[basenode] (n6) at (p6) {};

\tinybars{0}{0}{basegraydark}
\tinybars{2.0}{0}{basegraydark}
\tinybars{4.0}{0}{basegraydark}
\tinybars{6.0}{0}{basegraydark}
\tinybars{8.0}{0}{basegraydark}
\tinybars{10.0}{0}{basegraydark}
\tinybars{12.0}{0}{basegraydark}

\draw[rlpath, ->] (p2) to[out=-28,in=172] (r1);
\node[rlnode] (nr1) at (r1) {};
\tinybars{4.8}{-1.7}{rlblue}

\draw[basepath, ->] (r1) -- (r1a);
\draw[basepath, ->] (r1a) -- (r1b);

\node[basenode] (nr1a) at (r1a) {};
\node[basenode] (nr1b) at (r1b) {};
\tinybars{7.0}{-1.7}{basegraydark}
\tinybars{9.2}{-1.7}{basegraydark}

\draw[rlpath, ->] (r1b) to[out=-28,in=172] (r2);
\node[rlnode] (nr2) at (r2) {};
\tinybars{9.9}{-3.4}{rlblue}

\draw[basepath, ->] (r2) -- (r2a);
\draw[basepath, ->] (r2a) -- (r2b);

\node[basenode] (nr2a) at (r2a) {};
\node[basenode] (nr2b) at (r2b) {};
\tinybars{12.1}{-3.4}{basegraydark}
\tinybars{14.3}{-3.4}{basegraydark}

\begin{scope}[on background layer]
    \fill[rlblue, opacity=0.07] (r1) circle [radius=0.38];
    \fill[rlblue, opacity=0.07] (r2) circle [radius=0.38];
\end{scope}

\node[rlblue, font=\bfseries\footnotesize, align=center] 
    at ($(r1)+(0,-0.78)$) {RL edit};

\node[rlblue, font=\bfseries\footnotesize, align=center] 
    at ($(r2)+(0,-0.78)$) {RL edit};

\node[basegraydark, font=\footnotesize, align=center] 
    at ($(p4)+(0,0.82)$) {base trajectory};

\node[basegraydark, font=\footnotesize, align=center] 
    at ($(r1a)+(0,0.82)$) {continue with base};

\node[basegraydark, font=\footnotesize, align=center] 
    at ($(r2a)+(0,0.82)$) {continue with base};

\begin{scope}[shift={(0.2,1.9)}]
    \node[basenode, scale=0.82] at (2.0,0) {};
    \tinybars{2.0}{0}{basegraydark}
    \node[anchor=west, font=\footnotesize] at (2.38,0) {Base token distribution};

    \node[rlnode, scale=0.82] at (7.0,0) {};
    \tinybars{7.0}{0}{rlblue}
    \node[anchor=west, font=\footnotesize] at (7.38,0) {RL-modified distribution};
\end{scope}

\end{tikzpicture}
    }
    \caption{Visualization of RLVR as a sparse trajectory steering mechanism}
    \label{fig:intuitive}
\end{subfigure}

\caption{
\textbf{Overview: RLVR acts as sparse, high-impact token-level refinement.}
RL fine-tuning induces sparse distributional shifts: divergence between base and RL token distributions remains near zero at most positions, with only a small subset of tokens exhibiting substantial changes.
}
\label{fig:overview}
\end{figure*}
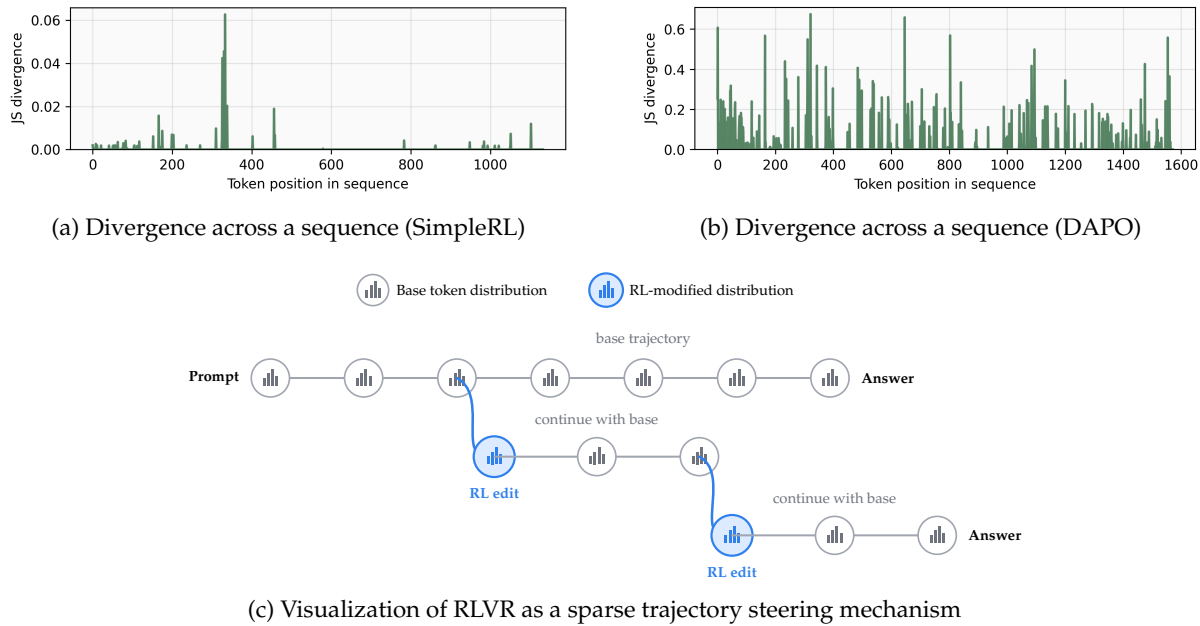

\section{Introduction}
\label{sec:intro}

Recent advances in reinforcement learning with verifiable rewards (RLVR) \citep{rlvr} for reasoning in large language models (LLMs), such as Group Relative Policy Optimization (GRPO) \citep{shao2024deepseekmath}, have enabled substantial performance improvements on challenging reasoning and mathematical benchmarks. Despite this empirical success, the mechanisms through which RLVR modifies model behavior remain unclear. Most evaluations of RL fine-tuning focus on aggregate response-level metrics such as accuracy, reward, and response length. While informative, these provide only a high-level view of improvement and offer limited insight into the mechanisms by which model behavior changes. In particular, a central unresolved question is: \textit{how does RLVR reshape the token-level predictive distributions of a base model, and which of these changes actually drive downstream reasoning gains?}

Recent work has begun to analyze RL fine-tuning through token-level entropy and uncertainty perspectives \citep{wang2025high_entropy_minority_tokens, cheng2025entropy_exploration_rlvr, cui2025entropymechanismreinforcementlearning}, highlighting the role of high-entropy tokens and exploration dynamics. Complementary analyses study RL-induced changes through token-level KL divergence and rank-shift statistics \citep{huan2025doesmathreasoningimprove} as well as through the perspective of reasoning patterns \citep{chen2026reshaping}. However, a more detailed distributional view of change remains missing: how such shifts are structured across positions and contexts, how probability mass is reallocated across candidate tokens, how they evolve over training, and to what extent they are responsible for RLVR's performance gains.

In this paper, we develop a fine-grained, token-level perspective on RLVR through the lens of distributional change. We perform a systematic empirical study of how RLVR alters next-token distributions relative to the base model, and connect these distributional shifts directly to  sequence-level reasoning performance. Our analyses reveal that RLVR acts primarily as a sparse and targeted refinement process: most token distributions remain nearly unchanged, while a small subset of high-divergence positions carries disproportionate functional importance, guiding generation toward more effective reasoning trajectories otherwise accessible under the base model.

Our contributions are organized as follows:

\begin{itemize}

\item \textbf{Structure of Token-Level Distributional Shifts.}
We show that RLVR induces sparse token-level distribution shifts relative to the base model. We characterize the structure of these shifts through divergence, entropy, and positional analyses, and compare across multiple RLVR methods, revealing differences in exploration and refinement behavior.

\item \textbf{Cross-Sampling Interventions.}
We use forward and reverse cross-sampling interventions to measure the role of divergent token decisions. We show that modifying only a small fraction of token choices is sufficient to recover (in base-model generations) or erase (in RL-model generations) RLVR performance gains, linking the sparse distributional shifts directly to sequence-level reasoning outcomes. These results demonstrate that RL and base policies are behaviorally similar across most tokens but differ critically at a sparse set of high-impact decisions that steer reasoning trajectories.

\item \textbf{Fine-Grained Distribution Mechanics.}
We analyze how RLVR modifies token distributions at high-divergence positions and show that it primarily reallocates probability mass within an existing candidate set rather than introducing new tokens. We support this with top-$k$ overlap, rank, tail-probability, and training-evolution analyses.

\item \textbf{Divergence-Weighted Advantage.}
Motivated by these findings, we study divergence-weighted variants of the RLVR advantage signal as a diagnostic objective modification and show that they can improve over baselines.

\end{itemize}

Taken together, our results provide a unified token-level picture of RLVR fine-tuning: rather than globally rewriting model behavior, RLVR predominantly performs sparse, structured probability reallocation in a small set of critical token positions that steer downstream reasoning trajectories. This distributional and functional perspective helps clarify the mechanisms in which RLVR improves reasoning in LLMs.

\section{Token Distribution Analysis between Base and Fine-tuned Models}
\label{sec:sparse_shifts}

We begin by analyzing the general structure of distributional shifts induced by RLVR, with the goal of characterizing how token-level predictions differ between the base model and its RL-finetuned counterpart. Our analysis compares next-token distributions under identical sequence contexts: we take sequences generated by the RL policy and evaluate both models’ conditional distributions at each token position. This framing treats the RL-generated trajectory as a reference path and allows us to quantify how the base model would need to adapt in order to emulate it.

\subsection{Preliminaries}
For each token position $t$ and prefix $x_{<t}$, let $\pi_{\text{base}}(\cdot \mid x_{<t})$ and $\pi_{\text{RL}}(\cdot \mid x_{<t})$ denote the conditional next-token distributions of the base and RL models, respectively, defined on a vocabulary space $\mathcal{V}$. To quantify distributional differences, we use the Jensen--Shannon (JS) divergence, defined as
\[
\mathrm{JS}(x_{<t}) = D_{\mathrm{JS}}( \pi_{\text{base}}(\cdot \mid x_{<t})\parallel \pi_{\text{RL}}(\cdot \mid x_{<t}) )
= \tfrac{1}{2} D_{\mathrm{KL}}\!\left( \pi_{\text{base}}(\cdot \mid x_{<t}) \parallel  M_t \right) + \tfrac{1}{2} D_{\mathrm{KL}}\!\left( \pi_{\text{RL}}(\cdot \mid x_{<t}) \parallel  M_t \right),
\]
where $M_t = \tfrac{1}{2}\big( \pi_{\text{base}}(\cdot \mid x_{<t}) + \pi_{\text{RL}}(\cdot \mid x_{<t}) \big)$.

One could use any notion of divergence or distance between probability measures, but we use JS divergence over something like KL divergence because: (i) it is symmetric, avoiding directional considerations; (ii) it is bounded in $[0, \log 2]$, preventing extreme values from dominating aggregate statistics; and (iii) it remains well-defined even when the measures lack absolute continuity with respect to each other. The latter is particularly important in practice, as memory constraints often limit the retrieval of the full distribution over the entire vocabulary, and also when comparing top-$p$ truncated distributions, for which KL divergence may be undefined.

Unless otherwise stated, divergences are computed on top-$p$ truncated distributions using the same sampling configuration employed during generation, while entropy and probabilities are from the full estimated distributions. This is so that the comparisons reflect the models’ effective differences under the actual sampling regime, while still grounding the entropy and probabilities in their complete output distributions. Robustness checks across different top-$p$ values and against estimates of the full distributions are provided in Appendix~\ref{subsec:additional_distributions} (Figures~\ref{fig:js_topp_comparison} and~\ref{fig:js_topp1_comparison}).

\paragraph{Models and Datasets.} Our primary analysis focuses on Qwen2.5-32B \citep{qwen2025qwen25technicalreport} as the base model, with RLVR variants trained using DAPO \citep{yu2025dapo_system} and GRPO, the latter paired with the corresponding SimpleRL model \citep{zeng2025simplerl_zoo}. For evaluation on AIME 2024 and AIME 2025, we sample 32 responses per problem for robustness. We further extend the analysis to additional models (Qwen2.5-Math-7B \citep{yang2024qwen25mathtechnicalreportmathematical} with two variants corresponding to different upper clip settings, Qwen3-8B-Base \citep{yang2025qwen3technicalreport} with DAPO,  and Mistral-Small-24B \citep{mistral2025small3}) with SimpleRL, datasets (AIME25, GPQA \citep{rein2023gpqagraduatelevelgoogleproofqa}, and the models' respective fine-tuning datasets), and to comparisons with supervised fine-tuning (SFT). These extensions, reported in Appendix~\ref{subsec:rlvr_vs_sft} and Appendix~\ref{subsec:additional_distributions}, confirm that our findings generalize across models, datasets, and training paradigms. 

\subsection{Distribution Shifts Are Highly Targeted and Sparse}
\label{subsec:sparse_js}

A natural starting question is: \emph{how broadly are distributional shifts distributed across token positions?} To answer this, we examine the token-level JS divergence between the base and RL-finetuned models. Figure~\ref{fig:js_combined} presents log-scaled histograms and percentile curves of JS divergence for DAPO and SimpleRL on their respective generated responses for AIME 2024.

\begin{figure}[!htbp]
    \centering
    \begin{subfigure}{0.46\textwidth}
        \includegraphics[width=\linewidth]{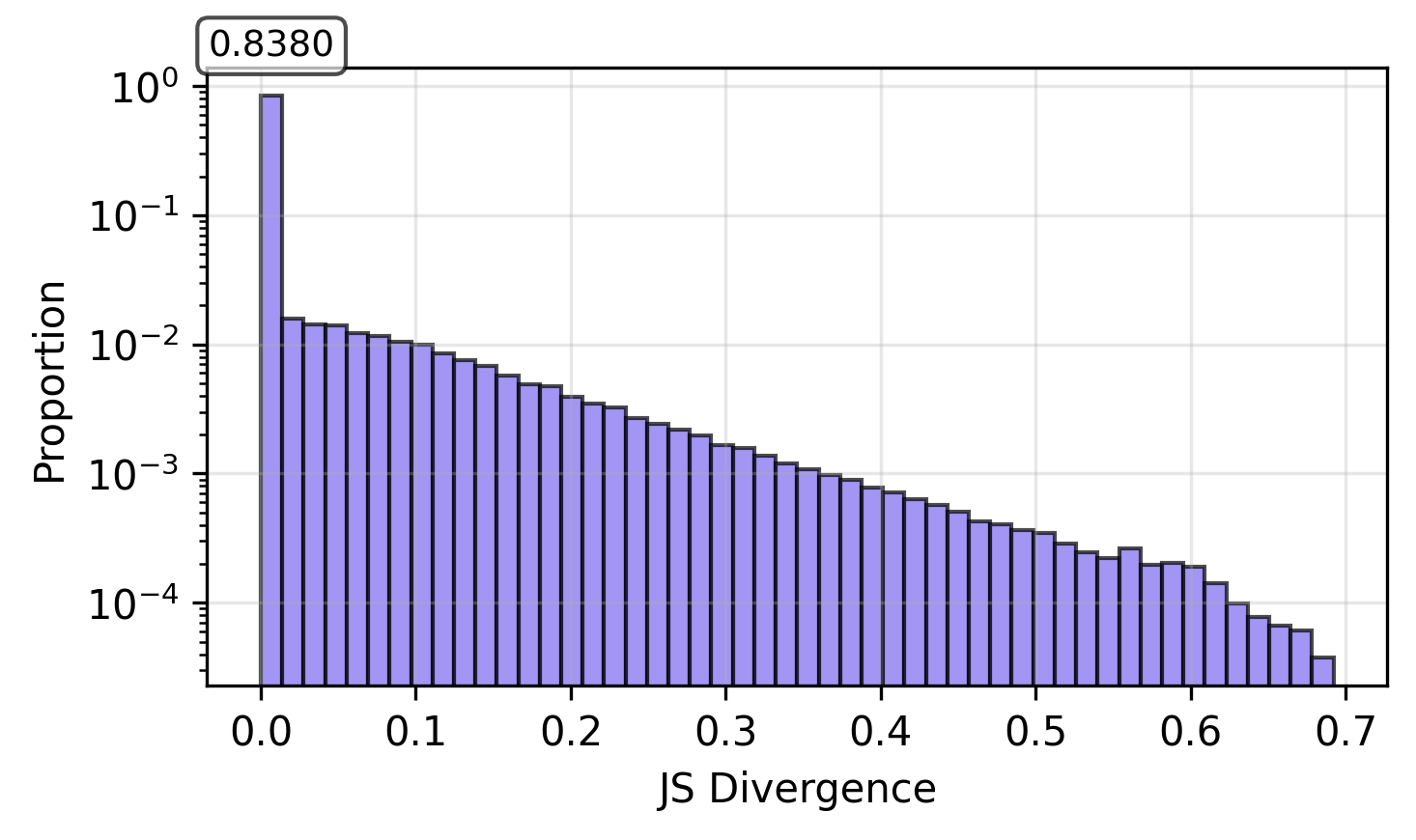}
        \caption{DAPO: Histogram (log y-axis)}
    \end{subfigure}
    \hfill
    \begin{subfigure}{0.46\textwidth}
        \includegraphics[width=\linewidth]{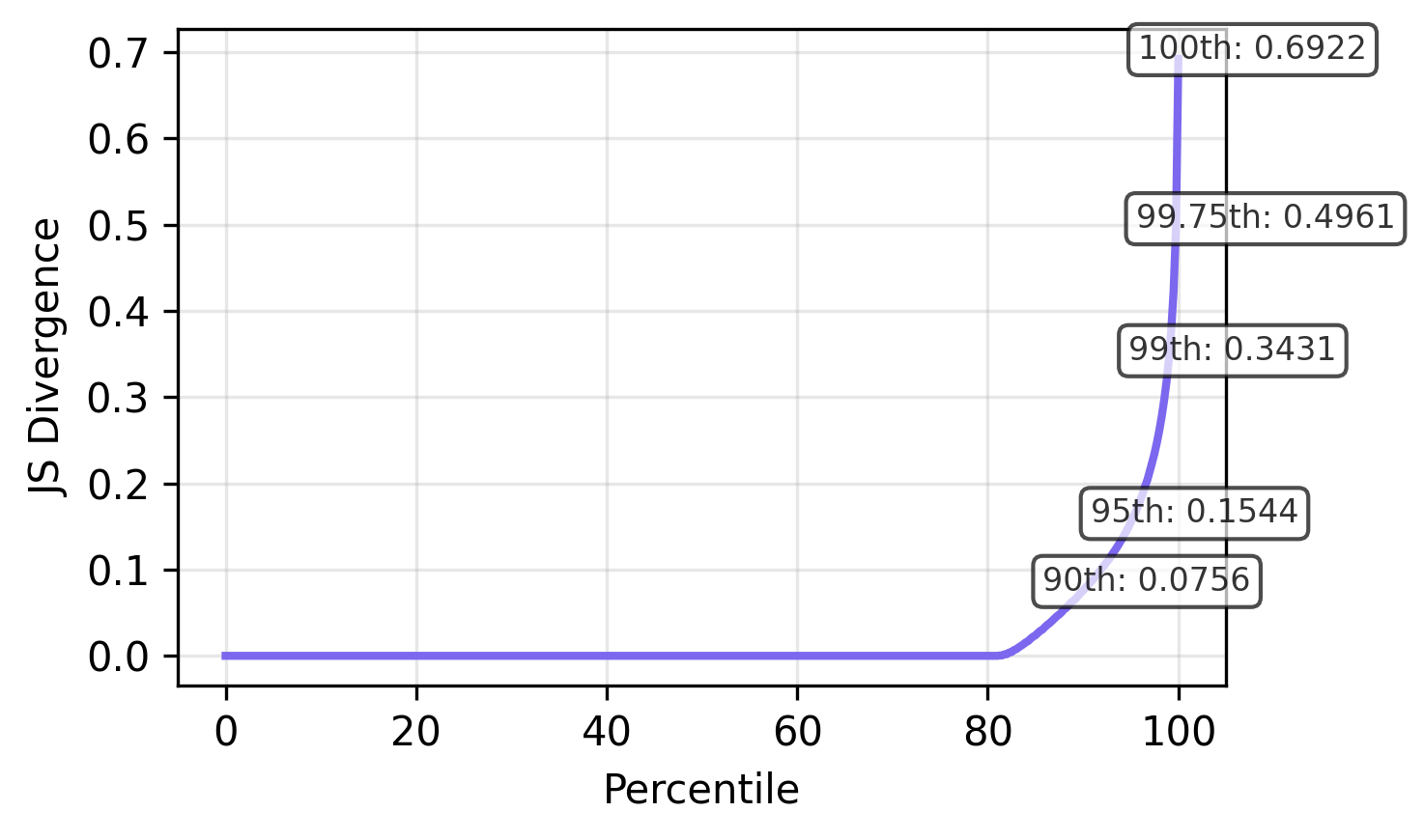}
        \caption{DAPO: Percentile curve}
    \end{subfigure}
    
    \vspace{1em}
    \begin{subfigure}{0.46\textwidth}
        \includegraphics[width=\linewidth]{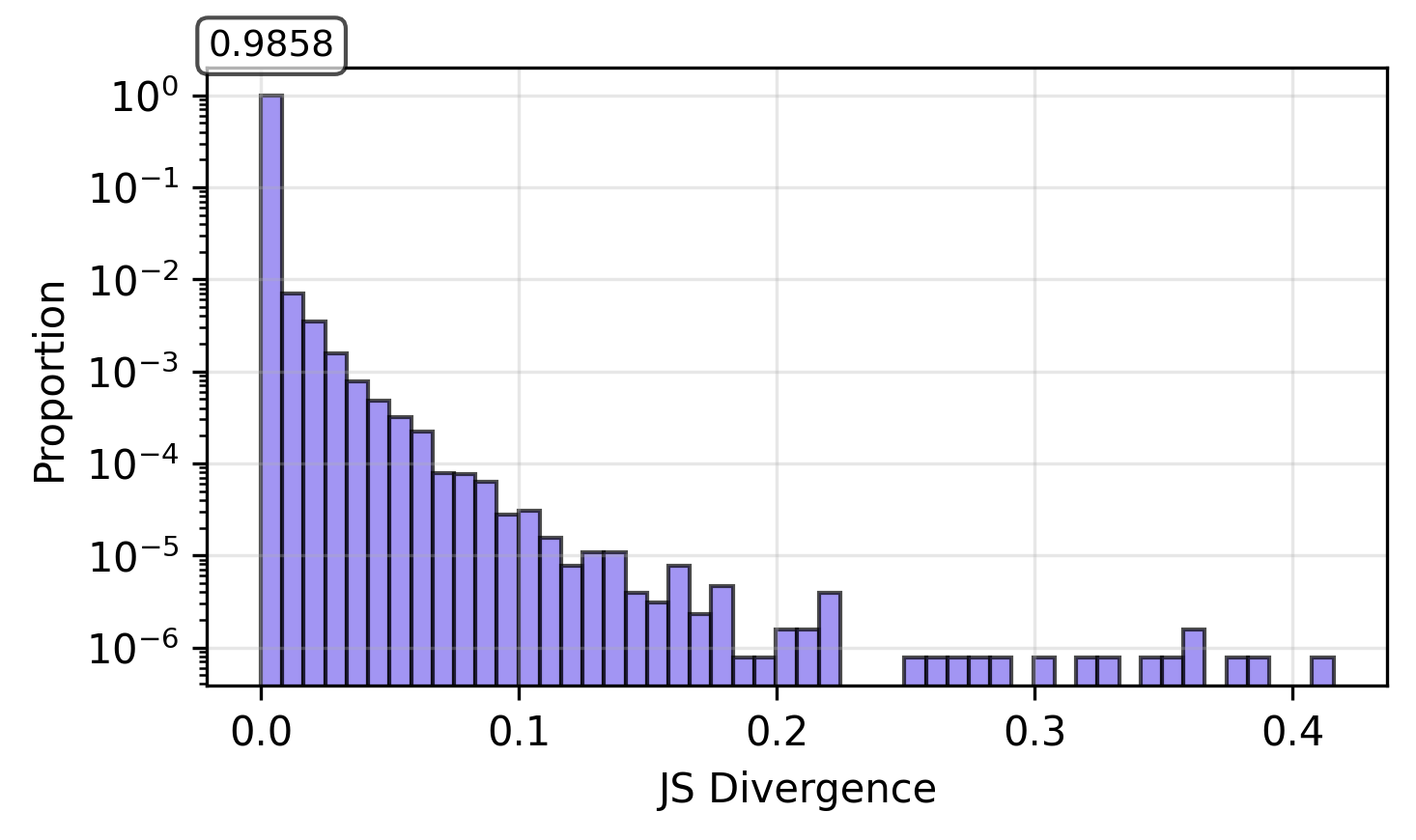}
        \caption{SimpleRL: Histogram (log y-axis)}
    \end{subfigure}
    \hfill
    \begin{subfigure}{0.46\textwidth}
        \includegraphics[width=\linewidth]{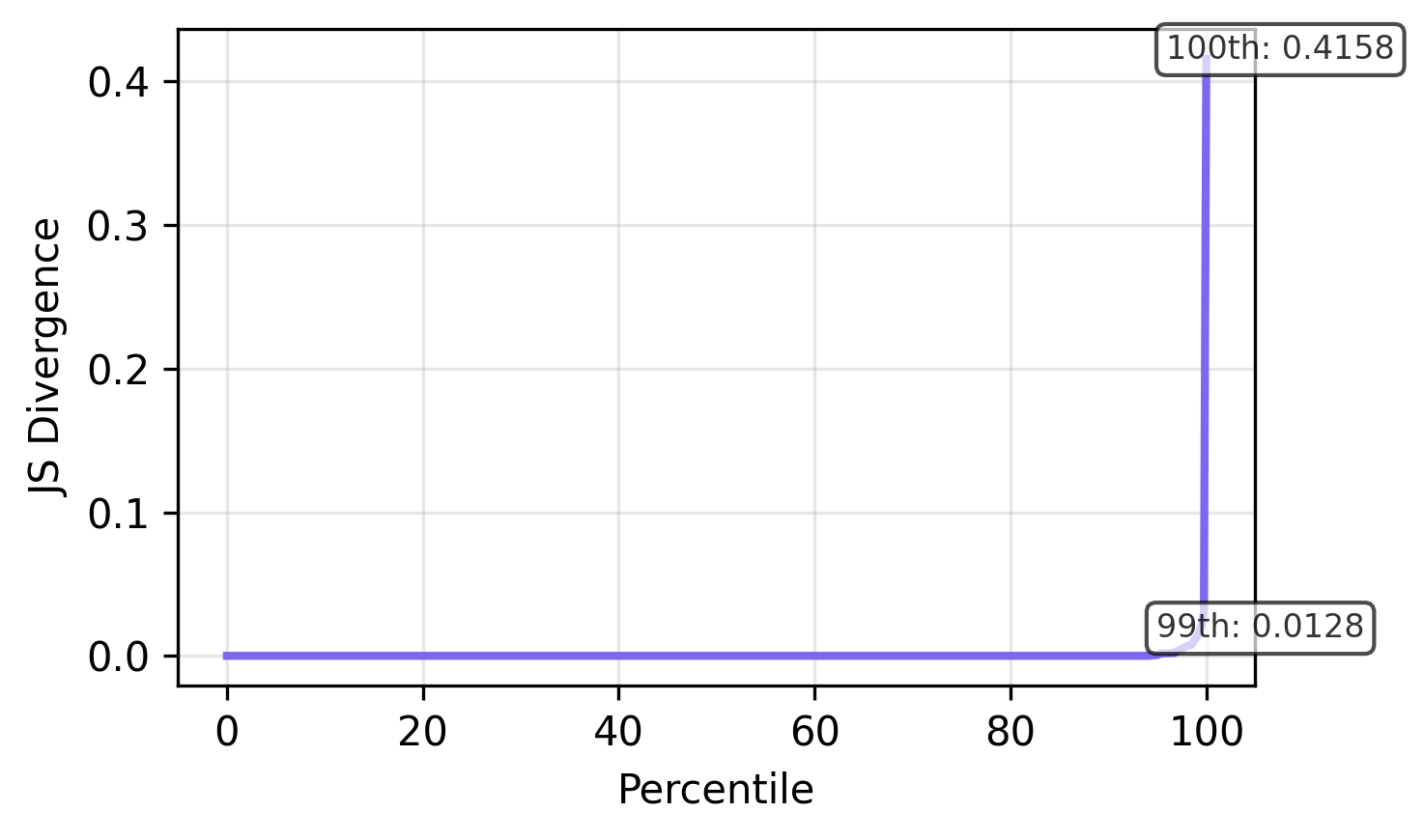}
        \caption{SimpleRL: Percentile curve}
    \end{subfigure}

    \caption{
        JS divergence distributions for Qwen2.5 32B DAPO and SimpleRL on AIME 2024.
    }
    \label{fig:js_combined}
\end{figure}

The results reveal that \textbf{RLVR refinement is \emph{highly sparse} at the token distribution level.} Under DAPO, more than 83\% of token positions exhibit near-zero divergence, while this proportion exceeds 98\% under SimpleRL. The clear spike at zero on the histograms and the steep rise of the percentile curves indicate that only a small subset of token positions undergo substantial distributional change as a result of RLVR. Comparing the two, DAPO exhibits a broader divergence distribution and a more gradual percentile curve, consistent with its clip-higher mechanism and lack of KL regularization, permitting broader exploratory updates. In contrast, SimpleRL imposes stricter constraints, resulting in more tightly concentrated changes. Importantly, even in the absence of KL regularization, the DAPO policy maintains near-zero divergence at most token distributions.

For a more controlled comparison for models fine-tuned on the same dataset, Appendix~\ref{subsec:dapo_clip_comparison} presents the results for Qwen2.5-Math-7B trained with DAPO, comparing upper clip settings of 0.28 and 0.2. We see that, analogous to the results of the 32B models, the more restrictive 0.2 upper clip setting results in sparser distributional shifts, as shown by the percentiles corresponding to near-zero divergence (Figure~\ref{fig:js_dapo_variants}). However, on its high-divergence set, the JS values are higher for the $0.2$ clip, as indicated by the higher upper percentiles. This indicates that \textbf{clip-higher admits a wider set of high-divergence token distributions but with reduced divergence magnitude at the extremes.} We observe consistent behavior on AIME 2025 (Figure~\ref{fig:js_32b_aime25}) and GPQA-Diamond (Figure~\ref{fig:gpqa_js_percentiles}), and the observed sparsity remains stable under different top-$p$ settings and when using estimated full distributions instead of truncated ones (Appendix~\ref{subsec:additional_distributions}, Figures~\ref{fig:js_topp_comparison} and~\ref{fig:js_topp1_comparison}). 

\subsection{Positional Concentration}
\label{subsec:positional_concentration}

Beyond how broadly changes are distributed across token positions, we next ask: \emph{where within a generated sequence do distributional shifts tend to occur?} Figure~\ref{fig:mean_js_by_pos} plots the mean and median JS divergence as a function of normalized token position (token index divided by sequence length), with percentile bands, for DAPO and SimpleRL on AIME 2024.

\begin{figure}[!htbp]
    \centering
    \begin{subfigure}{0.48\linewidth}
        \includegraphics[width=\linewidth]{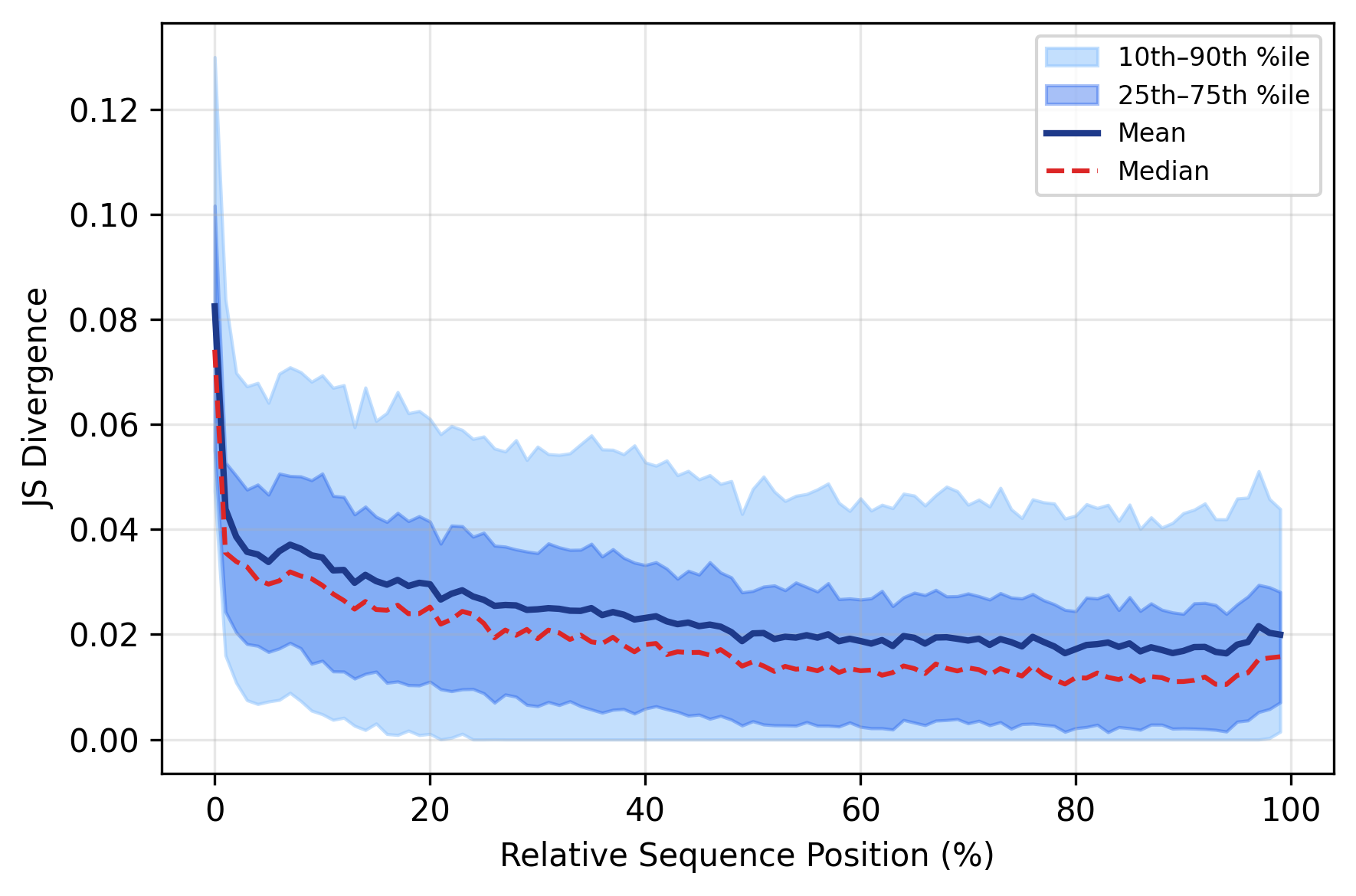}
        \caption{DAPO}
    \end{subfigure}
    \hfill
    \begin{subfigure}{0.48\linewidth}
        \includegraphics[width=\linewidth]{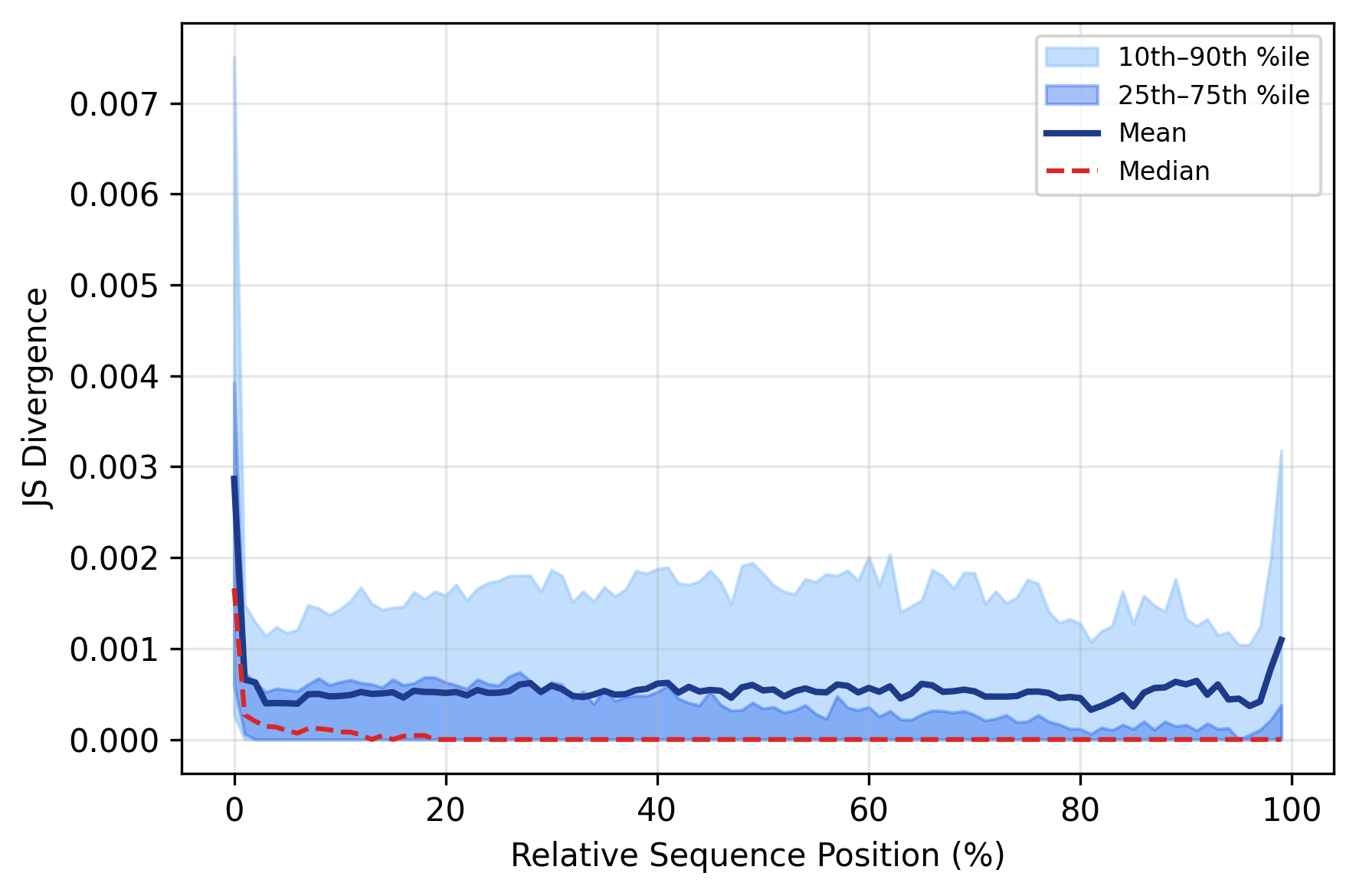}
        \caption{SimpleRL}
    \end{subfigure}
    \caption{
        Mean and median JS divergence by normalized token position, with percentile bands. 
        Both methods concentrate updates at the start and, to a lesser degree, at the end of responses.
    }
    \label{fig:mean_js_by_pos}
\end{figure}

Both models exhibit a clear positional structure: \textbf{average divergence across sequences is consistently higher near the beginning of the response, decreases through the middle, and increases modestly again toward the end.} The early concentration aligns with the modification of initial high-level branching decisions, while the late increase aligns with adjustments to answer formatting and termination behavior. However, this aggregate trend masks substantial variability at the level of individual sequences; as reflected in Figures~\ref{fig:overview_divergence_simplerl} and \ref{fig:overview_divergence_dapo} and the wide percentile spread in Figure~\ref{fig:mean_js_by_pos}, \textbf{high divergence occurs sporadically throughout the sequence.}

Comparing the two upper clip variants of Qwen2.5-Math-7B DAPO (Figure~\ref{fig:positional_dapo_variants}), both clip settings exhibit larger average divergences at the beginning of the sequence, with a smaller increase near the end, consistent with the behavior seen in the 32B models. Notably, the 0.2 clip setting shows higher average divergence at the beginning of the sequence compared to the 0.28 setting. 

\subsection{Divergence--Entropy Relationship}
\label{subsec:entropy_vs_divergence}
To further understand the general structure underlying these sparse distributional shifts, we ask: \emph{How are such shifts related to the model’s token-level entropy?}
We thus examine the relationship between distributional divergence and predictive entropy on the token level. At each token position $t$, we compute the token-level entropy
\[
H_{\pi}(x_{<t}) = -\sum_{v\in\mathcal{V}} \pi(v \mid x_{<t}) \log \pi(v \mid x_{<t}),
\]
and analyze how entropy relates to the distributional shifts from the base to RL model. Prior work suggests that RLVR updates may primarily affect high-entropy predictions while leaving low-entropy predictions largely unchanged \citep{wang2025high_entropy_minority_tokens}. We explore this perspective by comparing entropy distributions across low- and high-divergence token positions.

Specifically, token positions are grouped into low- and high-divergence bins, and we compare the entropy distributions of both the base and RL models within each bin. Figure~\ref{fig:entropy_dapo} shows these results for DAPO, with corresponding SimpleRL results provided in Appendix~\ref{subsec:additional_distributions} (Figure~\ref{fig:entropy_simplerl}).

\begin{figure}[!htbp]
    \centering
    \begin{subfigure}{0.48\linewidth}
        \includegraphics[width=\linewidth]{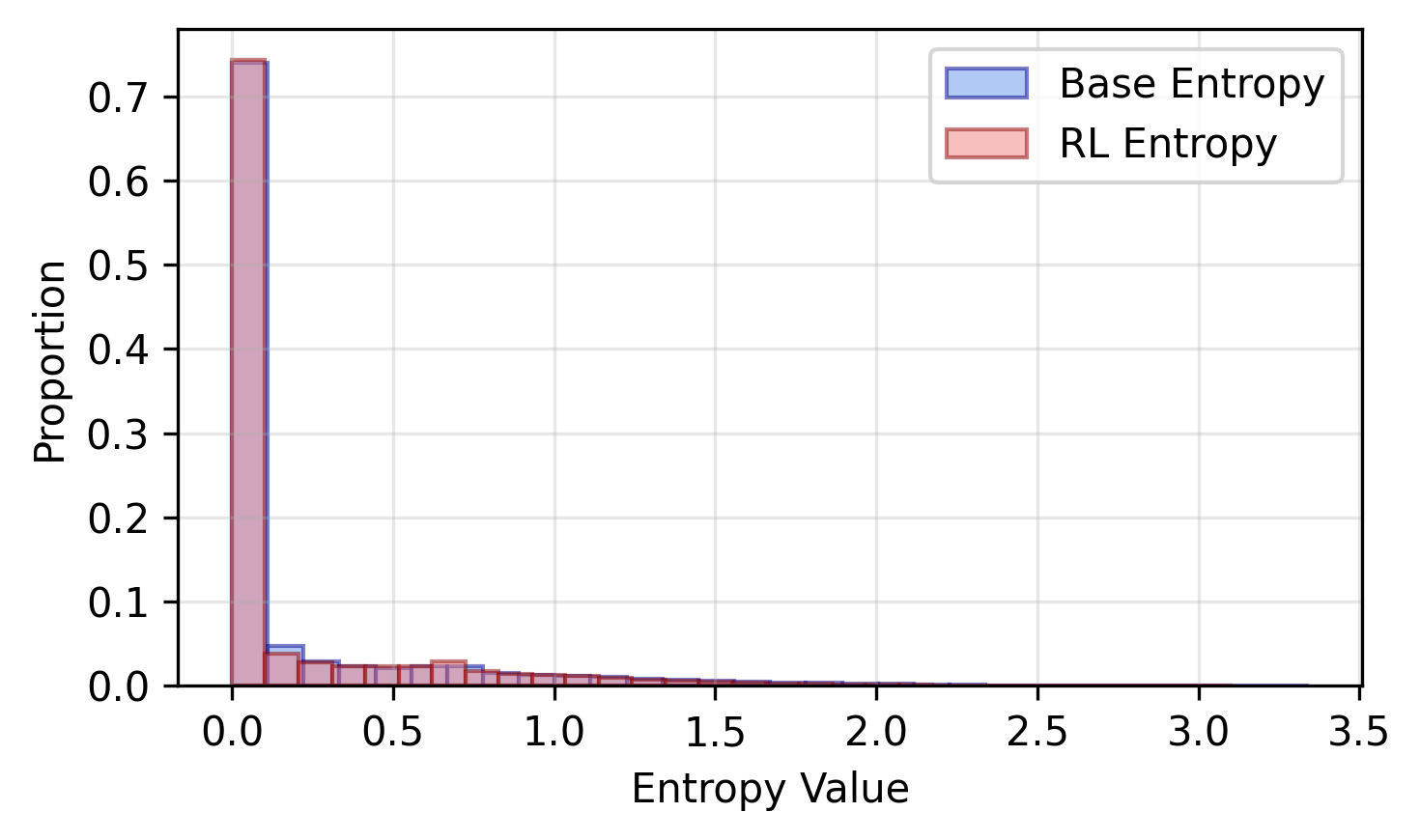}
        \caption{Low JS divergence distributions ($<0.1$).}
    \end{subfigure}
    \hfill
    \begin{subfigure}{0.48\linewidth}
        \includegraphics[width=\linewidth]{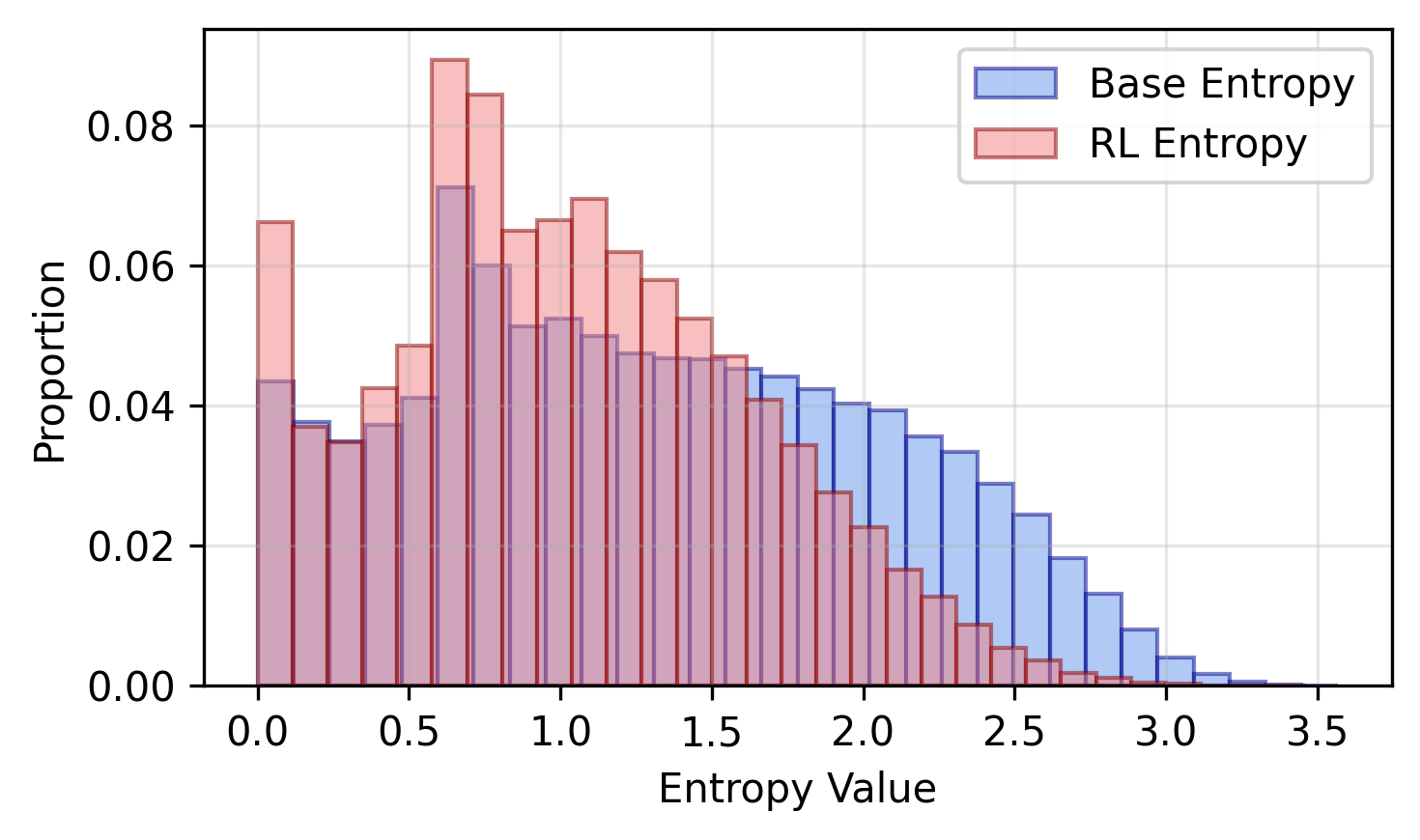}
        \caption{High JS divergence distributions ($>0.1$).}
    \end{subfigure}
    \caption{
        Entropy distributions for low and high divergence distributions for \textbf{DAPO}. 
        Low-divergence tokens are generally low-entropy, while high-divergence tokens span both high- and low-entropy regions, indicating that DAPO can modify even initially confident predictions.
    }
    \label{fig:entropy_dapo}
\end{figure}

The results show that low-divergence token distributions are largely low-entropy, indicating that distributions that are preserved are mostly initially low-entropy, though with a non-negligible portion of them that lie in the high-entropy regime. High-divergence contexts, however, can span a broad entropy range. In particular, DAPO \textbf{modifies both initially high- and low-entropy predictions, demonstrating its ability to override even confident base-model outputs.} By contrast, SimpleRL concentrates divergence more strongly in higher base entropy regions, reflecting a more conservative update regime. Isolating the effect of clip-higher, Figure~\ref{fig:entropy_7b_dapo} illustrates this contrast more clearly. At high-divergence positions, the higher $0.28$ upper clip produces a greater proportion of distributions with low base entropy, whereas the $0.2$ clip concentrates its high-divergence distributions in the higher base entropy regime. Additionally, the resulting RL entropy is higher under clip-higher, while for the $0.2$ clip it is concentrated at lower values, consistent with the overall entropy collapse observed under standard clipping \citep{yu2025dapo_system} and the steadily increasing entropy induced by clip-higher. Per-sequence scatter plots (Appendix~\ref{subsec:additional_distributions}, Figure~\ref{fig:entropy_vs_js_scatter}) show some variability across sequences, but with DAPO exhibiting an overall broader entropy spread among divergent positions and SimpleRL showing a tighter concentration, consistent with our aggregate analysis.

\subsection{Semantic Identity of Divergent Tokens}
\label{subsec:semantic_identity}

Given the sparsity and general structure of these shifts, a natural next question is: \emph{Which types of tokens are actually being targeted by RL fine-tuning?}

\begin{figure}[!htbp]
    \centering
    \begin{subfigure}{0.43\linewidth}
        \includegraphics[width=\linewidth]{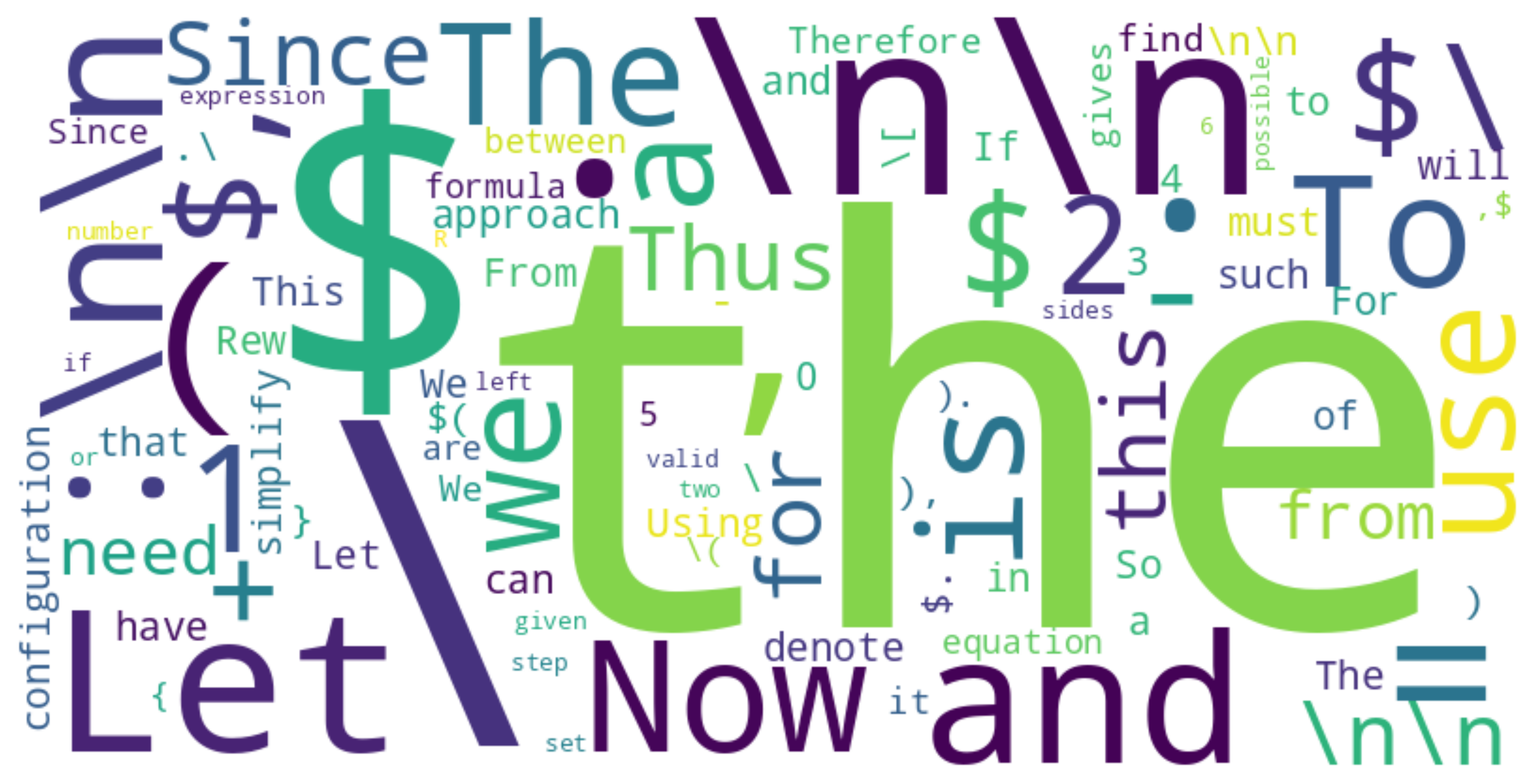}
        \caption{Tokens with high JS divergence ($\js>0.1$).}
    \end{subfigure}
    \hspace{1cm}
    \begin{subfigure}{0.43\linewidth}
        \includegraphics[width=\linewidth]{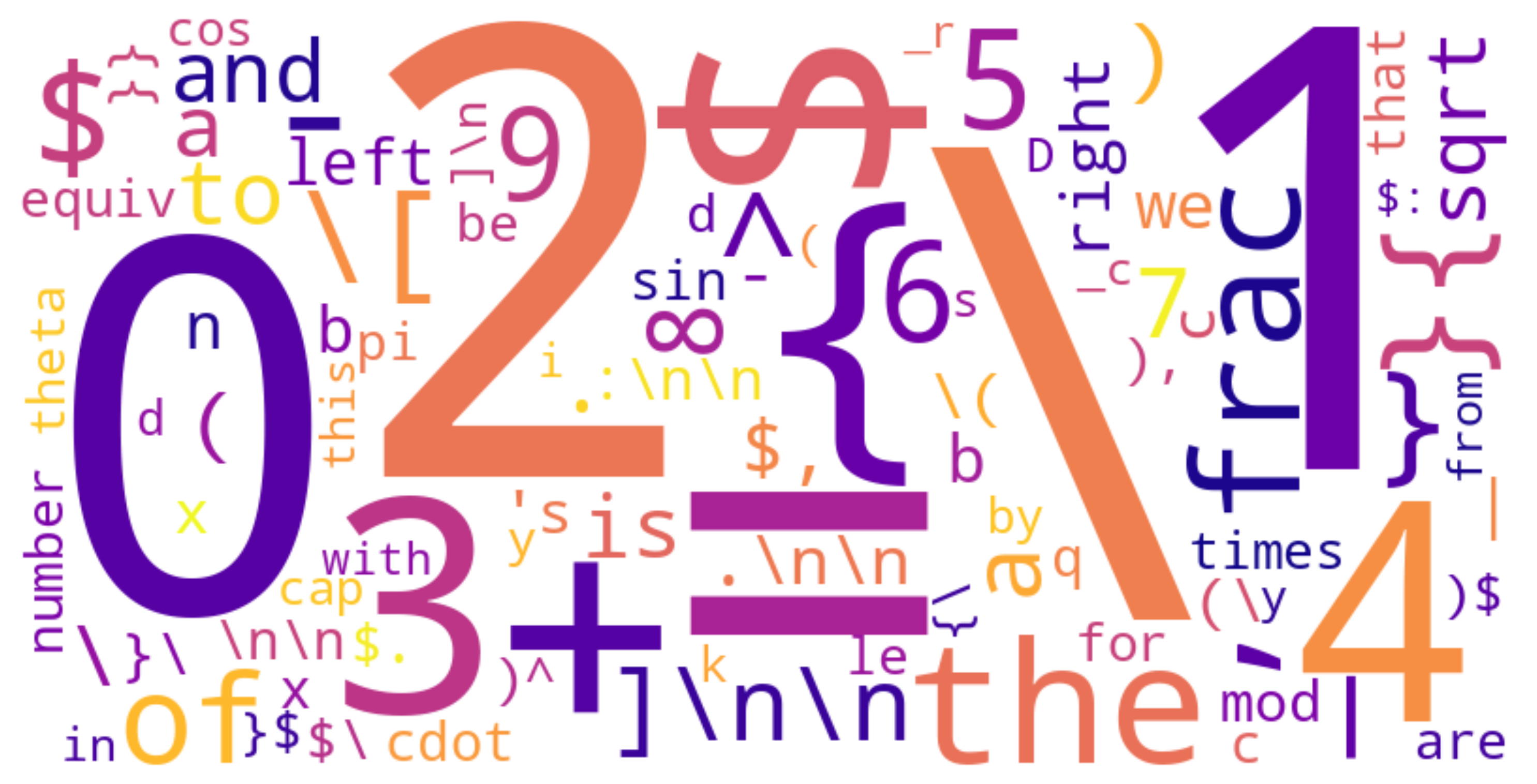}
        \caption{Tokens with low JS divergence ($\js<0.01$).}
    \end{subfigure}
    \caption{
        Word clouds of high and low divergence tokens under DAPO.
    }
    \label{fig:top_bottom_tokens}
\end{figure}
To investigate this, we examine which types of tokens tend to be sampled from high versus low divergence distributions. 
Figure~\ref{fig:top_bottom_tokens} visualizes representative examples using word clouds, where the size of each token is proportional to its frequency. Upon an initial examination, tokens appearing in high-divergence distributions include common function words, reasoning-related terms, and certain equation fragments, whereas those in low-divergence distributions are dominated by numerals, operators, and structural components of mathematical expressions.

\textbf{However, token identity alone does not determine divergence behavior}. Figure~\ref{fig:high_low_js_tokens} shows the full JS divergence distributions for the tokens sampled most frequently from high- and low-divergence distributions, revealing substantial context dependence. For example, the word ``the'' appears among the most frequent high-divergence tokens, yet its full divergence distribution across all sampled occurrences is overwhelmingly concentrated in the lower regime. This suggests that token identity alone is insufficient to characterize divergence, and that a contextual perspective is essential, rather than solely by token semantics. Instead, what is likely more important is the role the token plays within the reasoning trajectory and in the (base) model’s predictive distribution (as we'll see in the cross-sampling experiments in Section~\ref{sec:resampling}).

\subsection{Comparison with Supervised Fine-Tuning (SFT)}
\label{subsec:sft_comparison_main}

While the above analyses reveal that RLVR induces sparse distributional shifts, it remains unclear whether this behavior is unique to RL fine-tuning. This raises the question: \emph{Is such sparsity a distinctive property of RLVR, or a more general feature of fine-tuning?}
A natural point of comparison is supervised fine-tuning (SFT), which optimizes models to imitate target tokens rather than optimizing verifiable rewards on self-generated trajectories. Appendix~\ref{subsec:rlvr_vs_sft} presents a controlled comparison between SFT and RLVR (DAPO) on Qwen2.5-32B. Under the same JS divergence measurements (Section~\ref{subsec:sparse_js}), SFT exhibits a substantially larger high-divergence set and a broader divergence distribution than RLVR (Figure~\ref{fig:js_rlvr_vs_sft}). This demonstrates that the \textbf{sparsity of distributional shifts observed under RLVR is not a generic consequence of fine-tuning.}

\begin{figure}[!htbp]
    \centering
    \begin{subfigure}{0.46\textwidth}
        \includegraphics[width=\linewidth]{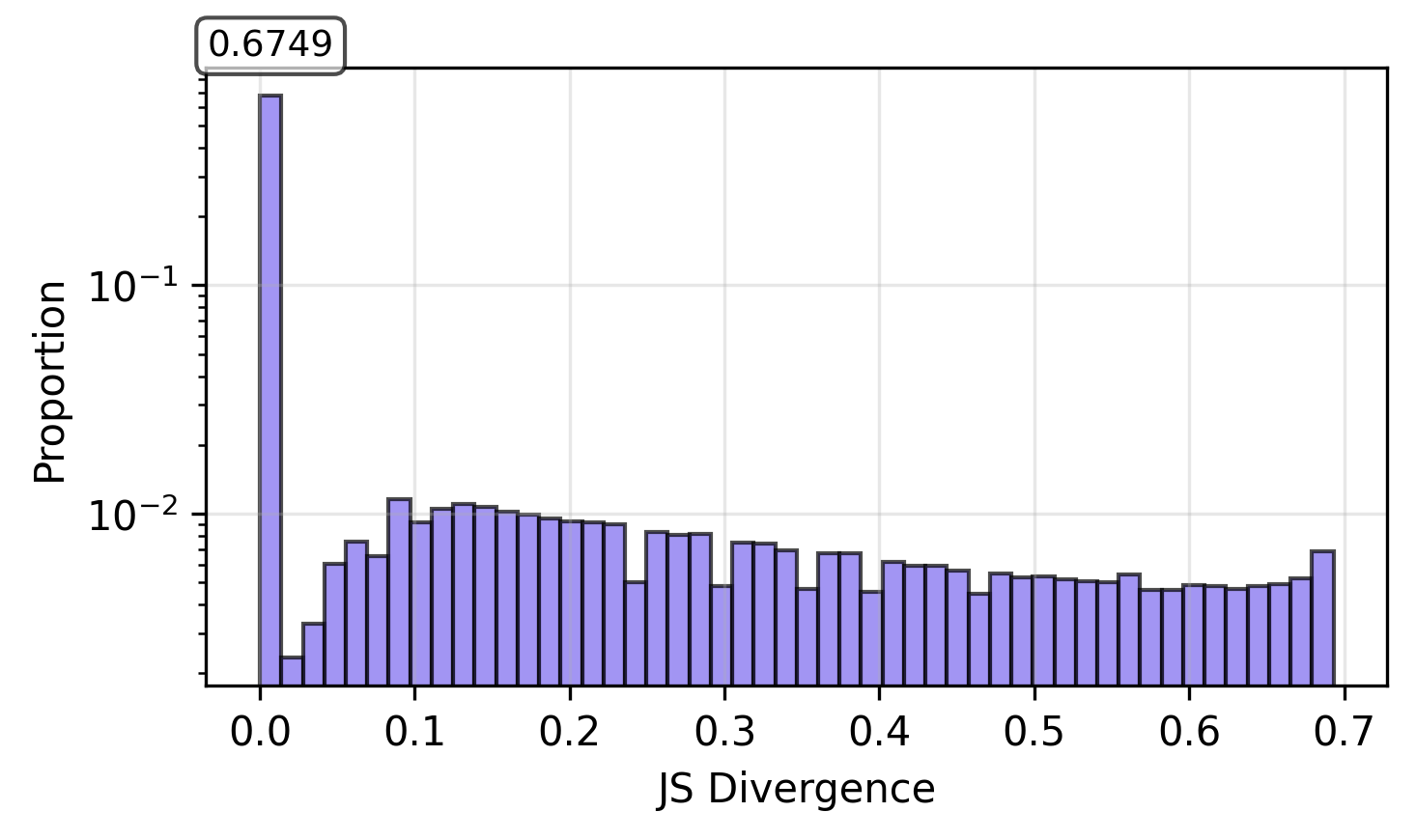}
        \caption{SFT: Histogram}
    \end{subfigure}
    \hspace{1cm}
    \begin{subfigure}{0.46\textwidth}
        \includegraphics[width=\linewidth]{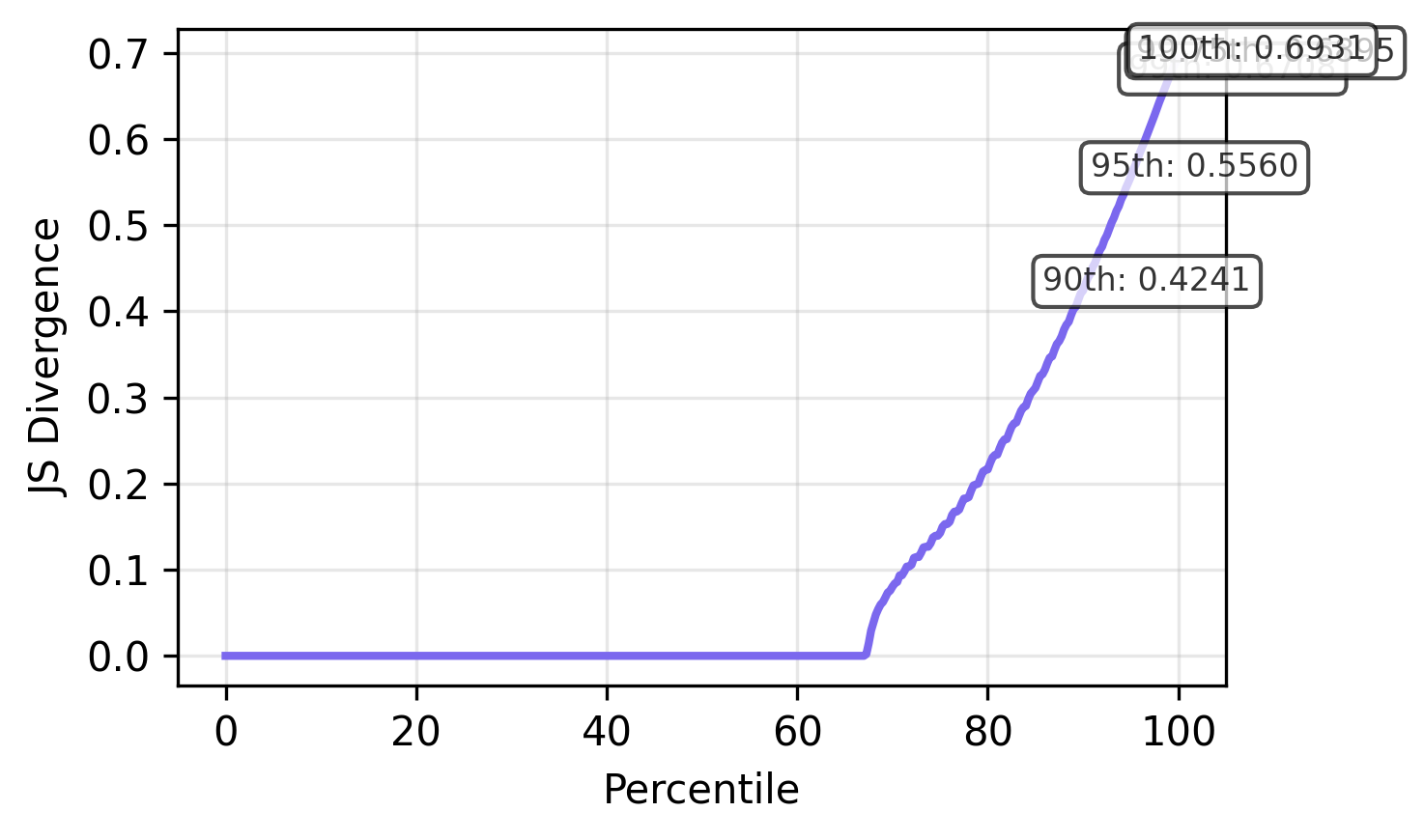}
        \caption{SFT: Percentiles}
    \end{subfigure}
    
    \caption{
        JS divergence distributions for Qwen2.5 32B fine-tuned with SFT on AIME 2024.
    }
    \label{fig:js_rlvr_vs_sft_main}
\end{figure}

Positional analysis further shows that SFT induces elevated divergence across the entire response, while still exhibiting increased divergence near the start of the sequence (Figure~\ref{fig:positional_distill_vs_rlvr}), mirroring, the early-position effects seen in RLVR. Finally, under the divergence–entropy analysis (Section~\ref{subsec:entropy_vs_divergence}), SFT’s divergent tokens concentrate more strongly in regions of high base-model entropy (compared with DAPO). While this concentration may be partially influenced by SFT outputs appearing more uncertain when evaluated under the base model, the resulting fine-tuned entropy values are nevertheless substantially lower than those of the base model (Figure~\ref{fig:entropy_distill}). These results are consistent with SFT’s objective of directly learning target outputs, leading to globally broader and sharper distributional updates.

\begin{tcolorbox}[breakable, title=Takeaways: General Structure of RLVR Distribution Shifts, colback=purple1!5!white,
  colframe=purple1]
\textbf{RLVR induces sparse and structured token-level distribution shifts.}
Across models and datasets, we observe:

\begin{itemize}
    \item \textbf{Token-level sparsity of shifts:} The vast majority of token positions exhibit near-zero JS divergence between base and RL models (often $>80\%$ and up to $>98\%$), indicating that RLVR modifies only a small subset of token distributions (even without KL regularization), despite significant performance differences between the base and fine-tuned model.
    
    \item \textbf{Method-dependent spread:} Less constrained methods (e.g., clip-higher DAPO) produce a broader divergence distribution, while more constrained methods (e.g., SimpleRL or lower clip settings) concentrate updates on fewer token distributions. 
    
    \item \textbf{Positional concentration:} Across sequences, divergence is consistently high near the beginning of responses and increases again near the end, however individual sequences exhibit varying divergence throughout.
    
    \item \textbf{Entropy interaction:} Low-divergence positions are largely low-entropy (confident) predictions, though with some of high-entropy. High-divergence token distributions span a wide entropy range under DAPO, showing that RLVR can override even low-entropy base-model predictions, while more conservative methods focus on higher-entropy regions.
    
    \item \textbf{Context dependence:} High-divergence token distributions are not determined solely by token identity; the same token can be sampled from both low and high divergence distributions depending on context.
    
    \item \textbf{Contrast with SFT:} Supervised fine-tuning produces substantially denser and more globally distributed shifts, indicating that the observed sparsity is not a generic feature of fine-tuning.
\end{itemize}
\end{tcolorbox}

\FloatBarrier

\section{Cross-Sampling: Functional Importance of Divergent Distributions}
\label{sec:resampling}

In the previous section, we showed that only a small fraction of token distributions exhibit substantial shifts between the base and RL models. This observation motivates a fundamental question: 

\emph{Are these divergent token distributions directly responsible for the performance gains induced by RLVR? More generally, to what extent are the base and RL policies functionally different on their entire sequence distributions?}

More concretely, can the accuracy improvements of the RL model be recovered by generating primarily under $\pi_{\text{base}}$ while selectively substituting a small number of tokens sampled from $\pi_{\text{RL}}$? On the other hand, does the RL model’s performance degrade when a small number of its token choices are replaced with those sampled from $\pi_{\text{base}}$? If RLVR’s gains are indeed concentrated in these sparse locations, then selectively intervening at such positions should have a disproportionate impact on performance. Furthermore, what happens if we intervene up to a certain number of interventions and then continue generation under the primary policy? Does the performance progressively improve/degrade as we increase the number of interventions, or does performance only change once most or all of the intervention-induced modifications are applied?

To answer these questions, we conduct controlled \textit{cross-sampling} experiments that selectively swap token choices between the base model $\pi_{\text{base}}$ and the RL-trained model $\pi_{\text{RL}}$. We consider two complementary interventions: (i) \textit{forward cross-sampling}, which injects RL-sampled tokens into base-model generations, and (ii) \textit{reverse cross-sampling}, which replaces RL-sampled tokens with base-model tokens during RL generation. Together, these interventions probe the contribution of RL-induced token-level changes by evaluating how introducing them into base-model trajectories or reverting them in RL trajectories influences reasoning performance. The general implementation procedure is summarized in Algorithm~\ref{alg:cross_sampling} of Appendix~\ref{subsec:additional_cross_sampling}.

\subsection{Cross-Sampling Framework}

Let $(X_t)_{t \ge 1}$ denote the sequence of random variables on $\mathcal{V}$ generated during decoding, and define the stopping time
\(
\tau := \inf\{t \ge 1 : X_t = \text{EOS}\}\wedge T_{\max},
\)
where EOS is the end-of-sequence token and $T_{\max}$ is the maximum number of tokens to generate.  
The generated response is then the finite sequence $X_{1:\tau}$.

Let $\pi_{\mathrm{prim}}$ denote the \emph{primary policy}, which governs generation by default, and let $\pi_{\mathrm{int}}$ denote the \emph{intervention policy}, which is used only at selected positions. These policies induce sequence-level distributions
$P_{\mathrm{prim}}$ and $P_{\mathrm{int}}$ over finite sequences.

To model cross-sampling, we introduce a \emph{switching rule} $\mathcal{S} : \mathcal{V}^{<\mathbb{N}} \to \{0,1\}$, where $\mathcal{V}^{<\mathbb{N}}$ is the set of finite sequences over the vocabulary $\mathcal{V}$, which determines, at each generation step, whether the next token is sampled from the intervention policy ($S_t=1$) or the primary policy ($S_t=0$).

Given a partial sequence $X_{<t}$, we define the switching variable
\(
S_t := \mathcal{S}(X_{<t}) \in \{0,1\},
\) and the resulting mixed policy governing the law of the next token:
\[
X_t\sim\pi_{\mathrm{mix}}^{(\mathrm{prim},\mathrm{int})}(\cdot \mid X_{<t})
=
(1-S_t)\, \pi_{\mathrm{prim}}(\cdot \mid X_{<t})
+
S_t \, \pi_{\mathrm{int}}(\cdot \mid X_{<t}).
\]
The corresponding sequence-level distribution is then denoted by
$P_{\mathrm{mix}}^{(\mathrm{prim},\mathrm{int})}$.

In our experiments, to align with the analysis in Section~\ref{sec:sparse_shifts}, the switching rule $\mathcal{S}$ is defined in terms of the token-level Jensen--Shannon divergence
$D_{\mathrm{JS}}$ between $\pi_{\mathrm{prim}}(\cdot \mid X_{<t})$ and
$\pi_{\mathrm{int}}(\cdot \mid X_{<t})$, but one could use any notion of divergence or distance between probability measures.  
Given a fixed threshold $\varepsilon \geq 0$, we set
\[
\mathcal{S}(X_{<t}) = \mathds{1}\{D_{\mathrm{JS}}(\pi_{\mathrm{prim}}(\cdot \mid X_{<t})\parallel \pi_{\mathrm{int}}(\cdot \mid X_{<t})) > \varepsilon\},
\]
so that cross-sampling intervenes only at high-divergence positions.

In Appendix~\ref{subsec:seq_div_bounds}, we provide simple bounds on divergences between the sequential distributions $P_{\text{mix}}$ and $P_\text{int}$ under different cross-sampling settings.

\paragraph{Forward Cross-Sampling.}
In forward cross-sampling, the response is generated primarily under the base policy,
which serves as the \emph{primary policy}, i.e., $\pi_{\mathrm{prim}} = \pi_{\text{base}}$.
The \emph{intervention policy} is the RL policy, $\pi_{\mathrm{int}} = \pi_{\text{RL}}$.
At positions where $S_t = 1$, the next token is sampled from $\pi_{\text{RL}}(\cdot \mid X_{<t})$, after which generation continues under the base policy until the next intervention point or termination.
This procedure tests whether selectively injecting RL token choices into trajectories that are otherwise generated by the base model is sufficient to recover RL-level performance.

\paragraph{Reverse Cross-Sampling.}
In reverse cross-sampling, the roles of the base and RL policies are reversed. The response is generated primarily under the RL policy,
which serves as the \emph{primary policy}, i.e., $\pi_{\mathrm{prim}} = \pi_{\text{RL}}$, while the \emph{intervention policy} is the base policy, $\pi_{\mathrm{int}} = \pi_{\text{base}}$.
At positions where $S_t = 1$, the next token is sampled from $\pi_{\text{base}}(\cdot \mid X_{<t})$, while all other positions follow the RL policy.
This intervention selectively replaces RL-sampled tokens at high-divergence
positions with base-model choices, allowing us to quantify how rapidly RL performance
degrades when those decisions are removed.

\paragraph{Evaluation with Intervention Budgets.} We further evaluate the effect of cross-sampling by limiting the number of interventions, generating responses under the mixed policy $P_{\text{mix}}$ with a fixed number of cross-sampling interventions, after which generation proceeds under the primary policy. This can be viewed as enforcing an intervention budget by setting $S_t = 0$ for all $t$ such that $\sum_{s=1}^{t-1} S_s \ge k$, where $k$ denotes the maximum number of cross-sampling interventions. We then measure accuracy for a fixed divergence threshold $\varepsilon$ as intervention count $k$ increases. This setup probes whether early, limited interventions are sufficient to induce downstream performance differences when generation is subsequently completed by the primary policy. In forward cross-sampling, this tests whether a small number of RL-induced interventions can steer trajectories such that the base model can complete them with improved accuracy, even if high-divergence positions remain later in the sequence. In reverse cross-sampling, it evaluates whether the RL model can still produce correct solutions after introducing a small number of base-sampled tokens, or whether such perturbations instead lead to a corresponding degradation in performance.

\paragraph{Connection to Speculative Decoding and BiLD.}
Our cross-sampling framework is conceptually related to speculative decoding \citep{leviathan2023fastinferencetransformersspeculative, kim2023speculativedecodingbiglittle} in that generation depends on \emph{two} next-token distributions.
It is closer to \emph{Big Little Decoder} (BiLD) \citep{kim2023speculativedecodingbiglittle}, which defines a routing policy (rather than an exact sampling scheme that preserves a designated target distribution).
BiLD trades off latency and quality via fallback and rollback rules, while our approach defines a mixed policy $\pi_{\mathrm{mix}}^{(\mathrm{prim},\mathrm{int})}$ to investigate the role of high-divergence decisions between base and RL models for reasoning performance.
BiLD triggers fallback based on small-model confidence (e.g., a max-probability threshold) and uses rollback based on a discrepancy between small/large predictive distributions (with a cross-entropy–based quantity), whereas we intervene when $D_{\mathrm{JS}}(\pi_{\mathrm{prim}}(\cdot\mid X_{<t})\parallel \pi_{\mathrm{int}}(\cdot\mid X_{<t}))$ exceeds a threshold, and investigate the impact on performance as we increase the number of interventions for a fixed threshold.
BiLD further proposes prediction alignment by fine-tuning the small model on large-model outputs to reduce avoidable disagreements, while our setting typically uses models that already agree at most positions (base vs.\ RL) to isolate fine-tuning effects.

\subsection{Results and Findings}

We now evaluate the functional impact of cross-sampling interventions on downstream task performance. Figures~\ref{fig:resample} and \ref{fig:reverse_resample} present the accuracy curves for forward and reverse cross-sampling on Qwen2.5-32B fine-tuned with SimpleRL, evaluated on AIME 2024. Each point along the curve corresponds to the Mean@32 accuracy obtained by generating responses under the mixed policy $P_{\text{mix}}$ with a fixed number of cross-sampling interventions, after which generation is completed under the primary policy.

\begin{figure}[!htbp]
    \centering
    \begin{subfigure}{0.48\textwidth}
        \includegraphics[width=\linewidth]{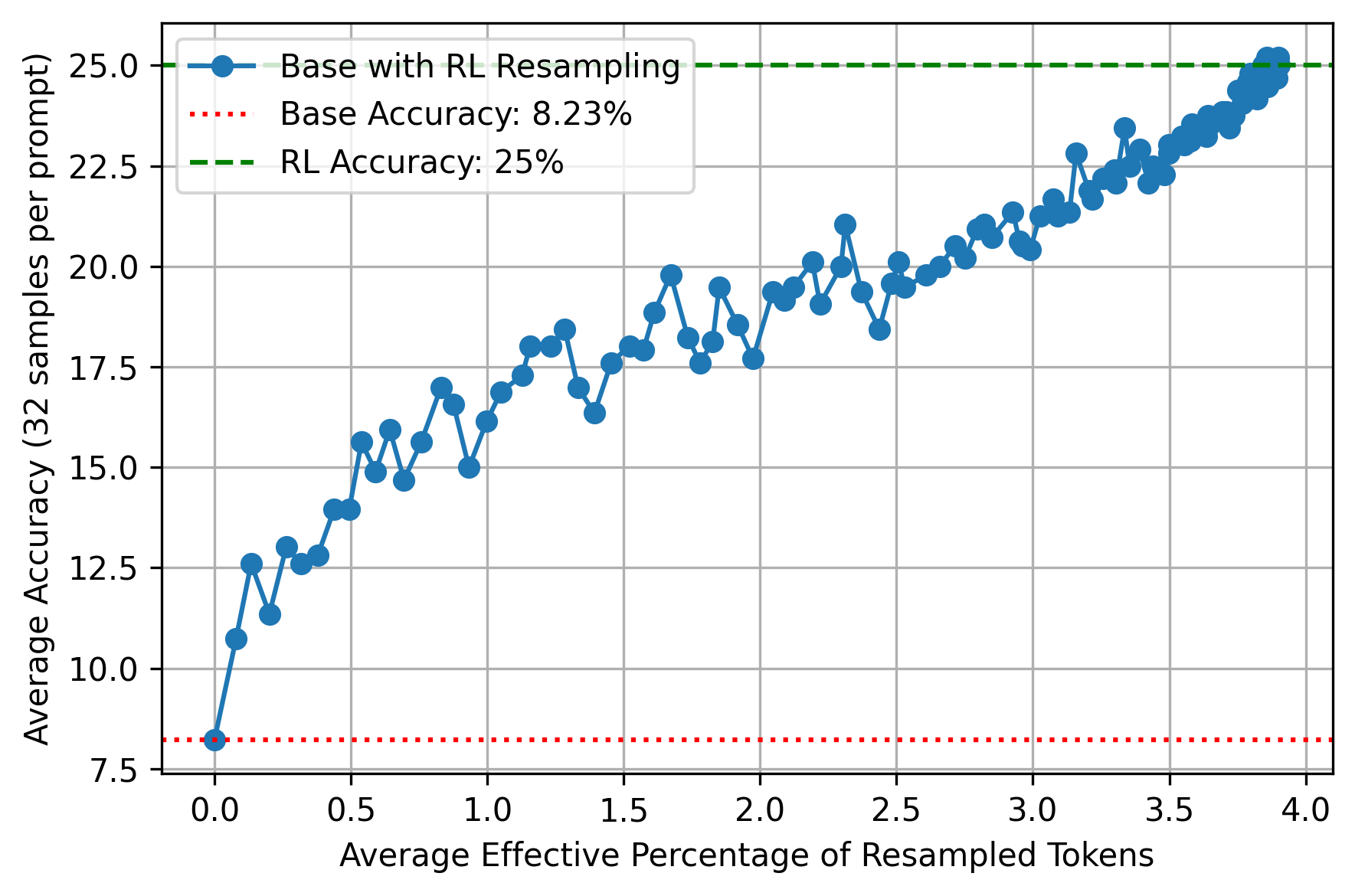}
        \caption{AIME 2024}
        \label{fig:forward_resample_aime24}
    \end{subfigure}
    \hfill
    \begin{subfigure}{0.48\textwidth}
        \includegraphics[width=\linewidth]{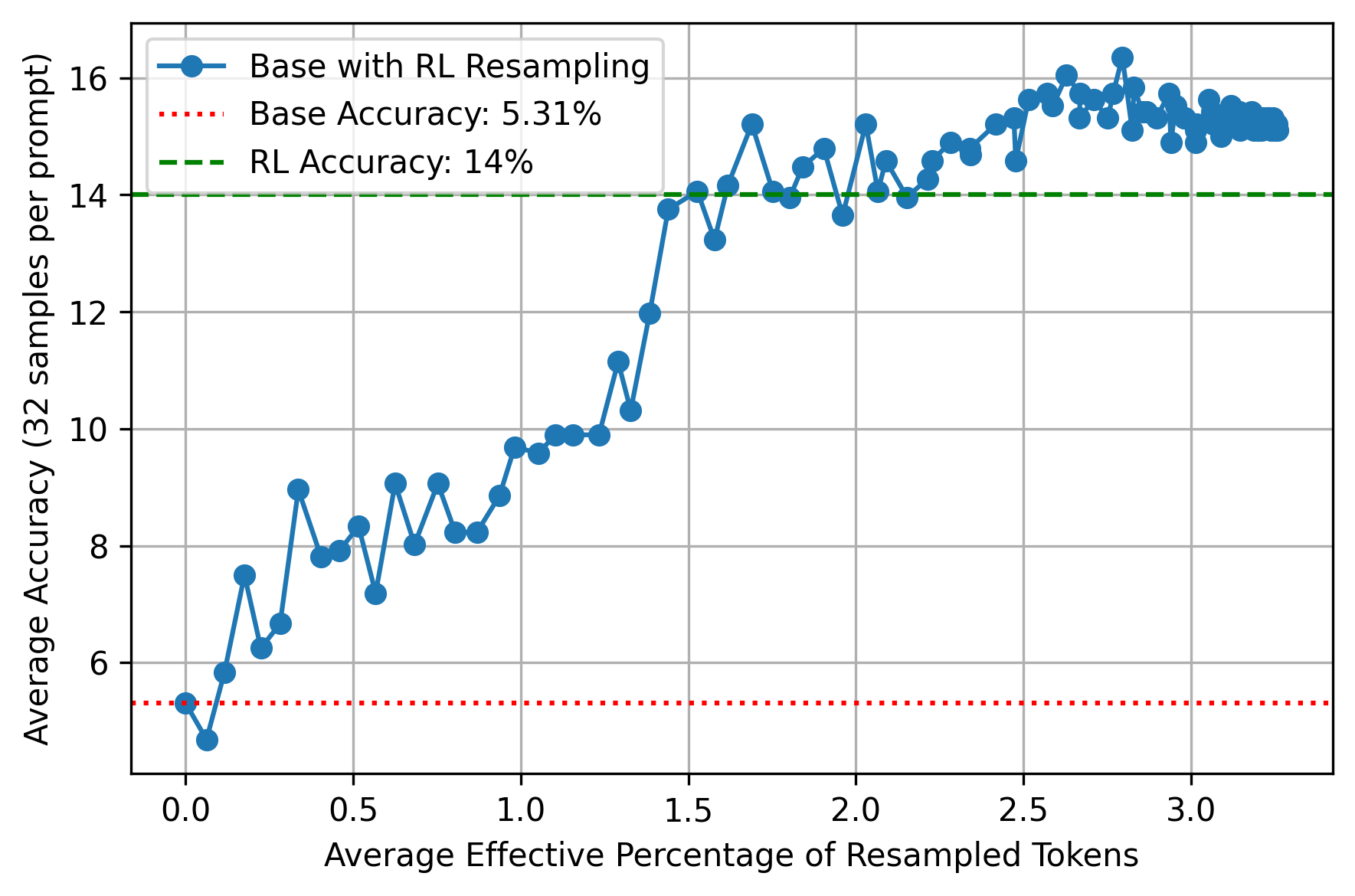}
        \caption{AIME 2025}
        \label{fig:forward_resample_aime25}
    \end{subfigure}
    \caption{Forward cross-sampling results (Qwen2.5 32B SimpleRL): injecting RL tokens into base generations progressively recovers RL accuracy.}
    \label{fig:resample}
\end{figure}

Table~\ref{tab:token_stats} summarizes the number and proportion of cross-sampled tokens required to approximately recover RL-level performance (forward cross-sampling) or collapse to base-level performance (reverse cross-sampling) for Qwen2.5 32B SimpleRL DAPO and AIME 2024 and AIME 2025. Additional cross-sampling results, including experiments on additional model configurations and datsets, are provided in Appendix~\ref{subsec:additional_cross_sampling}.  


\begin{tcolorbox}[colback=purple1!5!white,
  colframe=purple1]
\textbf{Forward Cross-Sampling:} A small fraction of RL-sampled tokens suffices to recover or even exceed RL-level performance when generation otherwise follows the base model.
\end{tcolorbox}

\paragraph{Forward Cross-Sampling:}
From Figure~\ref{fig:resample} and Table~\ref{tab:token_stats}, we observe that for Qwen2.5-32B with SimpleRL, forward cross-sampling under the mixed policy
\( \pi_{\mathrm{mix}}^{(\mathrm{base},\mathrm{RL})} \)
recovers RL-level accuracy with remarkably few interventions. Injecting fewer than \(4\%\) RL-sampled tokens per sequence on average, corresponding to fewer than \(40\) effective token substitutions per response, is sufficient to close the performance gap from the base model (approximately \(8\%\)) to the RL model (approximately \(25\%\)) on AIME 2024. On AIME 2025, the effect is even more pronounced: using only \(1.53\%\) effective cross-sampling, or roughly \(13\) average token substitutions per response, raises accuracy from about \(5\%\) to over \(14\%\). Interestingly, this level of performance exceeds that of the RL policy $\pi_{\text{RL}}$ itself, indicating that \textbf{the mixed policy
\( \pi_{\mathrm{mix}}^{(\mathrm{base},\mathrm{RL})} \)
can, in some cases, outperform the standalone RL policy.} This potentially occurs because the mixed policy is close but not identical to the RL policy (for $\varepsilon>0$), which may sometimes avoid failures induced by the RL model.

Recovering RL-level performance for DAPO requires a larger number of interventions, reflecting its substantially stronger fine-tuned performance. On AIME 2024, approximately \(7.8\%\) effective tokens are are enough to boost accuracy from roughly \(8\%\) to over \(44\%\), while on AIME 2025, fewer than \(6.5\%\) effective interventions suffice to increase performance from around \(5\%\) to over \(33\%\). Importantly, even though the performance gains are substantially larger, \textbf{the number of critical token-level decisions remains small relative to sequence length.}

\begin{figure}[!htbp]
    \centering
    \begin{subfigure}{0.48\textwidth}
        \includegraphics[width=\linewidth]{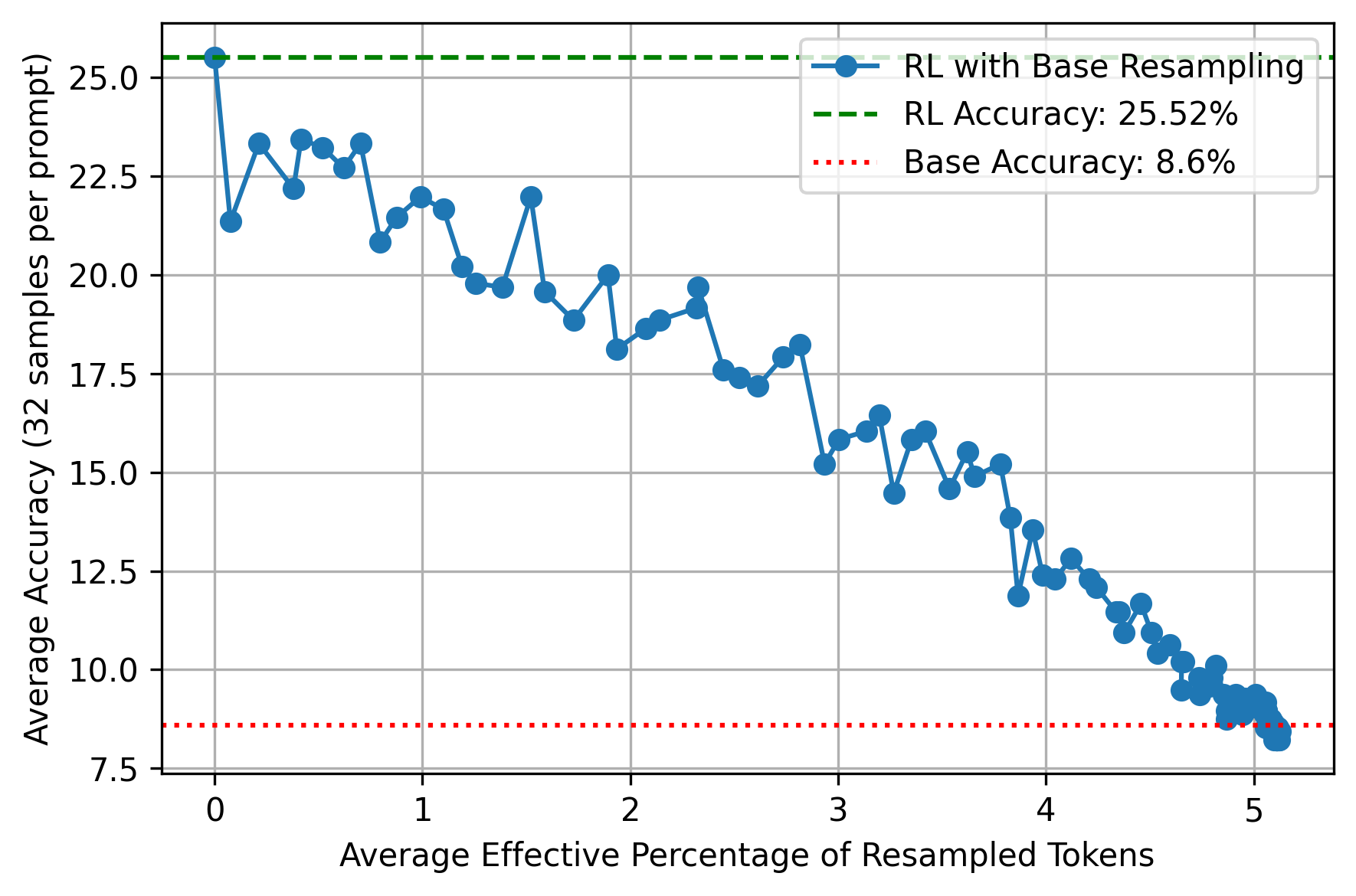}
        \caption{AIME 2024}
        \label{fig:reverse_resample_aime24}
    \end{subfigure}
    \hfill
    \begin{subfigure}{0.48\textwidth}
        \includegraphics[width=\linewidth]{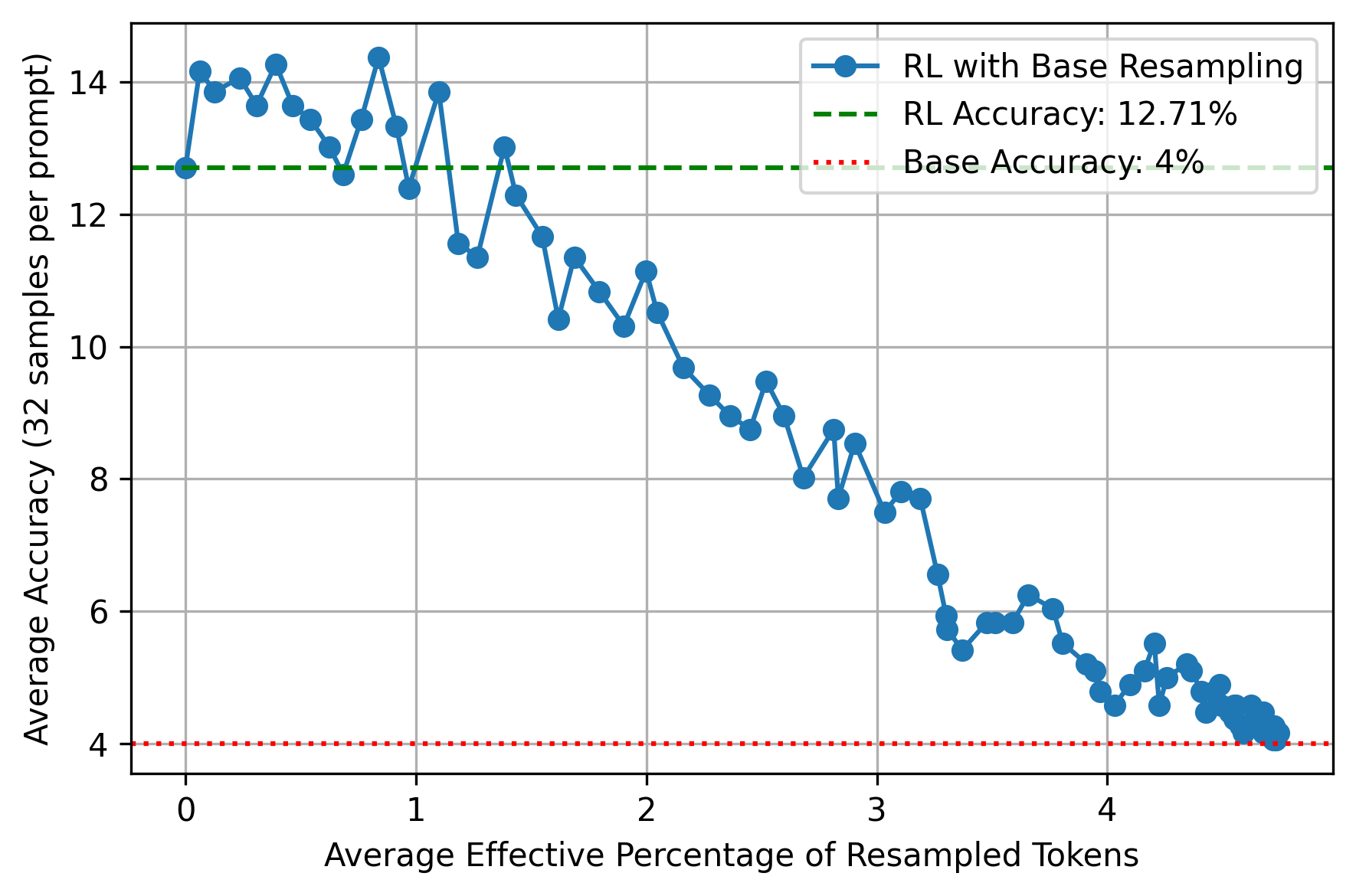}
        \caption{AIME 2025}
        \label{fig:reverse_resample_aime25}
    \end{subfigure}
    \caption{Reverse cross-sampling results (Qwen2.5 32B SimpleRL): swapping RL tokens with base tokens in RL generations causes near-monotonic degradation toward base performance.}
    \label{fig:reverse_resample}
\end{figure}

\begin{tcolorbox}[colback=purple1!5!white,
  colframe=purple1]
\textbf{Reverse Cross-Sampling:} Reverting a small fraction of RL tokens causes performance to collapse to, or even below, base-model levels when generation otherwise follows the RL model.
\end{tcolorbox}

\paragraph{Reverse Cross-Sampling:}
Reverse cross-sampling results show that the RL policy is highly sensitive to a small number of its token-level decisions, operating similarly as the base policy apart from the small number of token positions with high-divergence. From Figure~\ref{fig:reverse_resample} and Table~\ref{tab:token_stats}, we observe that for Qwen2.5-32B with SimpleRL, generating under the mixed policy
\( \pi_{\mathrm{mix}}^{(\mathrm{RL},\mathrm{base})} \), that is generating primarily with the RL policy except at high-divergence positions, only a small fraction of base-sampled tokens is enough to rapidly degrade performance. On AIME 2024, replacing approximately \(5\%\) of high-divergence distributions, corresponding to less than $30$ effective base-sampled tokens per response, is sufficient to collapse accuracy from RL levels (around \(25\%\)) back to base-level performance (around \(8\%\)). This phenomenon also holds on AIME 2025: roughly \(4.7\%\) of effective base-sampled tokens (around $30$ effective tokens per response) reduces accuracy from approximately \(12.7\%\) to below \(4\%\), which is below base-level performance.

For DAPO, a larger number of substitutions is required to erase its gains, consistent with its substantially stronger performance. On AIME 2024, reverting roughly \(10\%\) of effective tokens suffices to reduce accuracy from over \(44\%\) to near base levels (around \(8\%\)), while on AIME 2025, fewer than \(10\%\) effective reversions collapse performance from over \(33\%\) to below \(4.5\%\). Importantly, even in these cases, \textbf{the required reversions constitute a small fraction of the total generated tokens}, reinforcing that RL-level performance, across both SimpleRL and DAPO, depends critically on a sparse set of token-level decisions.

\begin{tcolorbox}[colback=purple1!5!white,
  colframe=purple1]
Notably, \textbf{the substituted base tokens are mostly plausible and semantically reasonable (Figure~\ref{fig:resample_token_pairs}), yet they nonetheless progressively derail the reasoning process.} Even when two token choices are locally equivalent or interchangeable to a human reader, they can induce different downstream conditional distributions and lead to diverging reasoning responses, revealing substantial trajectory sensitivity of the model.
\end{tcolorbox}

\paragraph{Progressive Steering of Reasoning Trajectories.}
Across both forward and reverse cross-sampling, reasoning performance varies smoothly as the cross-sampling intervention budget increases. In the forward direction, accuracy improves steadily as additional RL-sampled tokens are introduced, with no sharp threshold, indicating that each intervention contributes positively, on average, to performance and that \textbf{RL’s gains are sparsely distributed across multiple decision points rather than requiring all RL-induced changes.} In reverse cross-sampling, performance degrades in a similarly smooth and near-monotonic manner as RL token choices are reverted, demonstrating that RL-level performance depends on preserving a sparse set of token-level shifts throughout the generation, even when the remainder of the response is produced under the RL policy.

An important aspect underlying both settings is that cross-sampling interventions are applied sequentially along the generation trajectory, while decoding otherwise proceeds under a single primary policy. Intuitively, one might expect that modifying only a small number of token choices, particularly early ones, would have limited impact once generation continues under the primary policy. However, our results show that this is not the case: \textbf{injecting even just the first few RL-sampled tokens can already yield measurable performance gains in forward cross-sampling, while swapping the earliest few divergent tokens can noticeably degrade performance in the reverse setting.} These effects arise not necessarily because early tokens are inherently dominant, but because small, local edits can redirect the generation process toward different reasoning trajectories, which are then continued by subsequent decoding under the primary policy.

\begin{tcolorbox}[colback=purple1!5!white,
  colframe=purple1]
Rather than introducing entirely new reasoning behaviors, RLVR refines a sparse set of local token choices that reliably steer generation toward more effective reasoning trajectories that remain accessible to the base model, but are unlocked through these targeted edits. Taken together, the forward and reverse cross-sampling results show that the RL fine-tuned model operates similarly to the base model, modulo a small number of token-level decisions.
\end{tcolorbox}

\paragraph{Summary.}
Overall, these results establish that RLVR refinement operates in a highly targeted manner. Across datasets and training settings, forward and reverse cross-sampling show that RLVR’s performance gains are functionally concentrated in a sparse set of high-divergence token positions. Forward interventions show that introducing only a small number of RL-sampled tokens into base-model generations is sufficient to recover, and in some cases exceed, RL-level accuracy, while reverse interventions demonstrate that reverting a similarly small number of RL token choices with base-sampled tokens can rapidly erase these gains, and in some cases degrade performance below base-model levels. Importantly, these effects emerge progressively as the intervention budget increases, indicating that the reasoning trajectories are steadily shaped by sequential, local token-level modifications along the sequence.

\begin{table}[htbp]
\centering
\small
\caption{Summary of cross-sampled tokens required to reach approximate RL-level performance (forward) or base-level performance (reverse) for Qwen2.5-32B on AIME 2024 and AIME 2025 with a token generation budget of 8000.
Effective token counts/percentages exclude identity swaps during cross-sampling. Token percentages are computed at the sequence level.}
\label{tab:token_stats}
\begin{tabular}{l@{\hspace{0.5em}}l@{\hspace{0.2em}}c@{\hspace{0.2em}}c@{\hspace{0.2em}}|c@{\hspace{0.2em}}|c}
\toprule
\textbf{Dataset} & \textbf{Method} & \textbf{Eff. \%} & \textbf{Eff. \#} & \textbf{Initial} \textbf{(}$\pi_{\text{prim}}$\textbf{) } & \textbf{Final} \textbf{(}$\pi_{\text{mix}}$\textbf{)} \\
 &  & \textbf{Tokens } & \textbf{Tokens } & \textbf{Acc. (\%)} & \textbf{Acc. (\%)} \\
\midrule
\multirow{4}{*}{AIME24}
 & SimpleRL   & $3.86$\% & $38$ & $8.23$ & $>25$ \\
 & SimpleRL Reverse  & $5$\%  & $29$ & $25.52$ & $<8.3$ \\
 & DAPO       & $7.8$\% & $280$ & $8.23$ & $>44$ \\
 & DAPO Reverse       & $10.1$\% & $173$ & $44.8$ & $<8.5$ \\
\midrule
\multirow{4}{*}{AIME25}
 & SimpleRL   & $1.53$\% & $13$ & $5.3$ & $>14$ \\
 & SimpleRL Reverse  & $4.73$\% & $31$ & $12.71$ & $<4$ \\
 & DAPO       & $6.47$\%  & $230$ & $5$ & $>33$ \\
 & DAPO Reverse       & $9.89$\%  & $181$ & $32$ & $<4.5$ \\
\bottomrule
\end{tabular}
\end{table}

\begin{tcolorbox}[breakable, 
  title=Takeaways: Functional Role of Divergent Token Distributions via Cross-Sampling,
  colback=purple1!5!white,
  colframe=purple1,
  fonttitle=\bfseries
]

\textbf{Cross-sampling experiments show that RLVR performance gains are concentrated at a small set of high-divergence token positions, revealing that RL and base models are largely similar overall but differ critically at sparse, high-impact token decisions where RLVR guides generation toward more effective reasoning trajectories that are otherwise accessible to the base model.}

\begin{itemize}
    \item \textbf{Forward cross-sampling:} Injecting a small fraction of RL-sampled tokens into base-model generations is sufficient to recover RL-level accuracy. In multiple settings, modifying only $\sim$1–10\% of tokens per response closes most or all of the performance gap, and can sometimes exceed standalone RL decoding performance.

    \item \textbf{Reverse cross-sampling:} Reverting a similarly small fraction of RL token choices back to base-sampled tokens in RL-model generations progressively collapses RL performance to base levels, and in some cases below base accuracy. RL gains are therefore highly sensitive to these sparse token-level decisions.

    \item \textbf{Model improvement vs. number of required interventions:} Stronger RLVR models (e.g., DAPO) in general require more interventions to recover or erase their gains, but the required substitutions still form a small fraction of tokens overall.

   \item \textbf{Progressive trajectory shaping:} Performance changes progressively and near-monotonically with the number of interventions, indicating that gains accumulate across multiple divergence points rather than requiring most or all RL- or base-induced changes to impact reasoning performance. Even a small number of early interventions can propagate forward to produce globally different reasoning trajectories with sustained performance impact, even when generation is continued under the primary policy.

    \item \textbf{Sensitivity to locally interchangeable tokens:} Cross-sampled substitutions generally correspond to tokens that are locally plausible and often semantically equivalent from a human perspective, yet still produce significant downstream performance differences. This indicates that reasoning trajectories can be sensitive to distributionally distinct but interchangeable token choices. This sensitivity highlights the limited invariance of the generation process in such LLMs to locally equivalent token choices.

    \item \textbf{Insight into RLVR vs.\ base policies:} The base and RL models behave similarly on most token decisions, but differ at a sparse set of high-divergence positions that have disproportionate impact on reasoning outcomes. RLVR thus acts as a targeted modification mechanism on the base model rather than a global policy shift.

\end{itemize}

\end{tcolorbox}

\FloatBarrier

\section{Fine-Grained Mechanics of Distribution Shifts}
\label{sec:fine_grained}

Having established the general aspects of RLVR-induced distribution shifts, namely their sparsity, positional concentration, relationship to entropy, as well as their functional importance to reasoning performance, a natural next question is:

\textit{At token positions where substantial changes occur, how novel are these updates? Do they introduce (effectively) new candidate tokens, or primarily redistribute probability mass among existing ones?}

This question targets the \emph{mechanism} underlying the sparse shifts observed earlier. While previous analyses reveal \emph{where} and \emph{how much} change occurs, and their importance to reasoning, they do not resolve \emph{what kind} of change is taking place at the level of individual token distributions. To address this, we conduct a fine-grained analysis of high-divergence positions, examining how probability mass is reallocated within the next-token distributions. Concretely, we move beyond general changes and directly study the structural changes in candidate tokens and their rankings. We examine this through multiple lenses: (i) overlap in top-$k$ candidate sets and token rank reordering, (ii) low-probability token behavior, and (iii) the evolution over the course of training.

\begin{tcolorbox}[colback=purple1!5!white,
  colframe=purple1]
This fine-grained analysis reveals that, even at positions with substantial divergence, current RLVR methods mostly do not fundamentally change the candidate space of predictions. Instead, they primarily reorder and selectively amplify tokens that are already plausible under the base model, with limited promotion of low-probability tokens. Nevertheless, these comparatively rare cases of substantial re-ranking or promotion of low-probability tokens may still play an important role in enabling improved reasoning.
\end{tcolorbox}

We further validate these findings across additional datasets, models, and RLVR hyperparameter settings in Appendix~\ref{subsec:additional_distributions}.

\subsection{Top-\textit{k} Overlap and Rank Reordering}
\label{subsec:topk_overlap}

We first investigate whether RLVR changes \emph{which} tokens are considered plausible, or mainly changes \emph{how} they are prioritized. Concretely, we examine (1) the overlap between the base and RL models’ top-$k$ candidate sets, and (2) how the relative ranking of shared candidates shifts.

Figure~\ref{fig:topk_overlap} reports the fraction of shared tokens between the top-$k$ sets of the base and RL fine-tuned models, restricted to high-divergence token distributions. Despite only considering high-divergence positions, \textbf{overlap between top-$k$ token sets remains high once $k \geq 2$.} SimpleRL exhibits over 80\% average overlap (often exceeding 85\%), while DAPO shows slightly lower but still substantial overlap. Both methods display a sharp increase in overlap from $k=1$ to $k=2$, suggesting that while the top-1 token often
changes at high-divergence positions, \textbf{the replacement was typically already among the base model’s top-3.}

\begin{figure}[!htbp]
    \centering
    \begin{subfigure}{0.48\linewidth}
        \includegraphics[width=\linewidth]{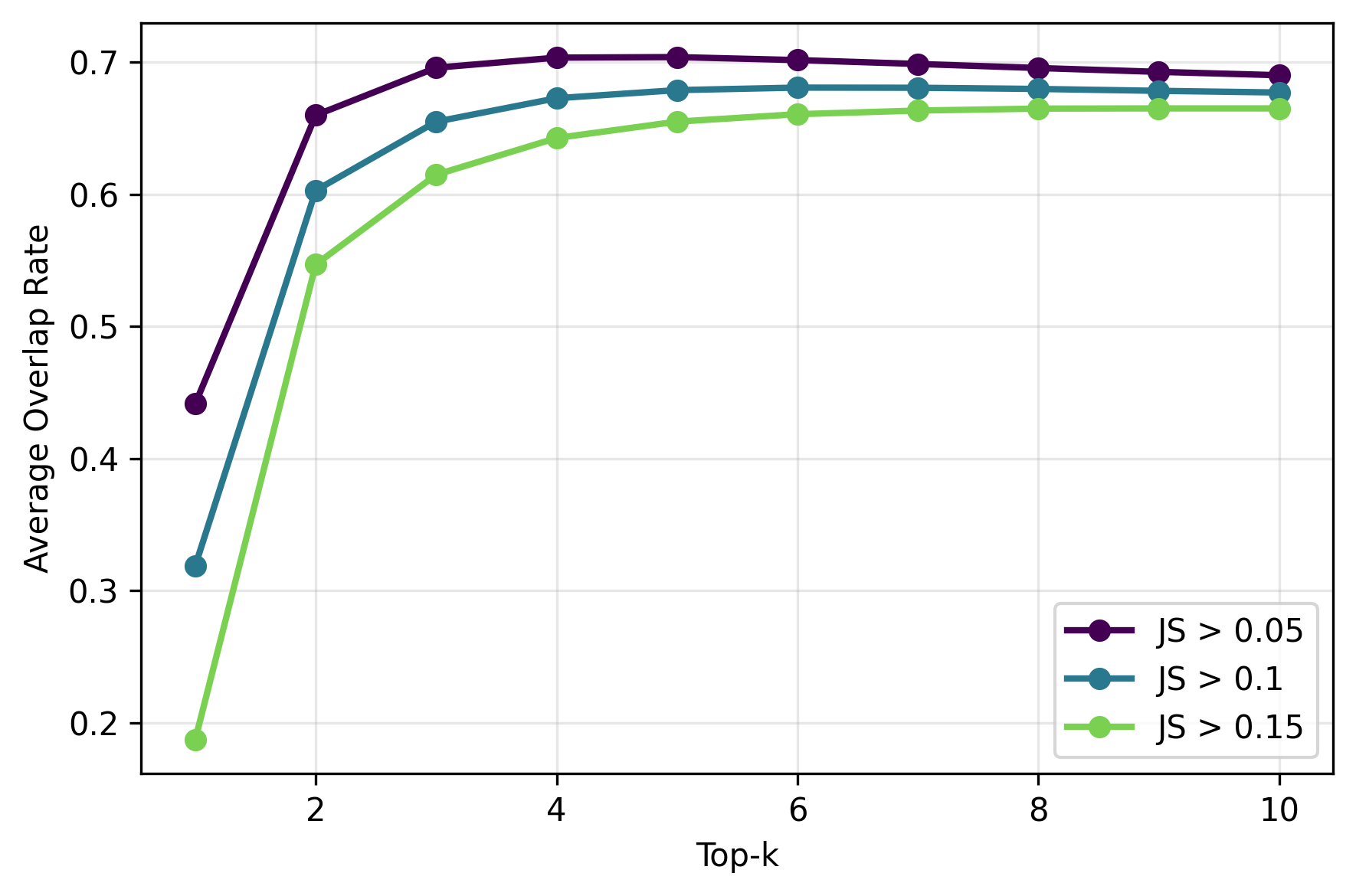}
        \caption{DAPO: Top-$k$ overlap across thresholds.}
        \label{fig:topk_dapo}
    \end{subfigure}
    \hfill
    \begin{subfigure}{0.48\linewidth}
        \includegraphics[width=\linewidth]{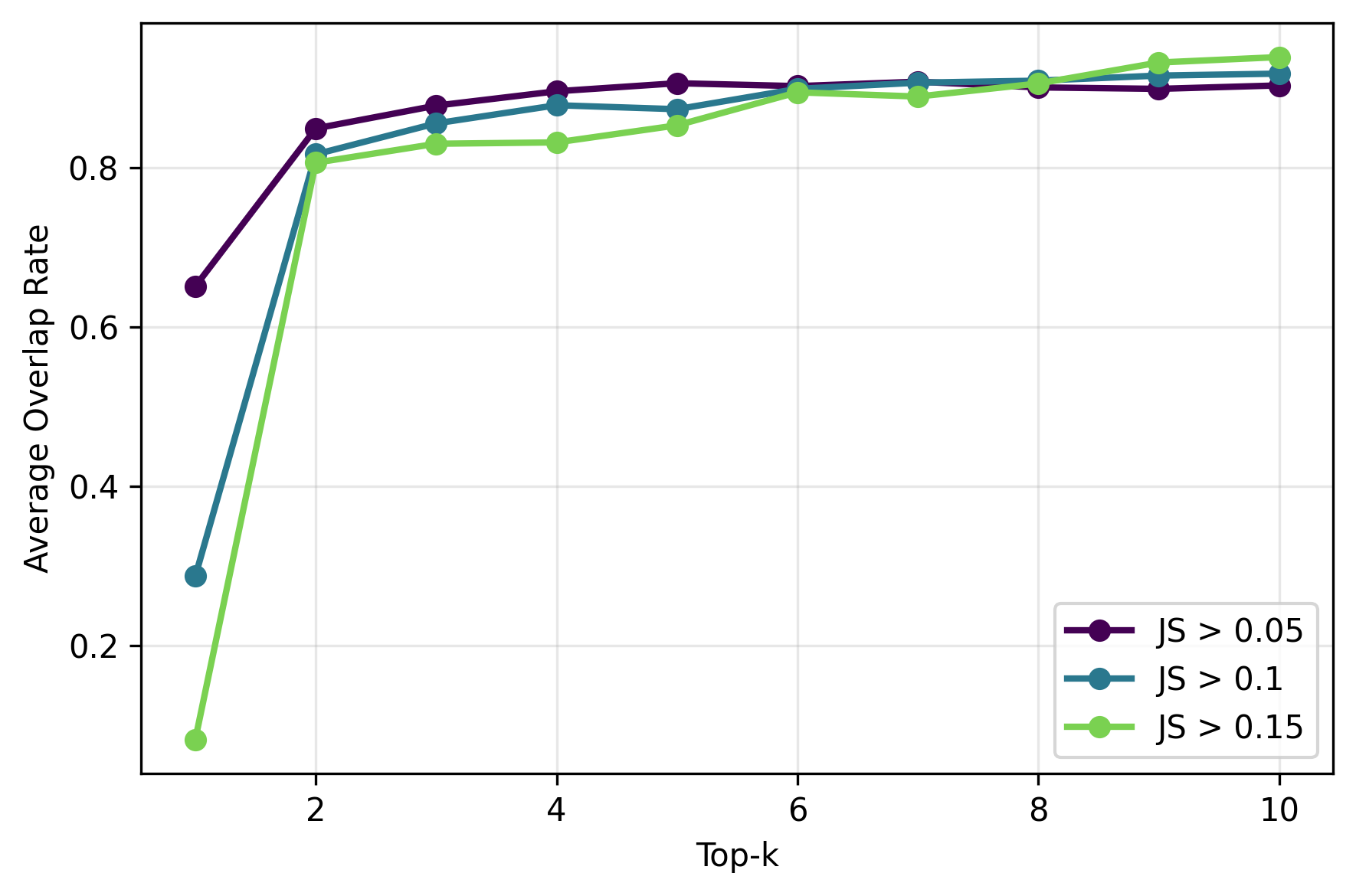}
        \caption{SimpleRL: Top-$k$ overlap across thresholds.}
        \label{fig:topk_simplerl}
    \end{subfigure}
    \caption{
        Top-$k$ token overlap between base and RL models at divergent positions (JS > 0.1). Computed as the size of the intersection divided by $k$.
        High overlap for $k \geq 2$ shows that distributional shifts occur mostly within shared candidate sets.
    }
    \label{fig:topk_overlap}
\end{figure}

This observation is further clarified in Figure~\ref{fig:rl_vs_base_ranks}, which shows where the RL model’s top-3 tokens appear in the base model’s ranking, among high divergence positions. Around 30\% of RL top-1 tokens are already ranked first under the base model, and over 80\% (DAPO) and 90\% (SimpleRL) fall within the base top-3. RL top-2 tokens typically lie within the base top-3–4, with SimpleRL exhibiting consistently stronger alignment.

Comparing the upper-clip mechanism with Qwen2.5-Math-7B, we observe a similar but more nuanced behavior. For small $k$, DAPO with $0.28$ upper clip yields lower average overlap with the base model’s top-$k$ set than the variant with 0.2 clip (Figure~\ref{fig:topk_dapo_variants}), indicating more frequent changes among the highest-ranked tokens when using clip-higher. Interestingly, this trend reverses for larger $k$, where the model trained with $0.2$ upper clip exhibits smaller overlap, suggesting that agreement with the base model deteriorates in the tail of the candidate set. A consistent picture emerges from the rank-shift analysis in Figure~\ref{fig:ranks_dapo_variants}. For the 0.2 clip setting, the base-model ranks of the RL model’s top-3 tokens are more often preserved, but a non-negligible fraction of the RL model's 3rd ranked token originate from much lower base ranks compared to the 0.28 clip setting. 

Taken together, these results suggest that clip-higher primarily redistributes probability mass within an already plausible candidate set, leading to more frequent reordering among the highest-ranked tokens while largely preserving the broader candidate space. In contrast, removing clip-higher tends to concentrate probability more strongly on a small number of dominant tokens, resulting in higher agreement with the base model among the very top ranks but poorer agreement deeper in the candidate set. Consistent with this, under the 0.2 clip setting, some tokens promoted to the RL model’s 3rd ranked token originate from much lower base ranks. However, given the substantially lower RL entropy observed in this setting, such promoted tokens may still carry relatively little probability mass in practice, with most of the probability concentrated on the top one or two tokens. As a result, these apparent rank promotions in the 0.2 clip setting may often reflect secondary effects in the low-probability tail rather than substantial changes to the primary token choice driving generation.

\begin{figure}[!htbp]
    \centering
    \begin{subfigure}{0.45\linewidth}
    \includegraphics[width=\linewidth]{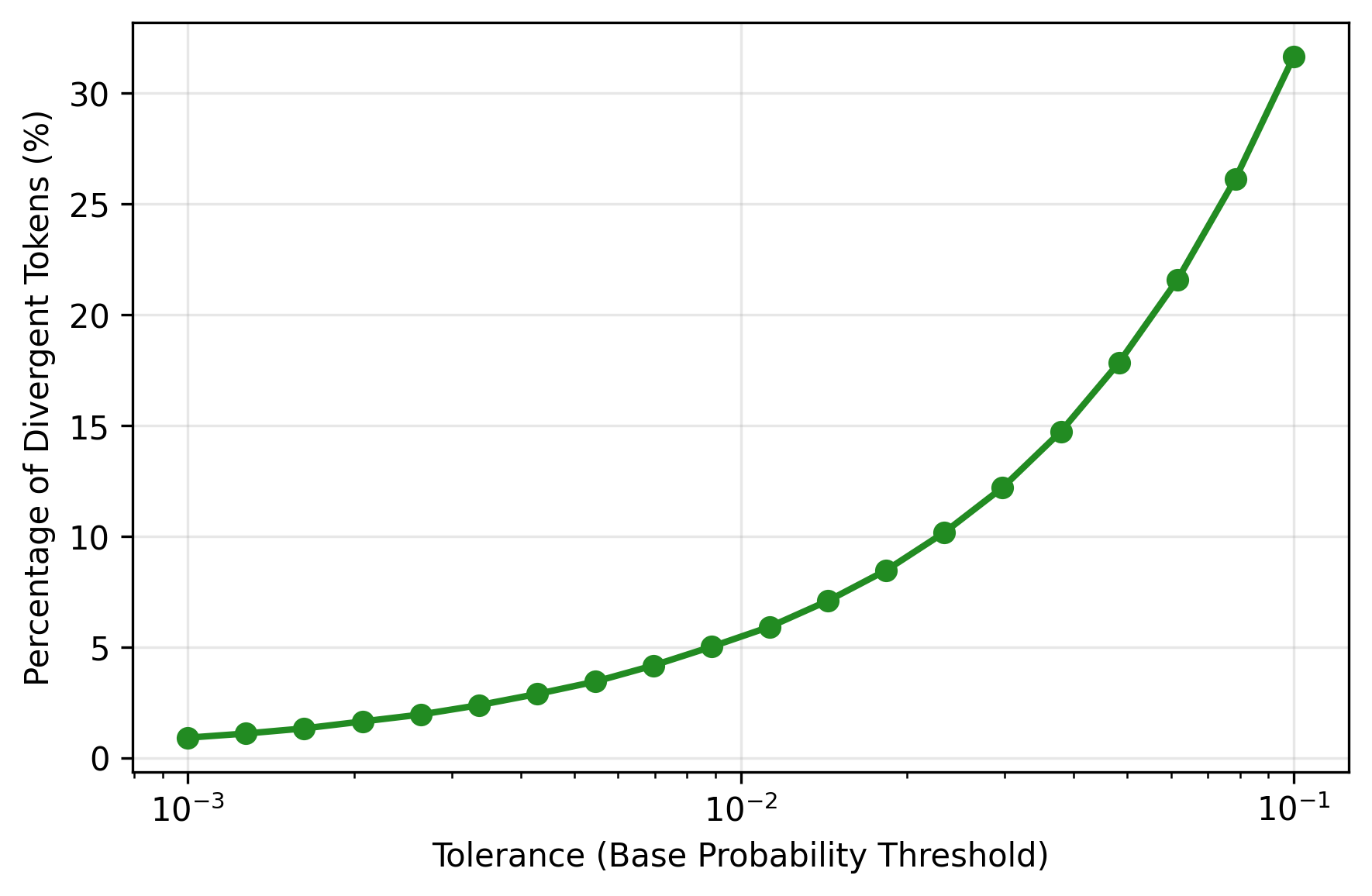}
    \caption{Percentage of divergent tokens with low base probability.}
    \end{subfigure}
    \hfill
    \begin{subfigure}{0.45\linewidth}
    \includegraphics[width=\linewidth]{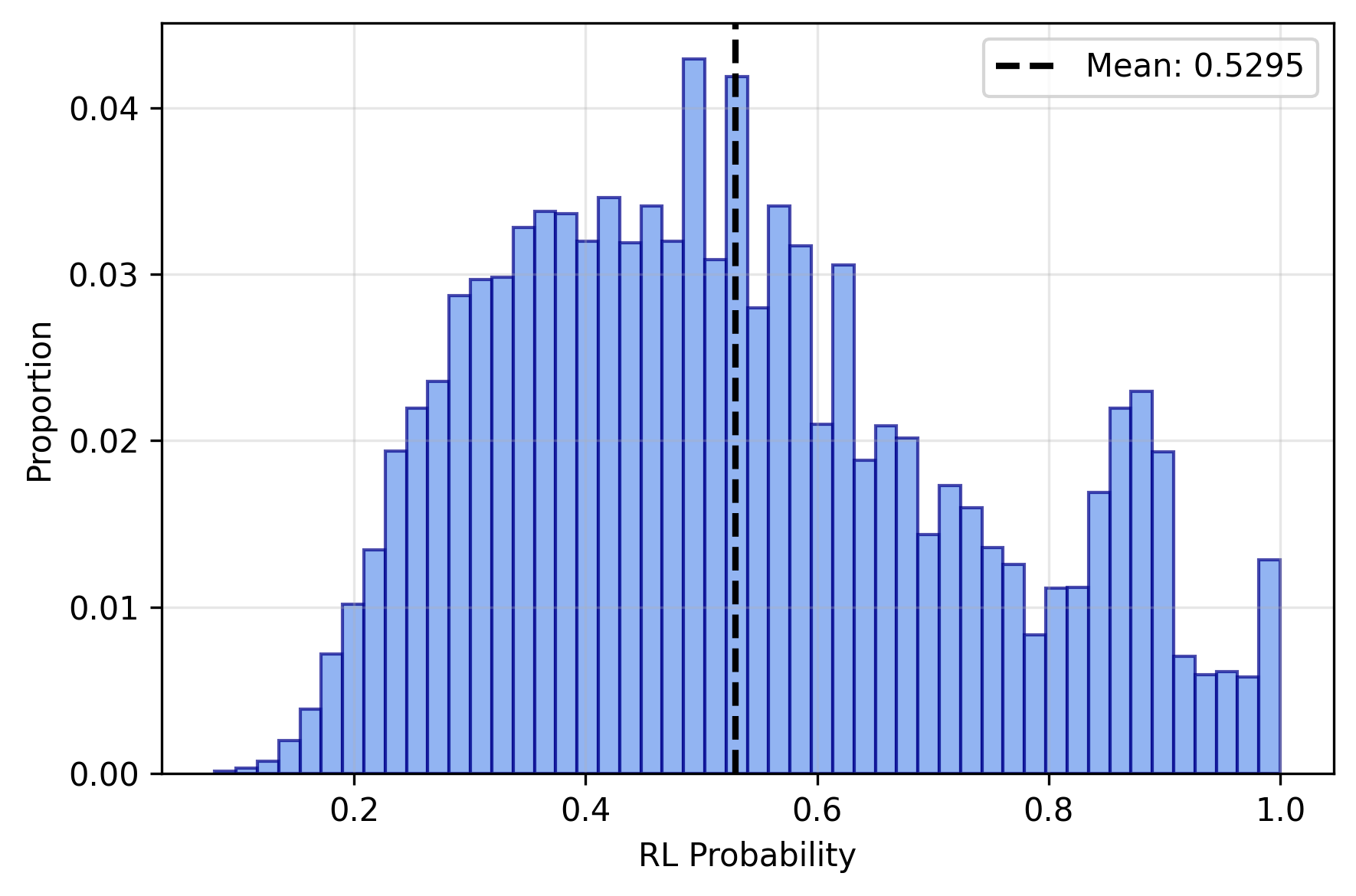}
    \caption{Histogram of RL probabilities for low base-probability tokens.}
    \end{subfigure}
    \caption{Analysis of tail behavior under DAPO for divergent token distributions ($\js>0.1$). \textbf{(a)} shows the percentage of divergent tokens whose RL top-1 choice had base probability below a given threshold. \textbf{(b)} shows the distribution of RL probabilities for the subset with base probability $<0.01$.}
    \label{fig:tolerance_vs_divergence}
    \end{figure}

\subsection{Low-Probability Behavior: Does RL Invent or Select?}
\label{subsec:tail_behavior}

We next examine whether RLVR promotes tokens that were highly unlikely under the base model, or instead amplifies alternatives that were already plausible but underweighted. For each high-divergence position, we take the RL model’s top-1 token and record its probability under the base distribution. We then compute the fraction of such tokens whose base probability falls below a given threshold among high-divergence positions.

Figure~\ref{fig:tolerance_vs_divergence} shows that under DAPO, only about 5\% (among high-divergence positions) of RL top-1 tokens have base probability below $0.01$, while under SimpleRL this fraction is nearly zero (Figure~\ref{fig:simplerl_tolerance}). Thus, even for DAPO, which encourages broader exploration through its clip-higher mechanism and lack of KL regularization, \textbf{RLVR rarely elevates tokens that were highly unlikely in the base model.} Comparing DAPO variants (Figure~\ref{fig:tolerance_7b_dapo}) with different upper clip settings, we observe that the clip-higher mechanism substantially increases the fraction of RL top-1 tokens (among high-divergence positions) whose base-model probability was initially very small compared to the variant without clip-higher. This aligns with earlier observations and supports the interpretation that clip-higher enables greater exploration, allowing tokens that were low probability in the base model distribution to be promoted more frequently. Importantly, although such low-probability promotions remain rare overall, they may still be consequential and important for improved reasoning performance.


\subsection{Evolution Across Training}
\label{sec:evolution}

Finally, we analyze how the distributional shifts develop over training. Using intermediate checkpoints from DAPO training on Qwen2.5-Math-7B, we track token-level distributions while conditioning on the final model’s (checkpoint 200) outputs, allowing us to follow the evolution of divergence for a fixed set of token sequences.

\begin{figure}[!htbp]
    \centering
    \begin{subfigure}{0.48\linewidth}
        \includegraphics[width=\linewidth]{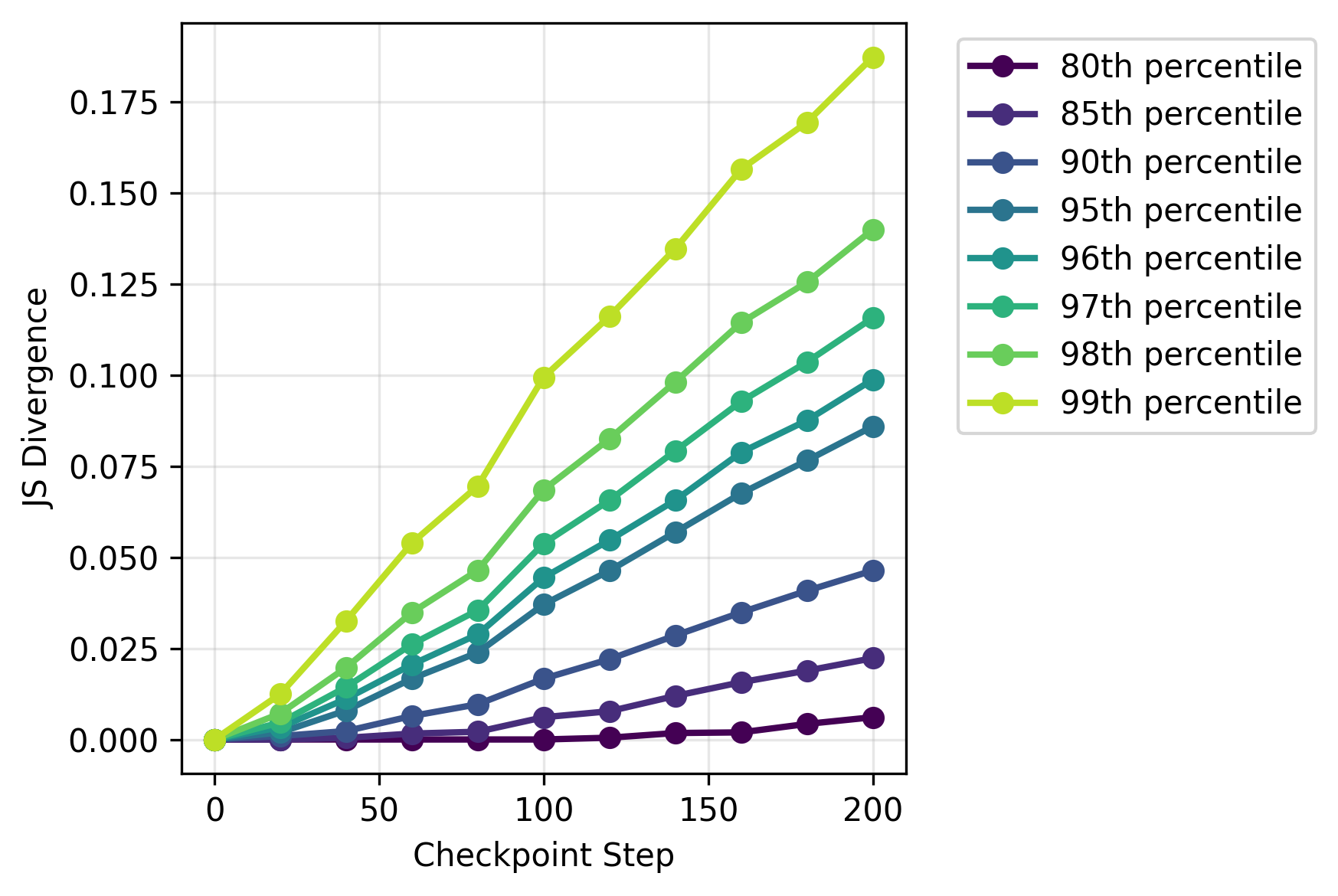}
        \caption{JS divergence percentiles. }
        \label{fig:js_percentiles}
    \end{subfigure}
    \hfill
    \begin{subfigure}{0.48\linewidth}
        \includegraphics[width=\linewidth]{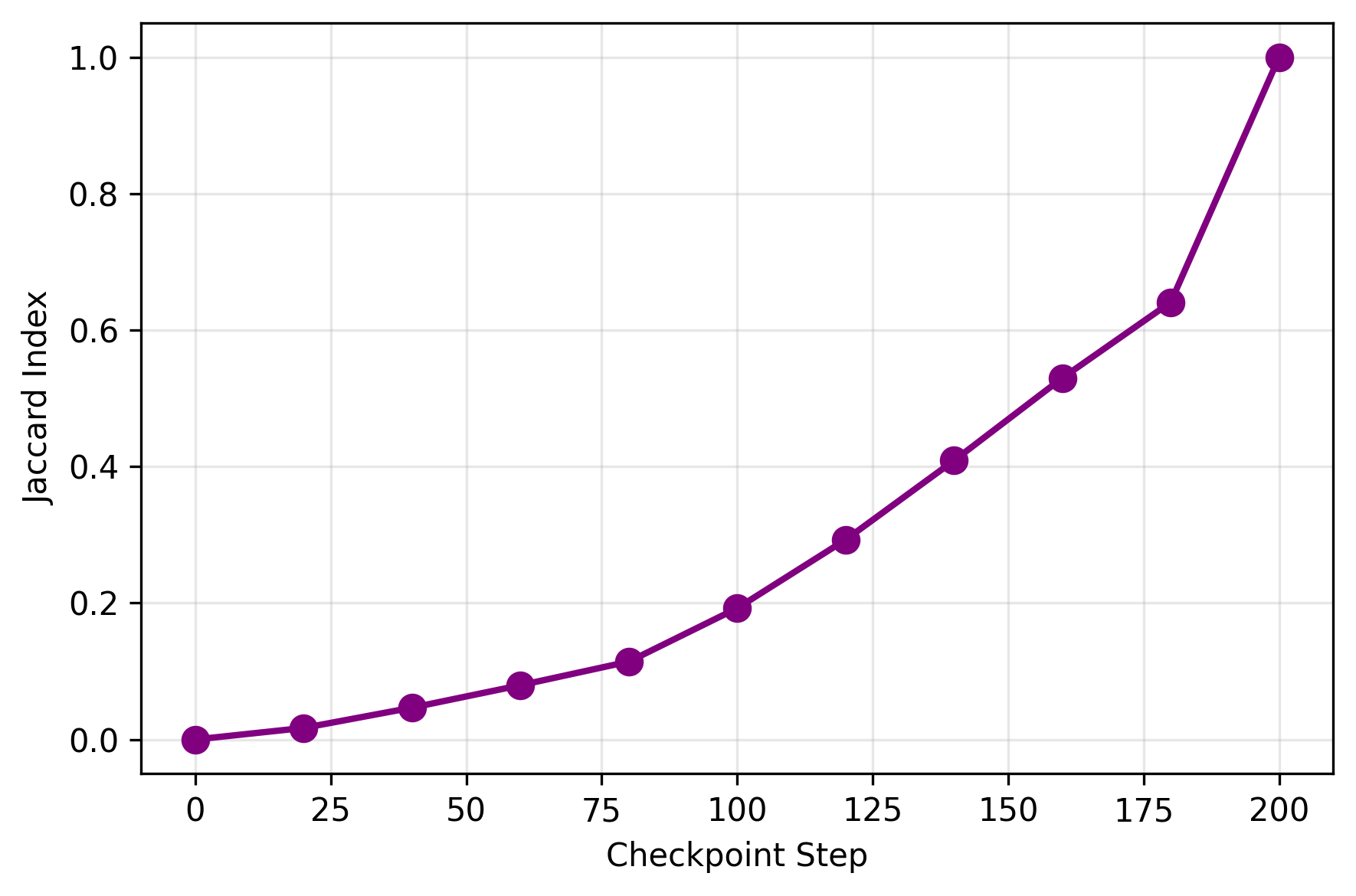}
        \caption{Jaccard index with final divergent set ($\mathrm{JS}_t > 0.1$).}
        \label{fig:jaccard}
    \end{subfigure}
    \caption{
        Distributional shifts grow increasingly focused and stable.
        Most tokens remain unchanged; updates concentrate in a sparse set late in training.
    }
    \label{fig:train_evo}
\end{figure}

Figure~\ref{fig:train_evo} shows that JS divergence increases monotonically throughout training, with higher percentiles (e.g., 95th and 99th) growing faster than lower ones. This widening gap indicates that distributional change becomes increasingly concentrated in a small subset of tokens, while the majority remain relatively stable. Consistent with this perspective, the Jaccard overlap between each checkpoint’s divergent-token set and the final set increases gradually before rising sharply near the end of training (Figure~\ref{fig:jaccard}).

\begin{tcolorbox}[
  title=Takeaways: Fine-Grained Mechanics of RLVR Distribution Shifts,
  colback=purple1!5!white,
  colframe=purple1,
  fonttitle=\bfseries
]

\textbf{At high-divergence positions, RLVR primarily reshapes probabilities within an existing candidate set rather than introducing new tokens, and this refinement becomes increasingly focused over training.}

\begin{itemize}
    \item \textbf{Shared candidate sets:} Even at high-divergence positions, base and RL models retain strong overlap in their top-$k$ token candidate sets, especially for $k \ge 2$. Distributional shifts therefore typically occur within a shared candidate set, rather than through expansion of the candidate space.
    
    \item \textbf{Re-ranking over replacement:} Within this largely shared support, RLVR mainly changes which already-plausible tokens are prioritized, frequently swapping the top-ranked choice with another token that was already among the base model’s top candidates.

    \item \textbf{Selection rather than invention:} RLVR rarely significantly promotes tokens that were highly unlikely under the base model. Most RL top-1 tokens at divergent positions already had nontrivial base probability, showing that RLVR predominantly amplifies underweighted but plausible alternatives rather than effectively inventing new ones.
    
    \item \textbf{Method differences:} More exploratory methods (e.g., clip-higher DAPO) permit larger rank movements and more frequent promotion of lower-probability tokens than more constrained methods (e.g., SimpleRL or lower clip-high settings), consistent with their broader update behavior and stronger performance gains. Notably, clip-higher produces larger reshuffling within an already plausible candidate set, whereas removing clip-higher tends to concentrate probability mass on a few dominant tokens, occasionally promoting lower-ranked base tokens but without substantially broadening the effective candidate space.
    
    \item \textbf{Overall picture:} Taken together, RLVR acts primarily as a targeted probability reallocation mechanism: sharpening and reordering choices within an existing plausibility set, rather than globally altering the (effective) token support or prediction structure. Although substantial reshuffling or promotion of low-probability tokens occurs only rarely, these events may still be important and contribute meaningfully to improved reasoning performance, a view consistent with the observation that stronger-performing models exhibit such behaviors more frequently.
\end{itemize}

\end{tcolorbox}

\FloatBarrier
\section{Exploratory Investigation: Divergence-Weighted Advantages}
\label{sec:method_design}

Our earlier analyses reveal that RL refinements are \emph{sparse and targeted}, with only a small subset of tokens exhibiting meaningful distributional change. Moreover, cross-sampling experiments demonstrate that these high-divergence tokens are functionally critical, with performance gains hinging on precisely these positions. This raises a natural question: \emph{if only a small fraction of tokens drive improvements, can training be more effectively guided by modulating token-level learning signals according to these divergences?} To investigate this possibility, we conduct a preliminary exploration of \textit{divergence-weighted advantages} as a diagnostic intervention, where advantages are reweighted by token-level distributional divergence. We explore two different approaches: \textit{high-KL boost}, which concentrates updates towards token distributions that are already changing substantially, and \textit{low-KL boost}, which focuses updates on distributions that have changed less, potentially encouraging updates in previously stable regions.

\subsection{GRPO-based Methods}
\label{subsec:grpo_dapo_objective}

\paragraph{GRPO in brief.}
GRPO \citep{shao2024deepseekmath} samples $G$ responses $\{o_i\}_{i=1}^G$ from a policy
$\pi_{\theta_{\text{old}}}(\cdot \mid q)$ for a prompt $q$ with ground-truth answer $a$, assigns sequence-level rewards
$\{R_i\}_{i=1}^G$, and computes a group-normalized advantage for each sample.
GRPO then applies a PPO-style~\citep{schulman2017ppo} clipped surrogate objective at the token level, typically with an explicit KL penalty to a reference model.

\paragraph{DAPO.}
DAPO~\citep{yu2025dapo_system} modifies GRPO with an asymmetric clip-higher mechanism, dynamic sampling of correct/incorrect completions, token-level averaging, and removal of the explicit KL penalty term. Its training objective is then given by
\begin{align*}
J_{\text{DAPO}}(\theta)
&= \; \mathbb{E}_{(q,a)\sim\mathcal{D}, \; \{o_i\}_{i=1}^G \sim \pi_{\theta_{\text{old}}}(\cdot\mid q)}
\\&\Bigg[
\frac{1}{\sum_{i=1}^{G} |o_i|}
\sum_{i=1}^{G} \sum_{t=1}^{|o_i|}
\min \Big(
r_{i,t}(\theta)\,\hat A_{i,t},\;
\mathrm{clip}\!\big(r_{i,t}(\theta),\, 1-\epsilon_{\text{low}},\, 1+\epsilon_{\text{high}}\big)\,\hat A_{i,t}
\Big)
\Bigg]
\\
&\text{s.t.}\quad 0 \;<\; \big|\{\,o_i \;\big|\; \mathtt{is\_equivalent}(a, o_i)\,\}\big| \;<\; G,
\end{align*}
with
\begin{equation*}
\label{eq:ratio_and_advantage}
r_{i,t}(\theta)
= \frac{\pi_{\theta}(o_{i,t} \mid q, o_{i,<t})}{\pi_{\theta_{\text{old}}}(o_{i,t} \mid q, o_{i,<t})},\quad \hat A_{i,t}
= \frac{R_i - \mathrm{mean}\!\big(\{R_j\}_{j=1}^G\big)}
       {\mathrm{std}\!\big(\{R_j\}_{j=1}^G\big)}.
\end{equation*}

\subsection{Divergence-Weighted Advantage}
\label{subsec:div_weighted_adv}

Standard RLVR objectives treat all tokens within a sequence uniformly in terms of their advantages (though the importance sampling ratios are defined on the token-level). Motivated by our observation that distributional shifts are sparse and concentrated, we investigate whether modulating token-level advantages according to divergence magnitude can help improve or control aspects of training. We explore modifications where advantages are rescaled depending on the per-token divergences.

\paragraph{General formulation.}
We define a divergence-weighted advantage: $\tilde{A}_t = w_t \cdot \hat{A}_t$, where $\hat{A}_t$ denotes the standard group-normalized advantage and $w_t$ is a per-token weight based on divergence. To ensure that the introduced divergence weight influences only the weighting, divergences are detached from the computation graph.

\paragraph{Choice of divergence.}
For ease of compatibility with standard frameworks, we employ KL divergence with respect to the old policy as our primary divergence quantity:
\[
\mathrm{KL}_t^{\text{old}} = D_{\mathrm{KL}}\!\left( \pi_{\theta_\text{old}}(\cdot \mid x_{<t}) \parallel  \pi_{\theta}(\cdot \mid x_{<t}) \right),
\]
where $\pi_{\theta_\text{old}}$ denotes the policy from the previous update iteration, as in PPO/GRPO. This old-policy KL quantifies the magnitude of recent policy updates at each token position, serving as a proxy for the extent of local distributional change. For computational efficiency and compatibility with existing training frameworks such as verl \citep{sheng2024hybridflow}, we estimate these quantities using KL estimators computed over sampled tokens only, which may not capture the full distributional structure. In particular, we use the $k_3$ estimator \citep{schulman2020approximatingkl} defined by 
\[
\mathrm{KL}_{\mathrm{est}}(\pi_{\theta_{\text{old}}} \parallel  \pi_{\theta})
\approx
k_3\!\left(
\frac{\pi_{\theta}(\cdot \mid x_{<t})}
     {\pi_{\theta_{\text{old}}}(\cdot \mid x_{<t})}
\right)
=
\frac{\pi_{\theta}(\cdot \mid x_{<t})}
     {\pi_{\theta_{
     \text{old}
     }}(\cdot \mid x_{<t})}
-
\log
\frac{\pi_{\theta}(\cdot \mid x_{<t})}
     {\pi_{\theta_{\text{old}}}(\cdot \mid x_{<t})}
- 1.
\]
\paragraph{Weighting scheme.}
We adopt a simple sigmoid weighting scheme (to ensure bounded weights), which transforms divergence into weights through:
\[
w_t = 1 + s \left( \sigma(\alpha \cdot \mathrm{KL}_t) - \tfrac12 \right), \quad \sigma(x) = \tfrac{1}{1+e^{-x}}.
\]
The parameter $\alpha$ controls the direction and magnitude of emphasis: $\alpha>0$ amplifies high-divergence tokens, whereas $\alpha<0$ emphasizes low-divergence ones. The sigmoid function provides a smooth, bounded nonlinear transformation that enables selective focus on either high- or low-divergence regions depending on the sign of $\alpha$. This formulation allows us to investigate whether concentrating the learning signal on regions that have already changed or those that remain unchanged yields more effective training dynamics. 

\paragraph{Evaluation.}
We train with divergence-weighted advantages using the DAPO training recipe and data on Qwen2.5-Math-7B, evaluating on AIME 2024, AIME 2025, and AMC. Results are presented in Table~\ref{tab:dw_results}. Detailed training hyperparameters and implementation details are documented in Appendix~\ref{subsubsec:train_details}.

\begin{table}[htbp]
\centering
\caption{Accuracy (\%) under divergence-weighted configurations on Qwen2.5-Math-7B. Results shown for KL divergence with $\pi_{\theta_{\text{old}}}$ and sigmoid weighting scheme across AIME 2024, AIME 2025, and AMC datasets. The results displayed are the Mean@32 scores (or, equivalently, the pass@1 scores computed using 32 samples). Results are each averaged over 3 runs.}
\label{tab:dw_results}
\begin{tabular}{lcccc}
\toprule
\textbf{Configuration} & \textbf{AIME24} & \textbf{AIME25} & \textbf{AMC} & \textbf{Overall Avg} \\
\midrule
Baseline DAPO & 33.61 & 18.75 & 75.08 & 42.48 $\pm$ 1.35 \\
Low-KL boost & 35.90 & 19.90 & 78.97 & 44.92 $\pm$ 0.05 \\
High-KL boost & 36.74 & 20.00 & 78.40 & 45.05 $\pm$ 0.79 \\
\bottomrule
\end{tabular}
\end{table}
These results demonstrate that weighting token-level updates by divergence can amplify performance gains, providing empirical support for the hypothesis that targeted tokens disproportionately drive improvements. Both low-KL and high-KL boost configurations can yield improvements over the baseline, suggesting that different divergence weighting strategies may be effective. However, the optimal choice between these approaches, and indeed whether divergence weighting provides benefits at all, may depend on the specific models and training methods used. Effective divergence weighting across training configurations may require model-specific paradigms or adaptive scheduling mechanisms to stabilize learning dynamics. We present this approach as a complementary diagnostic tool that may inform future refinements of token-level training strategies. These include approaches that aggregate information from token-level divergences into better signals, as well as those that more effectively promote the rare actions discussed in Section~\ref{sec:fine_grained}.

\section{Related Work}
\label{sec:related}

\paragraph{RL fine-tuning in LLMs.}
Reinforcement learning has become an important component of LLM fine-tuning, largely stemming from reinforcement learning with human feedback (RLHF) used to align LLM behavior to human preferences \citep{christiano2023deepreinforcementlearninghuman, ouyang2022traininglanguagemodelsfollow}.
Recently, reinforcement learning with verifiable rewards (RLVR) has emerged as a central paradigm for improving \emph{reasoning} by optimizing with verifiable reward signals of generated responses \citep{rlvr}.
A number of RLVR methods build on policy-gradient variants, including Group Relative Policy Optimization (GRPO) \citep{shao2024deepseekmath}, as well as extensions such as Dr.GRPO \citep{liu2025understandingr1zeroliketrainingcritical}, Decoupled
Clip and Dynamic sAmpling Policy Optimization (DAPO) \citep{yu2025dapo_system}, and sequence-level variants such as Group Sequence Policy Optimization (GSPO) \citep{zheng2025groupsequencepolicyoptimization}.
Beyond these core methods, several works propose complementary improvements to better target impactful updates and stabilize training, including entropy-based exploration or minority-token perspectives \citep{wang2025high_entropy_minority_tokens, cheng2025entropy_exploration_rlvr}, clipping/KL regularization strategies \citep{cui2025entropymechanismreinforcementlearning}, reweighting based on token probability, perplexity, or position \citep{yang2025letlowprobabilitytokensoverdominate, deng2025trial_to_improvement_rlvr}, as well as analyses of different reward designs \citep{shao2025spuriousrewardsrethinkingtraining}.
Additional works have also studied the reasoning boundaries of RLVR and explore ways to expand it \citep{yue2025doesreinforcementlearningreally, wen2025rlvr_incentivizes_correct_reasoning, liu2025prorlprolongedreinforcementlearning}.

\paragraph{Understanding RLVR and its differences with SFT.}
A growing body of work argues that RL fine-tuning often acts as a \emph{scalpel rather than a hammer}, amplifying existing capabilities through localized changes. This perspective is supported mainly through evaluations on different domains/tasks, catastrophic forgetting, parameter changes, and overall KL divergence \citep{rajani2025scalpel_vs_hammer, chu2025sft_memorizes_rl_generalizes, shenfeld2025rls_razor, huan2025doesmathreasoningimprove}. Recent work also analyze locality from the \emph{parameter} perspective: for example, \citet{mukherjee2025reinforcementlearningfinetunessmall} report that RL fine-tuning can concentrate effective updates into relatively small subnetworks, while \citet{zhu2025pathtakenrlvrprovably} provide theoretical insight into RLVR learning dynamics and the structure of these parameter-sparse updates. Our paper complements these perspectives by focusing on quantifying the changes induced by RLVR at the level of \emph{token distributions}. We show that RL not only induces smaller aggregate divergence than SFT, but that its changes are substantially sparser at the token level. Moreover, we connect these sparse distributional changes to sequence-level reasoning outcomes and to fine-grained reallocations of probability mass.

\paragraph{Token-Level analyses of RLVR.}
Several works seek to understand RLVR through a token-level lens.
\citet{wang2025high_entropy_minority_tokens} attribute a substantial portion of RL gains to high-entropy minority tokens, while \citet{cheng2025entropy_exploration_rlvr} connect such tokens to exploratory reasoning steps; related work also highlights entropy collapse risks and token-level regularization mechanisms \citep{cui2025entropymechanismreinforcementlearning}.
Other studies emphasize the disproportionate role of specific tokens or sampling decisions \citep{vassoyan2025ignoreklpenaltyboosting, lin2025criticaltokensmattertokenlevel, karan2025reasoningsamplingbasemodel}, and \citet{huan2025doesmathreasoningimprove} analyzes RL-induced changes using token-level KL divergence and token rank shifts.
Recent work by \citet{chen2026reshaping} shows that RL training modifies a sparse subset of tokens when viewed through rank-shift statistics, and develops a theoretical analysis based on reasoning patterns. Our contributions are complementary but distinct: we conduct a systematic empirical study of RLVR-induced token-level changes through the lens of quantities such as divergence, entropy, and probability mass redistribution, and connect these shifts to sequence-level reasoning via forward and reverse cross-sampling interventions that establish their impact on reasoning performance.

\section{Conclusion}
\label{sec:conclusion}

Our study reveals that reinforcement learning with verifiable rewards (RLVR) reshapes LLMs in a manner that is sparse, targeted, and structured rather than uniformly diffused across
tokens. By analyzing token-level distributional shifts, we show that only a small subset of tokens
undergo meaningful divergence, and that these divergences carry disproportionate functional
importance: cross-sampling interventions confirm that performance gains hinge on precisely these positions. Moreover, our fine-grained analyses suggest that, even at high-divergence positions, RLVR typically refines behavior by reallocating probability mass within an existing candidate set rather than introducing fundamentally new tokens. However, the comparatively rare cases of substantial re-ranking and promotion of initially low-probability tokens may still be important for the observed improvements in reasoning performance. To complement these analyses, we explored divergence-weighted advantage, a simple modification that scales token-level advantages by per-token divergence. Our results suggest that such weighting strategies can influence learning dynamics, though stabilizing performance may require
model-specific choices and further investigation.

Together, these findings advance a token-level understanding of RL fine-tuning. They highlight that
the essence of RLVR’s success lies not in widespread distributional changes, but in selective refinements aligned with varying entropy levels. Taken together, our findings suggest that RLVR operates not as a global policy shift, but as a sparse intervention on a small set of high-impact decision points that steer generation trajectories. Beyond clarifying the mechanics of existing methods,
our work offers a perspective for designing future RL objectives and diagnostics that further incorporate distributional structure, opening avenues for more effective, interpretable, and controllable LLM post-training.

\section{Authors}
\label{sec:authors}
Haoming Meng\footnote{University of Toronto} \quad Kexin Huang \quad Shaohang Wei\footnote{Peking University} \quad Chiyu Ma\footnote{Dartmouth College} \quad Shuo Yang\footnotemark[2] \quad Xue Wang \quad \\Guoyin Wang\footnote{Alibaba Group}\textsuperscript{,}\daggerfootnote{Qwen Pilot Team Lead} \quad Bolin Ding\footnotemark[4] \quad Jingren Zhou\footnotemark[4]
\clearpage
\bibliography{biblio}
\bibliographystyle{colm2024_conference}

\appendix
\section{Appendix}


\subsection{Sequence-level Divergence Bounds for Cross-Sampling}
\label{subsec:seq_div_bounds}

\paragraph{Setup.}
Let $(X_t)_{t\ge 1}$ be the decoded tokens (each in $\mathcal{V}$), with stopping time
\(
\tau := \inf\{t\ge 1: X_t=\text{EOS}\}\wedge T_{\max}.
\)
We may work on the fixed horizon $T_{\max}$ by absorbing the EOS token:
once EOS is generated, the process deterministically outputs EOS thereafter. This yields an equivalent
distribution over $X_{1:T_{\max}}$ and ensures all sequences are of length $T_{\max}$.
All results below therefore sum over $t=1,\dots,T_{\max}$; terms after $\tau$ contribute zero since both
policies become point masses on EOS.

Let $\pi_{\mathrm{prim}}$ be the primary policy and $\pi_{\mathrm{int}}$ the intervention policy.
Given a switching rule $\mathcal{S}:\mathcal{V}^{<\mathbb{N}}\to\{0,1\}$, define
$S_t := \mathcal{S}(X_{<t})$ and the mixed policy
\[
\pi_{\mathrm{mix}}(\cdot\mid X_{<t})
=
(1-S_t)\,\pi_{\mathrm{prim}}(\cdot\mid X_{<t})
+
S_t\,\pi_{\mathrm{int}}(\cdot\mid X_{<t}).
\]
Let $P_{\mathrm{mix}}$ and $P_{\mathrm{int}}$ denote the induced sequence-level distributions on
$X_{1:T_{\max}}$.

\subsubsection{KL Case}
We first provide a bound on the sequence-level divergence between the cross-sampled policy and the target intervention policy, in the simpler case of KL divergence with a token-level KL switching rule.

\begin{lemma}[KL decomposition]
\label{lem:kl_chain_rule}
Let $P,Q$ be probability distributions on $X_{1:T_{\max}}$ admitting factorizations
$P(x_{1:T_{\max}})=\prod_{t=1}^{T_{\max}} p_t(x_t\mid x_{<t})$ and
$Q(x_{1:T_{\max}})=\prod_{t=1}^{T_{\max}} q_t(x_t\mid x_{<t})$. For each $t$, let $P_{<t}$ and $Q_{<t}$ denote the marginals of $X_{<t}$ under $P$ and $Q$, respectively.
Then
\[
\KL(P\parallel Q)
=
\sum_{t=1}^{T_{\max}}
\EE_{X_{<t}\sim P_{<t}}\!\left[
\KL\!\big(p_t(\cdot\mid X_{<t}) \parallel \ q_t(\cdot\mid X_{<t})\big)
\right].
\]
\end{lemma}

\begin{proof}
By definition,
\(
\KL(P\parallel Q)=\EE_{X\sim P}\!\left[\log\frac{P(X)}{Q(X)}\right].
\)
Using the factorizations,
\[
\log\frac{P(X_{1:T_{\max}})}{Q(X_{1:T_{\max}})}
=
\sum_{t=1}^{T_{\max}}
\log\frac{p_t(X_t\mid X_{<t})}{q_t(X_t\mid X_{<t})}.
\]
Taking expectation under $P$ and exchanging sum and expectation gives
\begin{align*}
\KL(P\parallel Q)
&=
\sum_{t=1}^{T_{\max}}
\EE_{X\sim P}\!\left[\log\frac{p_t(X_t\mid X_{<t})}{q_t(X_t\mid X_{<t})}\right] \\
&= \sum_{t=1}^{T_{\max}}
\EE_{X_{<t}\sim P_{<t}}\!\left[
\EE_{P}\!\left[
\left.
\log\frac{p_t(X_t\mid X_{<t})}{q_t(X_t\mid X_{<t})}
\,\right|\, X_{<t}
\right]
\right] \\
&=
\sum_{t=1}^{T_{\max}}
\EE_{X_{<t}\sim P_{<t}}\!\left[
\KL\!\big(p_t(\cdot\mid X_{<t})\parallel q_t(\cdot\mid X_{<t})\big)
\right],
\end{align*}

where the second last equality follows by the definition of conditional expectation (or the law of total expectation) applied to the conditional distribution of $X_t$ given $X_{<t}$.

\end{proof}

\begin{proposition}[Token-level KL threshold $\Rightarrow$ sequence-level KL bound]
\label{prop:kl_eps_bound}
Assume the switching rule is defined by a KL threshold:
\[
\mathcal{S}(x_{<t})
=
\mathds{1}\!\left\{
\KL\!\big(\pi_{\mathrm{prim}}(\cdot\mid x_{<t}) \parallel \pi_{\mathrm{int}}(\cdot\mid x_{<t})\big)>\varepsilon
\right\}.
\]
Define the number of \emph{non-intervention} steps on a trajectory
\[
N_0
:=
\sum_{t=1}^{\tau}
\mathds{1}\{\mathcal{S}(X_{<t})=0\}.
\]
Then
\[
\KL(P_{\mathrm{mix}}\parallel P_{\mathrm{int}})
\le
\varepsilon\ \EE_{X\sim P_{\mathrm{mix}}}[N_0].
\]
\end{proposition}
\begin{proof}
We apply Lemma~\ref{lem:kl_chain_rule} with $P=P_{\mathrm{mix}}$ and $Q=P_{\mathrm{int}}$.
For any history $h=x_{<t}$, since $\mathcal{S}(h)\in\{0,1\}$,
\[
\pi_{\mathrm{mix}}(\cdot\mid h)=
\begin{cases}
\pi_{\mathrm{int}}(\cdot\mid h), & \mathcal{S}(h)=1,\\
\pi_{\mathrm{prim}}(\cdot\mid h), & \mathcal{S}(h)=0.
\end{cases}
\]
Hence
\[
\KL\!\big(\pi_{\mathrm{mix}}(\cdot\mid h) \parallel \pi_{\mathrm{int}}(\cdot\mid h)\big)
=
\mathds{1}\{\mathcal{S}(h)=0\}\,
\KL\!\big(\pi_{\mathrm{prim}}(\cdot\mid h) \parallel \pi_{\mathrm{int}}(\cdot\mid h)\big).
\]
By the definition of $\mathcal{S}$, whenever $\mathcal{S}(h)=0$ the token-level KL $\KL\!\big(\pi_{\mathrm{prim}}(\cdot\mid h) \parallel \pi_{\mathrm{int}}(\cdot\mid h)\big)$ is at most $\varepsilon$.
So, for all $h$,
\[
\KL\!\big(\pi_{\mathrm{mix}}(\cdot\mid h) \parallel \pi_{\mathrm{int}}(\cdot\mid h)\big)
\le
\varepsilon\,\mathds{1}\{\mathcal{S}(h)=0\}.
\]
Under absorbing EOS, 
\(
\KL\!\big(\pi_{\text{mix}}(\cdot\mid X_{<t}) \,\parallel\, \pi_{\text{int}}(\cdot\mid X_{<t})\big)=0
\)
for $t>\tau$, so we may multiply by $\mathds{1}\{t\le \tau\}$.
Substituting this into Lemma~\ref{lem:kl_chain_rule} yields
\[
\KL(P_{\mathrm{mix}}\parallel P_{\mathrm{int}})
\le
\sum_{t=1}^{T_{\max}}
\EE_{X_{<t}\sim (P_{\mathrm{mix}})_{<t}}\!\left[
\mathds{1}\{t\le \tau\}\,\varepsilon\,\mathds{1}\{S_t=0\}
\right]
=
\varepsilon
\sum_{t=1}^{T_{\max}}
\EE_{X_{<t}\sim (P_{\mathrm{mix}})_{<t}}\!\left[
\mathds{1}\{t\le \tau\}\,\mathds{1}\{S_t=0\}
\right].
\]
Note that $\{t\le \tau\}=\{\text{EOS}\notin X_{<t}\}$ depends only on $X_{<t}$, so $\mathds{1}\{t\le \tau\}\mathds{1}\{S_t=0\}$ is $\sigma(X_{<t})$-measurable. Thus, we may equivalently write
the expectation under the full trajectory $X\sim P_{\mathrm{mix}}$:
\[
\EE_{X_{<t}\sim (P_{\mathrm{mix}})_{<t}}\!\left[
\mathds{1}\{t\le \tau\}\,\mathds{1}\{S_t=0\}
\right]
=
\EE_{X\sim P_{\mathrm{mix}}}\!\left[
\mathds{1}\{t\le \tau\}\,\mathds{1}\{S_t=0\}
\right].
\]
Then,
\begin{align*}
\KL(P_{\mathrm{mix}}\parallel P_{\mathrm{int}})
&\le\varepsilon
\sum_{t=1}^{T_{\max}}
\EE_{X_{<t}\sim (P_{\mathrm{mix}})_{<t}}\!\left[
\mathds{1}\{t\le \tau\}\,\mathds{1}\{S_t=0\}
\right]\\
&=
\varepsilon
\sum_{t=1}^{T_{\max}}
\EE_{X\sim P_{\mathrm{mix}}}\!\left[
\mathds{1}\{t\le \tau\}\,\mathds{1}\{S_t=0\}
\right] \\
&=\varepsilon\,
\EE_{X\sim P_{\mathrm{mix}}}\!\left[
\sum_{t=1}^{T_{\max}}
\mathds{1}\{t\le \tau\}\,\mathds{1}\{S_t=0\}
\right] \\
&=\varepsilon\,
\EE_{X\sim P_{\mathrm{mix}}}\!\left[
\sum_{t=1}^{\tau}\mathds{1}\{S_t=0\}
\right] \\
&= \varepsilon\,\EE_{X\sim P_{\mathrm{mix}}}[N_0].
\end{align*}
\end{proof}

\begin{remark}[Effective KL on non-intervention steps]
Define the \emph{effective} token-level KL on non-intervention steps
\[
\bar{\kappa}
:=
\frac{
\EE_{X\sim P_{\mathrm{mix}}}\!\left[
\sum_{t=1}^{\tau}\mathds{1}\{S_t=0\}\,
\KL\!\big(\pi_{\mathrm{prim}}(\cdot\mid X_{<t})\parallel \pi_{\mathrm{int}}(\cdot\mid X_{<t})\big)
\right]
}{
\EE_{X\sim P_{\mathrm{mix}}}[N_0]
},
\qquad (\bar{\kappa}:=0 \text{ if } \EE[N_0]=0).
\]
Then
\[
\KL(P_{\mathrm{mix}}\parallel P_{\mathrm{int}})
=
\bar{\kappa}\,\EE_{X\sim P_{\mathrm{mix}}}[N_0]
\le
\varepsilon\,\EE_{X\sim P_{\mathrm{mix}}}[N_0],
\]
and typically $\bar{\kappa}\ll \varepsilon$ in regimes where the models are already close (eg. in the setting of a base model and its RL fine-tuned counterpart).
\end{remark}

\subsubsection{JS Case}

We now give the corresponding sequence-level result for Jensen--Shannon divergence.
Unlike KL divergence, the exact decomposition involves a history-dependent skew
Jensen--Shannon divergence rather than the ordinary JS used in our experiments.

\begin{definition}[Skew Jensen--Shannon divergence]
For probability measures $\mu,\nu$ on a common measurable space and $\alpha\in[0,1]$, define
\[
\JS^{\alpha}(\mu\parallel \nu)
:=
\alpha\,
\KL\!\big(\mu \parallel \alpha \mu + (1-\alpha)\nu\big)
+
(1-\alpha)\,
\KL\!\big(\nu \parallel \alpha \mu + (1-\alpha)\nu\big).
\]
When $\alpha=\tfrac12$, this reduces to the usual Jensen--Shannon divergence:
\[
\JS(\mu\parallel \nu)=\JS^{1/2}(\mu\parallel \nu).
\]
\end{definition}

\begin{lemma}[JS decomposition via skew JS]
\label{lem:js_chain_rule}
Let $P,Q$ be probability distributions on $X_{1:T_{\max}}$ admitting factorizations \(
P(x_{1:T_{\max}})=\prod_{t=1}^{T_{\max}} p_t(x_t\mid x_{<t})\) and \(Q(x_{1:T_{\max}})=\prod_{t=1}^{T_{\max}} q_t(x_t\mid x_{<t})
\) respectively.
Let \(
M:=\tfrac12(P+Q)
\) and for each $t$ let $P_{<t},Q_{<t},M_{<t}$ denote the marginals of $X_{<t}$
under $P,Q,M$, respectively.

Define, for $M_{<t}$-almost every history $h=x_{<t}$,
\[
\alpha_t(h)
:=
\frac12\frac{dP_{<t}}{dM_{<t}}(h)
=
\frac{P_{<t}(h)}{P_{<t}(h)+Q_{<t}(h)}.
\]
Then
\[
\JS(P\parallel Q)
=
\sum_{t=1}^{T_{\max}}
\EE_{H\sim M_{<t}}\!\left[
\JS^{\alpha_t(H)}\!\big(p_t(\cdot\mid H)\parallel q_t(\cdot\mid H)\big)
\right].
\]
\end{lemma}

\begin{proof}
By definition, we have
\[
\JS(P\parallel Q)
=
\frac12\KL(P\parallel M)+\frac12\KL(Q\parallel M),
\qquad
M=\tfrac12(P+Q).
\]
For each $t$, since \(
M_{<t}=\tfrac12(P_{<t}+Q_{<t})
\), we have $P_{<t},Q_{<t}\ll M_{<t}$, so the Radon--Nikodym derivatives
$\frac{dP_{<t}}{dM_{<t}}$ and $\frac{dQ_{<t}}{dM_{<t}}$ exist. Since the prefix space is discrete, these Radon--Nikodym derivatives reduce to ratios of probability masses wherever $M_{<t}(h)>0$.

For each $t$ and each history $h=x_{<t}$ with $M_{<t}(h)>0$, define
\[
m_t(\cdot\mid h)
:=
M(X_t\in\cdot \mid X_{<t}=h).
\]
We first identify this conditional law. For any $x_t\in\mathcal V$,
\[
m_t(x_t\mid h)
=
\frac{M(X_{1:t}=(h,x_t))}{M(X_{<t}=h)}
=
\frac{\frac12 P_{<t}(h)p_t(x_t\mid h)+\frac12 Q_{<t}(h)q_t(x_t\mid h)}
{\frac12 P_{<t}(h)+\frac12 Q_{<t}(h)},
\]
hence
\[
m_t(\cdot\mid h)
=
\alpha_t(h)\,p_t(\cdot\mid h)
+
\bigl(1-\alpha_t(h)\bigr)\,q_t(\cdot\mid h).
\]
Since 
\(
M(x_{1:T_{\max}})
=
\prod_{t=1}^{T_{\max}} m_t(x_t\mid x_{<t}),
\)
Lemma~\ref{lem:kl_chain_rule}, applied to $(P,M)$ and $(Q,M)$ yields
\begin{align*}
\KL(P\parallel M)
&=
\sum_{t=1}^{T_{\max}}
\EE_{H\sim P_{<t}}\!\left[
\KL\!\big(p_t(\cdot\mid H)\parallel m_t(\cdot\mid H)\big)
\right], \\
\KL(Q\parallel M)
&=
\sum_{t=1}^{T_{\max}}
\EE_{H\sim Q_{<t}}\!\left[
\KL\!\big(q_t(\cdot\mid H)\parallel m_t(\cdot\mid H)\big)
\right].
\end{align*}
Thus,
\begin{align*}
\JS(P\parallel Q)
&=
\frac12\KL(P\parallel M)+\frac12\KL(Q\parallel M) \\
&=
\frac12\sum_{t=1}^{T_{\max}}
\EE_{H\sim P_{<t}}\!\left[
\KL\!\big(p_t(\cdot\mid H)\parallel m_t(\cdot\mid H)\big)
\right]
+
\frac12\sum_{t=1}^{T_{\max}}
\EE_{H\sim Q_{<t}}\!\left[
\KL\!\big(q_t(\cdot\mid H)\parallel m_t(\cdot\mid H)\big)
\right].
\end{align*}
By definition of $\alpha_t$, we have
\(
\frac12\frac{dP_{<t}}{dM_{<t}}(H)=\alpha_t(H)\) and thus \(\frac12\frac{dQ_{<t}}{dM_{<t}}(H)=1-\alpha_t(H)\) 
\(M_{<t}\text{-a.s.}
\) (since $M_{<t}=\tfrac12 P_{<t}+\tfrac12 Q_{<t}$).

Then by the Radon--Nikodym Theorem,
\begin{align*}
&\frac12
\EE_{H\sim P_{<t}}\!\left[
\KL\!\big(p_t(\cdot\mid H)\parallel m_t(\cdot\mid H)\big)
\right]
+
\frac12
\EE_{H\sim Q_{<t}}\!\left[
\KL\!\big(q_t(\cdot\mid H)\parallel m_t(\cdot\mid H)\big)
\right] \\
&\qquad=
\EE_{H\sim M_{<t}}\!\Big[
\frac12\frac{dP_{<t}}{dM_{<t}}(H)\,
\KL\!\big(p_t(\cdot\mid H)\parallel m_t(\cdot\mid H)\big)
+
\frac12\frac{dQ_{<t}}{dM_{<t}}(H)\,
\KL\!\big(q_t(\cdot\mid H)\parallel m_t(\cdot\mid H)\big)
\Big] \\
&\qquad=
\EE_{H\sim M_{<t}}\!\Big[
\alpha_t(H)\,
\KL\!\big(p_t(\cdot\mid H)\parallel m_t(\cdot\mid H)\big)
+
\bigl(1-\alpha_t(H)\bigr)\,
\KL\!\big(q_t(\cdot\mid H)\parallel m_t(\cdot\mid H)\big)
\Big] \\
&\qquad=
\EE_{H\sim M_{<t}}\!\left[
\JS^{\alpha_t(H)}\!\big(p_t(\cdot\mid H)\parallel q_t(\cdot\mid H)\big)
\right].
\end{align*}
Summing over $t=1,\dots,T_{\max}$ gives the claim.

\end{proof}

\begin{proposition}[Token-level skew-JS control $\Rightarrow$ sequence-level JS bound]
\label{prop:js_eps_bound}
Let \(
P_{\mathrm{mix}},\; P_{\mathrm{int}}
\) be the sequence-level laws induced by $\pi_{\mathrm{mix}}$ and $\pi_{\mathrm{int}}$ on
$X_{1:T_{\max}}$, and let
\(
M:=\tfrac12\bigl(P_{\mathrm{mix}}+P_{\mathrm{int}}\bigr).
\)
For each $t$, let $(P_{\mathrm{mix}})_{<t}$, $(P_{\mathrm{int}})_{<t}$, and $M_{<t}$
denote the corresponding marginals of $X_{<t}$, and define
\[
\alpha_t(h)
:=
\frac12\frac{d(P_{\mathrm{mix}})_{<t}}{dM_{<t}}(h)
=
\frac{(P_{\mathrm{mix}})_{<t}(h)}
{(P_{\mathrm{mix}})_{<t}(h)+(P_{\mathrm{int}})_{<t}(h)},
\qquad M_{<t}\text{-a.s.}
\]

Assume that the switching rule satisfies, for every $t$ and history $h=x_{<t}$ with $M_{<t}(h)>0$,
\[
\mathcal{S}(h)=0
\quad\Longrightarrow\quad
\JS^{\alpha_t(h)}\!\big(
\pi_{\mathrm{prim}}(\cdot\mid h)\parallel \pi_{\mathrm{int}}(\cdot\mid h)
\big)
\le \varepsilon.
\]
Define
\[
N_0
:=
\sum_{t=1}^{\tau}\mathds{1}\{\mathcal{S}(X_{<t})=0\}.
\]
Then
\[
\JS(P_{\mathrm{mix}}\parallel P_{\mathrm{int}})
\le
\varepsilon\ \EE_{X\sim M}[N_0].
\]
Equivalently,
\[
\JS(P_{\mathrm{mix}}\parallel P_{\mathrm{int}})
\le
\frac{\varepsilon}{2}
\left(
\EE_{X\sim P_{\mathrm{mix}}}[N_0]
+
\EE_{X\sim P_{\mathrm{int}}}[N_0]
\right).
\]
\end{proposition}

\begin{proof}
We apply Lemma~\ref{lem:js_chain_rule} with
\(
P=P_{\mathrm{mix}}\) and \(Q=P_{\mathrm{int}}\).

For any history $h=x_{<t}$,
\[
\pi_{\mathrm{mix}}(\cdot\mid h)
=
\begin{cases}
\pi_{\mathrm{int}}(\cdot\mid h), & \mathcal{S}(h)=1,\\
\pi_{\mathrm{prim}}(\cdot\mid h), & \mathcal{S}(h)=0.
\end{cases}
\]
Hence
\[
\JS^{\alpha_t(h)}\!\big(
\pi_{\mathrm{mix}}(\cdot\mid h)\parallel \pi_{\mathrm{int}}(\cdot\mid h)
\big)
=
\mathds{1}\{\mathcal{S}(h)=0\}\,
\JS^{\alpha_t(h)}\!\big(
\pi_{\mathrm{prim}}(\cdot\mid h)\parallel \pi_{\mathrm{int}}(\cdot\mid h)
\big).
\]
By the assumed token-level control, for every such history $h$,
\[
\JS^{\alpha_t(h)}\!\big(
\pi_{\mathrm{mix}}(\cdot\mid h)\parallel \pi_{\mathrm{int}}(\cdot\mid h)
\big)
\le
\varepsilon\,\mathds{1}\{\mathcal{S}(h)=0\}.
\]

Under absorbing EOS, \(
\JS^{\alpha_t(X_{<t})}\!\big(
\pi_{\mathrm{mix}}(\cdot\mid X_{<t})\parallel \pi_{\mathrm{int}}(\cdot\mid X_{<t})
\big)=0\) for $t>\tau$, so we may multiply by $\mathds{1}\{t\le \tau\}$. Substituting this into
Lemma~\ref{lem:js_chain_rule} gives
\[
\JS(P_{\mathrm{mix}}\parallel P_{\mathrm{int}})
\le
\sum_{t=1}^{T_{\max}}
\EE_{X_{<t}\sim M_{<t}}\!\left[
\mathds{1}\{t\le \tau\}\,\varepsilon\,\mathds{1}\{S_t=0\}
\right].
\]
Since $\mathds{1}\{t\le \tau\}\mathds{1}\{S_t=0\}$ is $\sigma(X_{<t})$-measurable, we may equivalently view the expectation under the full trajectory
$X\sim M$:
\[
\EE_{X_{<t}\sim M_{<t}}\!\left[
\mathds{1}\{t\le \tau\}\,\mathds{1}\{S_t=0\}
\right]
=
\EE_{X\sim M}\!\left[
\mathds{1}\{t\le \tau\}\,\mathds{1}\{S_t=0\}
\right].
\]
Thus,
\begin{align*}
\JS(P_{\mathrm{mix}}\parallel P_{\mathrm{int}})
&\le
\varepsilon
\sum_{t=1}^{T_{\max}}
\EE_{X\sim M}\!\left[
\mathds{1}\{t\le \tau\}\,\mathds{1}\{S_t=0\}
\right] \\
&=
\varepsilon\,
\EE_{X\sim M}\!\left[
\sum_{t=1}^{T_{\max}}
\mathds{1}\{t\le \tau\}\,\mathds{1}\{S_t=0\}
\right] \\
&=
\varepsilon\,
\EE_{X\sim M}\!\left[
\sum_{t=1}^{\tau}\mathds{1}\{S_t=0\}
\right] \\
&=
\varepsilon\,\EE_{X\sim M}[N_0].
\end{align*}
Finally, since $M=\tfrac12(P_{\mathrm{mix}}+P_{\mathrm{int}})$,
\[
\EE_{X\sim M}[N_0]
=
\frac12\EE_{X\sim P_{\mathrm{mix}}}[N_0]
+
\frac12\EE_{X\sim P_{\mathrm{int}}}[N_0],
\]
giving the equivalent form.

\end{proof}

\begin{remark}[Effective skew-JS on non-intervention steps]
Let
\[
M:=\tfrac12(P_{\mathrm{mix}}+P_{\mathrm{int}}),
\]
and define the \emph{effective} token-level skew Jensen--Shannon divergence on
non-intervention steps by
\[
\bar{\jmath}
:=
\frac{
\EE_{X\sim M}\!\left[
\sum_{t=1}^{\tau}
\mathds{1}\{S_t=0\}\,
\JS^{\alpha_t(X_{<t})}\!\big(
\pi_{\mathrm{prim}}(\cdot\mid X_{<t})
\parallel
\pi_{\mathrm{int}}(\cdot\mid X_{<t})
\big)
\right]
}{
\EE_{X\sim M}[N_0]
},
\qquad
(\bar{\jmath}:=0 \text{ if } \EE_{X\sim M}[N_0]=0).
\]
Then
\[
\JS(P_{\mathrm{mix}}\parallel P_{\mathrm{int}})
=
\bar{\jmath}\,\EE_{X\sim M}[N_0].
\]
Moreover, under the assumption of Proposition~\ref{prop:js_eps_bound}, \(
\bar{\jmath}\le \varepsilon,
\) and hence
\[
\JS(P_{\mathrm{mix}}\parallel P_{\mathrm{int}})
=
\bar{\jmath}\,\EE_{X\sim M}[N_0]
\le
\varepsilon\,\EE_{X\sim M}[N_0].
\]
\end{remark}

Here $\bar{\jmath}$ averages the skew Jensen--Shannon terms appearing in the exact
decomposition of Lemma~\ref{lem:js_chain_rule}; this is the natural JS analogue of
the effective KL in the preceding remark.

\subsection{Experimental Details}

\subsubsection{Token Distribution Analyses}
We run model inference using \texttt{vllm} \citep{kwon2023efficient}. On AIME 2024 and 2025, we apply nucleus sampling \citep{holtzman2020curiouscaseneuraltext} with top-$p=0.7$ and $\text{temperature}=1$. For divergence calculations on AIME, we use the top-$p$ truncated distribution to reflect the effective sampling distribution, to provide a more accurate estimate for our cross-sampling experiments. We also examine the distribution of JS divergence values without truncation (and on other top-$p$ values) to ensure the main results are not impacted by the truncation. For experiments on the fine-training data, we use top-$p=1$ to reflect the training sampling distribution.

\subsubsection{Cross-Sampling}
For cross-sampling experiments, we use the same inference setup as token analysis. Cross-sampling experiments selectively swap tokens between base and RL models at positions where JS divergence exceeds a threshold, allowing us to measure the functional importance of divergent token distributions.

We perform forward and reverse cross-sampling experiments on the following model-dataset combinations. The divergence thresholds used for each configuration are as follows:

\begin{itemize}
    \item \textbf{Qwen2.5-32B with SimpleRL:}
    \begin{itemize}
        \item AIME 2024: Forward threshold $\varepsilon=0.03$, Reverse threshold $\varepsilon=0.05$
        \item AIME 2025: Forward threshold $\varepsilon=0.05$, Reverse threshold $\varepsilon=0.05$
    \end{itemize}
    \item \textbf{Qwen2.5-32B with DAPO:}
    \begin{itemize}
        \item AIME 2024: Forward threshold $\varepsilon=0.08$, Reverse threshold $\varepsilon=0.06$
        \item AIME 2025: Forward threshold $\varepsilon=0.1$, Reverse threshold $\varepsilon=0.08$
    \end{itemize}
    \item \textbf{Mistral-Small-24B with SimpleRL:}
    \begin{itemize}
        \item AIME 2024: Forward threshold $\varepsilon=0.002$, Reverse threshold $\varepsilon=0.02$
    \end{itemize}
\end{itemize}

\subsubsection{Additional Training Details}
\label{subsubsec:train_details}

We implement RLVR training experiments using \texttt{verl} \citep{sheng2024hybridflow} with the standard DAPO recipe \citep{yu2025dapo_system}.

\paragraph{Qwen2.5-Math-7B DAPO Training.}
We follow the public DAPO recipe, namely with clip ratios $\epsilon_{\text{low}}=0.2$ and $\epsilon_{\text{high}}=0.28$. However, for token analysis, we also train a variant with $\epsilon_{\text{high}}=0.2$ for comparison. We optimize with learning rate $1\times10^{-6}$, a $10$-step warmup using AdamW, and no explicit reference-KL penalty. Each RLVR step processes $512$ prompts with $16$ sampled responses per prompt; these are split into mini-batches of $32$ prompts, yielding $16$ gradient updates per RLVR step. Maximum generation length and the overlong-penalty threshold are set to $8\text{k}$ and $4\text{k}$ tokens.

\paragraph{Supervised Fine-Tuning (SFT) Training.} For the SFT model based on on Qwen2.5 32B, we sampled 42k instances from the \texttt{AM-DeepSeek-R1-Distilled-1.4M} dataset. The model underwent full parameter fine-tuning for 5 epochs, employing DeepSpeed ZeRO-3 optimization.


\paragraph{Divergence-weighted Training} For the divergence-weighted advantage experiments on Qwen2.5-Math-7B, under the \textbf{high-KL} setting we use $s=0.3$ and set $\alpha$ to increase linearly from $0$ to $50$ starting at step $100$. In the \textbf{low-KL} setting, we use $s=0.3$ and set $\alpha$ to decrease linearly from $0$ to $-50$ beginning at step $150$.

\subsection{Weight-Level Analysis of Changes}
Orthogonal to the analysis done in the main text, we also investigate the degree of modifications induced by RLVR at the parameter level. More specifically, we employ the relative gap ratio~\citep{wu2025shadowfttuninginstructbase}, denoted as $\sigma$, to quantify the magnitude of weight divergence pre- and post-fine-tuning. This ratio is formulated as:

\begin{equation*}
    \sigma = \frac{\sum|W_{\text{original}} - W_{\text{tuned}}|}{\sum|W_{\text{original}}| + \sum|W_{\text{tuned}}|}
\end{equation*}

where $W_{\text{original}}$ and $W_{\text{tuned}}$ represent the model parameters before and after fine-tuning, respectively. A lower $\sigma$ value signifies greater similarity between the parameter sets, indicating a smaller overall modification from the fine-tuning process.

In our experiment, we utilized the Qwen2.5-32B and Qwen2.5-Math-7B models as foundations. Each model was independently fine-tuned via two distinct methodologies: RL and Supervised Fine-Tuning (SFT). To ensure a controlled and equitable comparison, the training regimen for both methods was standardized, employing an identical dataset size and the same number of training steps. Subsequently, the $\sigma$ was computed between each original model and its corresponding tuned counterparts. The results are presented in the following table.

\begin{table}[h!]
\centering
\caption{Relative gap ratio ($\sigma$) after RL and SFT fine-tuning.}
\label{tab:sigma_appendix}
\begin{tabular}{@{}lcc@{}}
\toprule
\textbf{Model} & \textbf{Qwen2.5-32B} & \textbf{Qwen2.5-Math-7B} \\ \midrule
$\sigma$ after RL & 0.00143 & 0.00136 \\
$\sigma$ after SFT & 0.00347 & 0.00944 \\ \bottomrule
\end{tabular}
\end{table}

The results presented in the table demonstrate a consistent trend across both models: the $\sigma$ values corresponding to RL fine-tuning are substantially lower than those from SFT. This quantitative analysis at the parameter level suggests that the cumulative weight modifications induced by RL are significantly less extensive than those resulting from SFT. This finding provides empirical support for the hypothesis that RL achieves performance gains through sparse and targeted parameter adjustments, contrasting with the more distributed updates characteristic of SFT.

\subsection{Additional RLVR vs. Supervised Fine-Tuning Results}
\label{subsec:rlvr_vs_sft}

This section provides additional results to supplement the discussion in Section~\ref{subsec:sft_comparison_main}. A natural question is whether the sparse, targeted distributional shifts we observe are specific to RLVR, or if they also characterize other fine-tuning approaches. To address this, we compare RLVR-trained models with models refined through supervised fine-tuning (SFT). We analyze Qwen2.5-32B trained with SFT alongside Qwen2.5-32B DAPO.

Figure~\ref{fig:js_rlvr_vs_sft} shows JS divergence distributions for both approaches. SFT produces a noticeably larger high-divergence set with larger divergence values, whereas RLVR concentrates almost all token distributions below very small JS values. This directly reflects RLVR’s extreme selectivity and the broader edits introduced by SFT. The top-$k$ overlap analysis (Figure~\ref{fig:topk_distill_vs_rlvr}) highlights that SFT consistently achieves lower overlap with the base model, indicating more aggressive re-ranking, while RLVR largely stays within the base model’s existing candidate set. The rank reordering analysis (Figure~\ref{fig:ranks_distill_vs_rlvr}) further shows that SFT promotes more tokens outside the base model’s top-3, whereas RLVR mainly promotes candidates that were already high-ranked.

\begin{figure}[!htbp]
    \centering
    \begin{subfigure}{0.40\textwidth}
        \includegraphics[width=\linewidth]{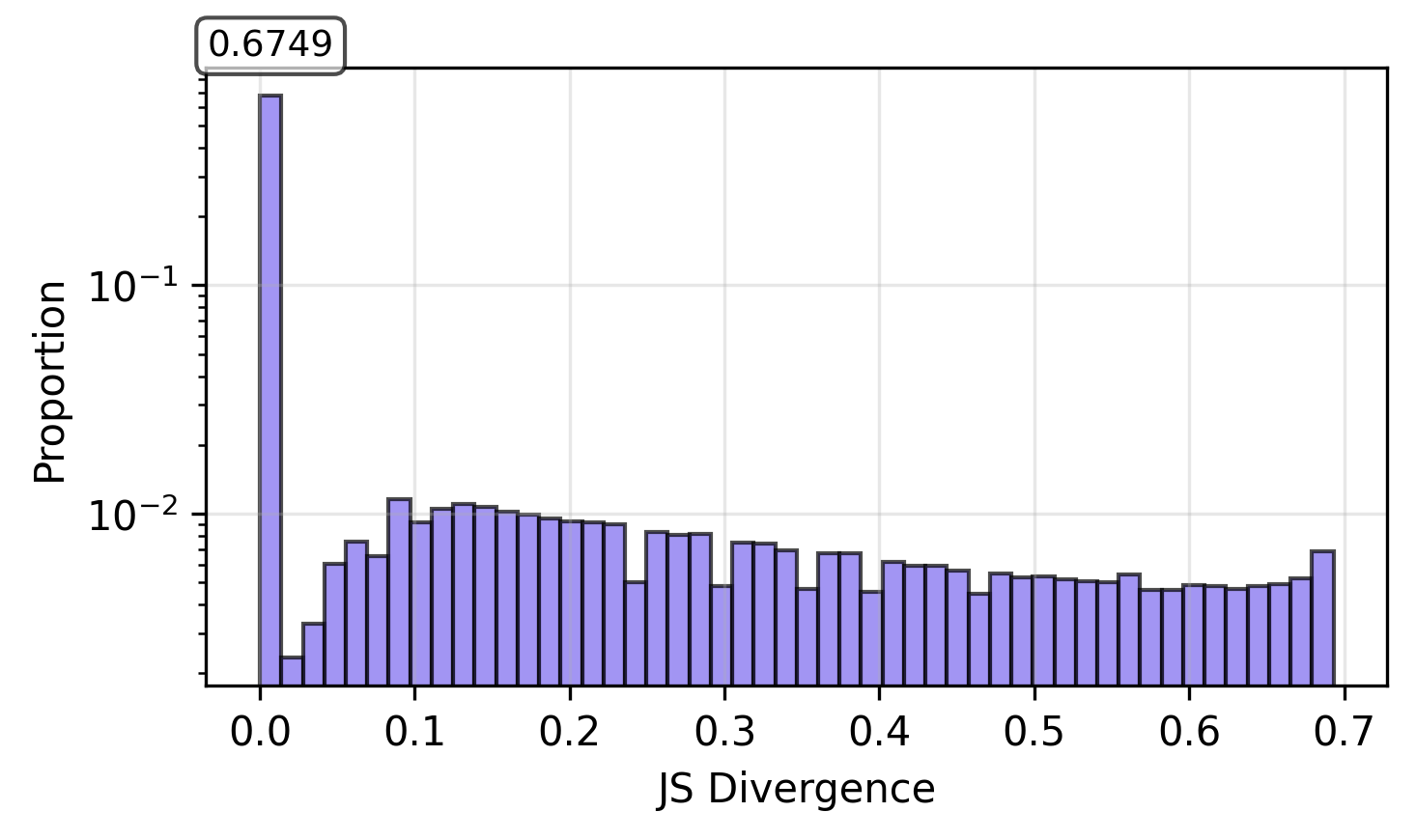}
        \caption{SFT: Histogram}
    \end{subfigure}
    \hspace{1cm}
    \begin{subfigure}{0.40\textwidth}
        \includegraphics[width=\linewidth]{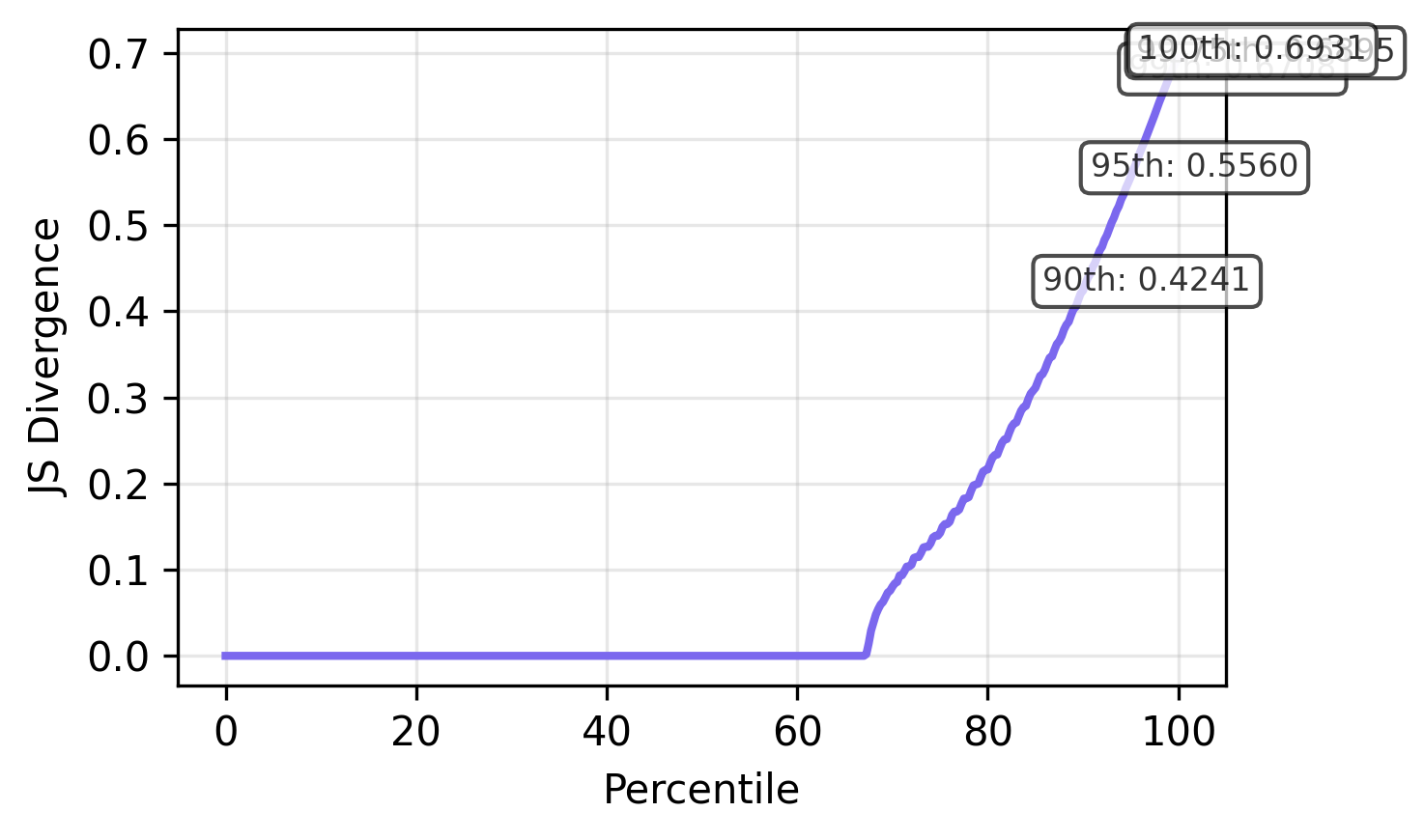}
        \caption{SFT: Percentiles}
    \end{subfigure}
    
    \vspace{1em}
    \begin{subfigure}{0.40\textwidth}
        \includegraphics[width=\linewidth]{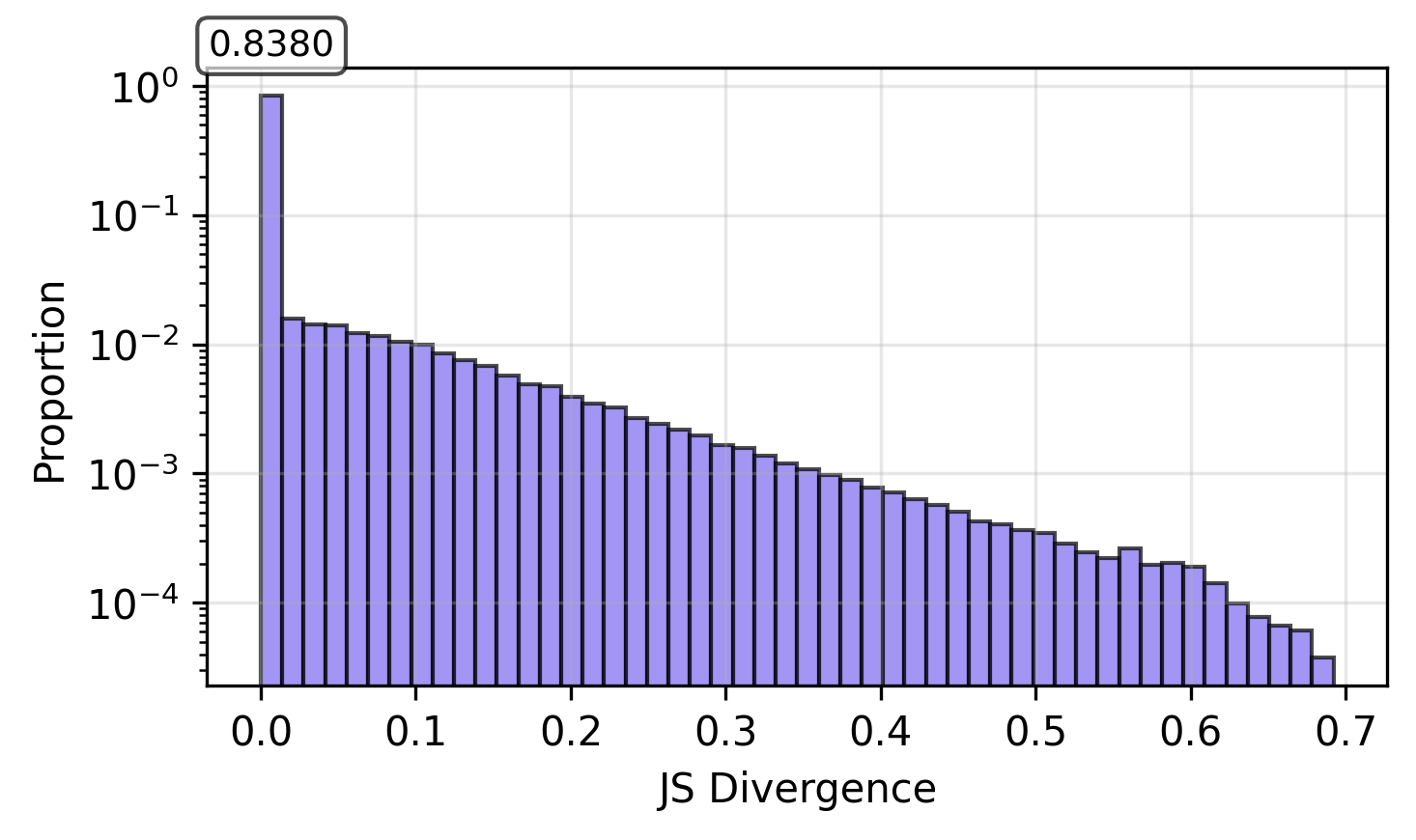}
        \caption{RLVR (DAPO): Histogram}
    \end{subfigure}
    \hspace{1cm}
    \begin{subfigure}{0.40\textwidth}
        \includegraphics[width=\linewidth]{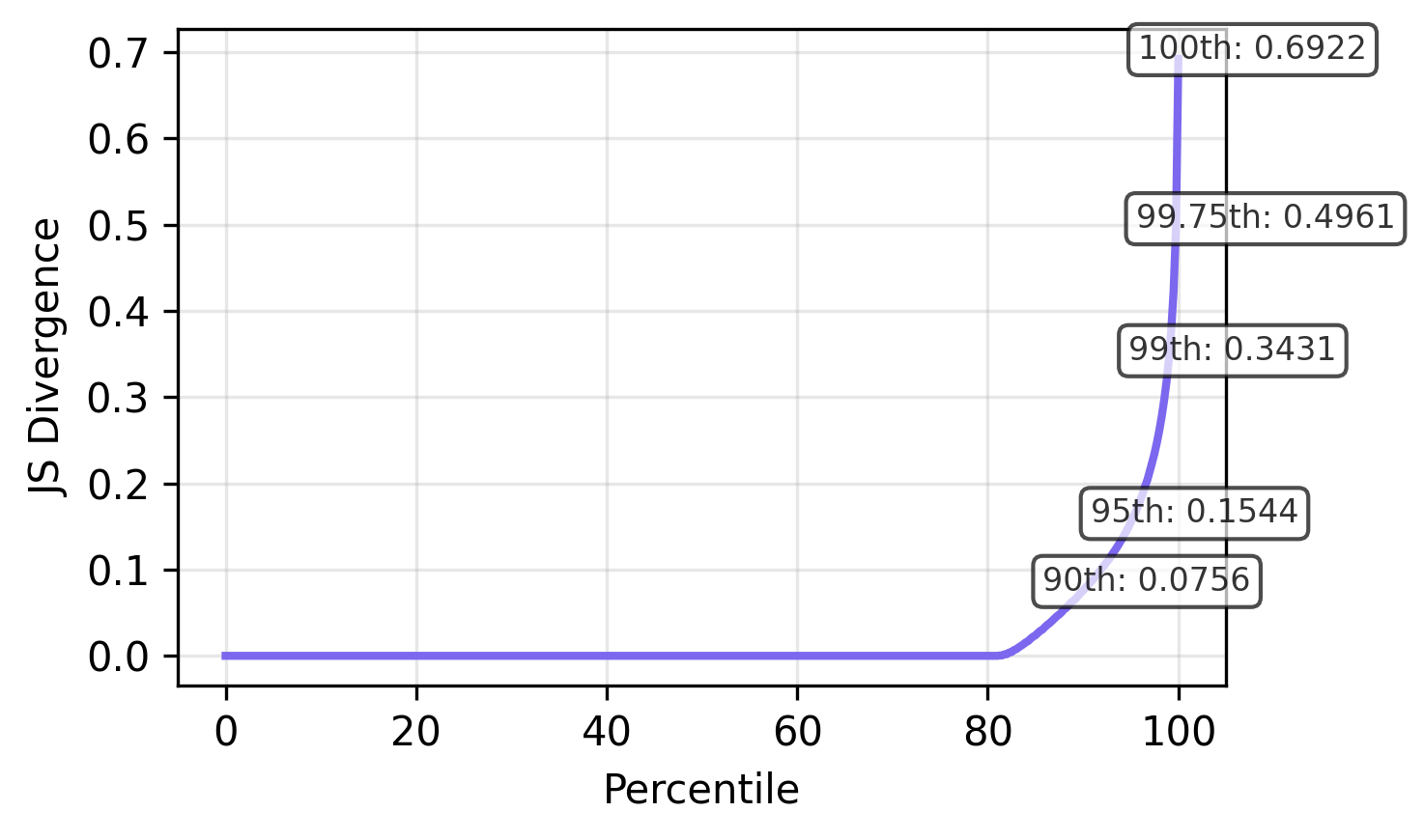}
        \caption{RLVR (DAPO): Percentiles}
    \end{subfigure}
    
    \caption{
        JS divergence distributions comparing supervised fine-tuning and RLVR on AIME 2024.
        RLVR exhibits sparser distributional shifts than SFT, suggesting more targeted refinement.
    }
    \label{fig:js_rlvr_vs_sft}
\end{figure}

\begin{figure}[!htbp]
    \centering
    \begin{subfigure}{0.43\textwidth}
        \includegraphics[width=\linewidth]{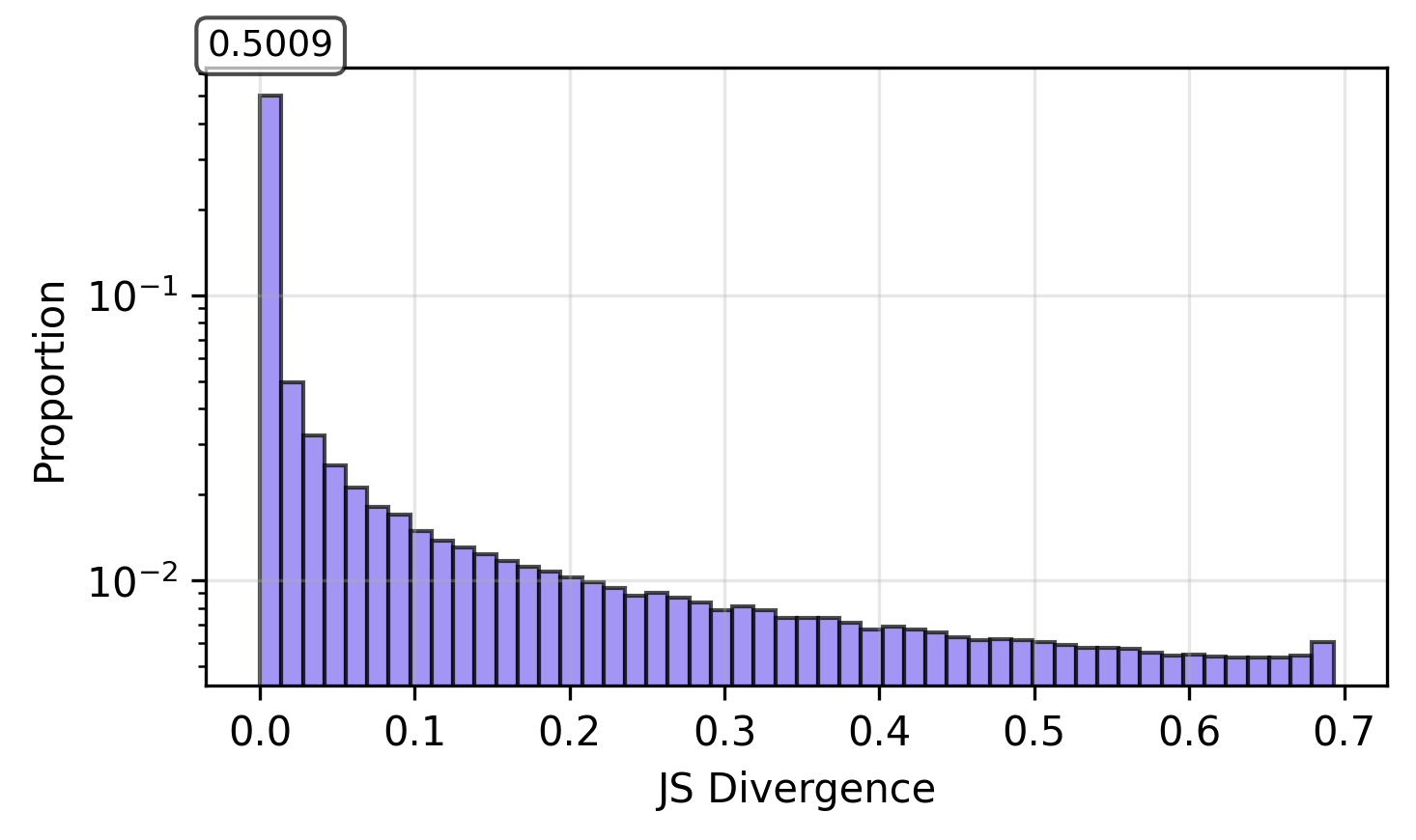}
        \caption{SFT: Histogram (topp1)}
    \end{subfigure}
    \hspace{1cm}
    \begin{subfigure}{0.43\textwidth}
        \includegraphics[width=\linewidth]{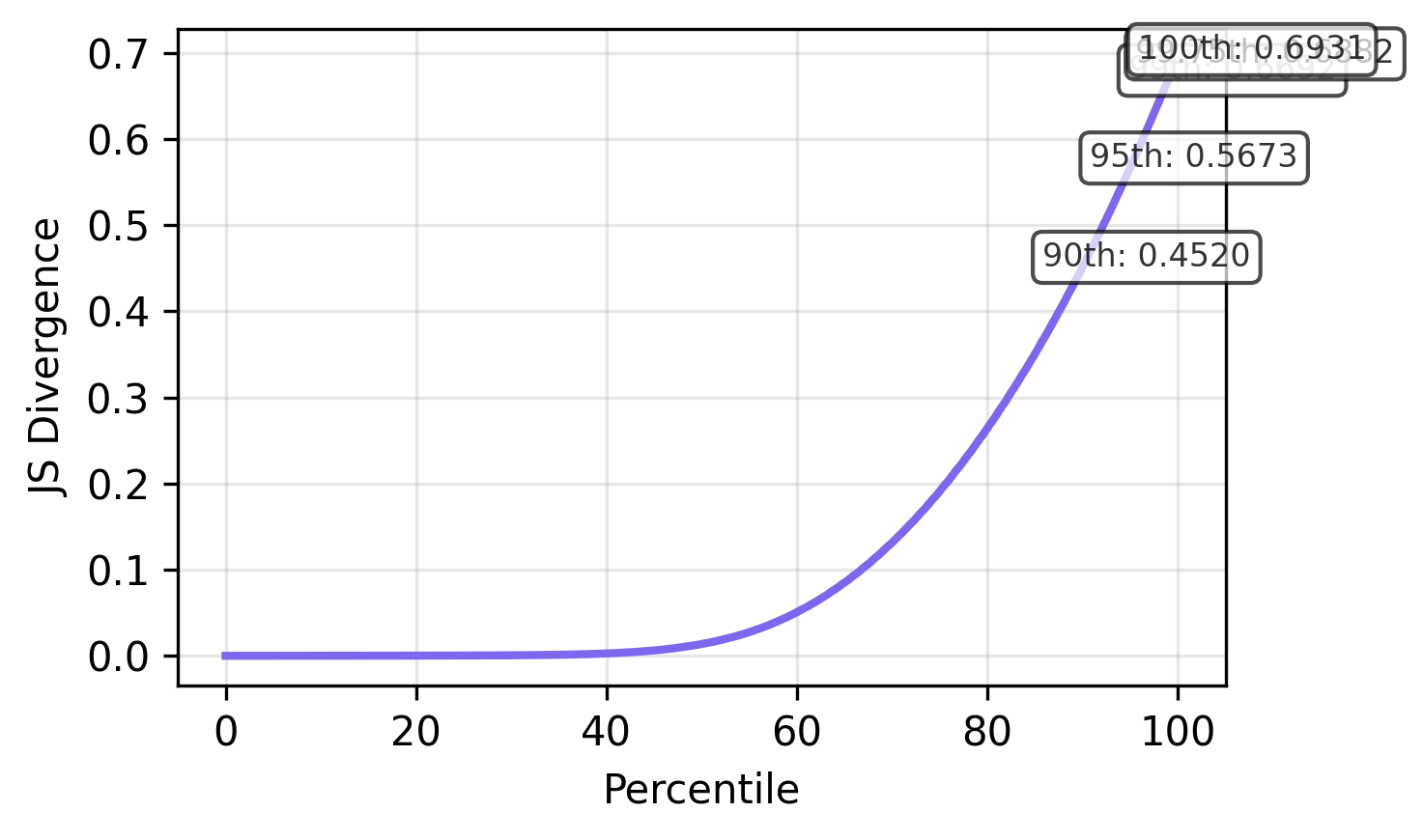}
        \caption{SFT: Percentile curve (topp1)}
    \end{subfigure}
    
    \caption{
        JS divergence distributions computed using top-$p=1$ distribution (while sampling is still with top-$p=0.7$) for Qwen2.5-32B SFT on AIME 2024.
    }
    \label{fig:js_topp1_comparison_sft_aime24}
\end{figure}

\begin{figure}[!htbp]
    \centering
    \begin{subfigure}{0.45\linewidth}
        \includegraphics[width=\linewidth]{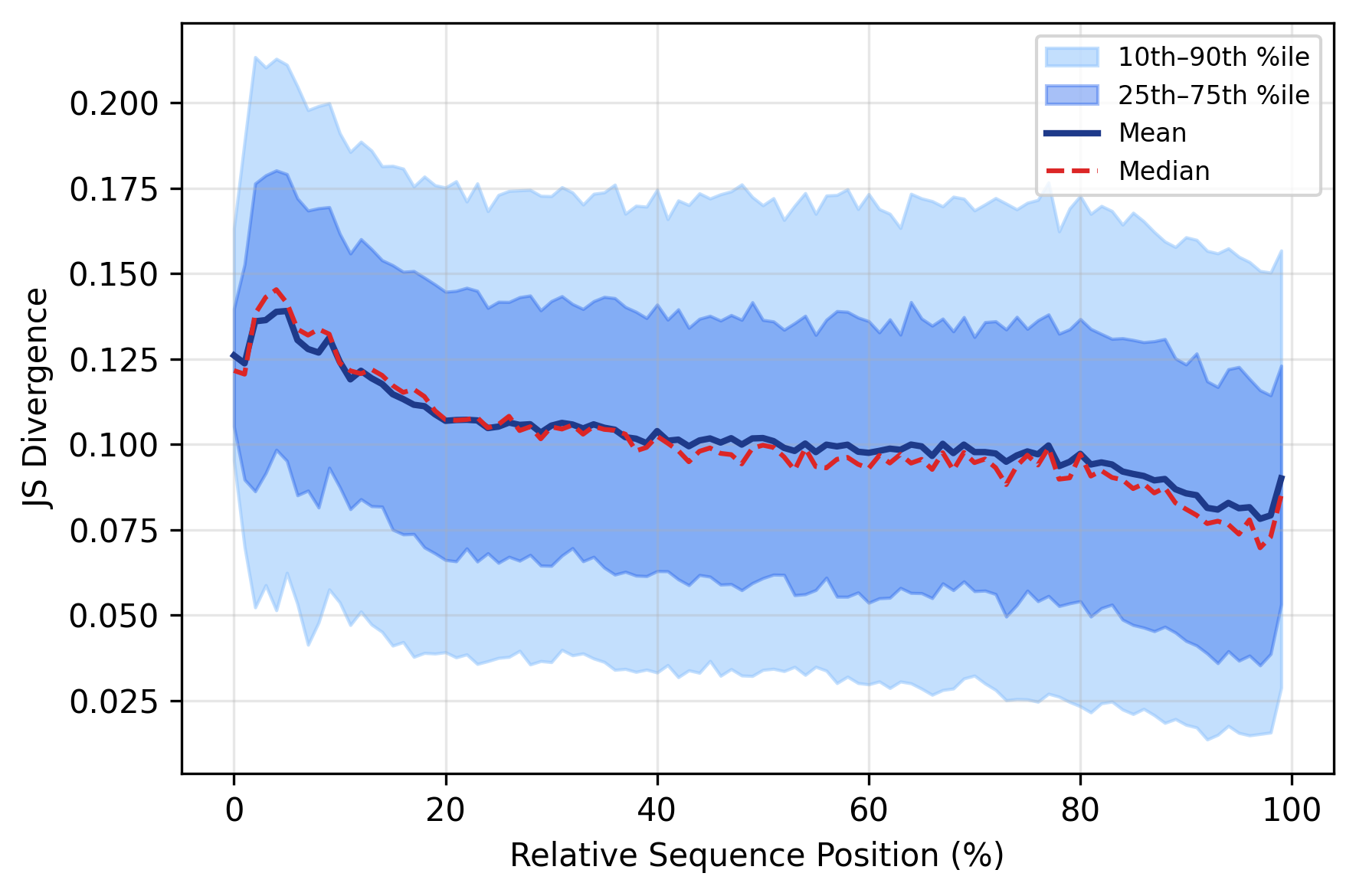}
        \caption{Qwen2.5-32B SFT}
    \end{subfigure}
    \hspace{1cm}
    \begin{subfigure}{0.45\linewidth}
        \includegraphics[width=\linewidth]{plots/Qwen2.5-32B_+_DAPO/AIME_2024/average_js_by_position.png}
        \caption{Qwen2.5-32B DAPO}
    \end{subfigure}
    
    \caption{
        Mean JS divergence by normalized token position comparing SFT and RLVR-trained models on AIME 2024.
        The positional patterns reveal differences in how SFT and RLVR concentrate their updates.
    }
    \label{fig:positional_distill_vs_rlvr}
\end{figure}

Taken together, the metrics highlight that SFT diverges from RLVR along several axes. The SFT model exhibits higher overall JS divergence as well as a larger mass of high-divergence tokens (Figure~\ref{fig:js_rlvr_vs_sft}), and attains lower top-$k$ overlap with the base model (Figure~\ref{fig:topk_distill_vs_rlvr}) alongside larger rank shifts (Figure~\ref{fig:ranks_distill_vs_rlvr}). Moreover, SFT’s divergent tokens more frequently elevate low base-probability choices compared to RLVR (Figure~\ref{fig:tolerance_distill_vs_rlvr}). These differences reinforce that RLVR acts as a targeted editor, while SFT drives broader, less selective reshaping of the distribution.

These findings align with recent work suggesting that RL fine-tuning acts as a \emph{scalpel} rather than a hammer, amplifying existing capabilities through localized changes compared to the broader modifications induced by supervised fine-tuning \citep{rajani2025scalpel_vs_hammer, chu2025sft_memorizes_rl_generalizes}. Our results align with this behavior but from the perspective of \emph{token-level distributional changes}: RLVR modifies far fewer token positions (as measured by JS divergence), and at those positions, the changes are more likely to be re-ranking within the base model's top candidates rather than introducing entirely new tokens. In contrast, SFT exhibits more widespread token-level distributional shifts across a larger fraction of positions, as it learns to mimic provided outputs while adjusting token probabilities more broadly across the vocabulary space.

\begin{figure}[!htbp]
    \centering
    \begin{subfigure}{0.45\linewidth}
        \includegraphics[width=\linewidth]{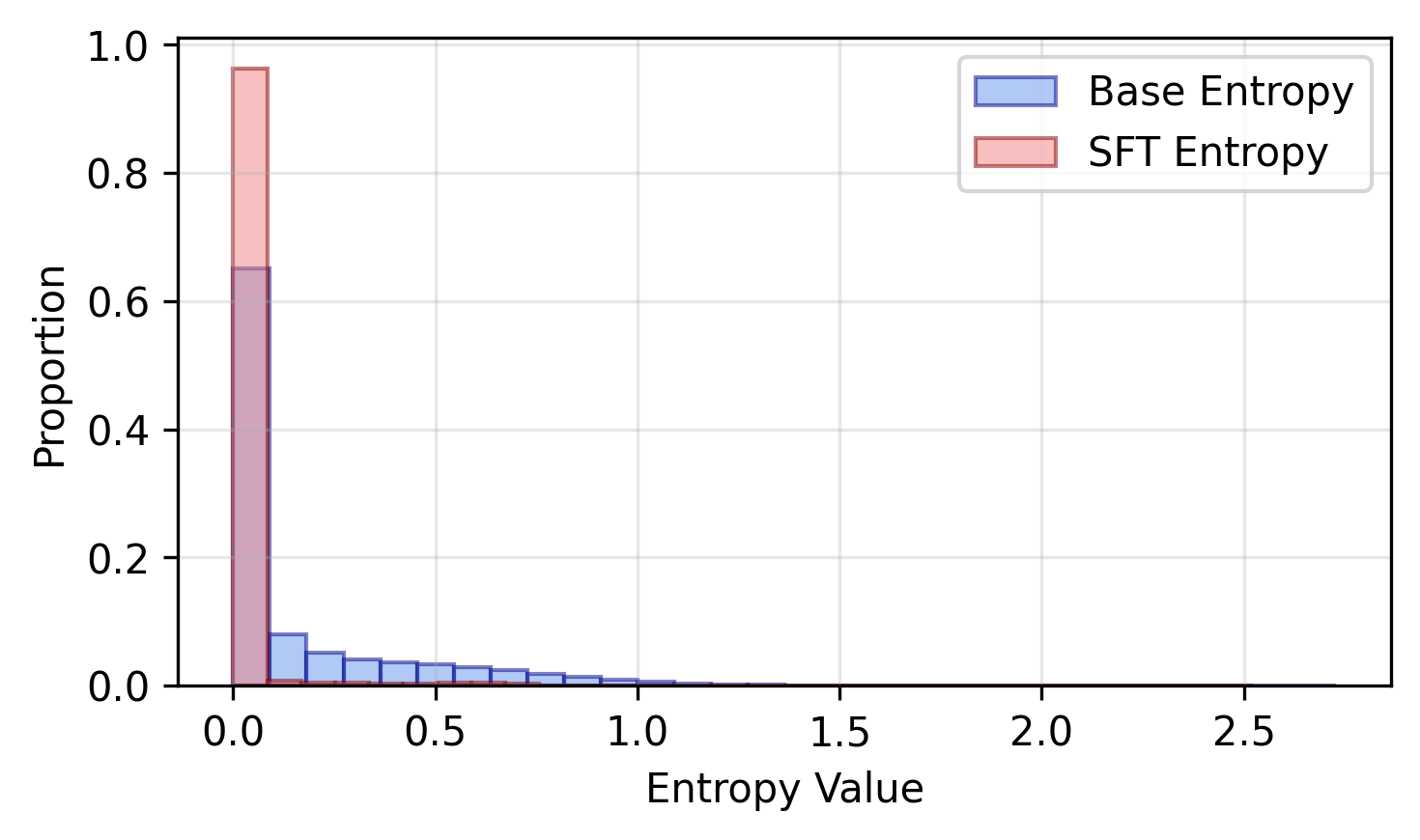}
        \caption{Qwen2.5-32B SFT Low JS bin ($<0.1$)}
    \end{subfigure}
    \hfill
    \begin{subfigure}{0.45\linewidth}
        \includegraphics[width=\linewidth]{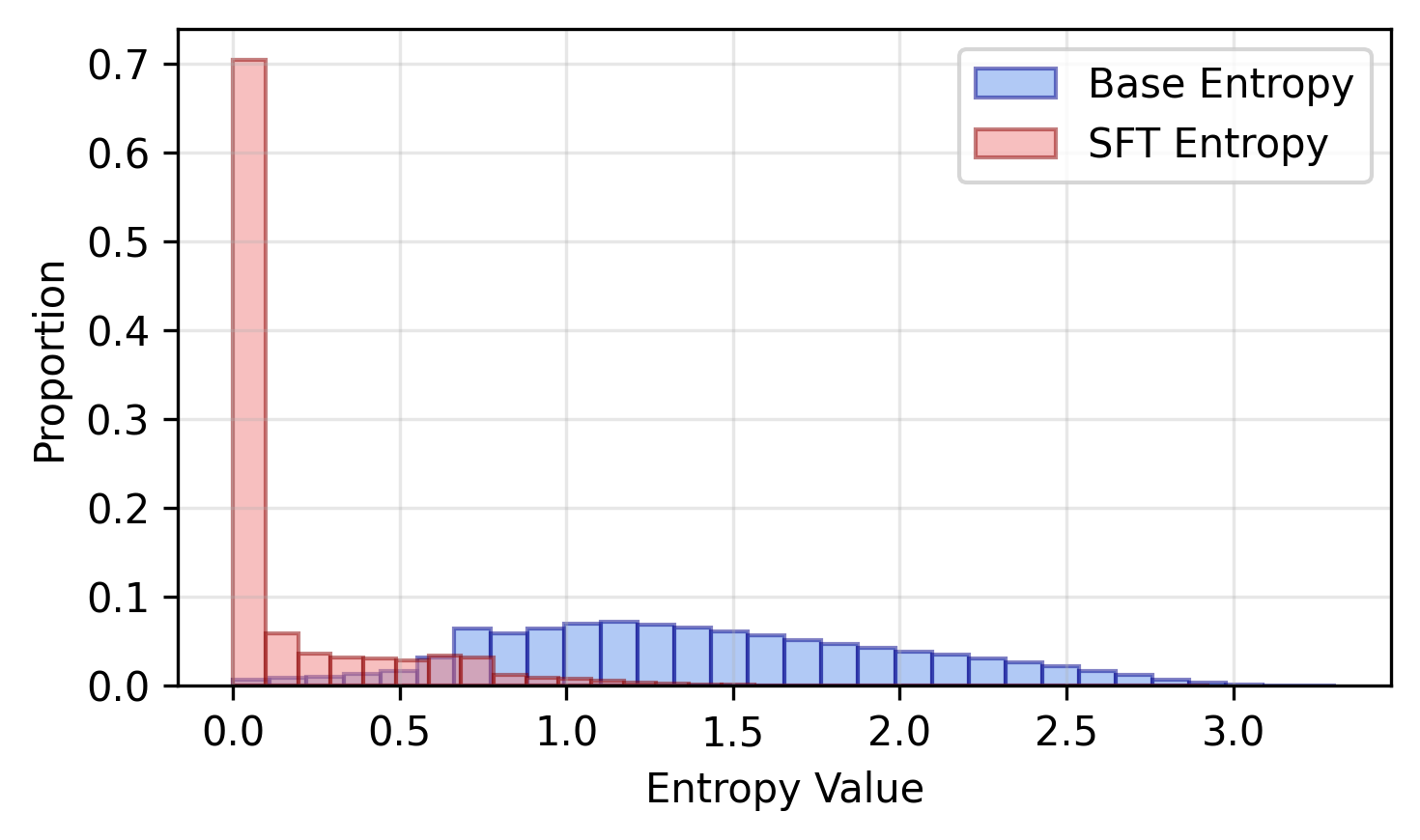}
        \caption{Qwen2.5-32B SFT High JS bin ($>0.1$)}
    \end{subfigure}
    
    \caption{
        Entropy distributions across divergence bins using full vocabulary for Qwen2.5-32B with SFT on AIME 2025.
    }
    \label{fig:entropy_distill}
\end{figure}


\begin{figure}[!htbp]
    \centering
    \begin{subfigure}{0.45\linewidth}
        \includegraphics[width=\linewidth]{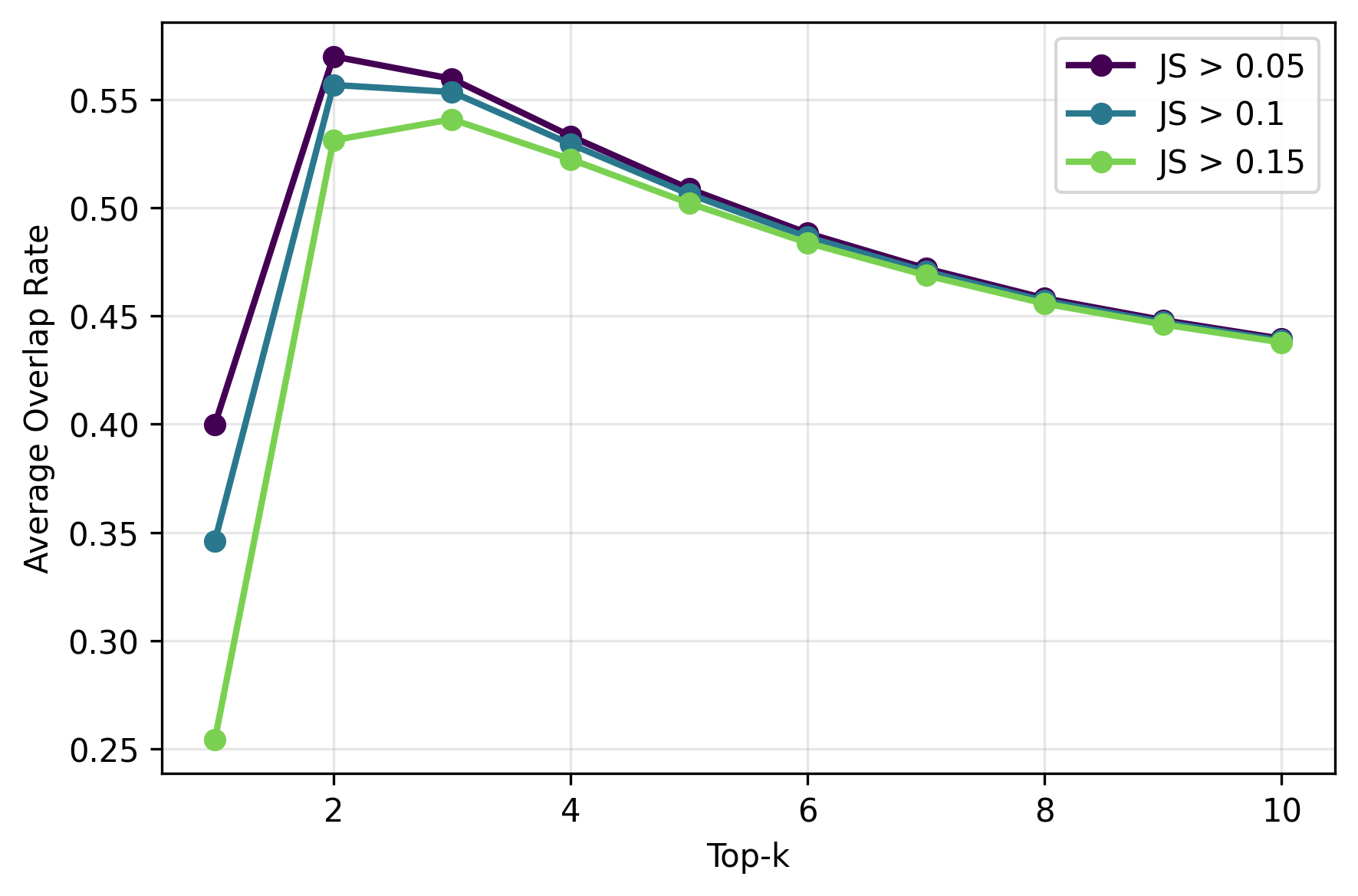}
        \caption{Qwen2.5-32B SFT}
    \end{subfigure}
    \hfill
    \begin{subfigure}{0.45\linewidth}
        \includegraphics[width=\linewidth]{plots/Qwen2.5-32B_+_DAPO/AIME_2024/topk_overlap_curves_by_threshold.png}
        \caption{Qwen2.5-32B DAPO}
    \end{subfigure}
    
    \caption{
        Top-$k$ token overlap between base and refined models at divergent positions ($\mathrm{JS}_t>0.1$) comparing SFT and RLVR-trained models on AIME 2024.
    }
    \label{fig:topk_distill_vs_rlvr}
\end{figure}

\begin{figure}[!htbp]
    \centering
    \begin{subfigure}{0.48\linewidth}
        \includegraphics[width=\linewidth]{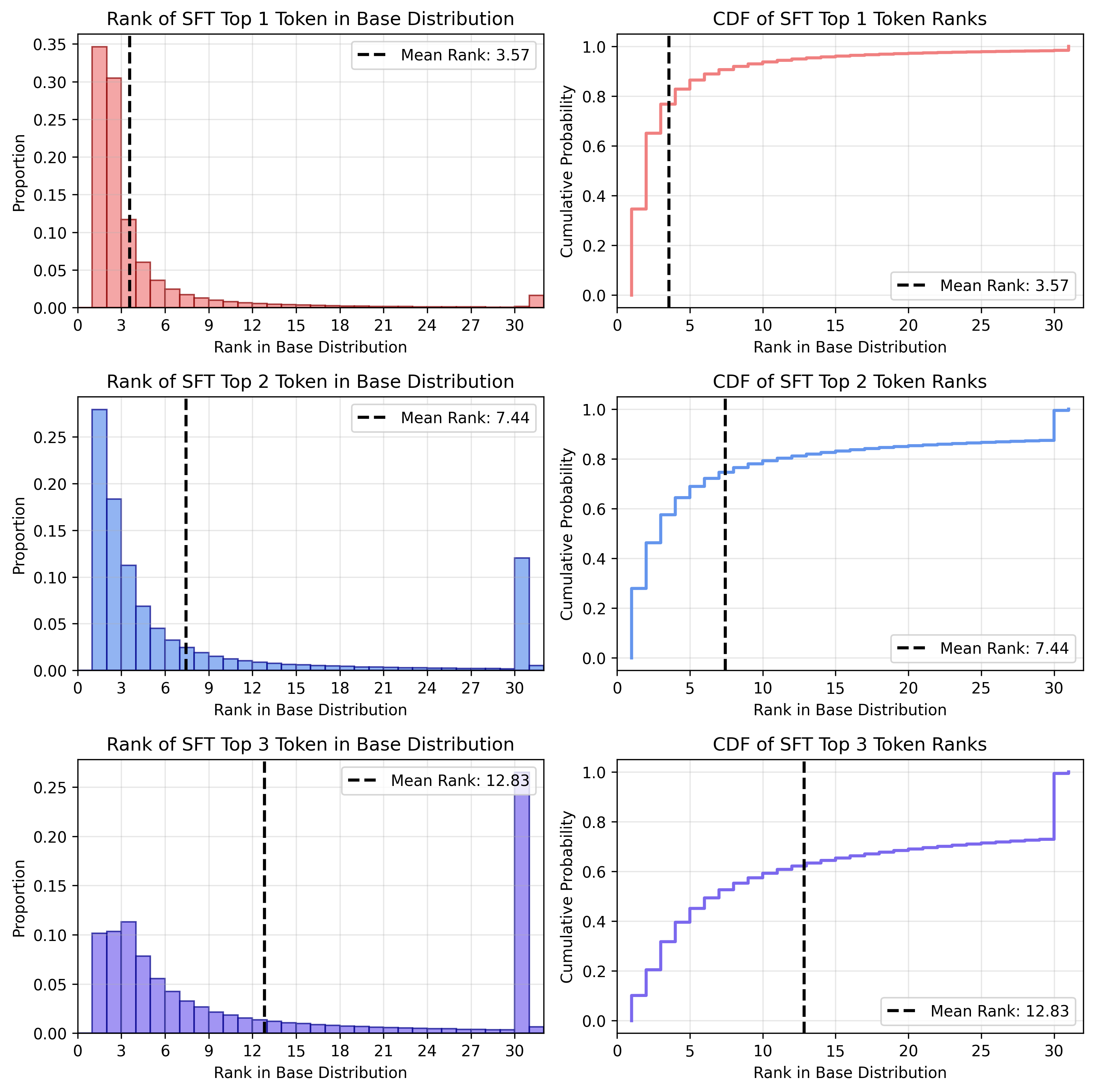}
        \caption{Qwen2.5-32B SFT}
    \end{subfigure}
    \hfill
    \begin{subfigure}{0.48\linewidth}
        \includegraphics[width=\linewidth]{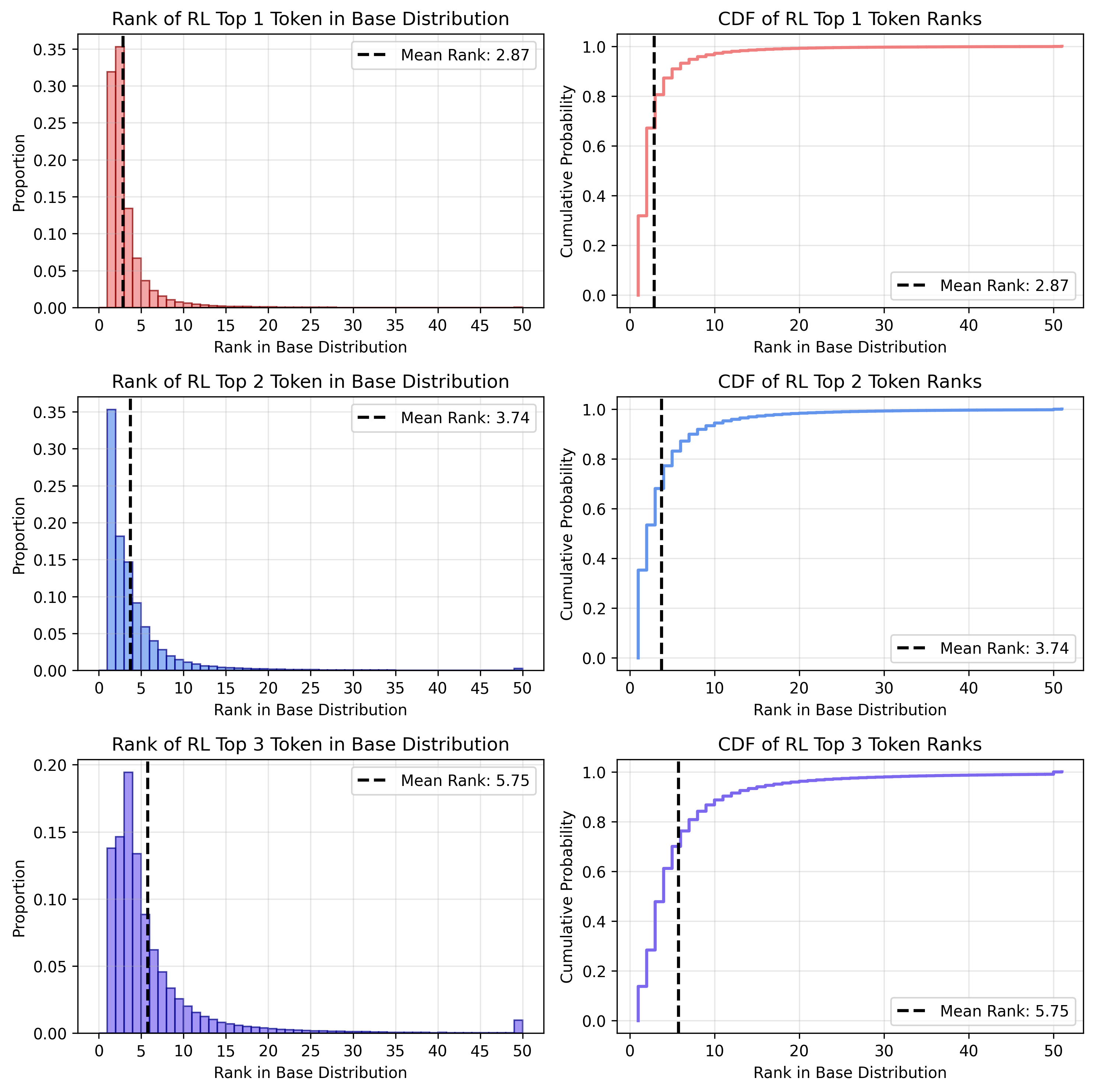}
        \caption{Qwen2.5-32B DAPO}
    \end{subfigure}
    
    \caption{
        Distribution of base-model ranks for fine-tuned models' top-3 tokens at high-divergence positions ($\js > 0.1$) comparing SFT and RLVR-trained models on AIME 2024.
    }
    \label{fig:ranks_distill_vs_rlvr}
\end{figure}

\begin{figure}[!htbp]
    \centering
    \begin{subfigure}{0.42\linewidth}
        \includegraphics[width=\linewidth]{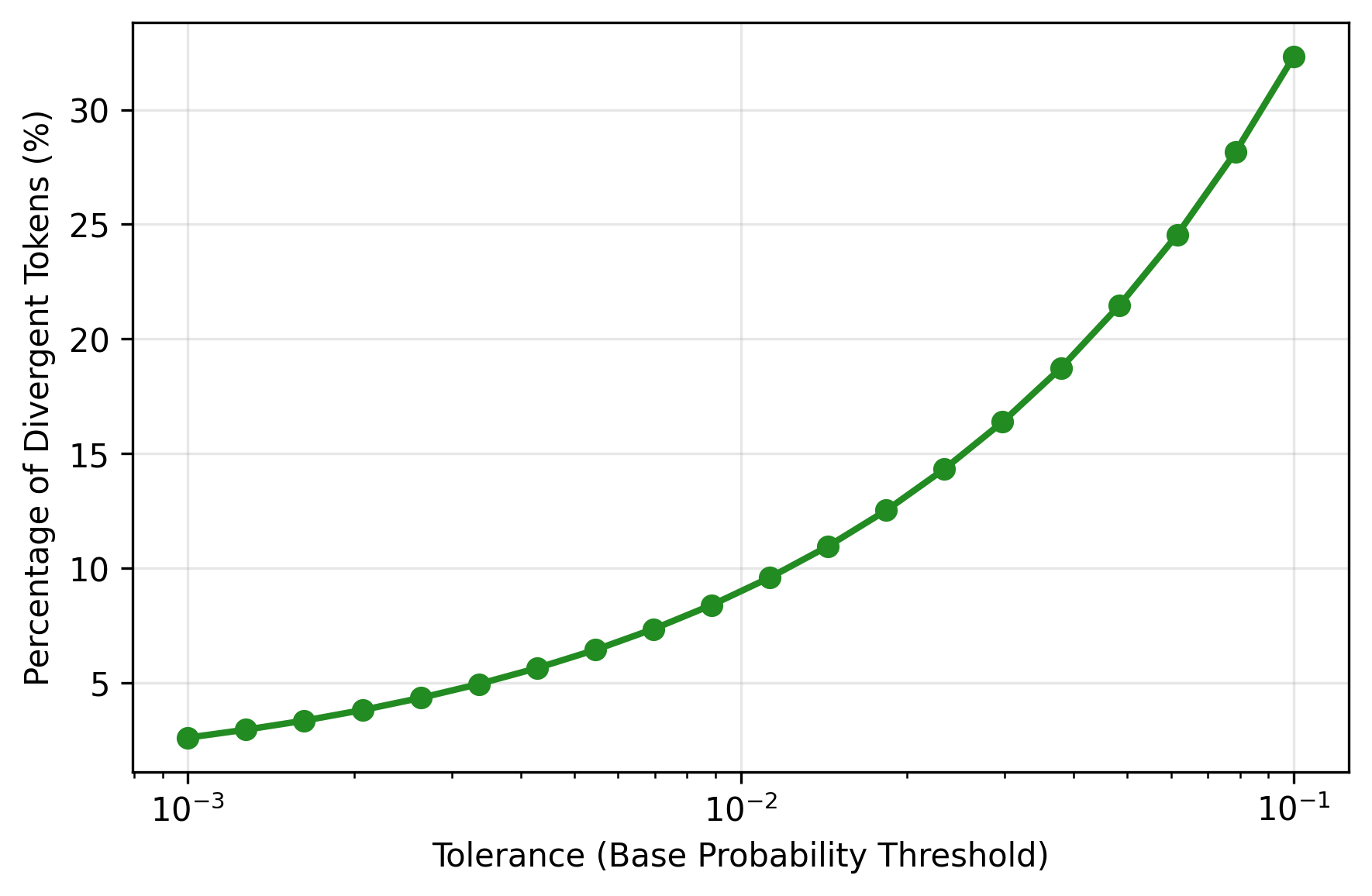}
        \caption{Qwen2.5-32B SFT}
    \end{subfigure}
    \hspace{1cm}
    \begin{subfigure}{0.42\linewidth}
        \includegraphics[width=\linewidth]{plots/Qwen2.5-32B_+_DAPO/AIME_2024/percentages_divergent_tokens_vs_tolerance_js0.100.png}
        \caption{Qwen2.5-32B DAPO}
    \end{subfigure}
    
    \caption{
        Percentage of divergent tokens whose RL top-1 choice had base probability below a given threshold comparing SFT and RLVR-trained models on AIME 2024.
    }
    \label{fig:tolerance_distill_vs_rlvr}
\end{figure}

\begin{figure}[!htbp]
    \centering
    \begin{subfigure}{0.43\textwidth}
        \includegraphics[width=\linewidth]{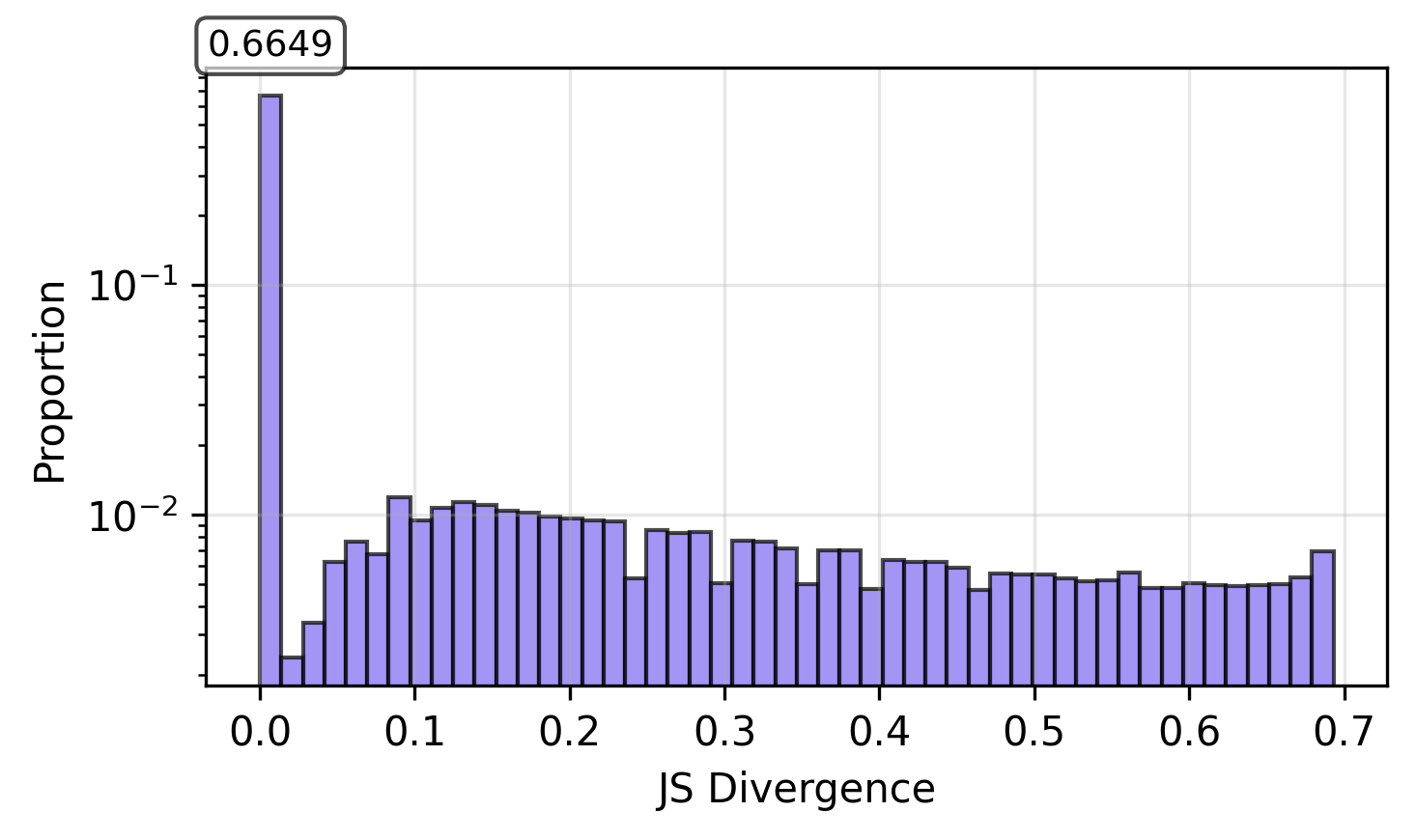}
        \caption{Qwen2.5-32B SFT: Histogram}
    \end{subfigure}
    \hspace{1cm}
    \begin{subfigure}{0.43\textwidth}
        \includegraphics[width=\linewidth]{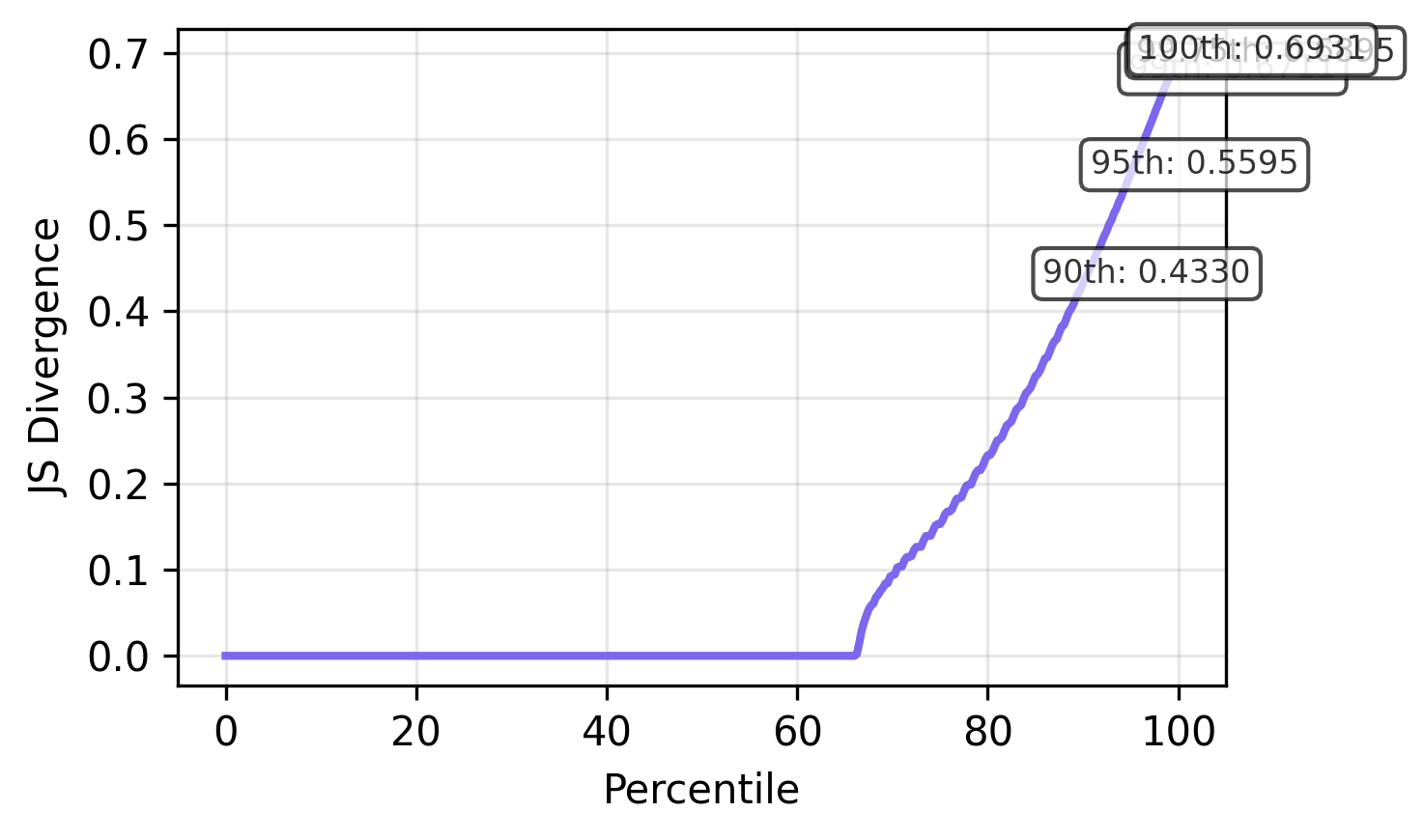}
        \caption{Qwen2.5-32B SFT: Percentiles}
    \end{subfigure}
    
    \caption{
        JS divergence distributions for Qwen2.5-32B SFT on AIME 2025.
    }
    \label{fig:js_distill_vs_rlvr_aime25}
\end{figure}

\begin{figure}[!htbp]
    \centering
    \begin{subfigure}{0.43\textwidth}
        \includegraphics[width=\linewidth]{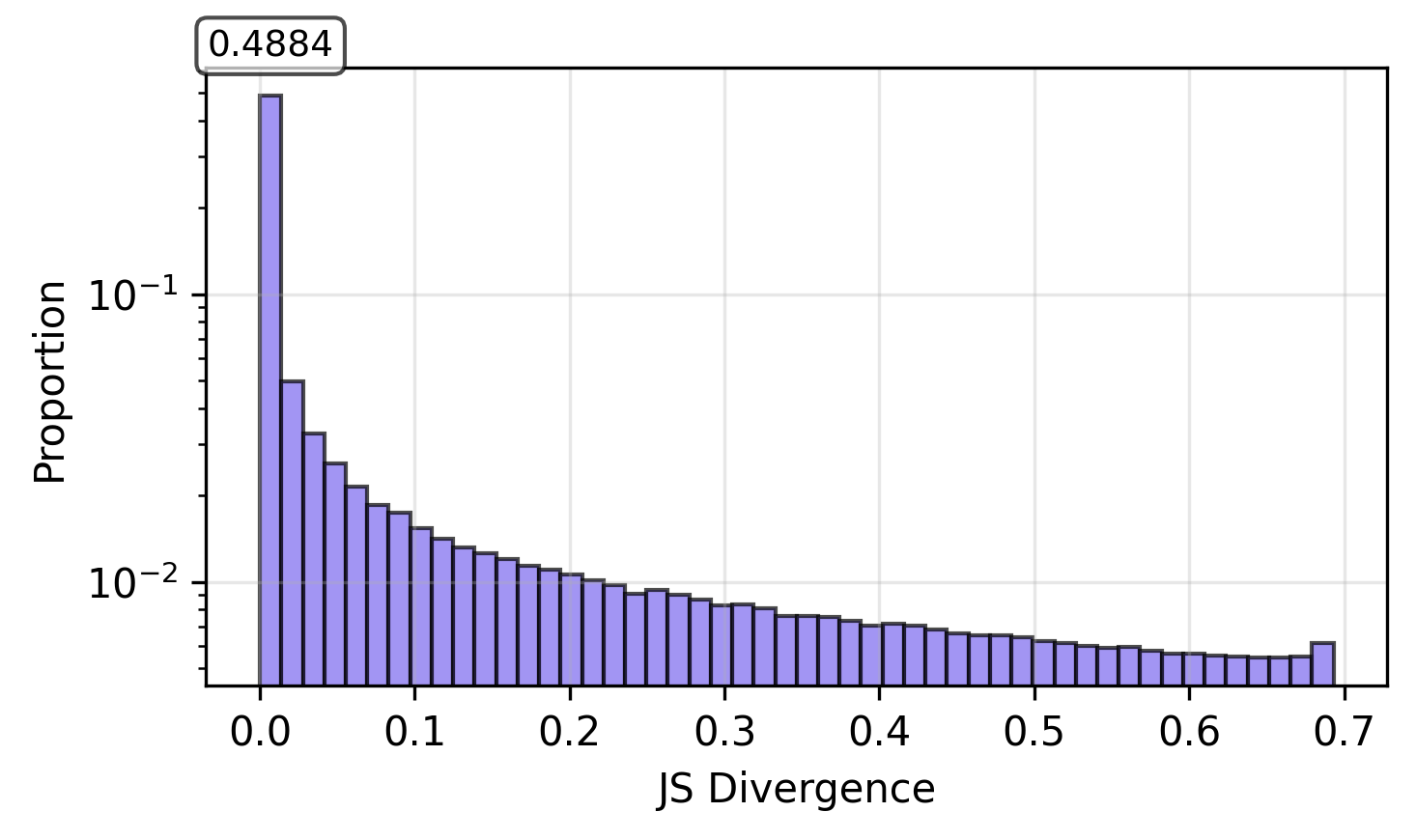}
        \caption{SFT: Histogram (topp1)}
    \end{subfigure}
    \hspace{1cm}
    \begin{subfigure}{0.43\textwidth}
        \includegraphics[width=\linewidth]{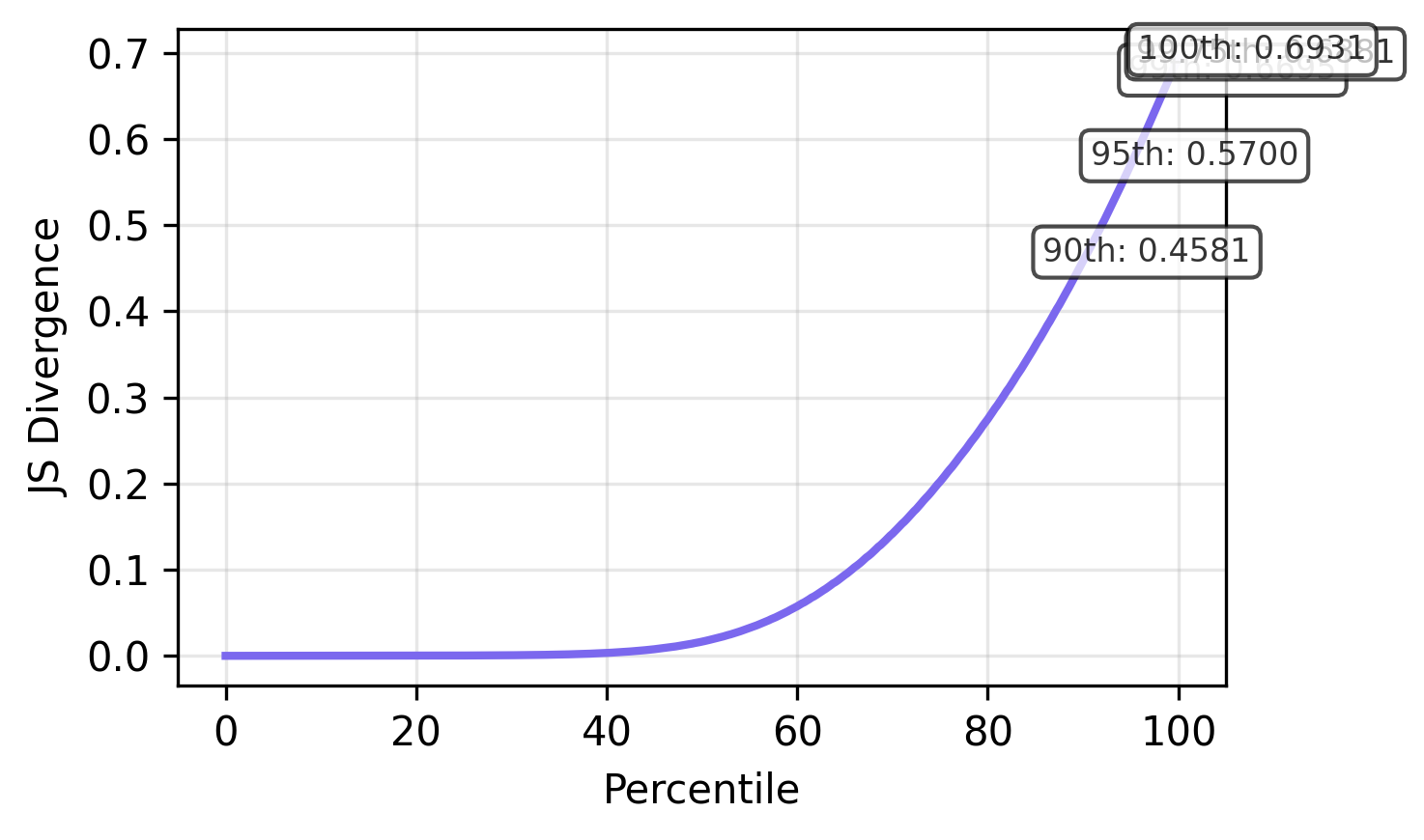}
        \caption{SFT: Percentile curve (topp1)}
    \end{subfigure}
    
    \caption{
        JS divergence distributions computed using top-$p=1$ for Qwen2.5-32B SFT on AIME 2025.
    }
    \label{fig:js_topp1_comparison_sft_aime25}
\end{figure}

\FloatBarrier
\subsection{Additional Token Distribution Analyses}
\label{subsec:additional_distributions}

This section provides supplementary and extended token distribution analyses. We first present supplementary figures for the main models (Qwen2.5-32B with DAPO and SimpleRL on AIME 2024), then extend the analysis to additional models and datasets to demonstrate the generalizability of our findings.

\subsubsection{Supplementary Figures for Main Models}

We provide additional figures for Qwen2.5-32B with DAPO and SimpleRL that complement the analyses in the main text.

\begin{figure}[!htbp]
    \centering
    \begin{subfigure}{0.40\linewidth}
        \includegraphics[width=\linewidth]{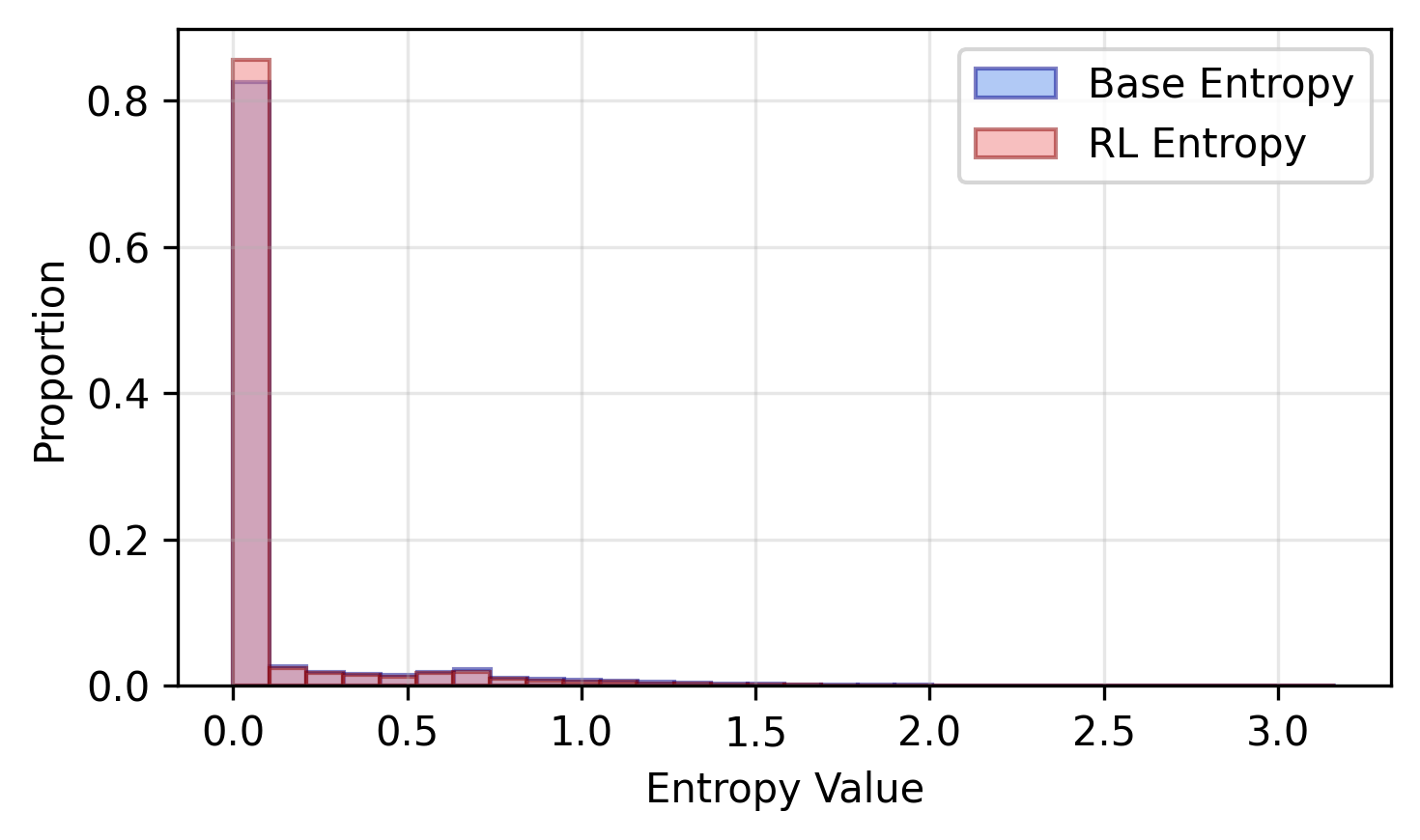}
        \caption{Low JS bin ($<0.1$).}
    \end{subfigure}
    \hspace{1cm}
    \begin{subfigure}{0.40\linewidth}
        \includegraphics[width=\linewidth]{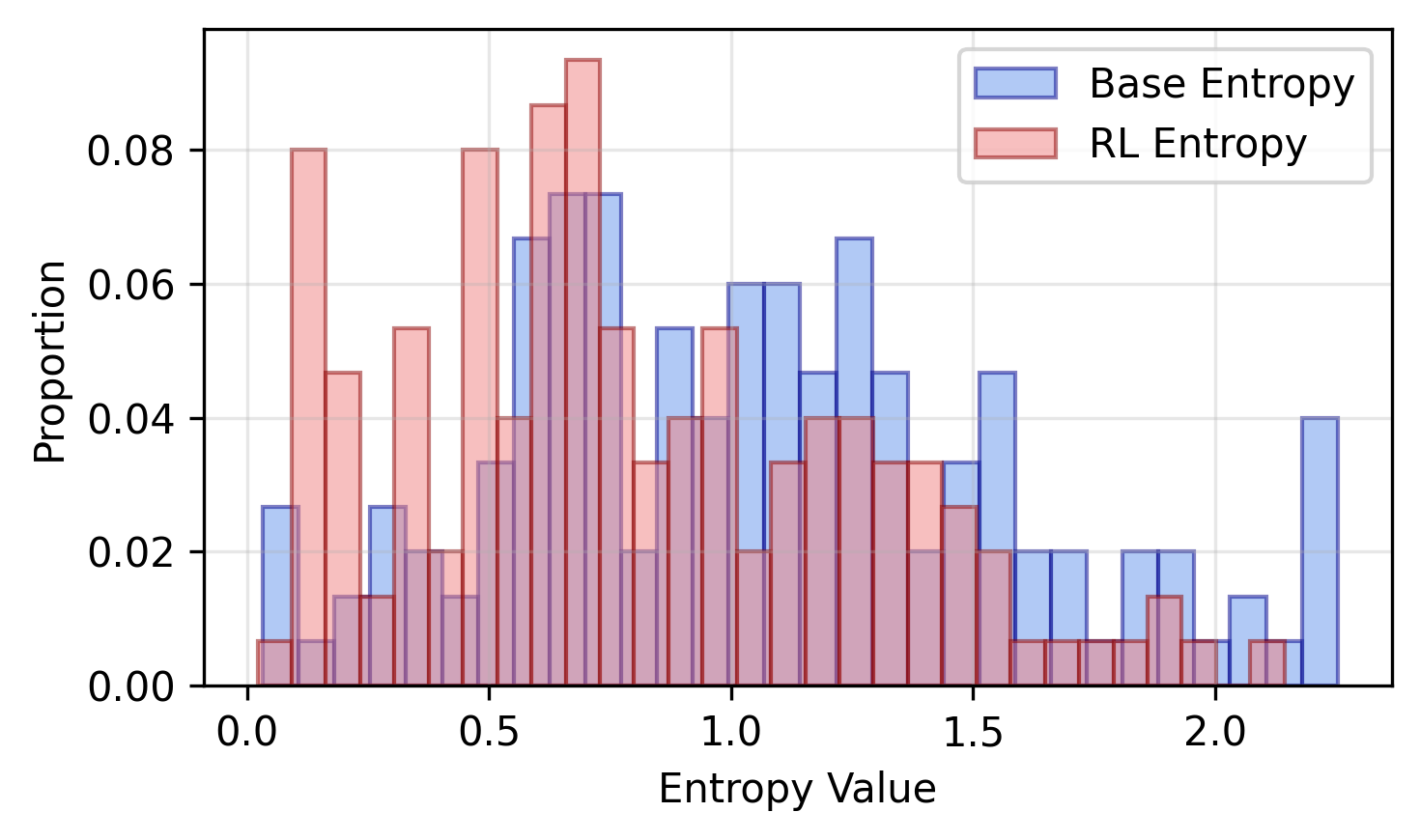}
        \caption{High JS bin ($>0.1$).}
    \end{subfigure}
    \caption{
        Entropy distributions across divergence bins for \textbf{SimpleRL}. 
        Low-divergence tokens are mostly low-entropy, while high-divergence tokens are concentrated in higher-entropy regions, reflecting a more conservative update strategy.
    }
    \label{fig:entropy_simplerl}
\end{figure}


\begin{figure}[!htbp]
    \centering
    \begin{subfigure}{0.48\linewidth}
        \includegraphics[width=\linewidth]{plots/Qwen2.5-32B_+_DAPO/AIME_2024/rl_vs_base_ranks_rl_top3_js0.100.png}
        \caption{Qwen2.5 32B DAPO.}
        \label{fig:ranks_dapo}
    \end{subfigure}
    \hfill
    \begin{subfigure}{0.48\linewidth}
        \includegraphics[width=\linewidth]{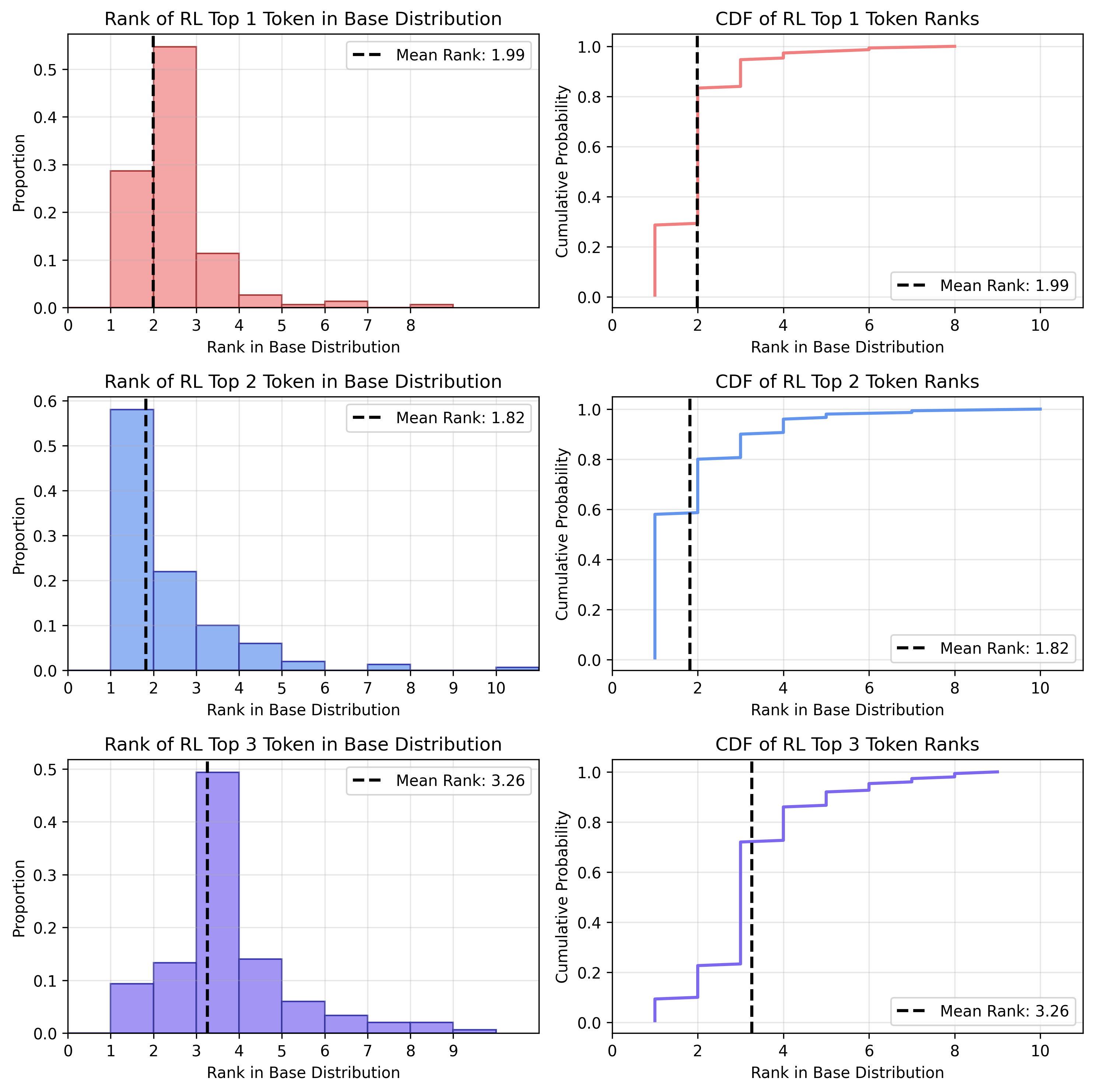}
        \caption{Qwen2.5 32B SimpleRL.}
        \label{fig:ranks_simplerl}
    \end{subfigure}
    \caption{
        Distribution of base-model ranks for RL's top-3 tokens at high-divergence positions ($\js > 0.1$).
        Most RL-selected tokens were already highly ranked in the base model, especially under SimpleRL.
    }
    \label{fig:rl_vs_base_ranks}
\end{figure}

\begin{figure}[!htbp]
    \centering
    \begin{subfigure}{0.48\linewidth}
        \includegraphics[width=\linewidth]{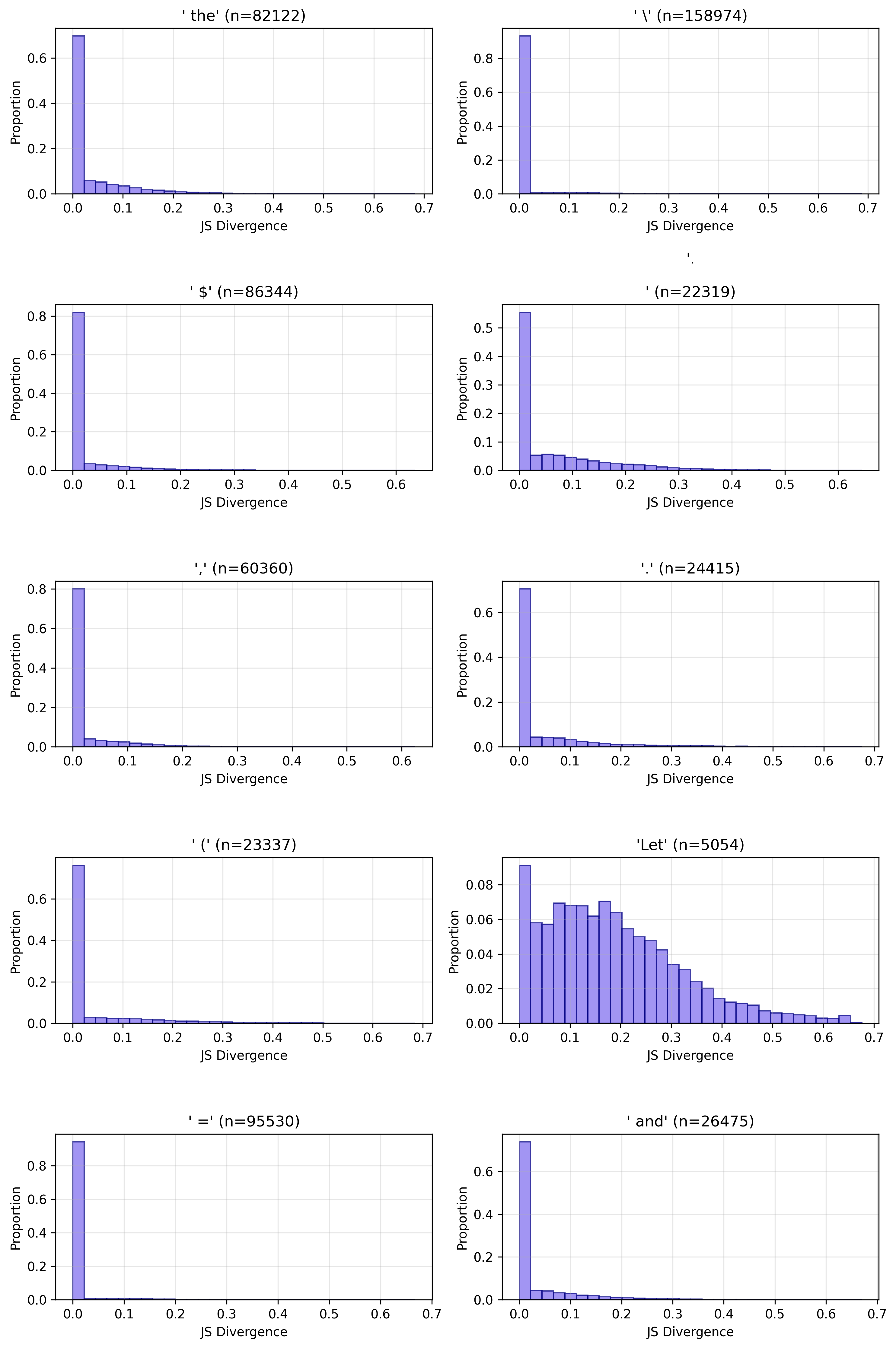}
        \caption{Frequent high JS tokens.}
        \label{fig:div_high_prob}
    \end{subfigure}
    \hfill
    \begin{subfigure}{0.48\linewidth}
        \includegraphics[width=\linewidth]{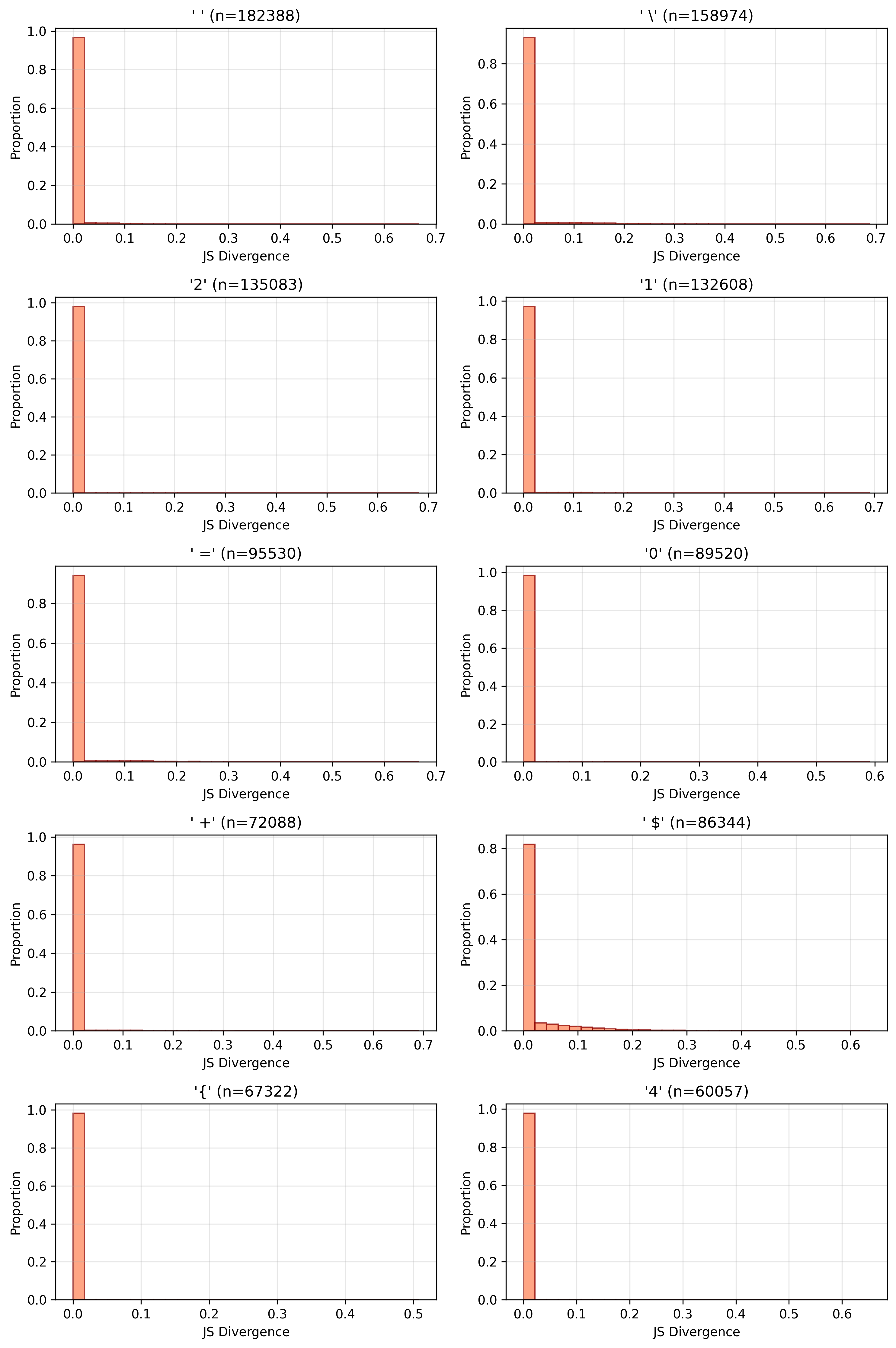}
        \caption{Frequent low JS tokens.}
        \label{fig:div_low_prob}
    \end{subfigure}
    \caption{Histogram of divergences for frequent high JS tokens and frequent low JS tokens (Qwen2.5 32B with DAPO). }
    \label{fig:high_low_js_tokens}
\end{figure}

\begin{figure}[!htbp]
\centering
\begin{subfigure}{0.52\linewidth}
\includegraphics[width=\linewidth]{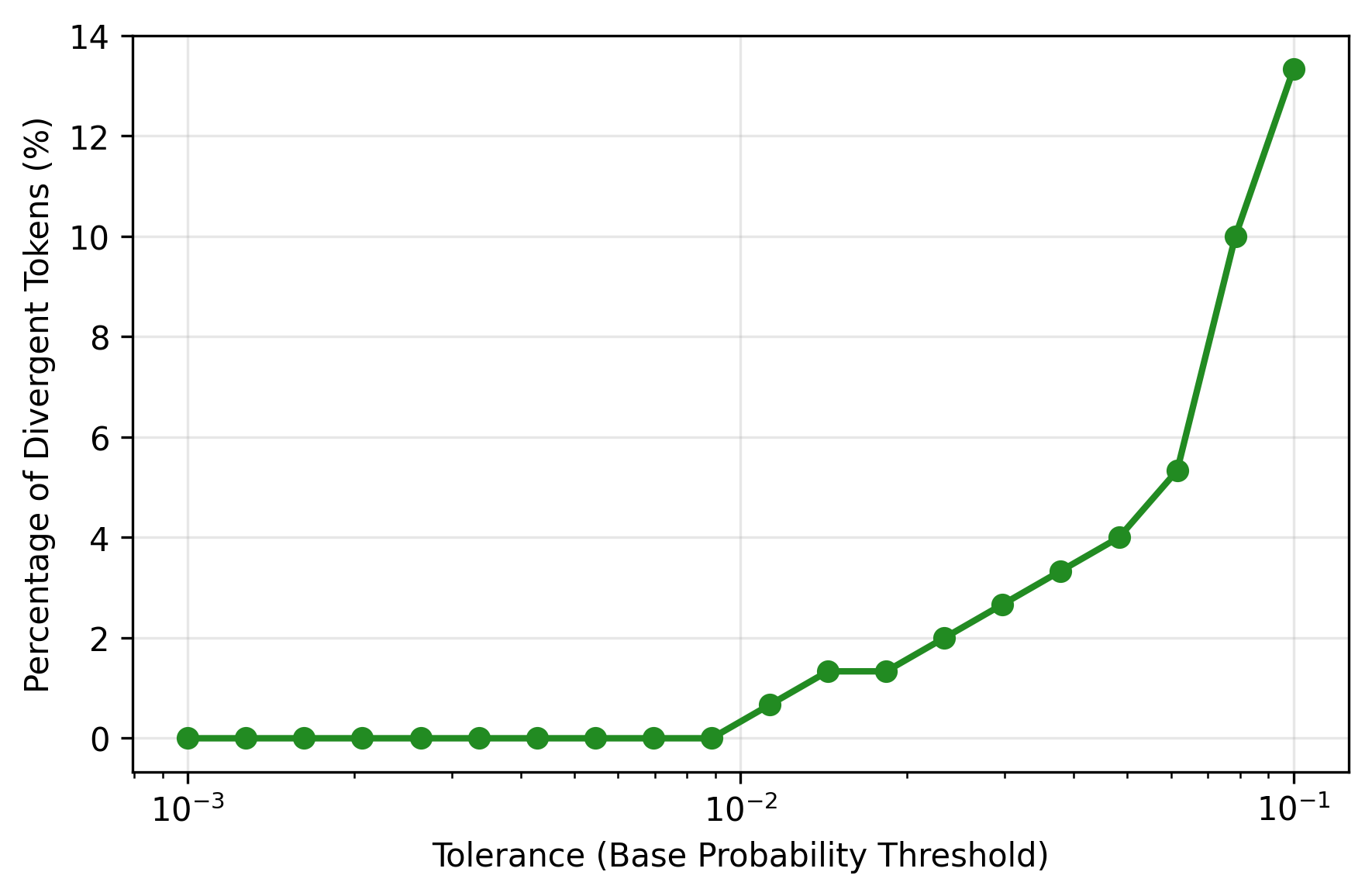}
\caption{SimpleRL}
\end{subfigure}
\caption{Percentage of divergent tokens whose RL top-1 choice had base probability below a given
threshold.}
\label{fig:simplerl_tolerance}
\end{figure}

\begin{figure}[!htbp]
\centering
\begin{subfigure}{0.47\linewidth}
\includegraphics[width=\linewidth]{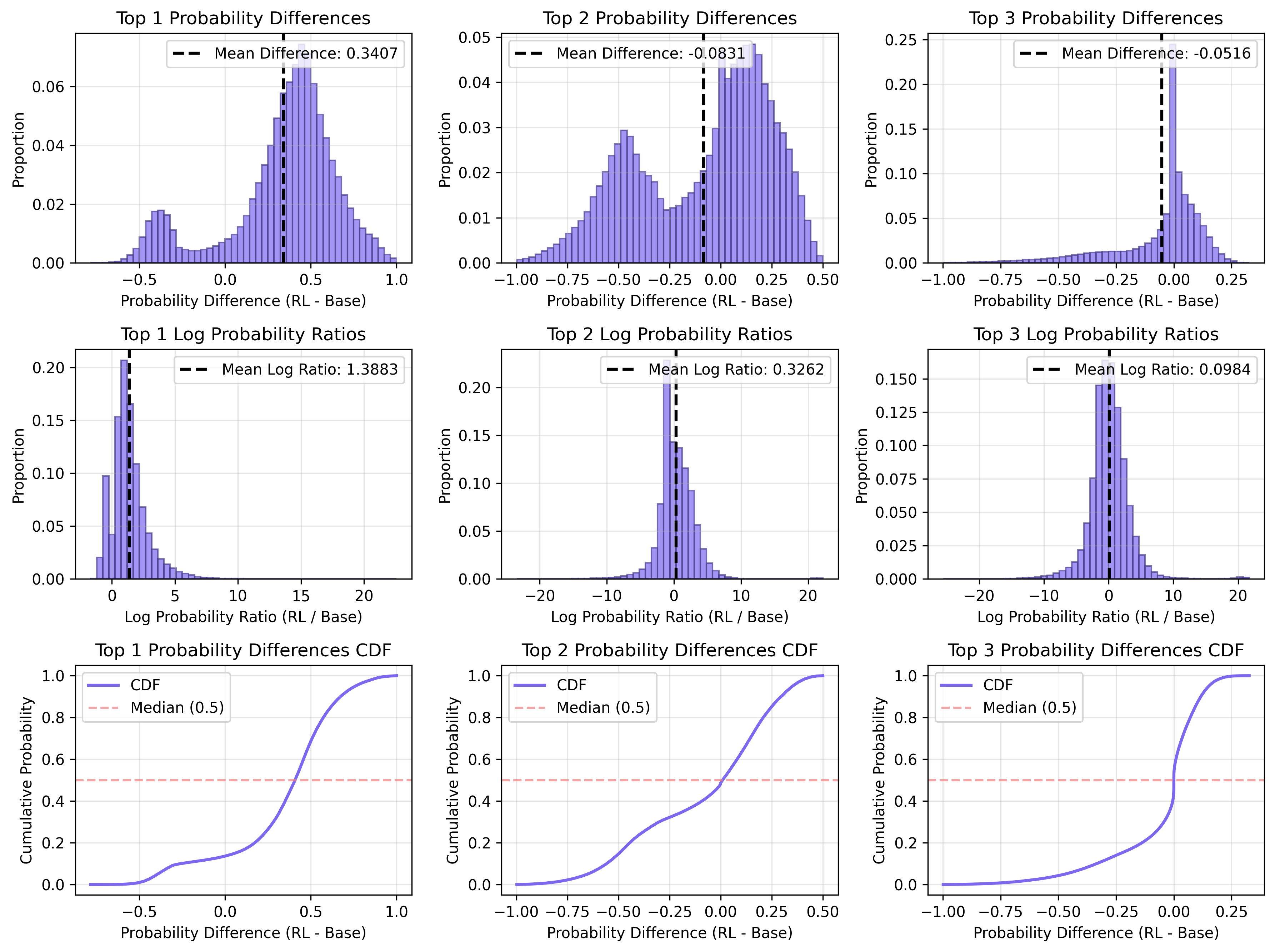}
\caption{DAPO.}
\end{subfigure}
\hfill
\begin{subfigure}{0.47\linewidth}
\includegraphics[width=\linewidth]{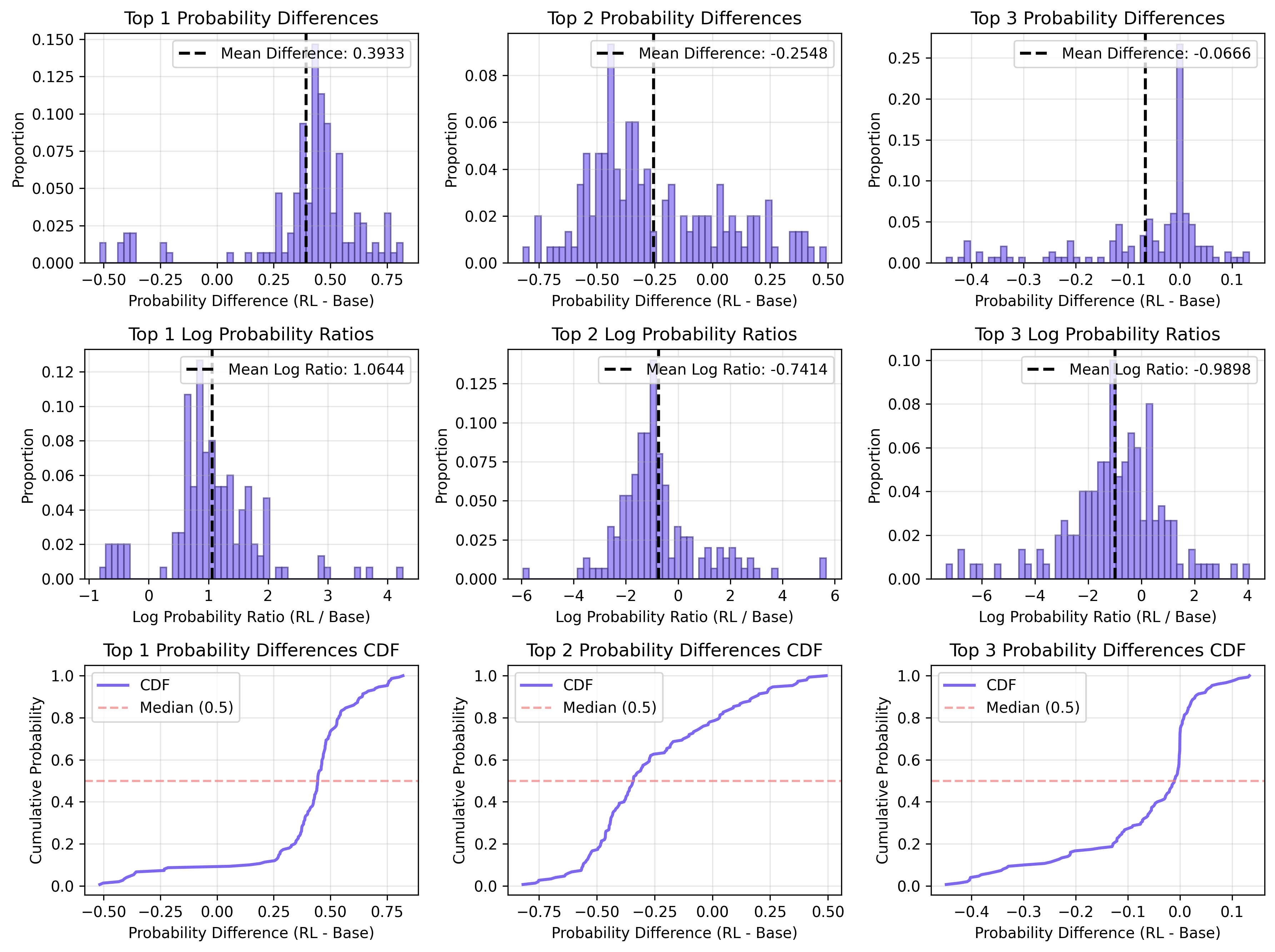}
\caption{SimpleRL.}
\end{subfigure}
    \caption{Probability differences and ratios for top-3 tokens under DAPO and SimpleRL among divergent distributions ($\js>0.1$).}
    \label{fig:probability_shifts}
\end{figure}

\begin{figure}[!htbp]
    \centering
    \begin{subfigure}{0.45\linewidth}
        \includegraphics[width=\linewidth]{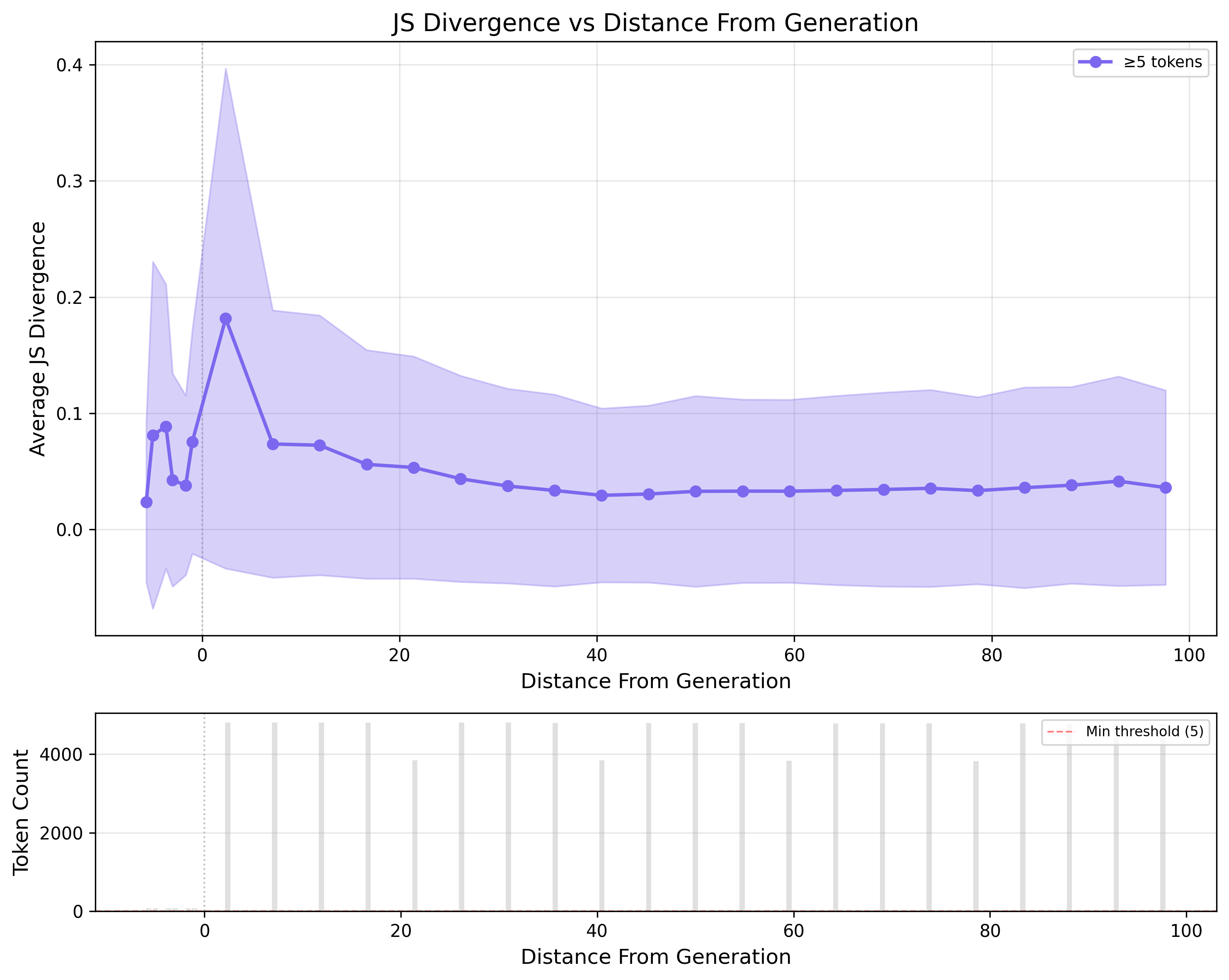}
        \caption{Near the start of generation}
    \end{subfigure}
    \hfill
    \begin{subfigure}{0.45\linewidth}
        \includegraphics[width=\linewidth]{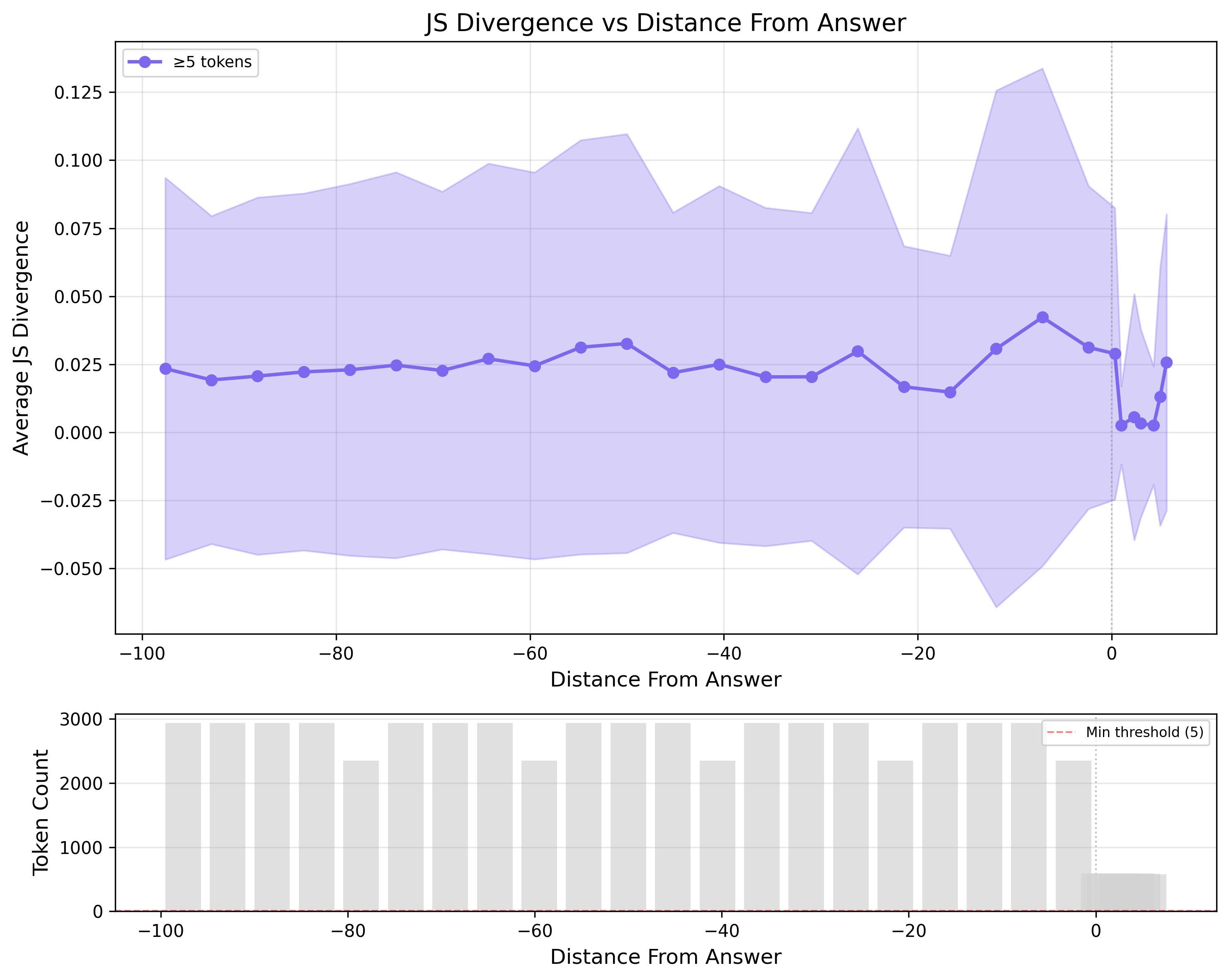}
        \caption{Near the answer span}
    \end{subfigure}
    \caption{
        Local averages of JS divergence as a function of distance from key regions (prompt beginning and answer) for Qwen2.5-32B models on AIME 2024.
        Average divergence peaks occur in the same early and late windows highlighted by the positional analysis.
    }
    \label{fig:local_js_windows}
\end{figure}

\begin{figure}[!htbp]
    \centering
    \begin{subfigure}{0.43\linewidth}
        \includegraphics[width=\linewidth]{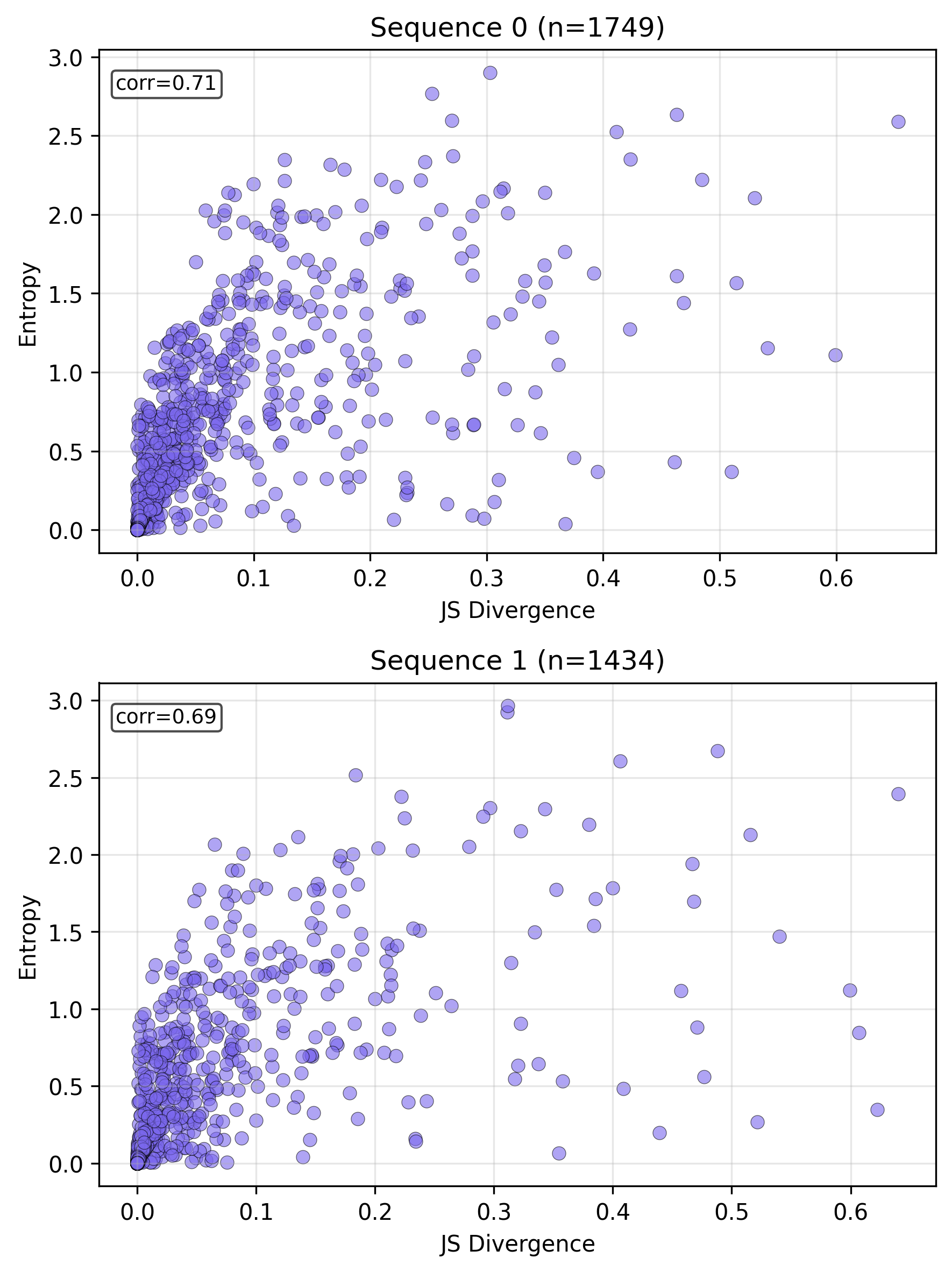}
        \caption{Qwen2.5-32B DAPO}
    \end{subfigure}
    \hspace{1cm}
    \begin{subfigure}{0.43\linewidth}
        \includegraphics[width=\linewidth]{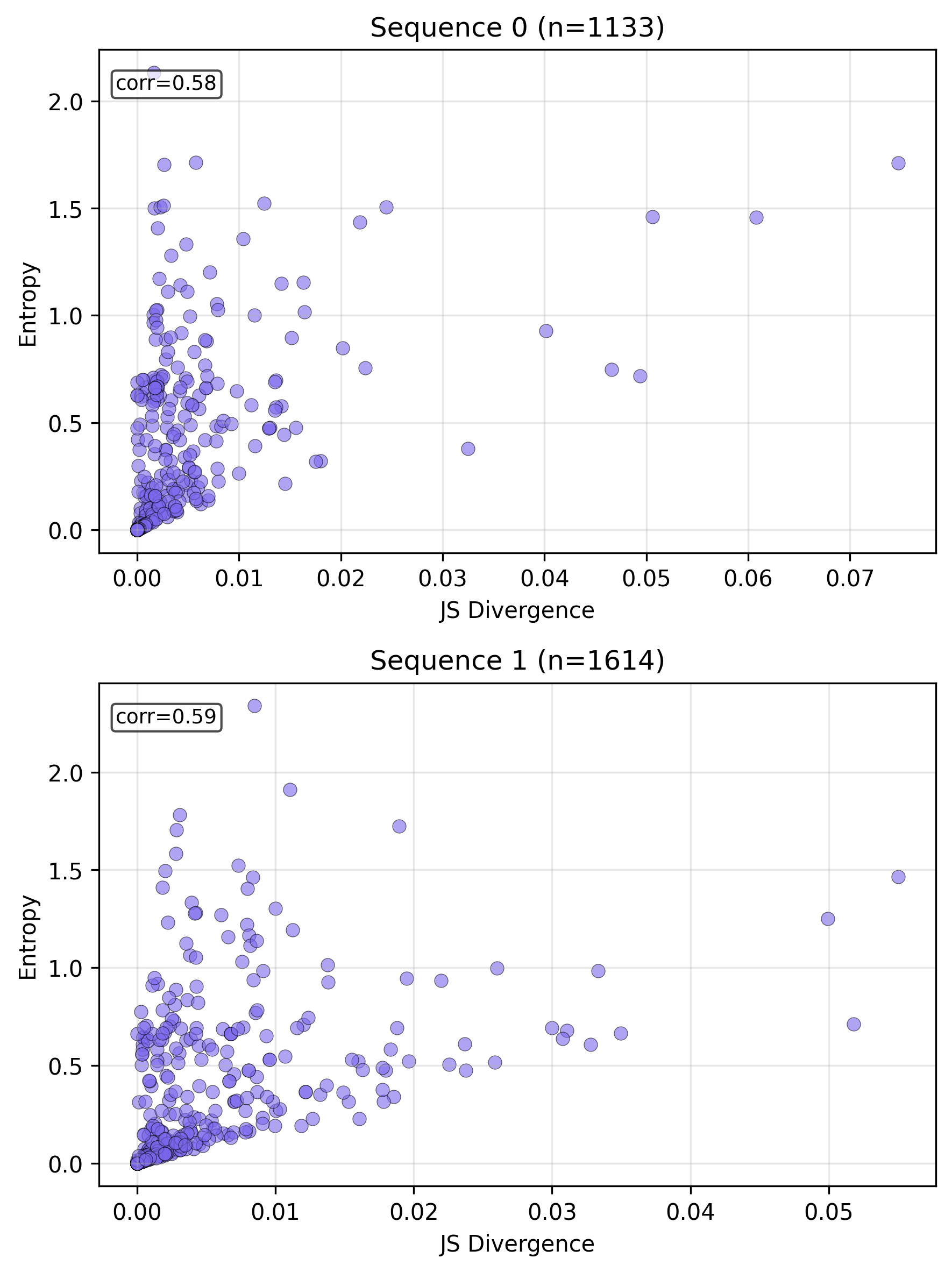}
        \caption{Qwen2.5-32B SimpleRL}
    \end{subfigure}
    \caption{
        Per-sequence scatter plots relating entropy to JS divergence for Qwen2.5-32B DAPO and SimpleRL on AIME 2024. DAPO exhibits a broader entropy spread among divergent tokens, whereas SimpleRL concentrates divergence in higher-entropy regions.
    }
    \label{fig:entropy_vs_js_scatter}
\end{figure}

\paragraph{Results on GPQA-Diamond.}
We extend our analysis to GPQA-Diamond to demonstrate the generalizability of our findings across different reasoning benchmarks. Figure~\ref{fig:gpqa_js_percentiles} shows JS divergence percentile curves and positional concentration for Qwen2.5-32B with DAPO on GPQA-Diamond, revealing consistent sparsity patterns. Figure~\ref{fig:gpqa_entropy} shows entropy distributions across divergence bins.

\begin{figure}[!htbp]
    \centering
    \begin{subfigure}{0.48\linewidth}
        \includegraphics[width=\linewidth]{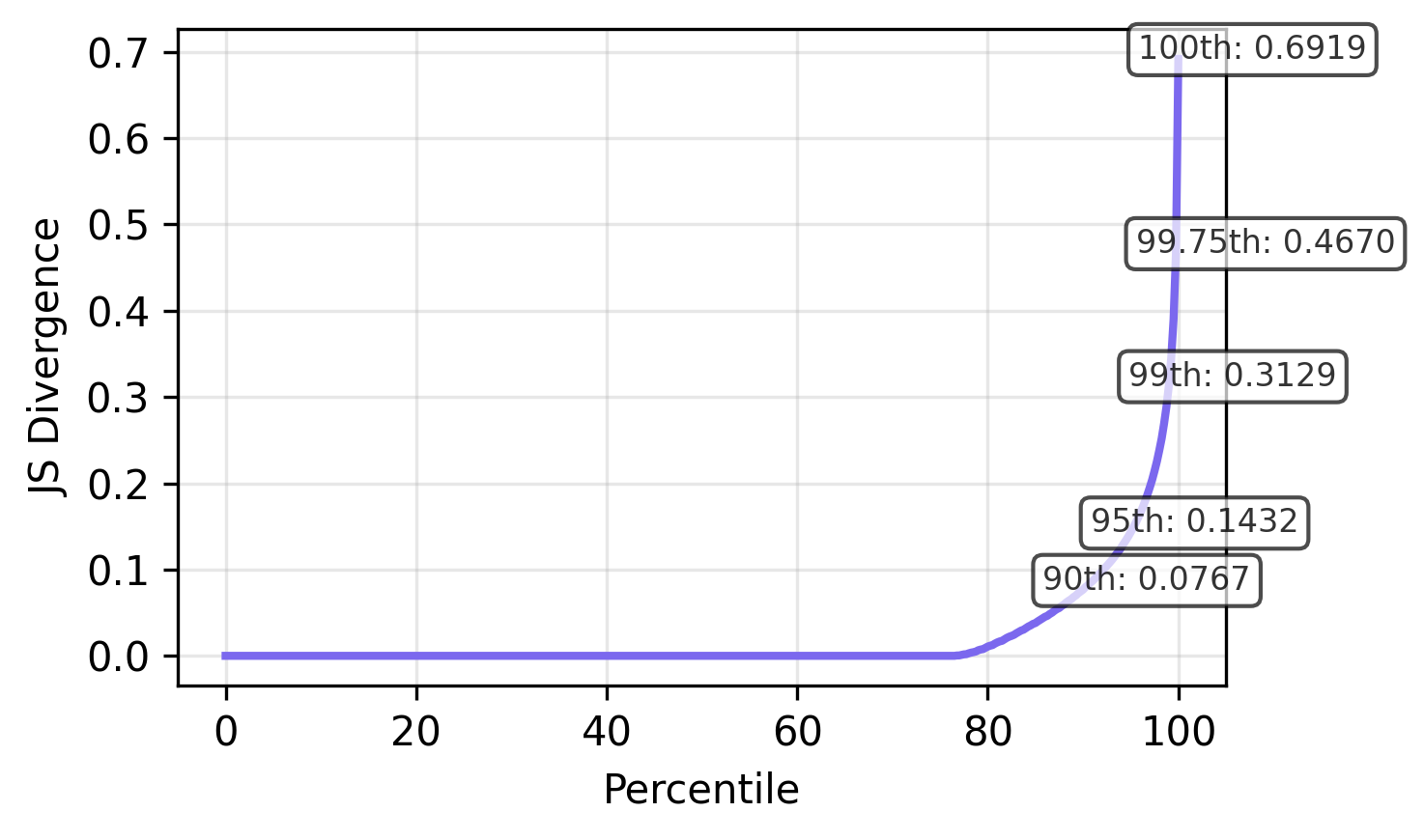}
        \caption{JS divergence percentiles}
    \end{subfigure}
    \hfill
    \begin{subfigure}{0.48\linewidth}
        \includegraphics[width=\linewidth]{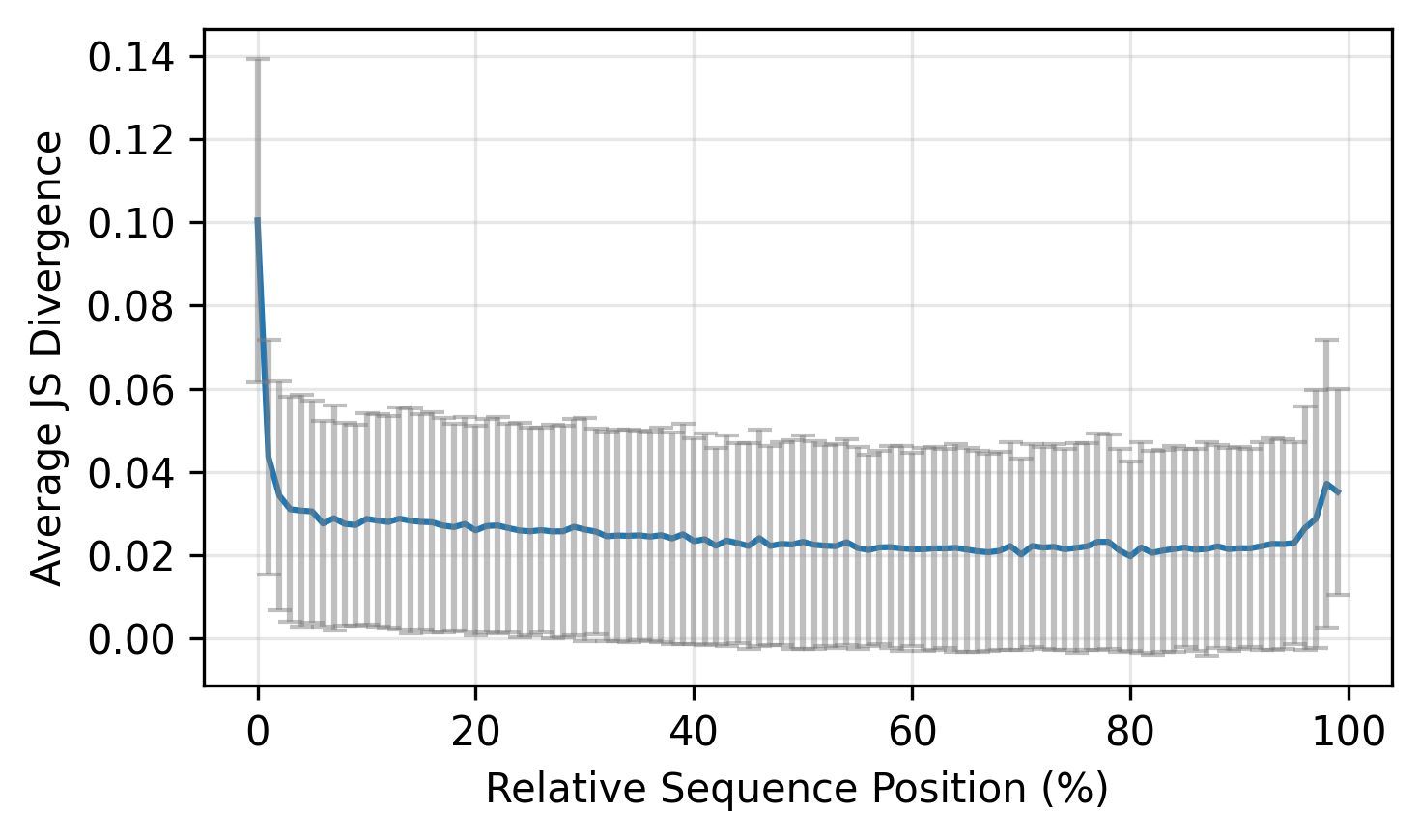}
        \caption{Positional concentration}
    \end{subfigure}
    \caption{
        JS divergence analysis for Qwen2.5-32B with DAPO on GPQA-Diamond.
        The sparsity patterns and positional concentration are consistent with findings on AIME datasets.
    }
    \label{fig:gpqa_js_percentiles}
\end{figure}

\begin{figure}[!htbp]
    \centering
    \begin{subfigure}{0.42\linewidth}
        \includegraphics[width=\linewidth]{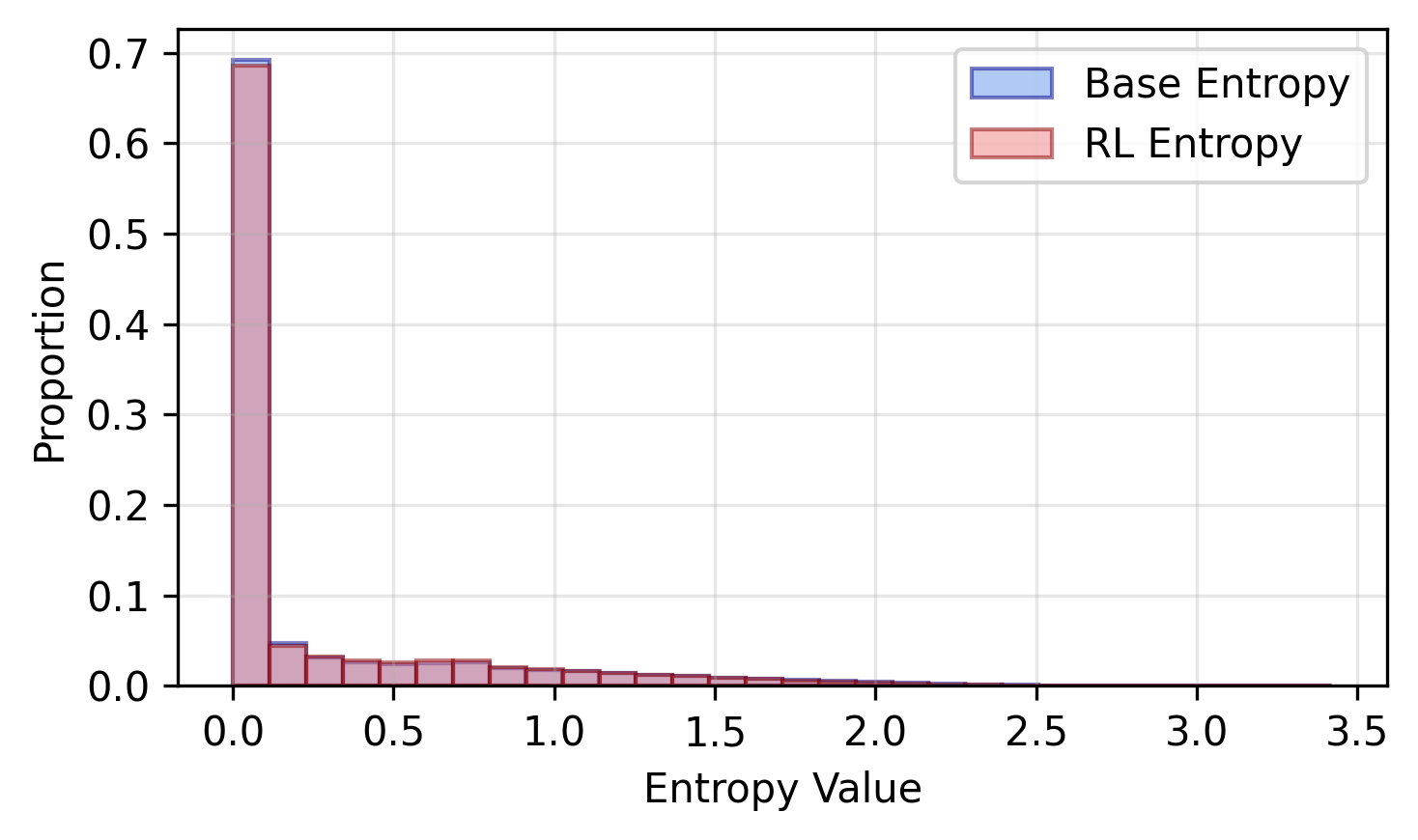}
        \caption{Low JS bin ($<0.1$).}
    \end{subfigure}
    \hspace{1cm}
    \begin{subfigure}{0.42\linewidth}
        \includegraphics[width=\linewidth]{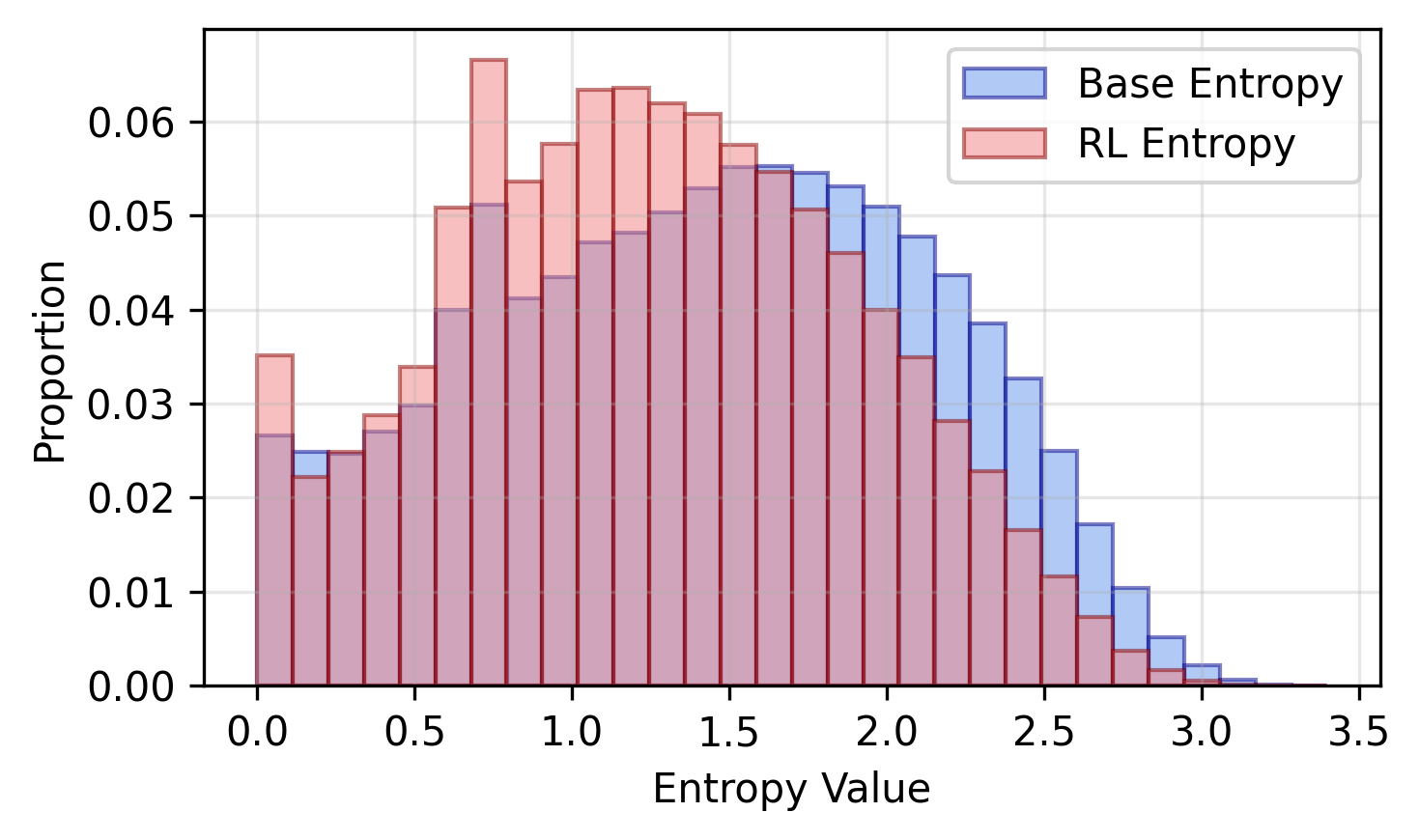}
        \caption{High JS bin ($>0.1$).}
    \end{subfigure}
    \caption{
        Entropy distributions across divergence bins for Qwen2.5-32B with DAPO on GPQA-Diamond.
        Patterns are consistent with those observed on AIME datasets.
    }
    \label{fig:gpqa_entropy}
\end{figure}

\FloatBarrier
\paragraph{Effect of Top-$p$ Sampling on JS Divergence.}
To verify that our findings are robust to different top-$p$ sampling settings, we compare JS divergence distributions across different sampling configurations. The default setting uses top-$p=0.7$ for sampling. We also evaluate configurations where sampling is performed with top-$p=0.8$ and top-$p=0.9$. Figure~\ref{fig:js_topp_comparison} shows that the sparsity patterns remain consistent across different sampling top-$p$ values, confirming that our main claims are not sensitive to the specific sampling top-$p$ value used.

\begin{figure}[!htbp]
    \centering
    \begin{subfigure}{0.4\textwidth}
        \includegraphics[width=\linewidth]{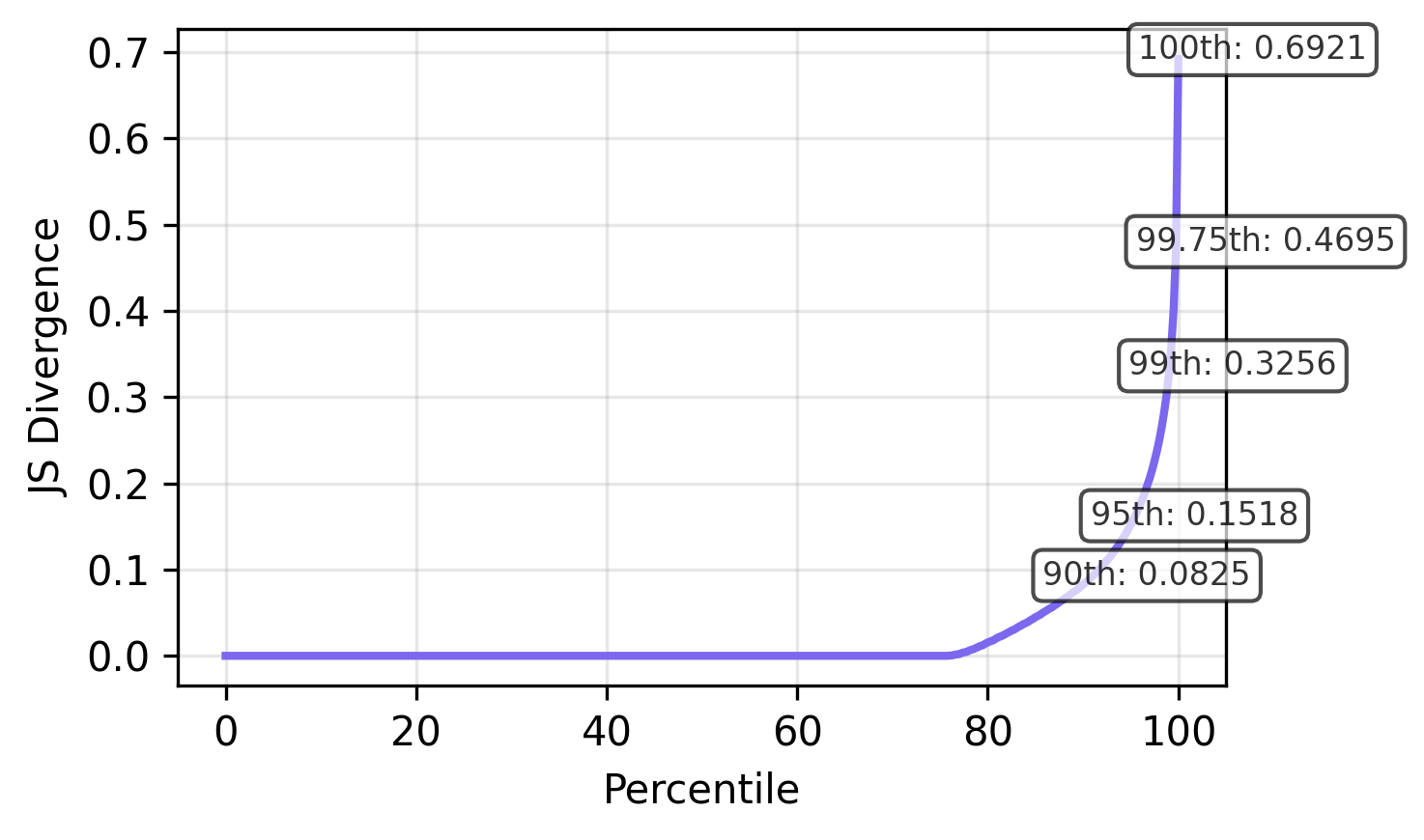}
        \caption{Sampling top-$p=0.8$}
    \end{subfigure}
    \hspace{1cm}
    \begin{subfigure}{0.4\textwidth}
        \includegraphics[width=\linewidth]{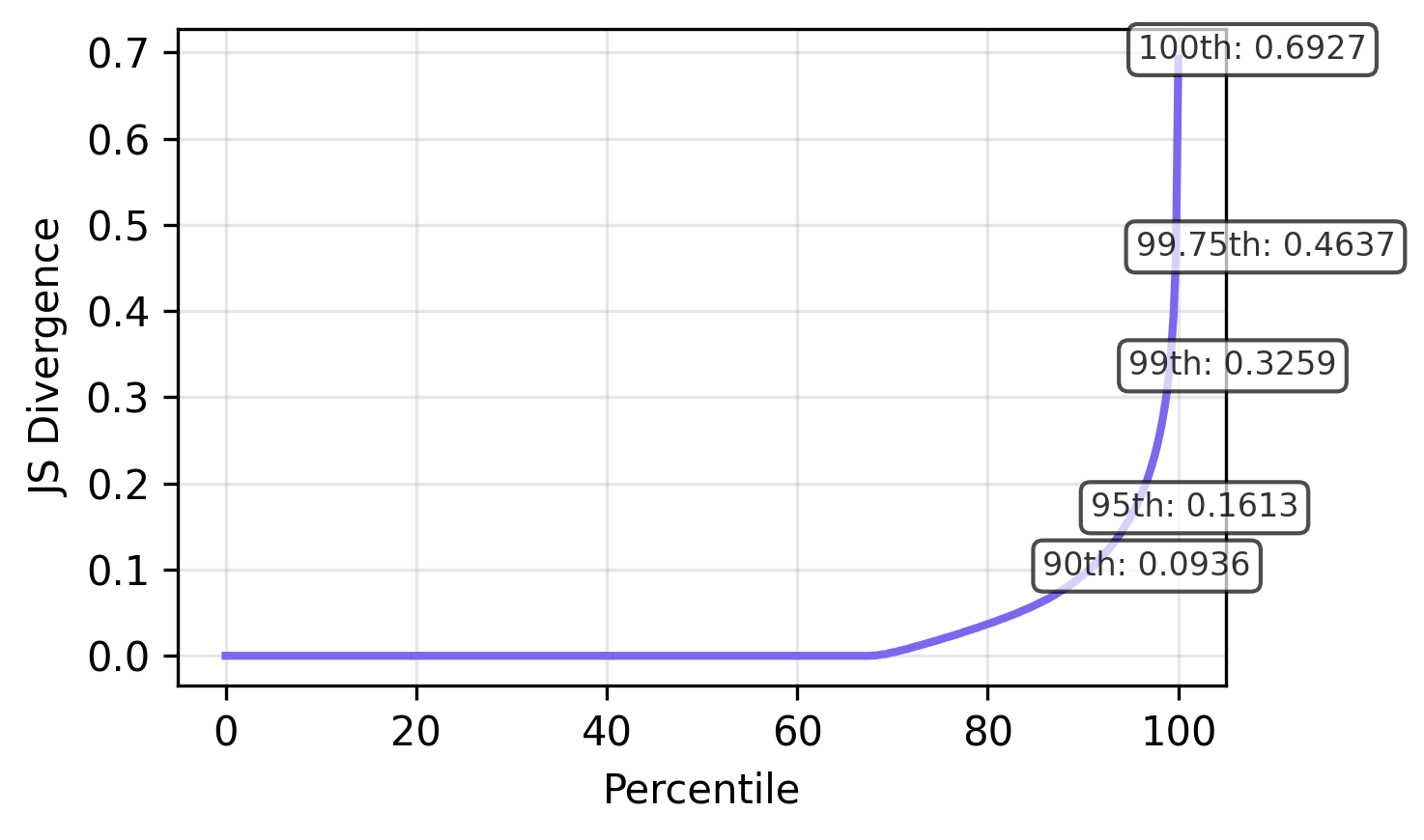}
        \caption{Sampling top-$p=0.9$}
    \end{subfigure}
    \caption{
        JS divergence percentile curves for Qwen2.5-32B with DAPO on AIME 2024 under different top-$p$ sampling settings.
        The sparsity patterns remain consistent across different sampling top-$p$ values, indicating robustness to the specific sampling configuration.
    }
    \label{fig:js_topp_comparison}
\end{figure}

\paragraph{JS Divergence on AIME 2025.}
Figure~\ref{fig:js_32b_aime25} shows JS divergence percentile curves for Qwen2.5-32B with DAPO and SimpleRL on AIME 2025, demonstrating consistent sparsity patterns across datasets.

\begin{figure}[!htbp]
    \centering
    \begin{subfigure}{0.40\textwidth}
        \includegraphics[width=\linewidth]{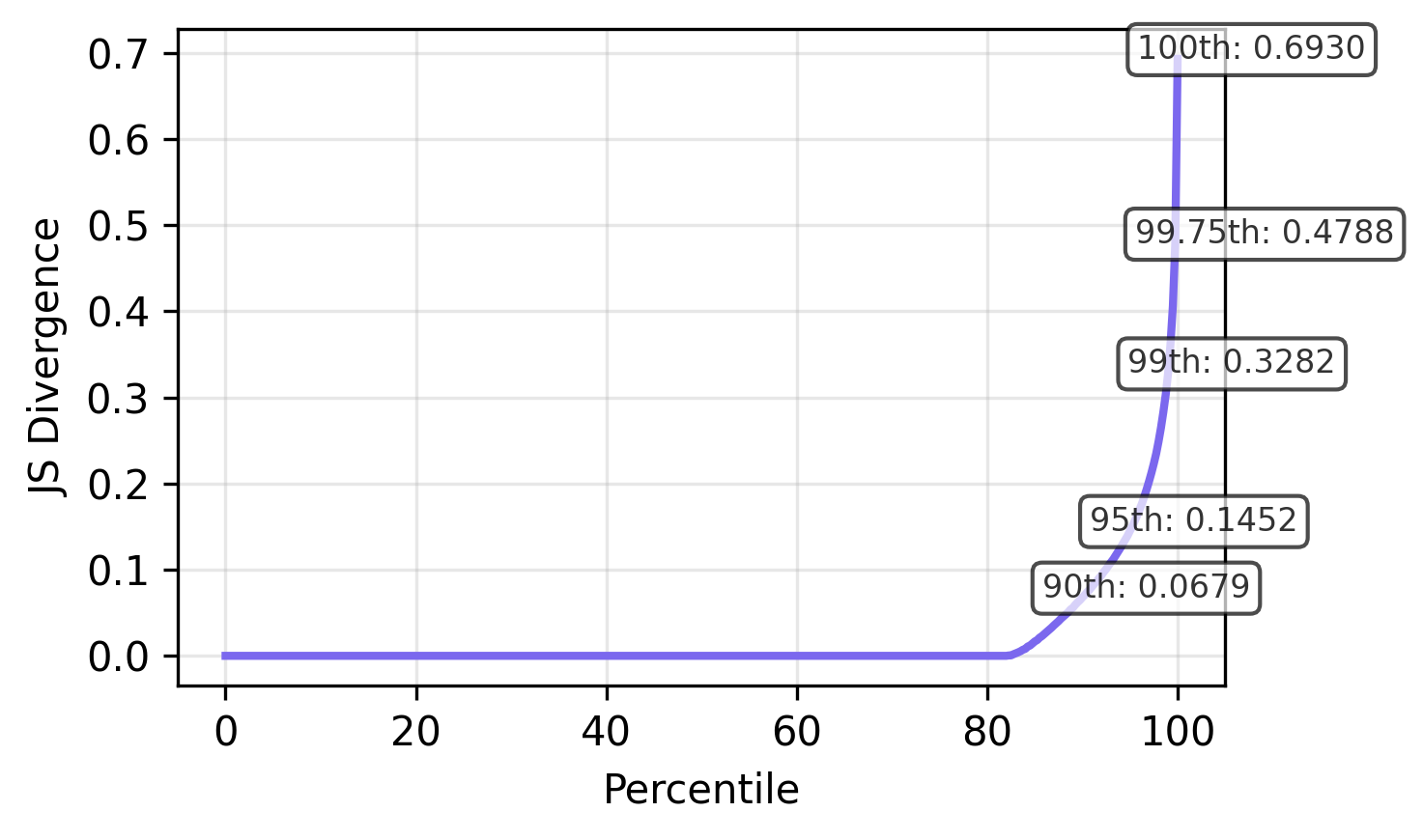}
        \caption{DAPO: Percentile curve}
    \end{subfigure}
    \hspace{1cm}
    \begin{subfigure}{0.40\textwidth}
        \includegraphics[width=\linewidth]{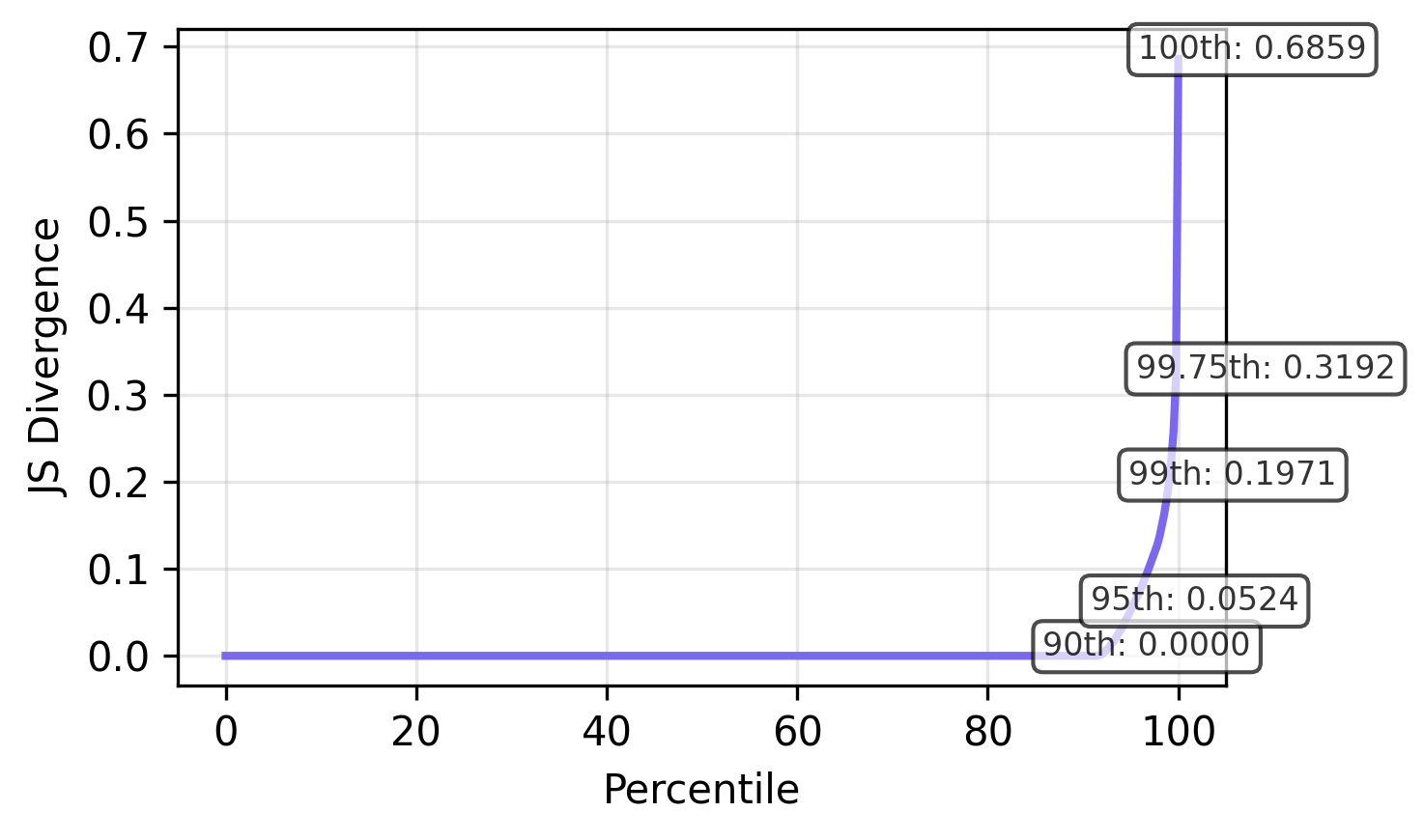}
        \caption{SimpleRL: Percentile curve}
    \end{subfigure}
    
    \caption{
        JS divergence distributions for Qwen2.5-32B with DAPO and SimpleRL on AIME 2025.
        The sparsity patterns are consistent with those observed on AIME 2024, confirming the robustness of our findings across datasets.
    }
    \label{fig:js_32b_aime25}
\end{figure}

\FloatBarrier
\paragraph{Effect of Top-$p$ Truncation on JS Divergence.}
To verify that our use of top-$p$ truncated distributions (with $\text{topp}=0.7$) does not significantly impact our findings, we compare JS divergence distributions computed using the estimated full distribution (top-$p=1$) under the original sampling setting (top-$p=0.7$) with those using truncated distributions. Figure~\ref{fig:js_topp1_comparison} shows that the patterns remain consistent: distributional shifts are highly sparse regardless of truncation, with the vast majority of tokens showing near-zero divergence.

\begin{figure}[!htbp]
    \centering
    \begin{subfigure}{0.42\textwidth}
        \includegraphics[width=\linewidth]{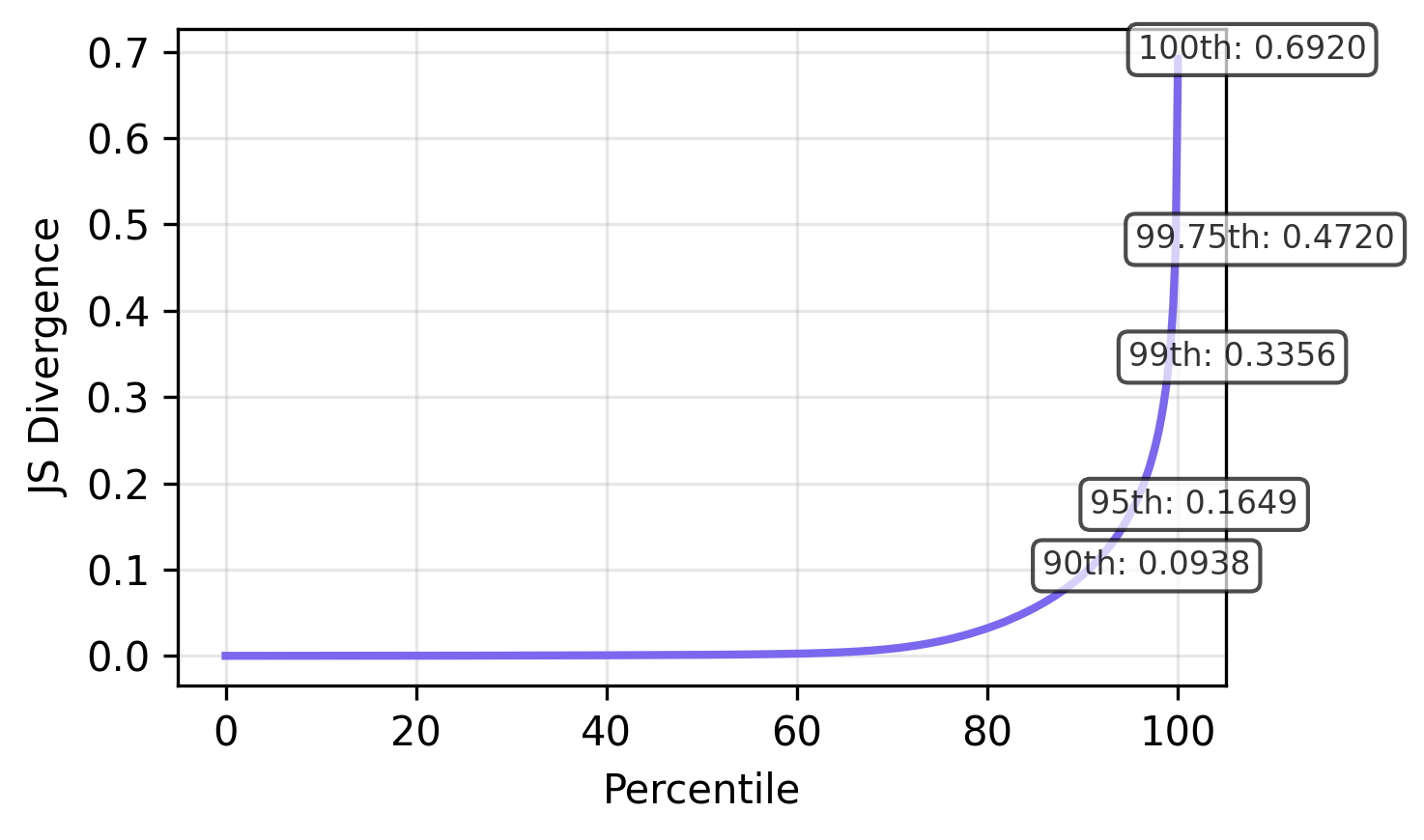}
        \caption{DAPO: Percentile curve (topp1)}
    \end{subfigure}
    \hspace{1cm}
    \begin{subfigure}{0.42\textwidth}
        \includegraphics[width=\linewidth]{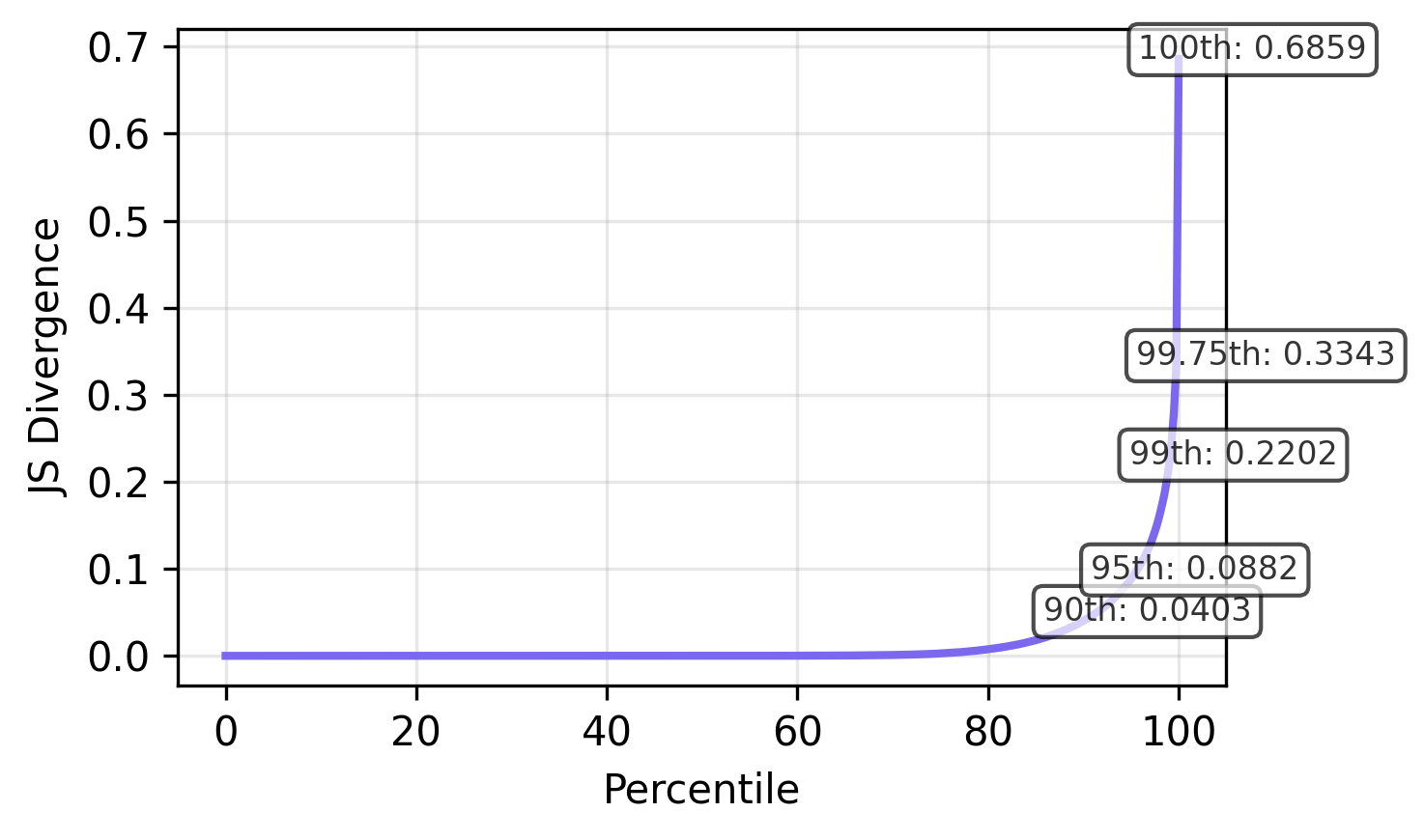}
        \caption{SimpleRL: Percentile curve (topp1)}
    \end{subfigure}
    
    \caption{
        JS divergence distributions computed using top-$p=1$ for Qwen2.5-32B with DAPO and SimpleRL on AIME 2025.
        The sparsity patterns are consistent with those observed using top-$p$ truncated distributions, confirming that truncation does not significantly impact our findings.
    }
    \label{fig:js_topp1_comparison}
\end{figure}

\FloatBarrier
\subsubsection{Comparison of DAPO Variants: Clip-Higher Settings}
\label{subsec:dapo_clip_comparison}

DAPO's clip-higher mechanism controls the degree of upper clip in the PPO updates. We compare two Qwen2.5-Math-7B models trained with DAPO: one with the default clip-higher setting (0.28) and another with a more restrictive setting (0.2). Figure~\ref{fig:js_dapo_variants} shows their JS divergence distributions on AIME 2024 and AIME 2025, revealing how the clip-higher parameter affects distributional shifts across datasets.

\begin{figure}[!htbp]
    \centering
    \begin{subfigure}{0.42\textwidth}
        \includegraphics[width=\linewidth]{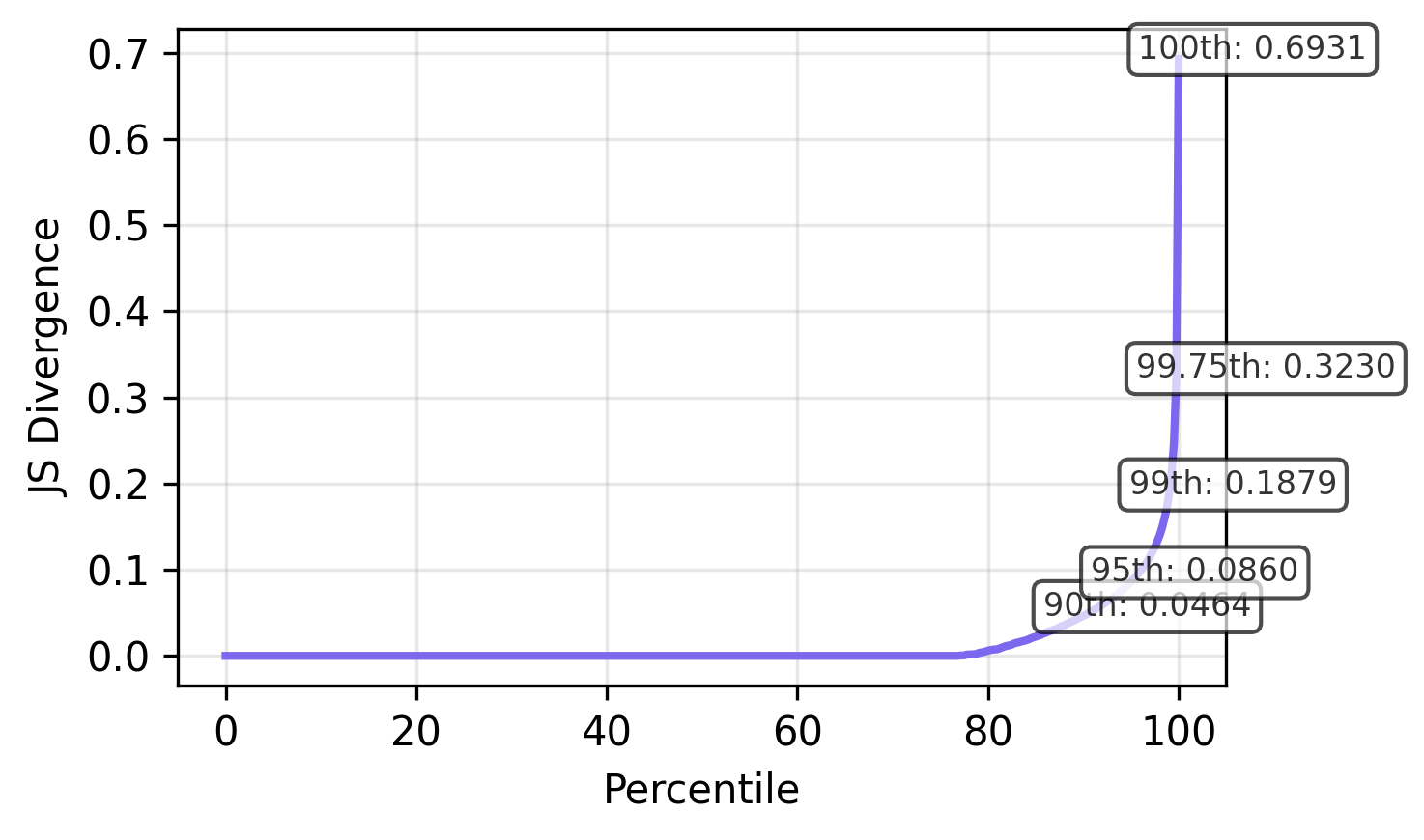}
        \caption{DAPO (0.28) AIME 2024: Percentiles}
    \end{subfigure}
    \hspace{1cm}
    \begin{subfigure}{0.42\textwidth}
        \includegraphics[width=\linewidth]{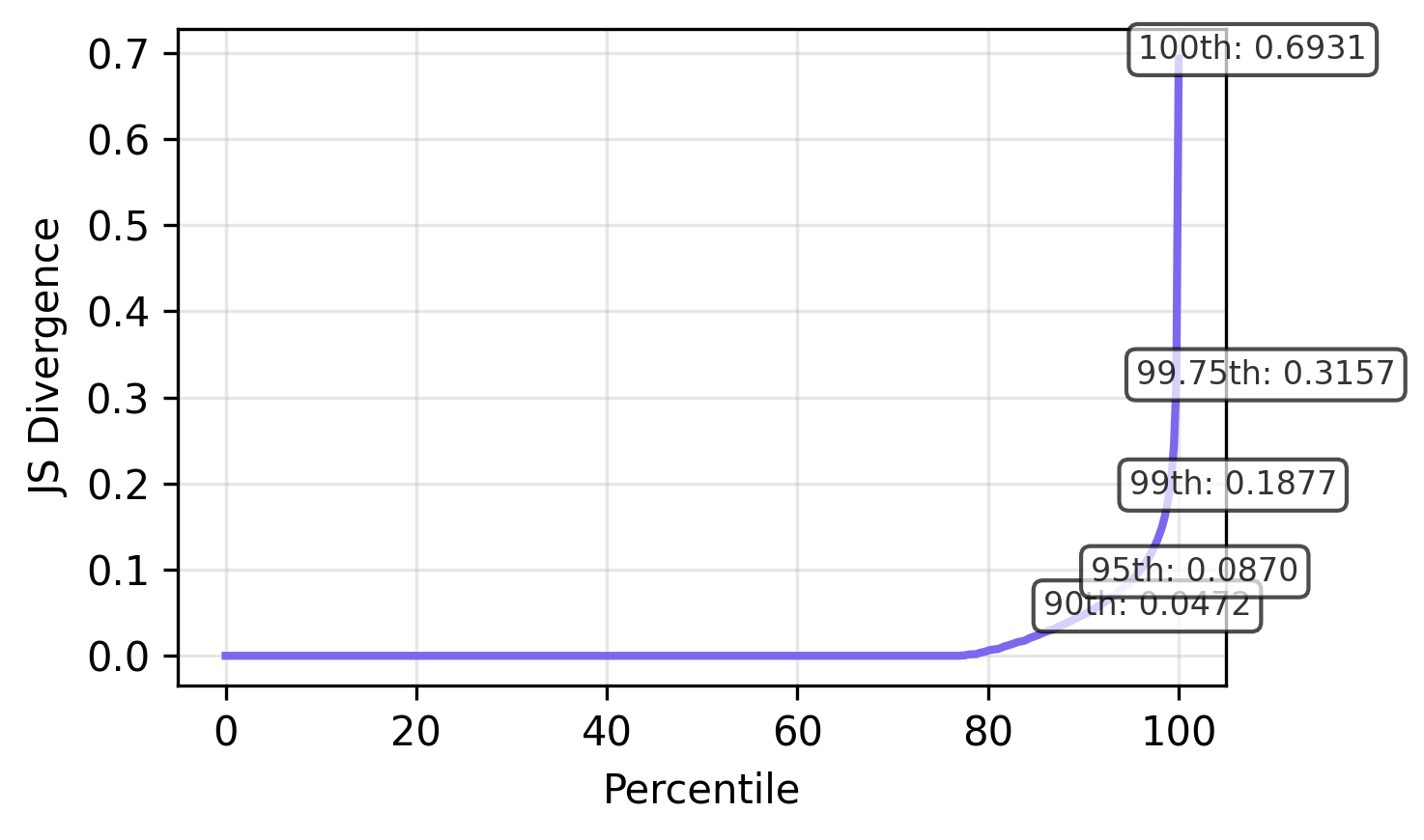}
        \caption{DAPO (0.28) AIME 2025: Percentiles}
    \end{subfigure}
    
    \vspace{1em}
    \begin{subfigure}{0.42\textwidth}
        \includegraphics[width=\linewidth]{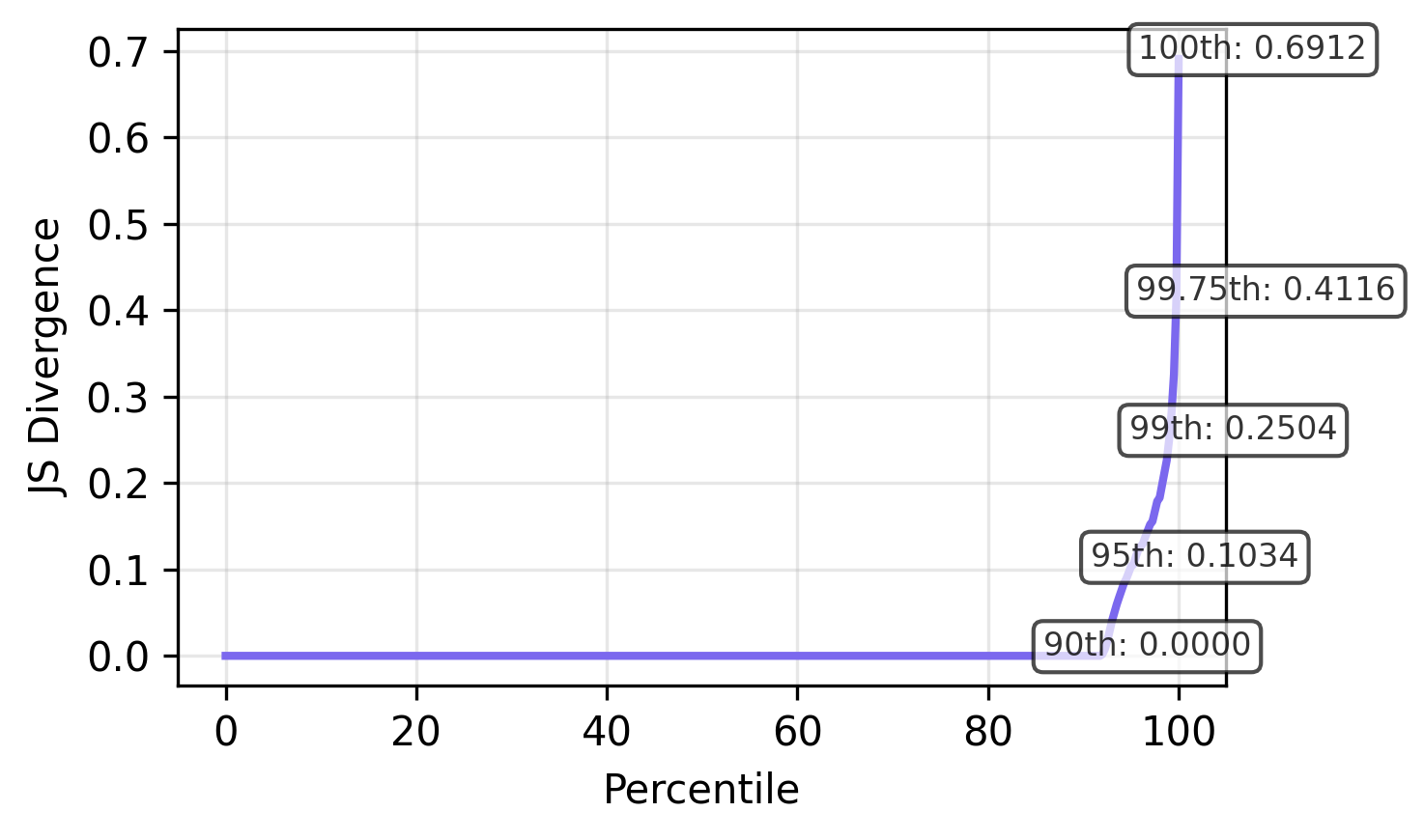}
        \caption{DAPO (0.2) AIME 2024: Percentiles}
    \end{subfigure}
    \hspace{1cm}
    \begin{subfigure}{0.42\textwidth}
        \includegraphics[width=\linewidth]{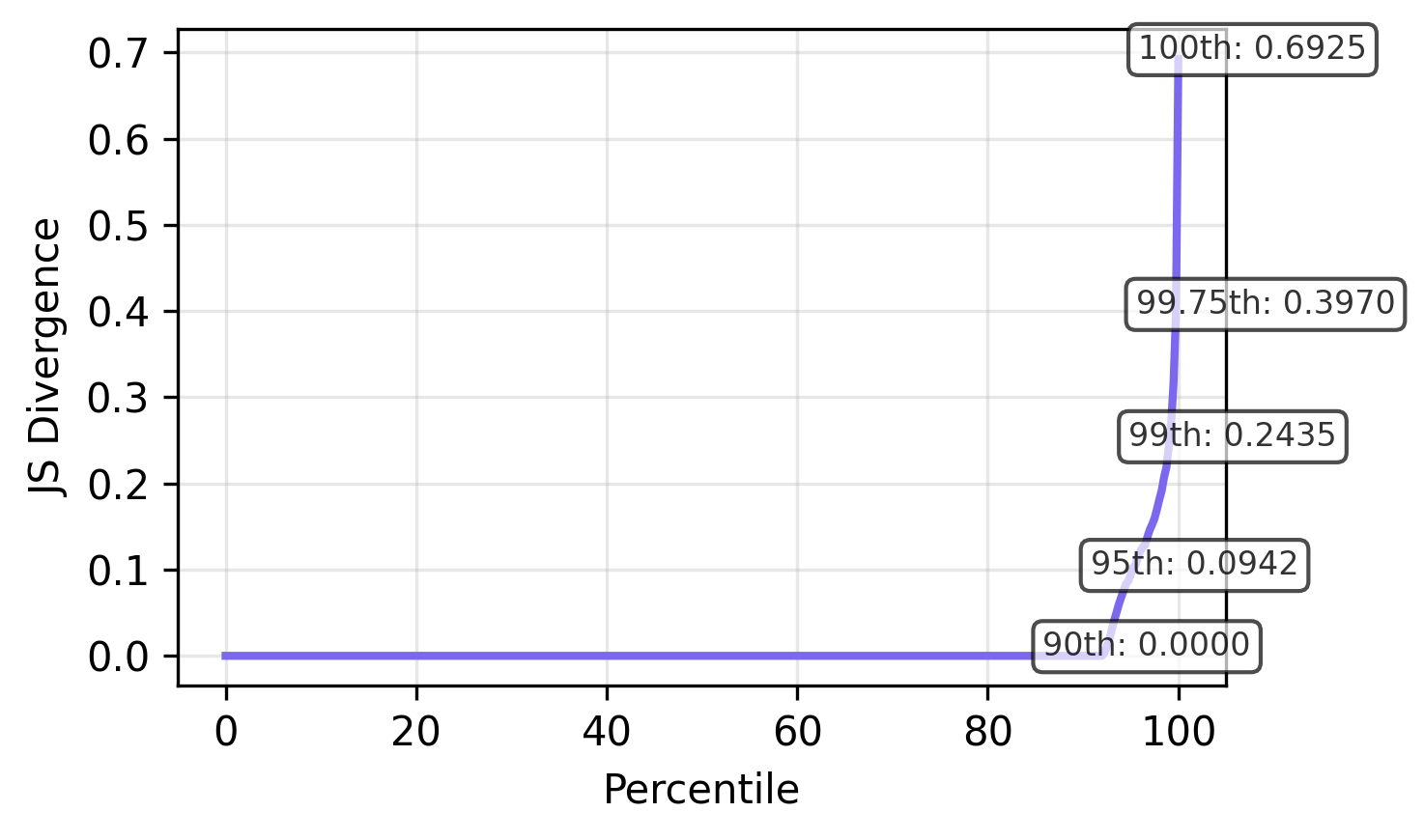}
        \caption{DAPO (0.2) AIME 2025: Percentiles}
    \end{subfigure}
    
    \caption{
        JS divergence distributions for Qwen2.5-Math-7B trained with DAPO under different clip-higher settings on AIME 2024 and AIME 2025.
        The more restrictive clip-high=0.2 setting leads to sparser distributional shifts compared to the default 0.28 setting across both datasets, with a smaller proportion of tokens exhibiting nonnegligible divergence. However, on its divergent token set, the JS values are higher as indicated by the higher upper percentiles.
    }
    \label{fig:js_dapo_variants}
\end{figure}

Figure~\ref{fig:positional_dapo_variants} compares positional concentration patterns on AIME 2024 and AIME 2025, while Figure~\ref{fig:topk_dapo_variants} and Figure~\ref{fig:ranks_dapo_variants} examine top-$k$ overlap and rank reordering, respectively. Figure~\ref{fig:tolerance_7b_dapo} shows the percentage of divergent tokens whose RL top-1 choice had base probability below a given threshold for both DAPO variants across different datasets. Figure~\ref{fig:entropy_7b_dapo} shows entropy distributions across divergence bins for both DAPO variants.

\begin{figure}[!htbp]
    \centering
    \begin{subfigure}{0.38\linewidth}
        \includegraphics[width=\linewidth]{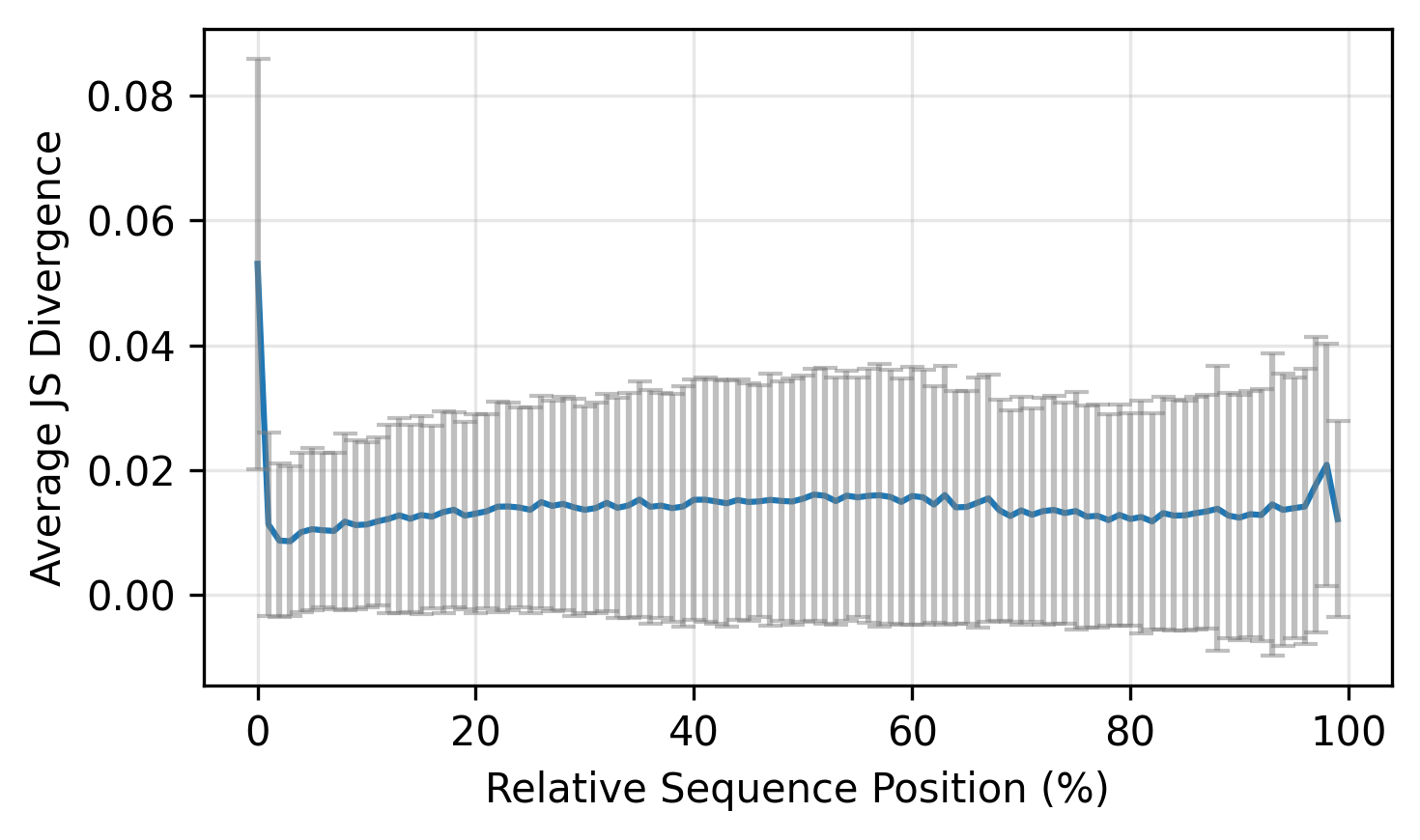}
        \caption{DAPO (0.28) AIME 2024}
    \end{subfigure}
    \hspace{1cm}
    \begin{subfigure}{0.38\linewidth}
        \includegraphics[width=\linewidth]{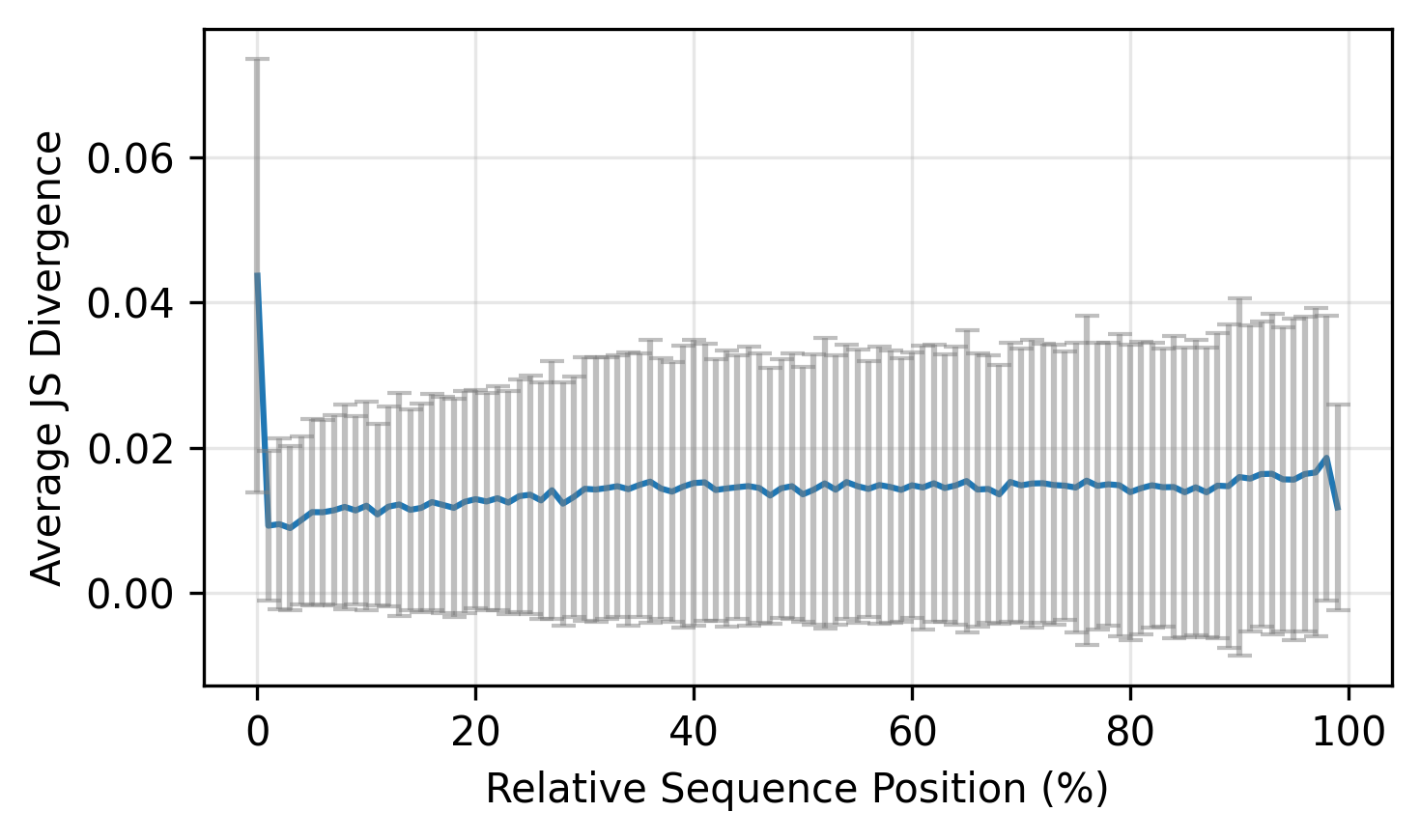}
        \caption{DAPO (0.28) AIME 2025}
    \end{subfigure}
    
    \begin{subfigure}{0.38\linewidth}
        \includegraphics[width=\linewidth]{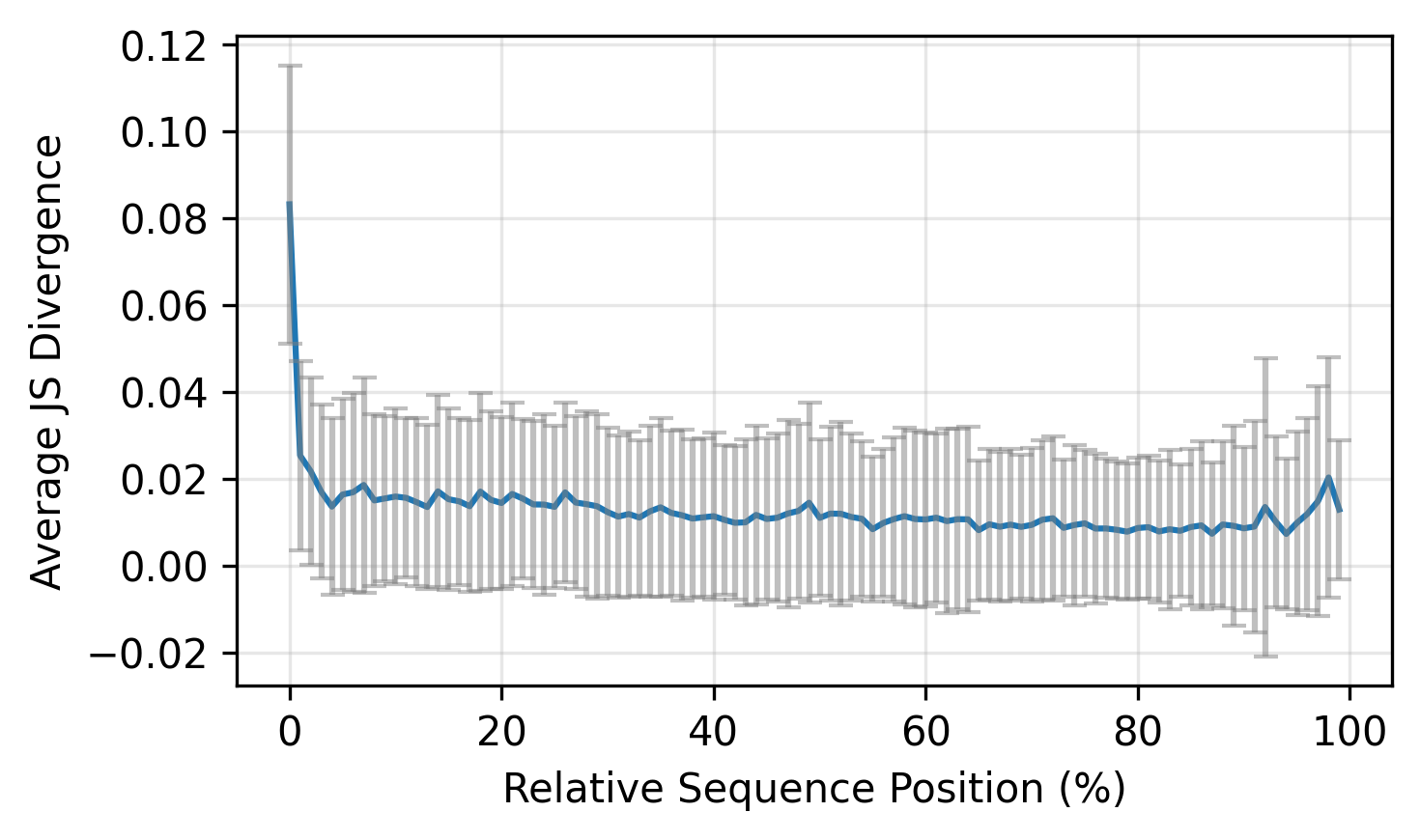}
        \caption{DAPO (0.2) AIME 2024}
    \end{subfigure}
    \hspace{1cm}
    \begin{subfigure}{0.38\linewidth}
        \includegraphics[width=\linewidth]{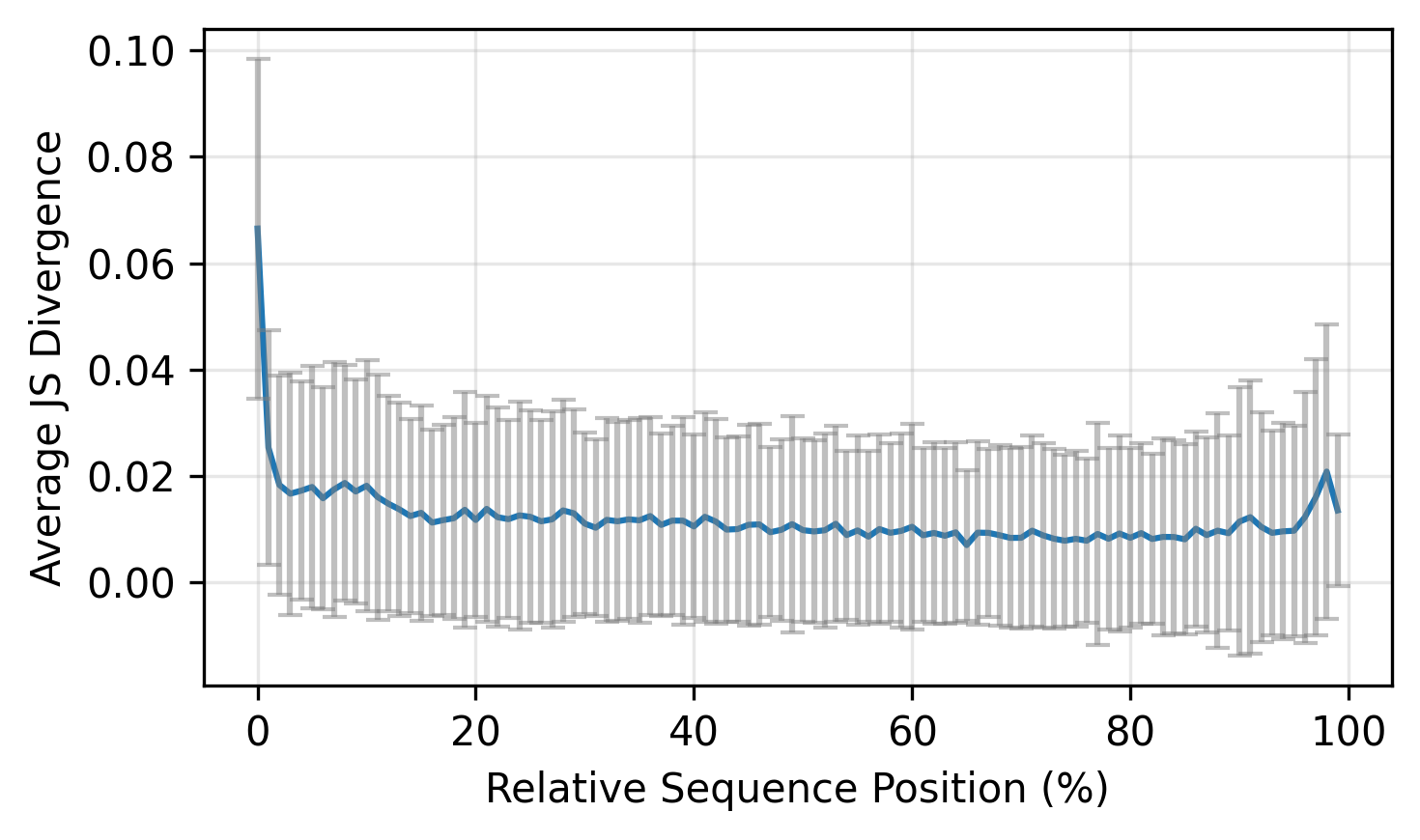}
        \caption{DAPO (0.2) AIME 2025}
    \end{subfigure}
    \caption{
        Mean JS divergence by normalized token position for DAPO variants with different clip-higher settings on AIME 2024 and AIME 2025.
    }
    \label{fig:positional_dapo_variants}
\end{figure}

\begin{figure}[!htbp]
    \centering
    \begin{subfigure}{0.42\linewidth}
        \includegraphics[width=\linewidth]{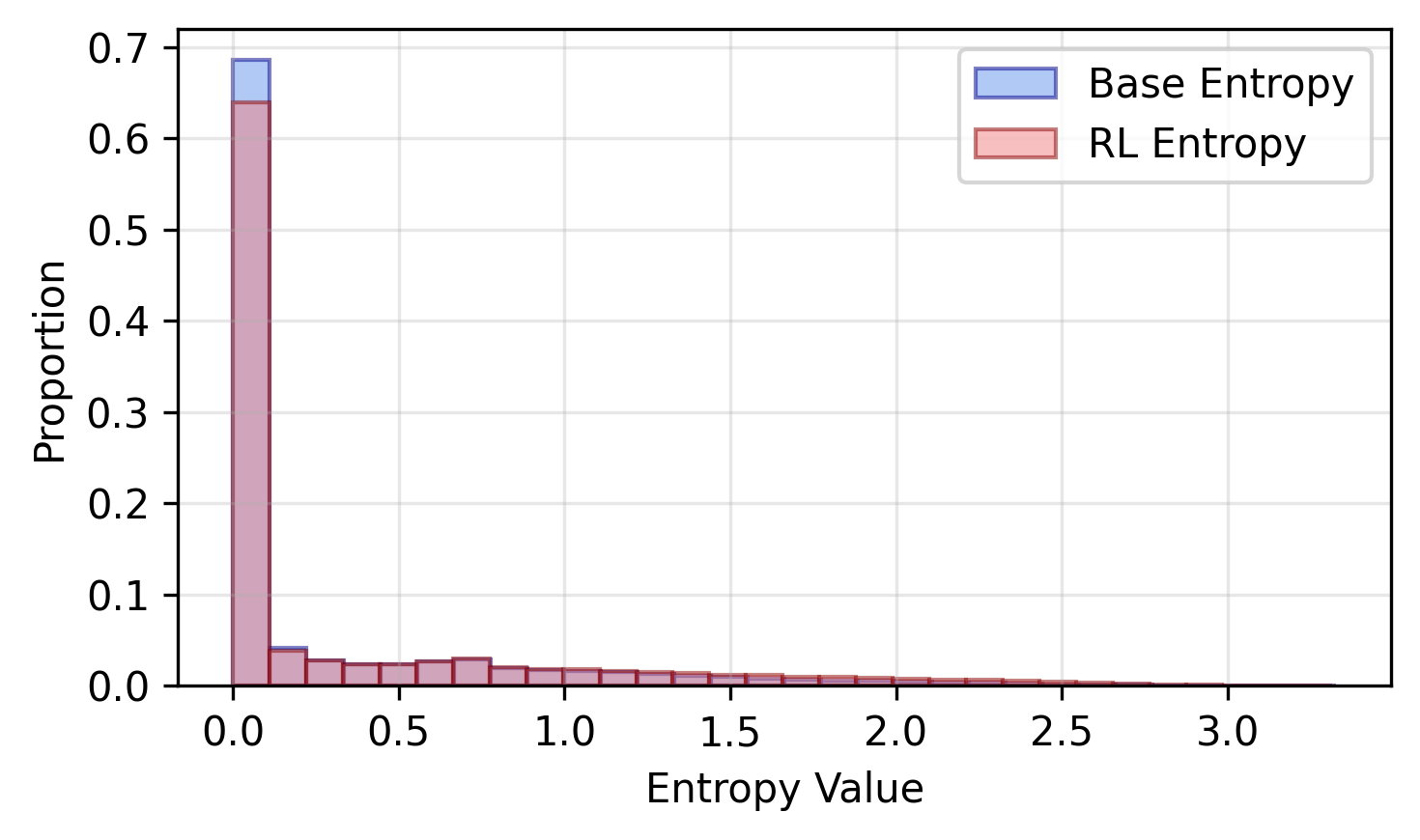}
        \caption{Low JS bin ($<0.1$)}
    \end{subfigure}
    \hspace{1cm}
    \begin{subfigure}{0.42\linewidth}
        \includegraphics[width=\linewidth]{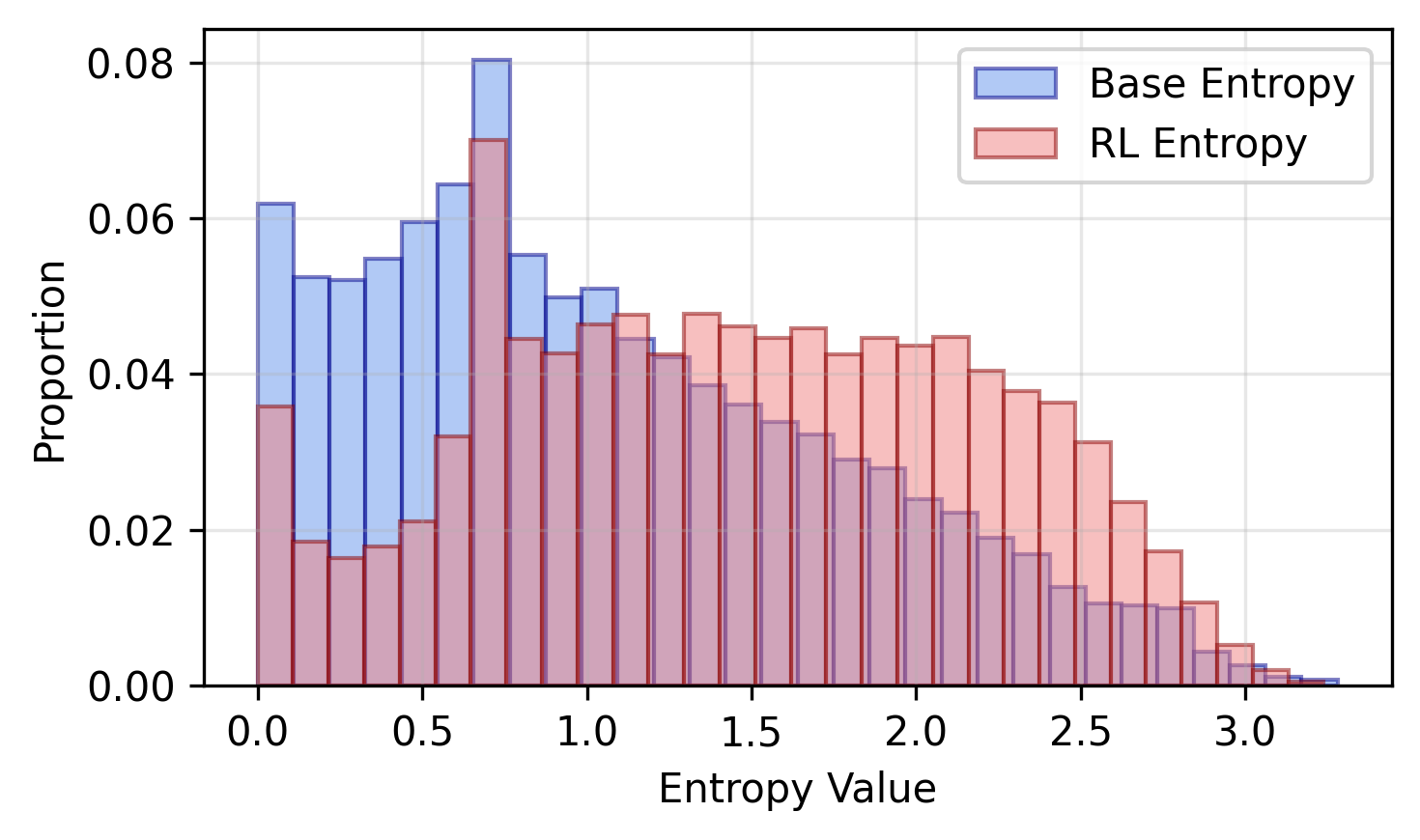}
        \caption{High JS bin ($>0.1$)}
    \end{subfigure}
    
    \vspace{1em}
    \begin{subfigure}{0.42\linewidth}
        \includegraphics[width=\linewidth]{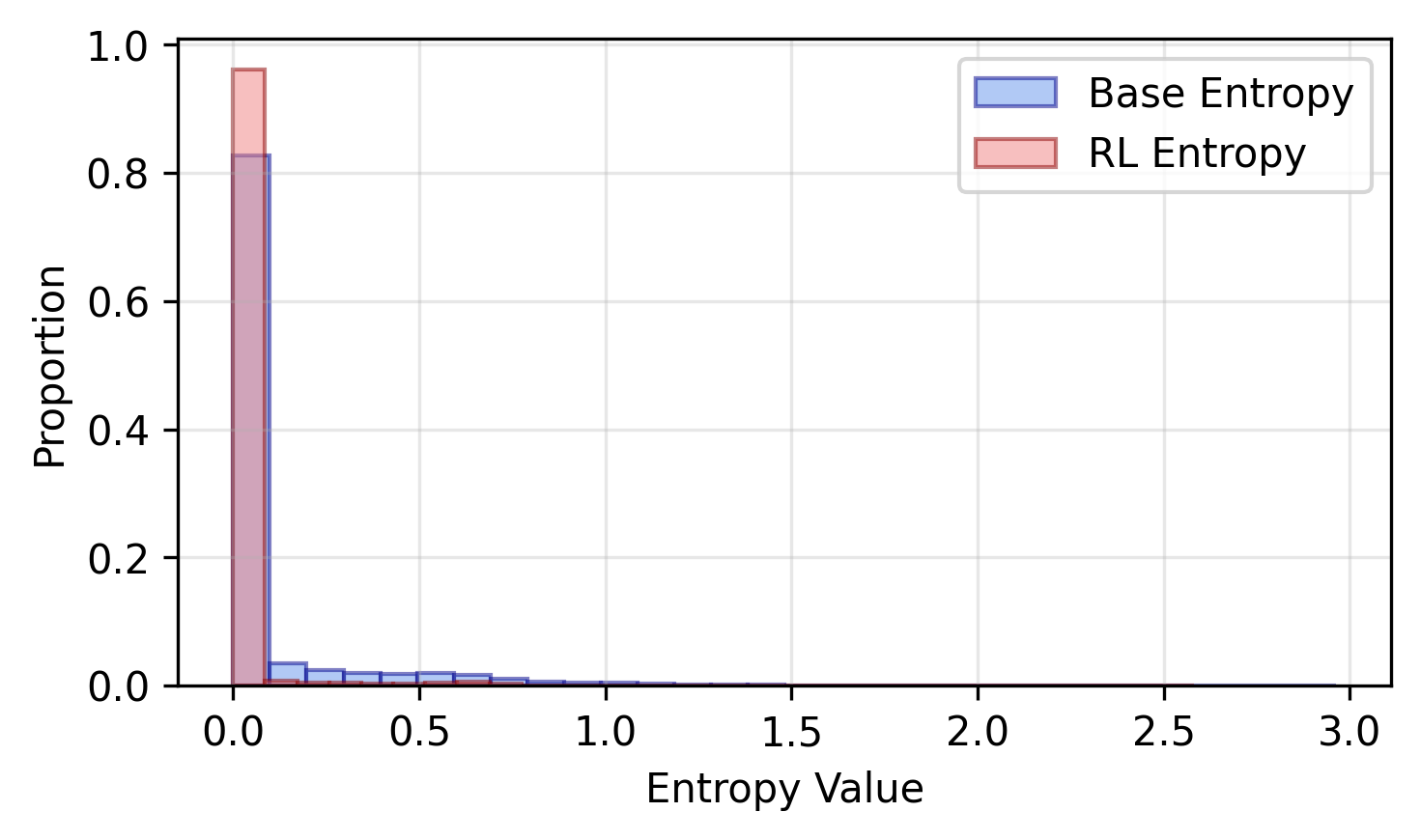}
        \caption{Low JS bin ($<0.1$)}
    \end{subfigure}
    \hspace{1cm}
    \begin{subfigure}{0.42\linewidth}
        \includegraphics[width=\linewidth]{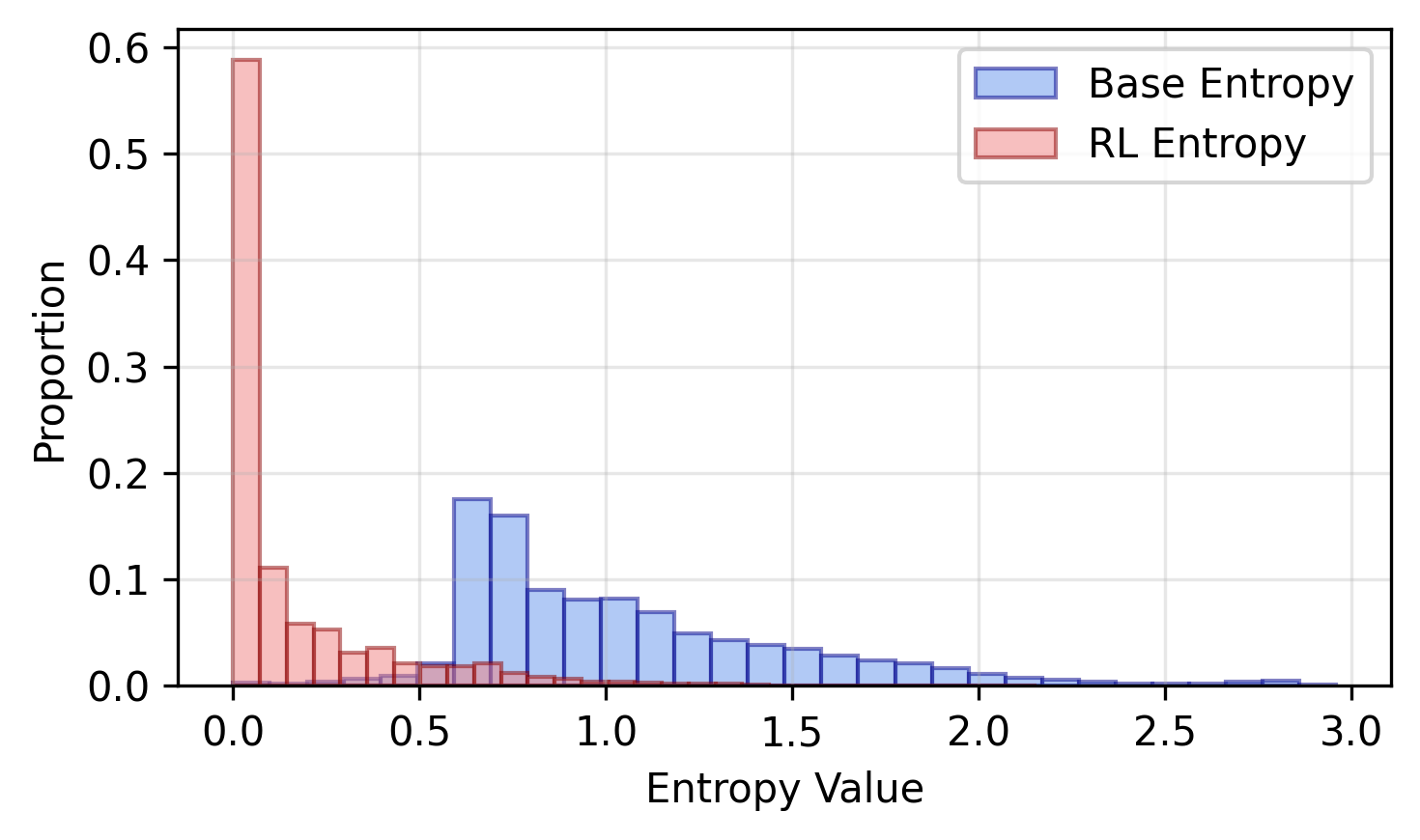}
        \caption{High JS bin ($>0.1$)}
    \end{subfigure}
    
    \caption{
        Entropy distributions across divergence bins for Qwen2.5-Math-7B with DAPO variants on AIME 2024. Top row: DAPO (clip-higher=0.28); bottom row: DAPO (clip-higher=0.2). Patterns are consistent with those observed in the main text, confirming the relationship between entropy and divergence across different clip-higher settings.
    }
    \label{fig:entropy_7b_dapo}
\end{figure}

\begin{figure}[!htbp]
    \centering
    \begin{subfigure}{0.42\linewidth}
        \includegraphics[width=\linewidth]{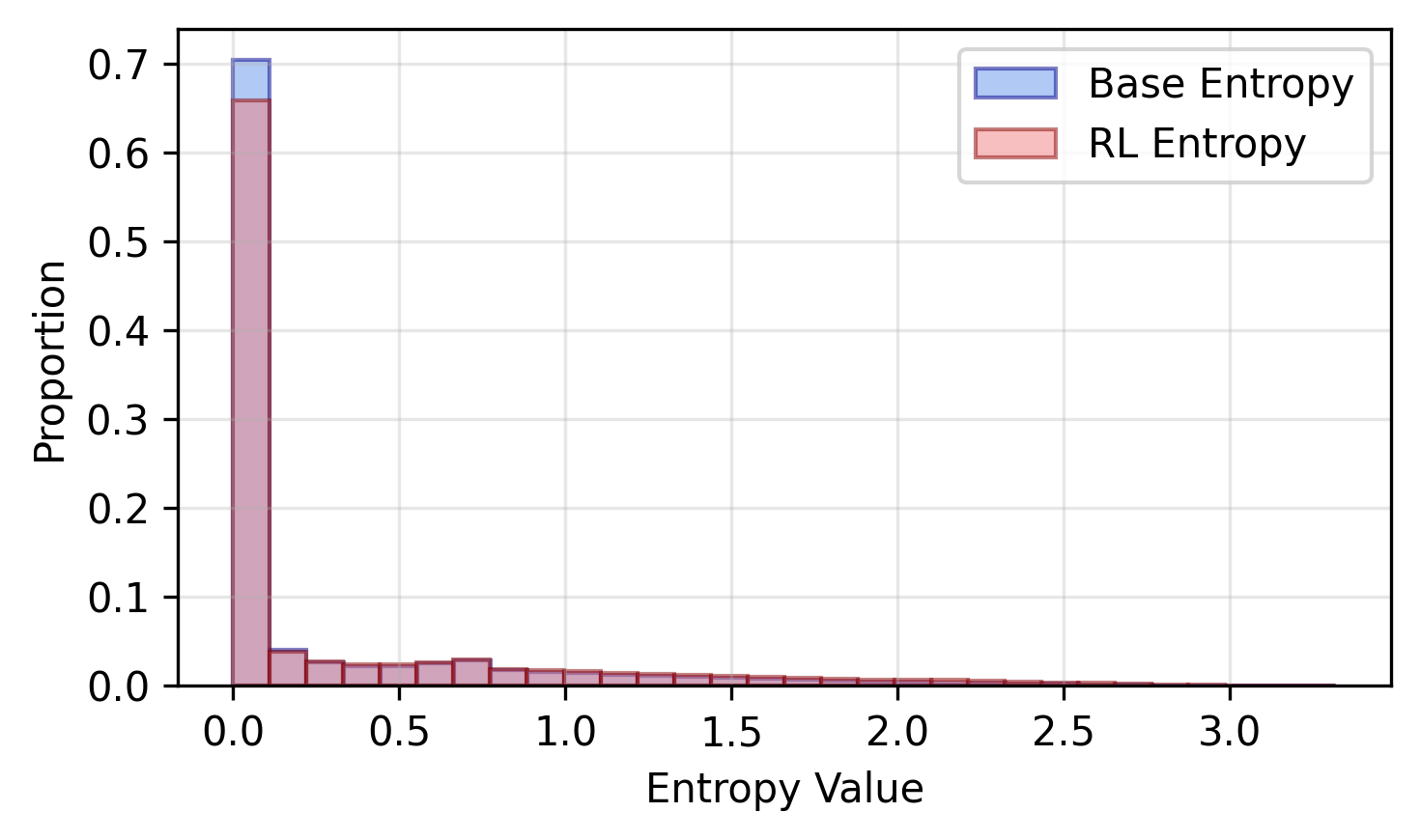}
        \caption{Low JS bin ($<0.1$)}
    \end{subfigure}
    \hspace{1cm}
    \begin{subfigure}{0.42\linewidth}
        \includegraphics[width=\linewidth]{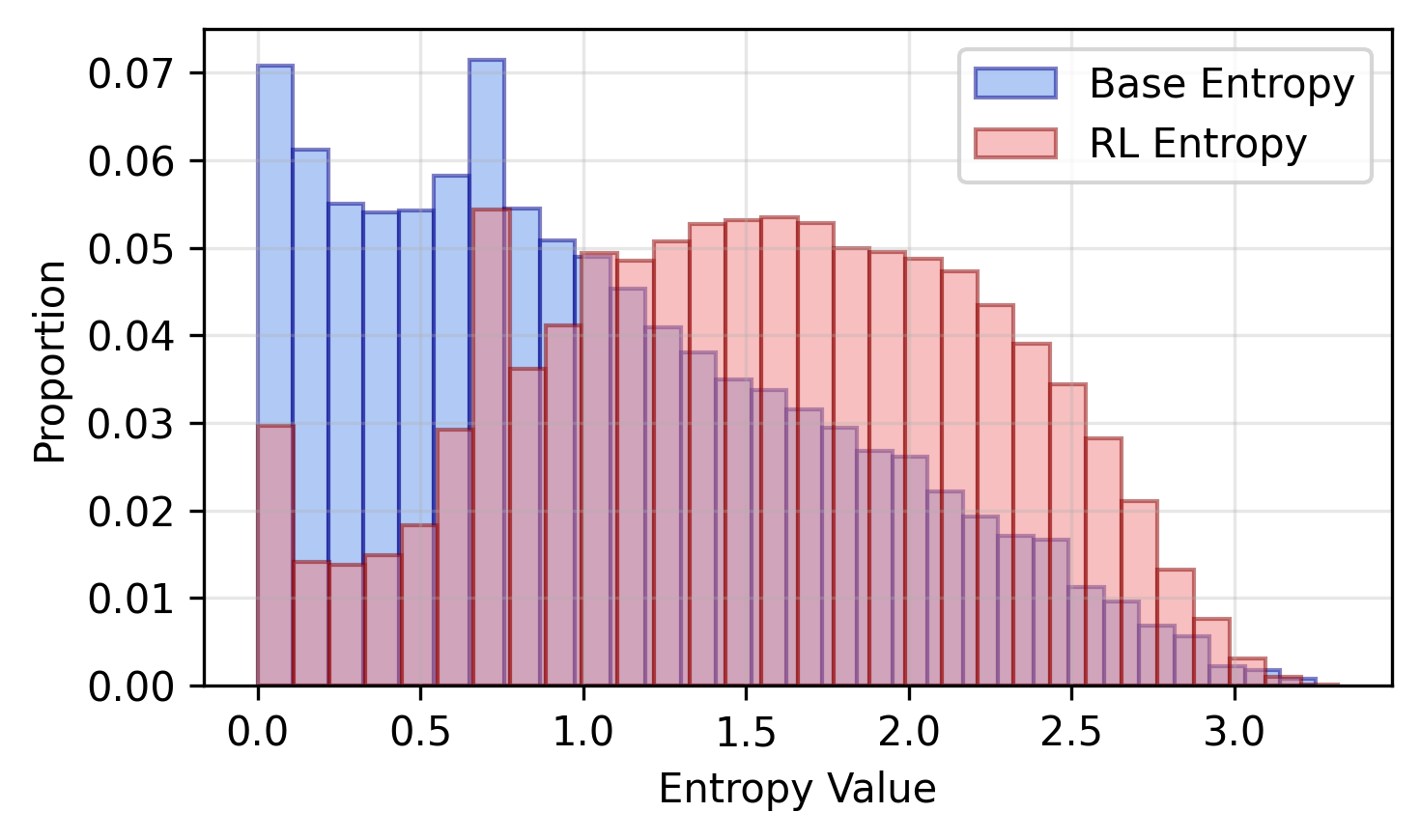}
        \caption{High JS bin ($>0.1$)}
    \end{subfigure}
    
    \vspace{1em}
    \begin{subfigure}{0.42\linewidth}
        \includegraphics[width=\linewidth]{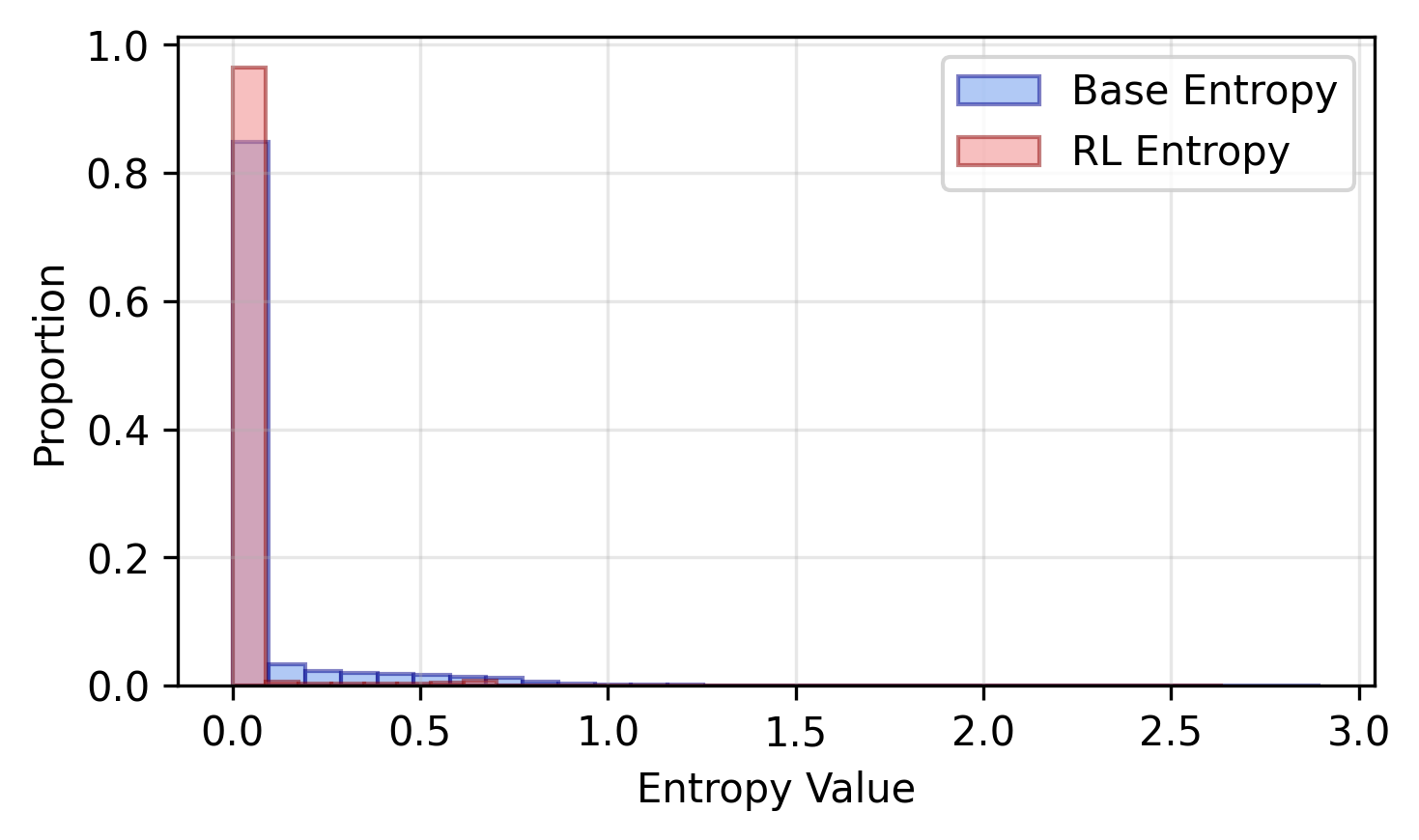}
        \caption{Low JS bin ($<0.1$)}
    \end{subfigure}
    \hspace{1cm}
    \begin{subfigure}{0.42\linewidth}
        \includegraphics[width=\linewidth]{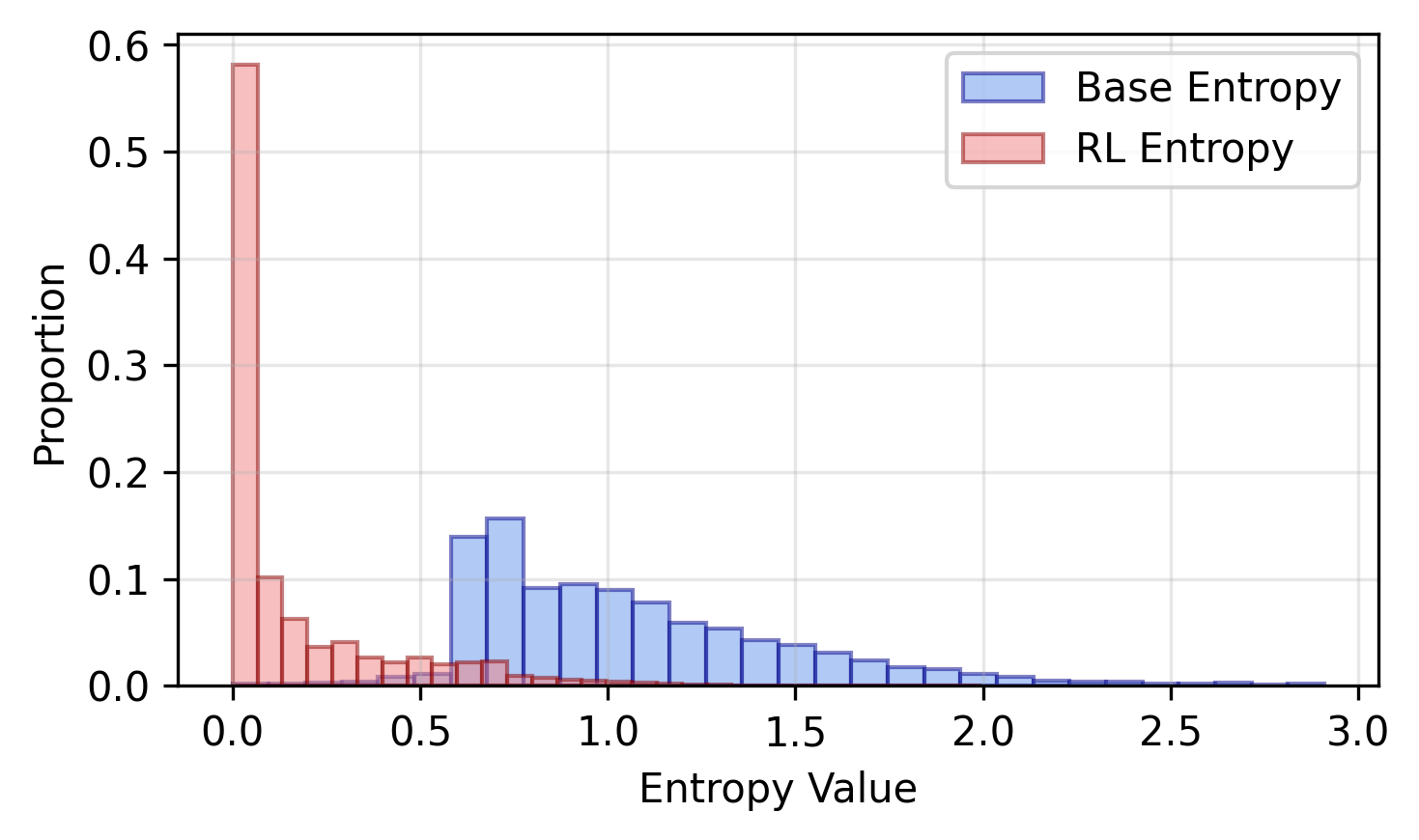}
        \caption{High JS bin ($>0.1$)}
    \end{subfigure}
    
    \caption{
        Entropy distributions across divergence bins for Qwen2.5-Math-7B with DAPO variants on AIME 2025. Top row: DAPO (clip-higher=0.28); bottom row: DAPO (clip-higher=0.2).
    }
    \label{fig:entropy_7b_dapo_aime25}
\end{figure}

\begin{figure}[!htbp]
    \centering
    \begin{subfigure}{0.31\linewidth}
        \includegraphics[width=\linewidth]{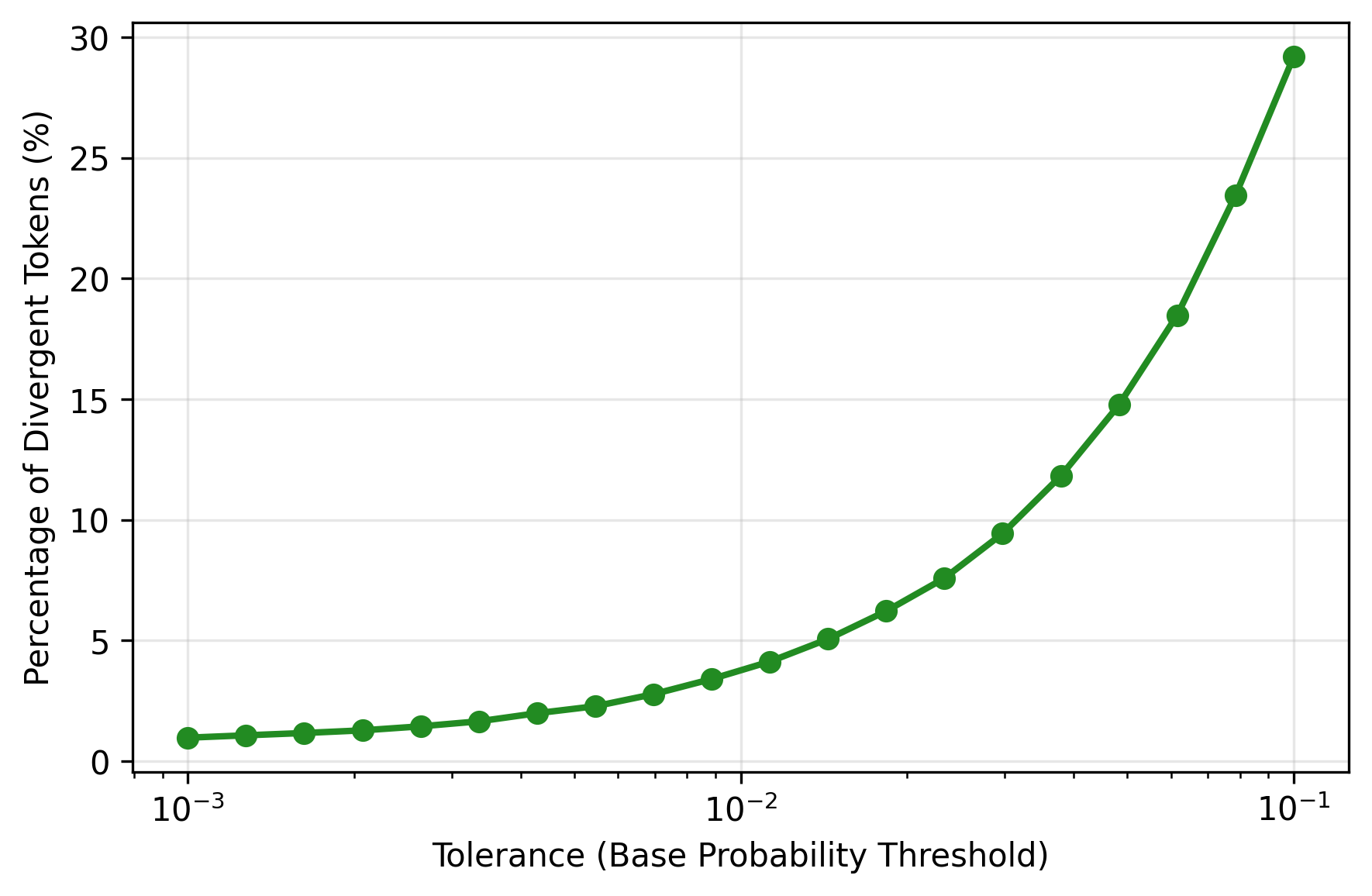}
        \caption{AIME 2024}
    \end{subfigure}
    \hfill
    \begin{subfigure}{0.31\linewidth}
        \includegraphics[width=\linewidth]{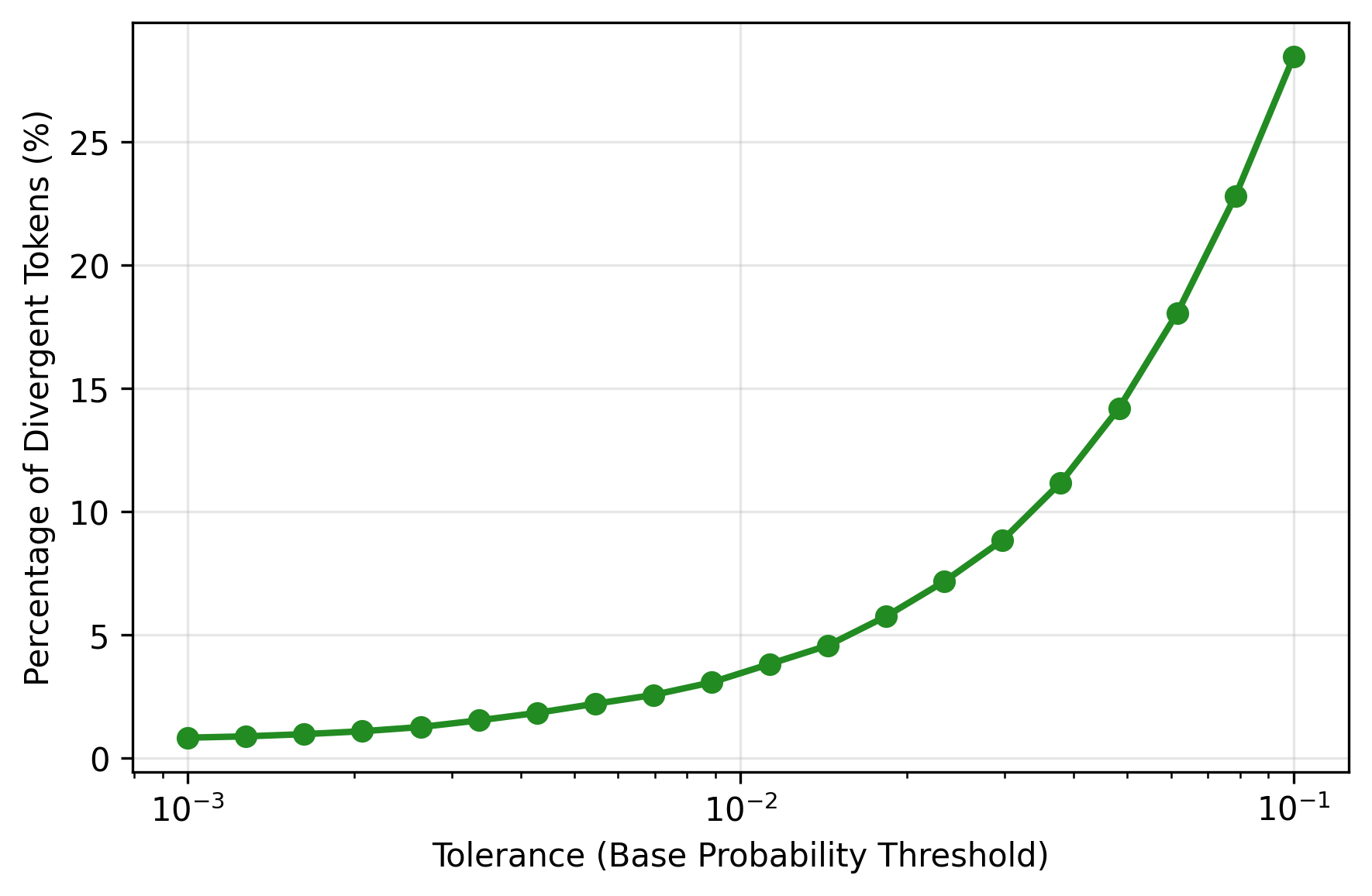}
        \caption{AIME 2025}
    \end{subfigure}
    \hfill
    \begin{subfigure}{0.31\linewidth}
        \includegraphics[width=\linewidth]{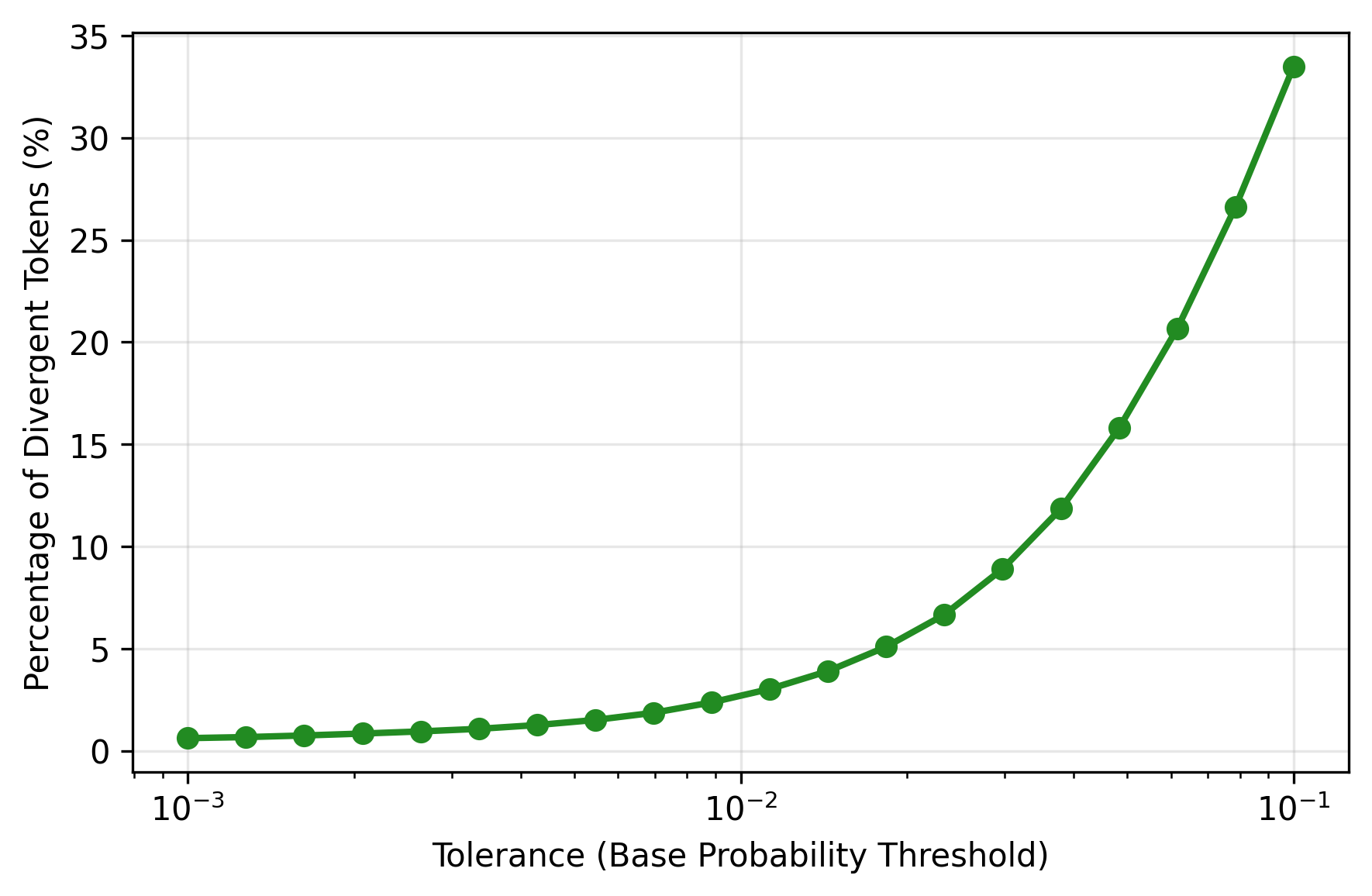}
        \caption{Fine-tune Data}
    \end{subfigure}
    
    \vspace{1em}
    \begin{subfigure}{0.31\linewidth}
        \includegraphics[width=\linewidth]{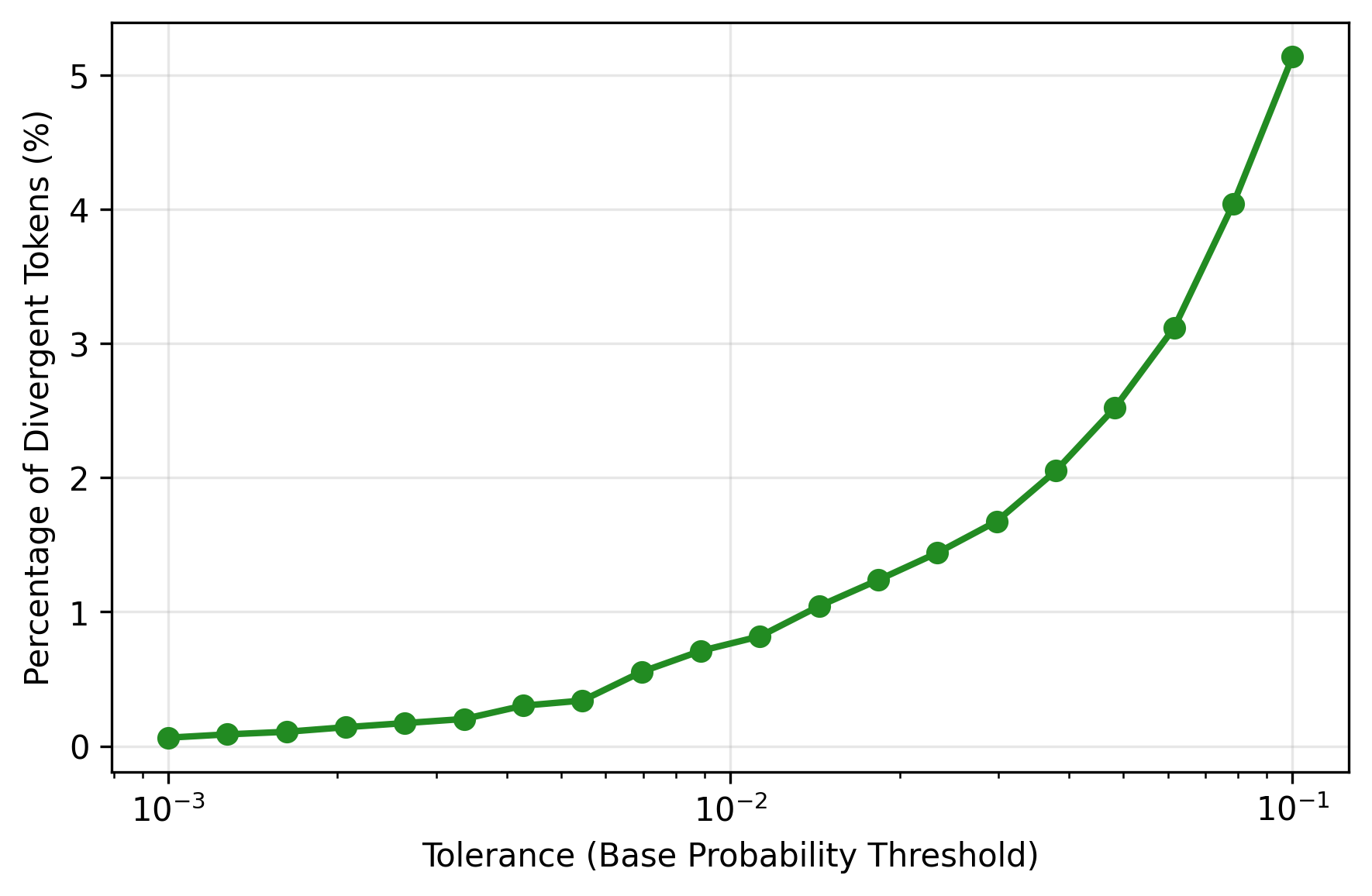}
        \caption{AIME 2024}
    \end{subfigure}
    \hfill
    \begin{subfigure}{0.31\linewidth}
        \includegraphics[width=\linewidth]{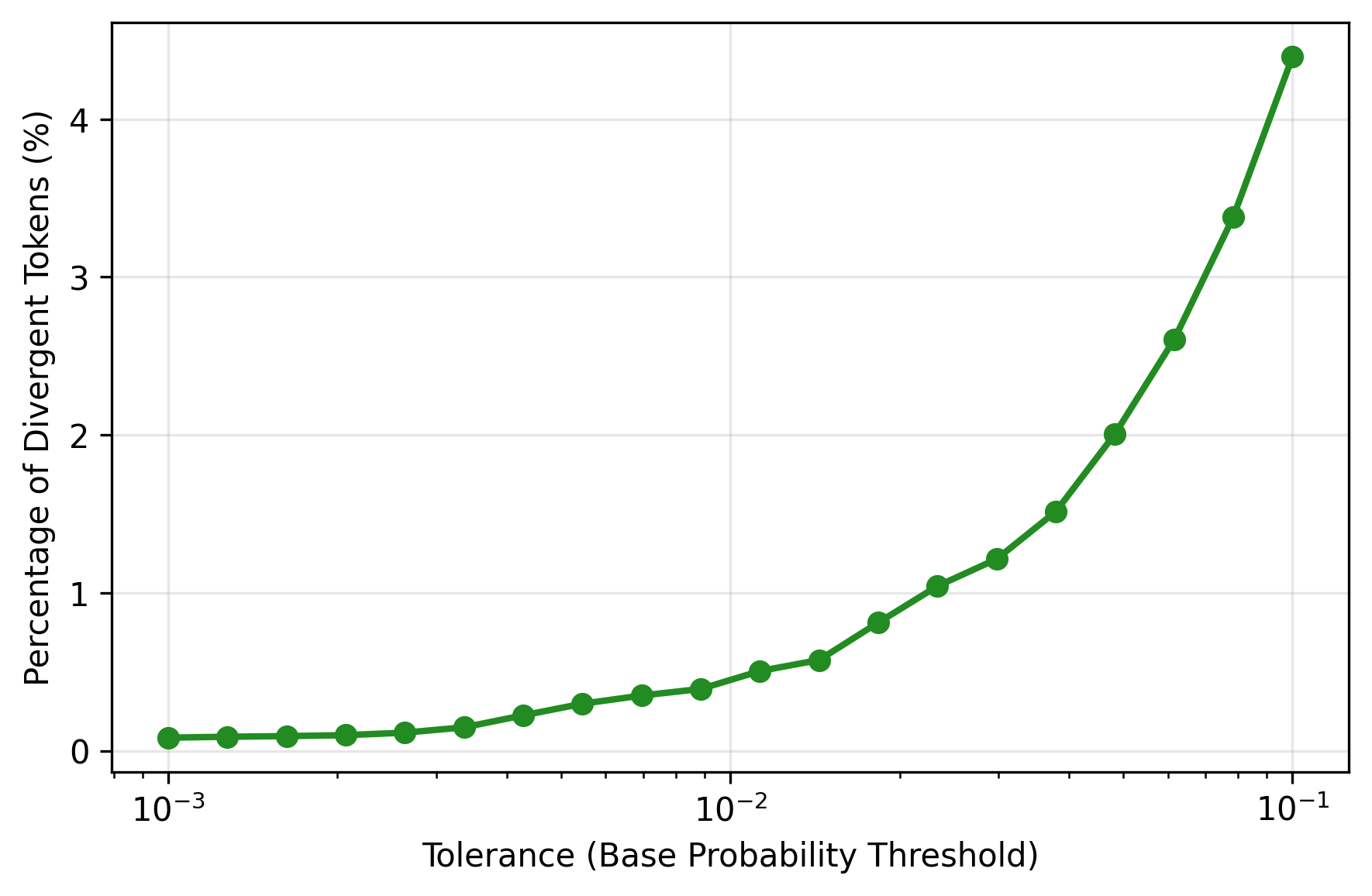}
        \caption{AIME 2025}
    \end{subfigure}
    \hfill
    \begin{subfigure}{0.31\linewidth}
        \includegraphics[width=\linewidth]{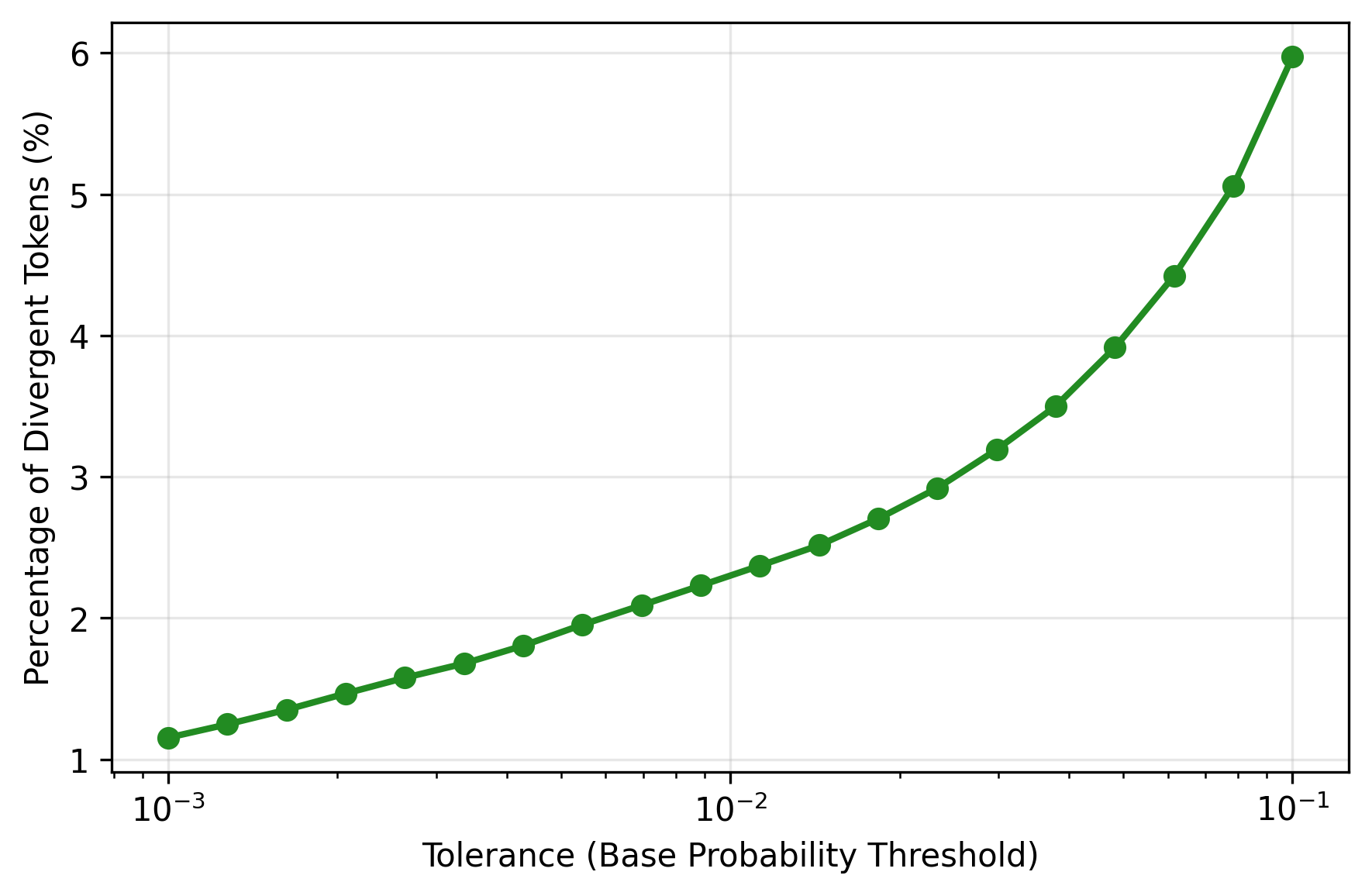}
        \caption{Fine-tune Data}
    \end{subfigure}
    
    \caption{
        Percentage of divergent tokens whose RL top-1 choice had base probability below a given threshold for Qwen2.5-Math-7B with DAPO variants. Top row: DAPO (clip-higher=0.28); bottom row: DAPO (clip-higher=0.2). Consistent with findings in the main text, RL rarely promotes tokens with very low base probability, even under more exploratory settings with clip-higher. We further observe a distinction between the two clip-high settings, with the more restrictive setting (0.2) promoting fewer tokens with low base probability.
    }
    \label{fig:tolerance_7b_dapo}
\end{figure}

\begin{figure}[!htbp]
    \centering
    \begin{subfigure}{0.42\linewidth}
        \includegraphics[width=\linewidth]{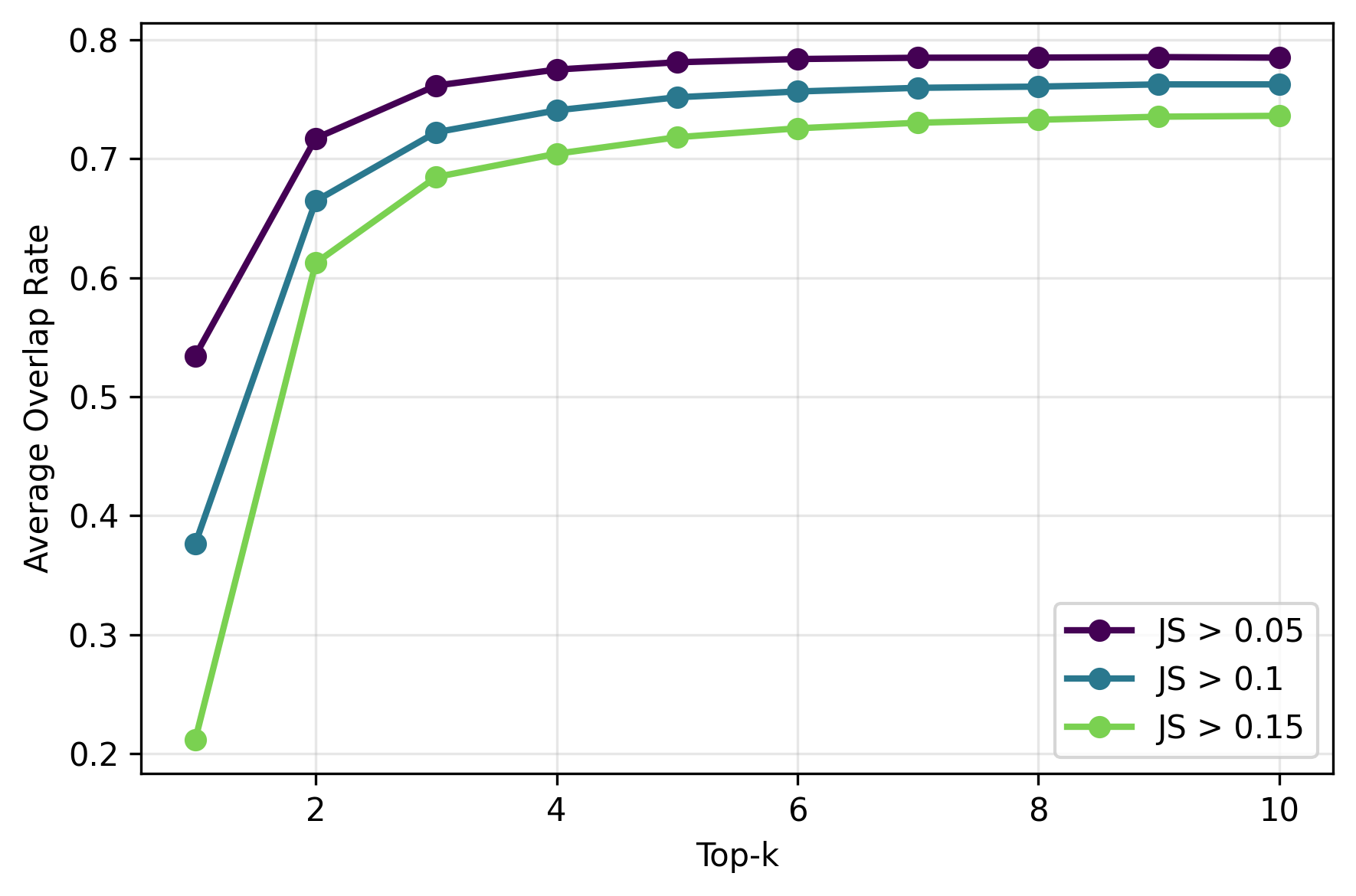}
        \caption{DAPO (0.28) AIME 2024}
    \end{subfigure}
    \hspace{1cm}
    \begin{subfigure}{0.42\linewidth}
        \includegraphics[width=\linewidth]{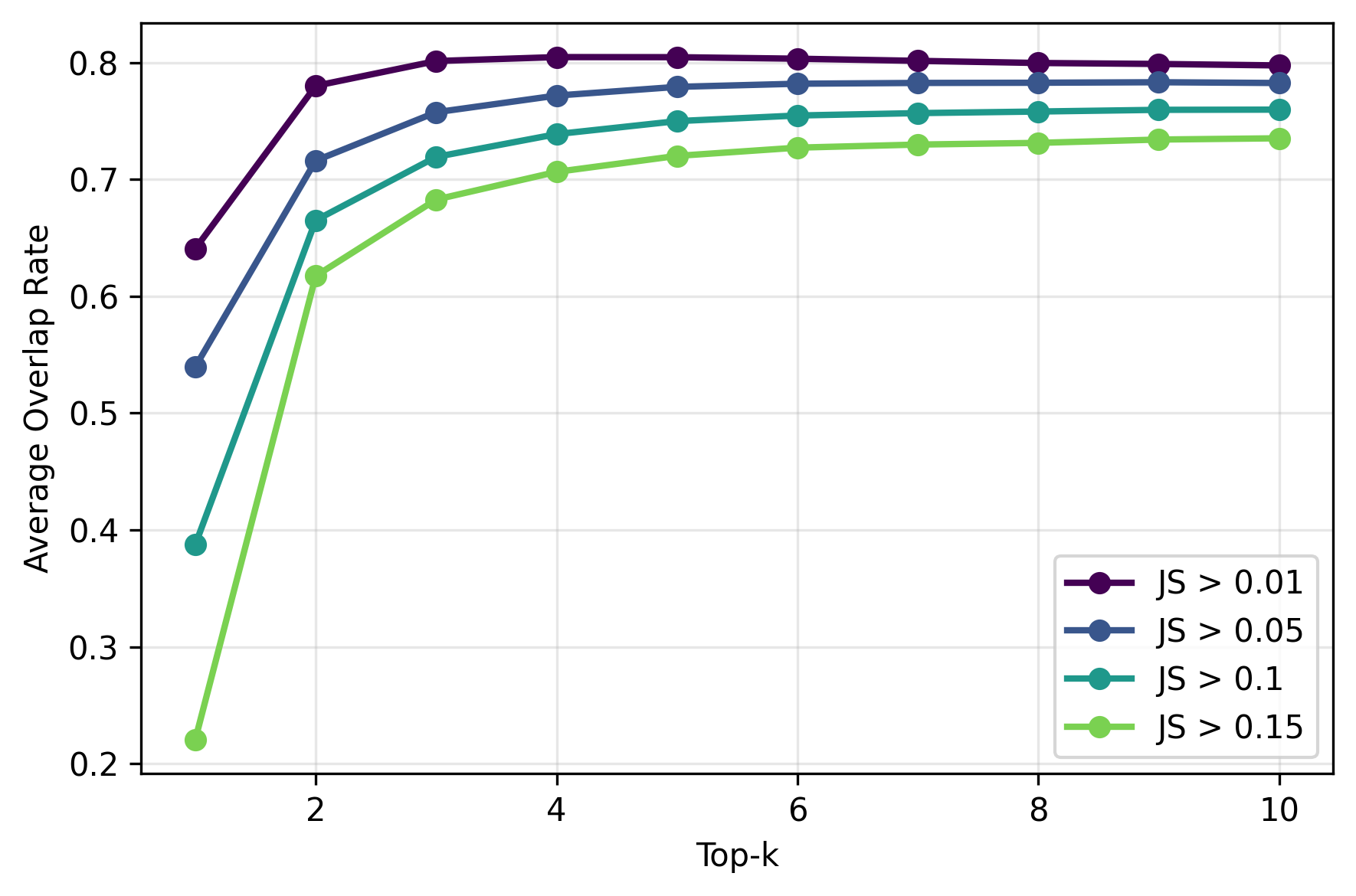}
        \caption{DAPO (0.28) AIME 2025}
    \end{subfigure}
    
    \vspace{1em}
    \begin{subfigure}{0.42\linewidth}
        \includegraphics[width=\linewidth]{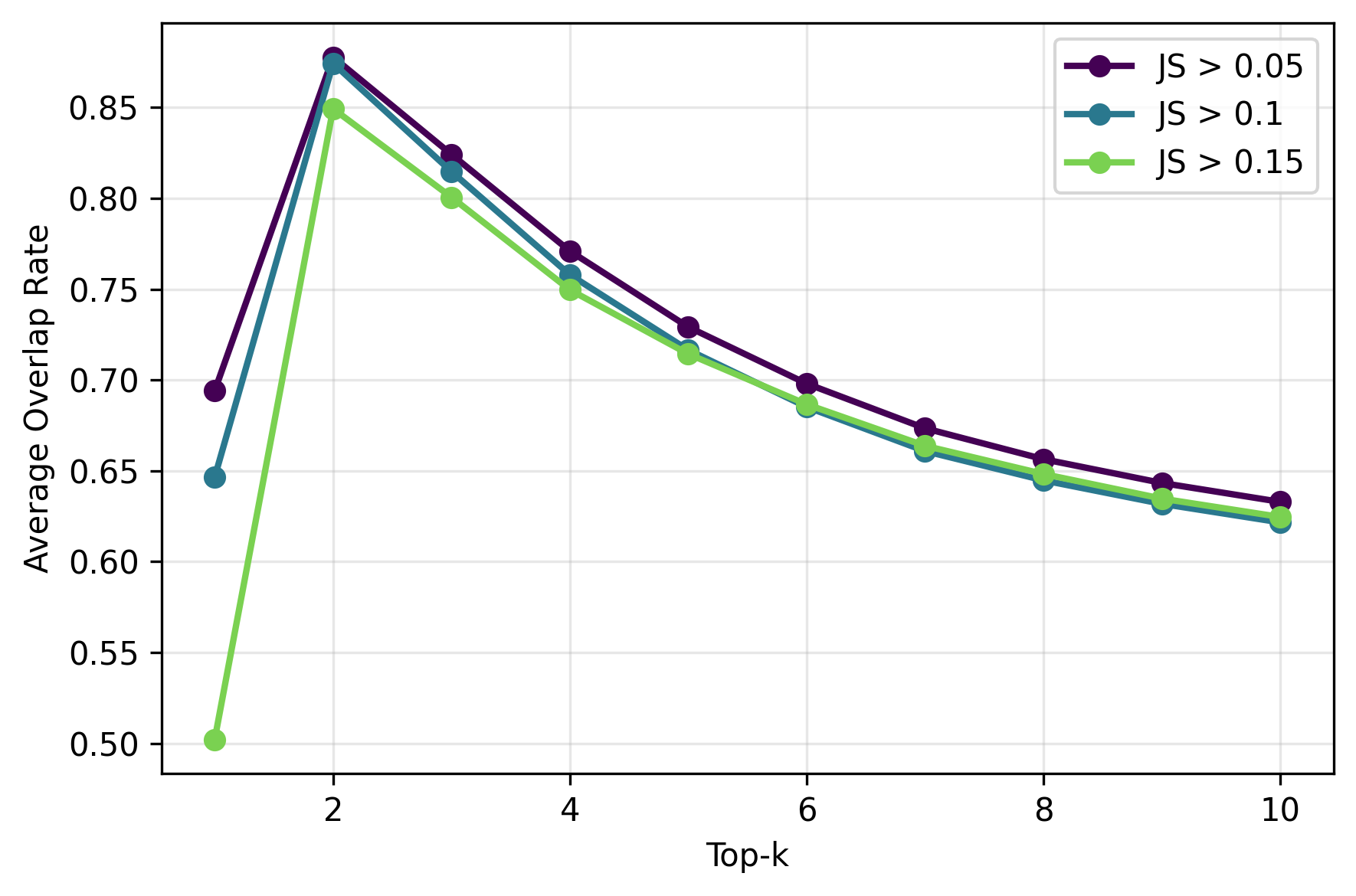}
        \caption{DAPO (0.2) AIME 2024}
    \end{subfigure}
    \hspace{1cm}
    \begin{subfigure}{0.42\linewidth}
        \includegraphics[width=\linewidth]{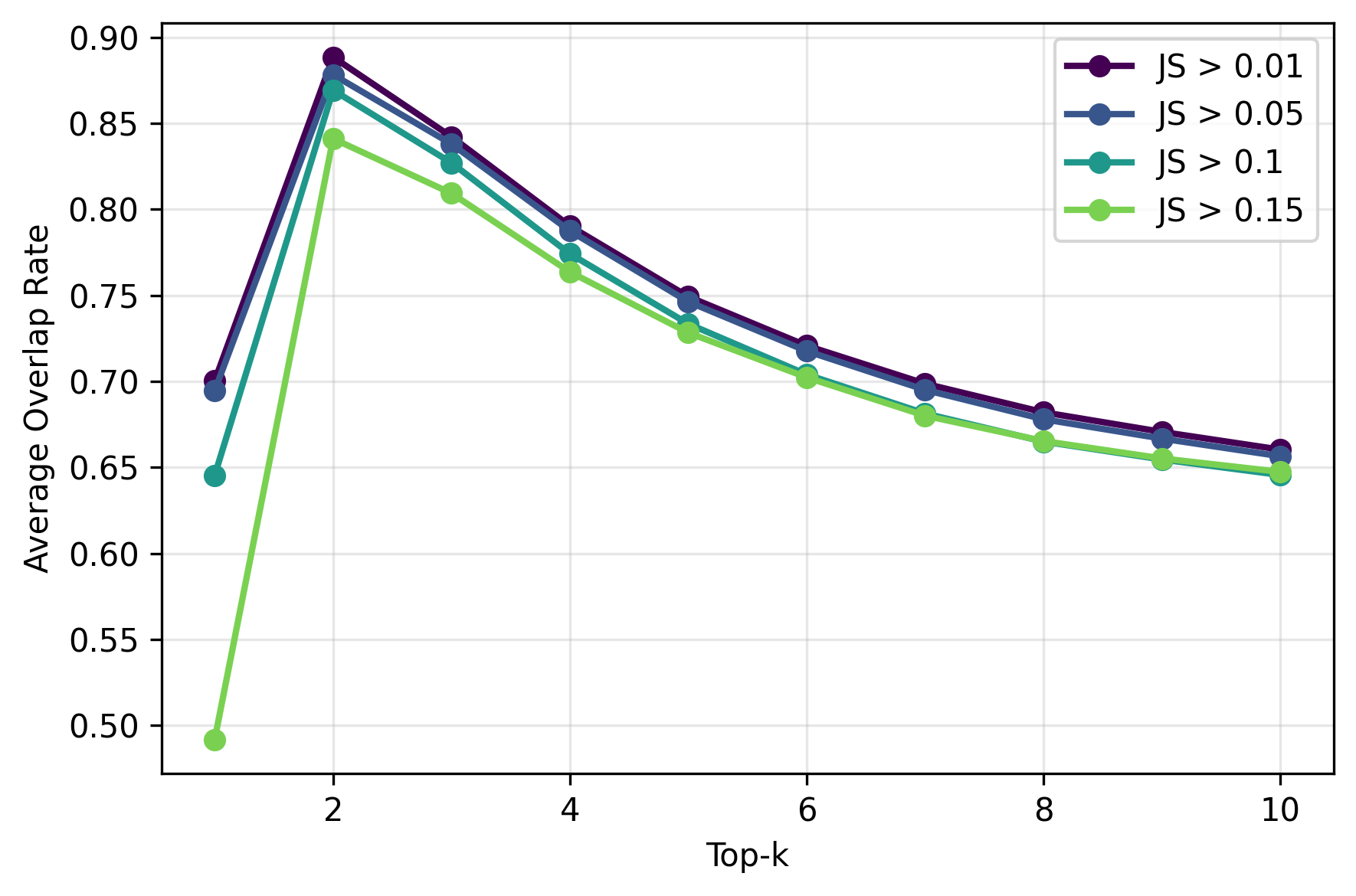}
        \caption{DAPO (0.2) AIME 2025}
    \end{subfigure}
    
    \caption{
        Top-$k$ token overlap between base and RL models at divergent positions ($\mathrm{JS}_t>0.1$) for DAPO variants on AIME 2024 and AIME 2025.
    }
    \label{fig:topk_dapo_variants}
\end{figure}

\begin{figure}[!htbp]
    \centering
    \begin{subfigure}{0.48\linewidth}
        \includegraphics[width=\linewidth]{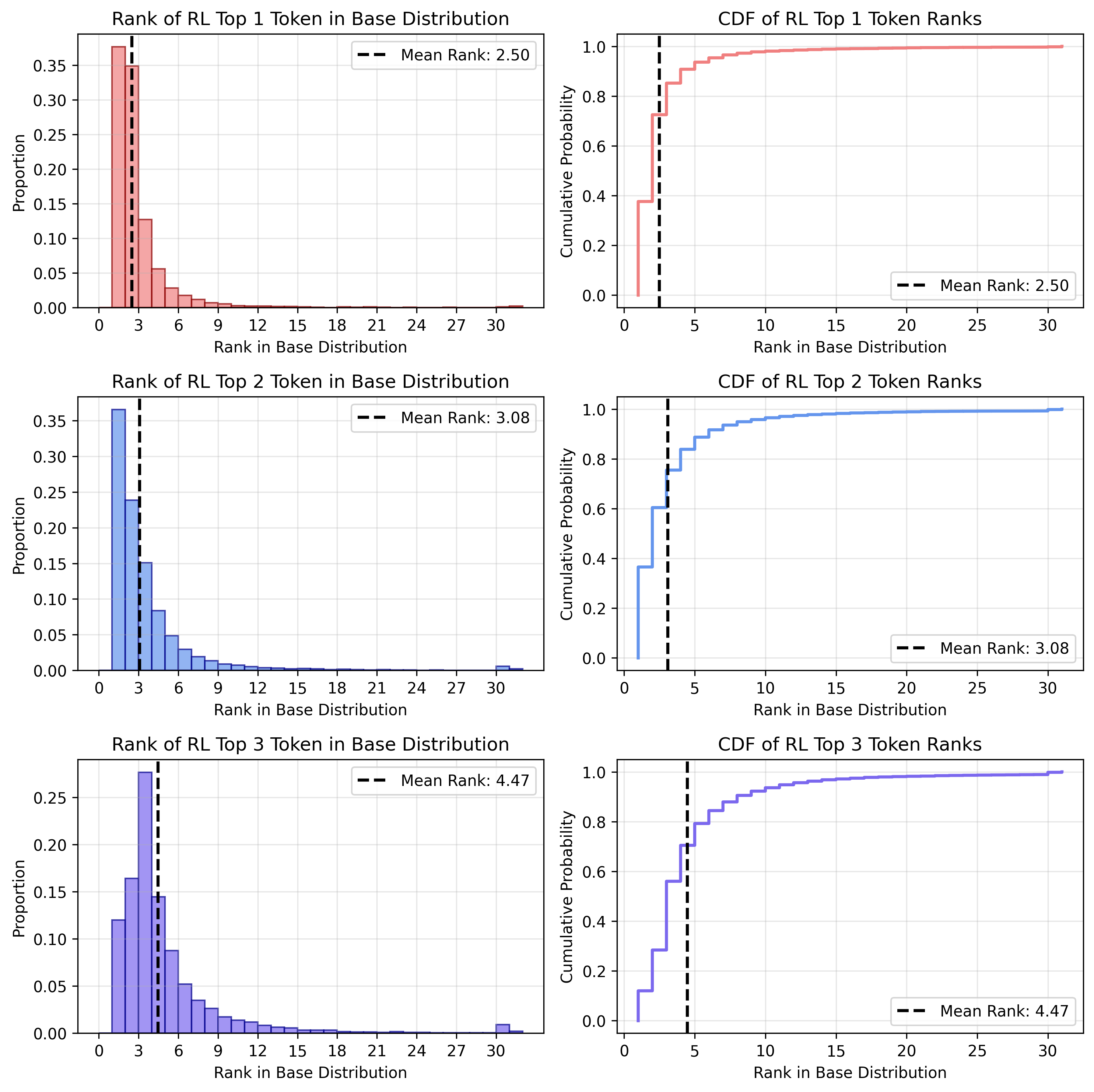}
        \caption{DAPO (0.28) AIME 2024}
    \end{subfigure}
    \hfill
    \begin{subfigure}{0.48\linewidth}
        \includegraphics[width=\linewidth]{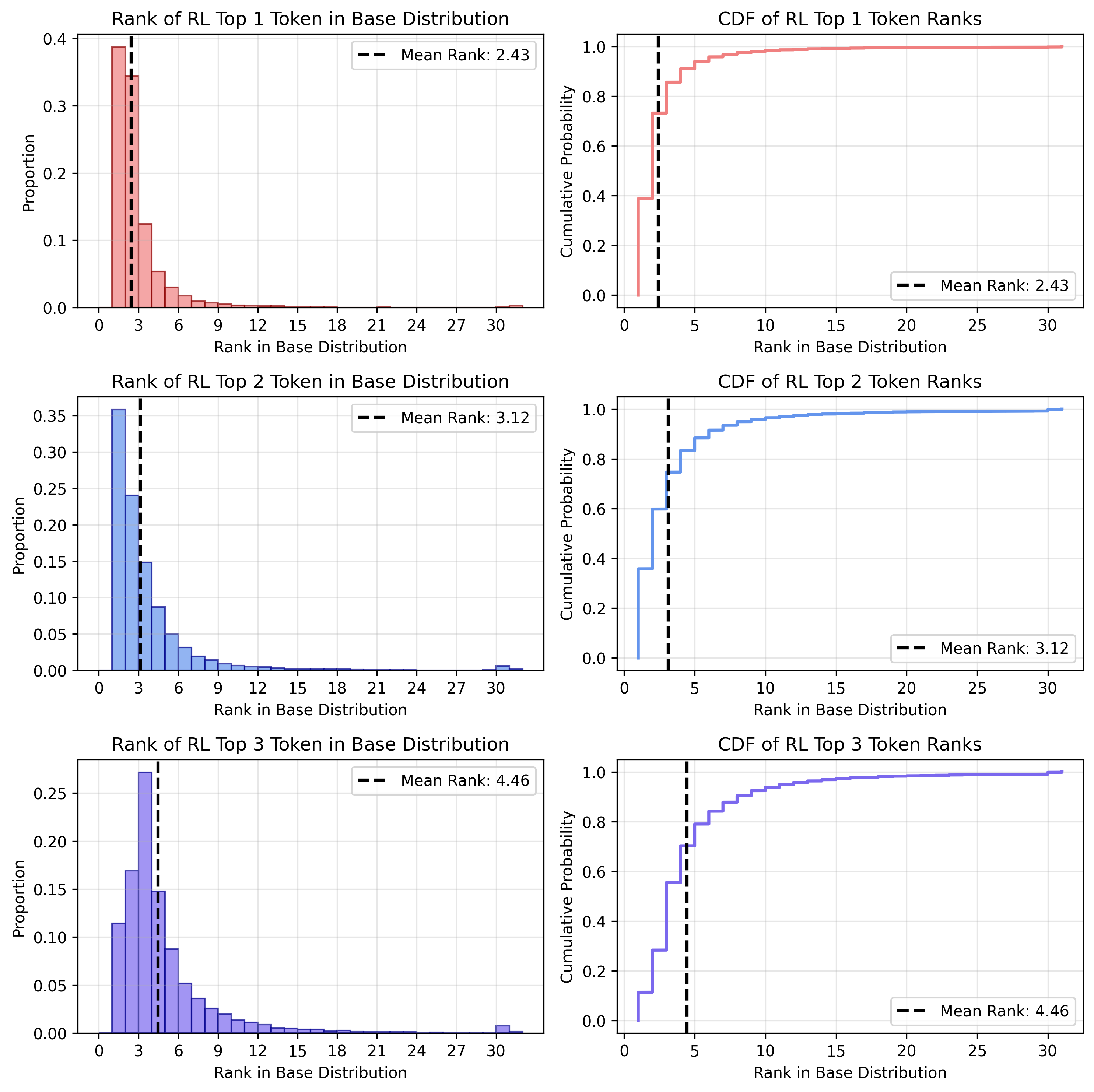}
        \caption{DAPO (0.28) AIME 2025}
    \end{subfigure}
    
    \vspace{1em}
    \begin{subfigure}{0.48\linewidth}
        \includegraphics[width=\linewidth]{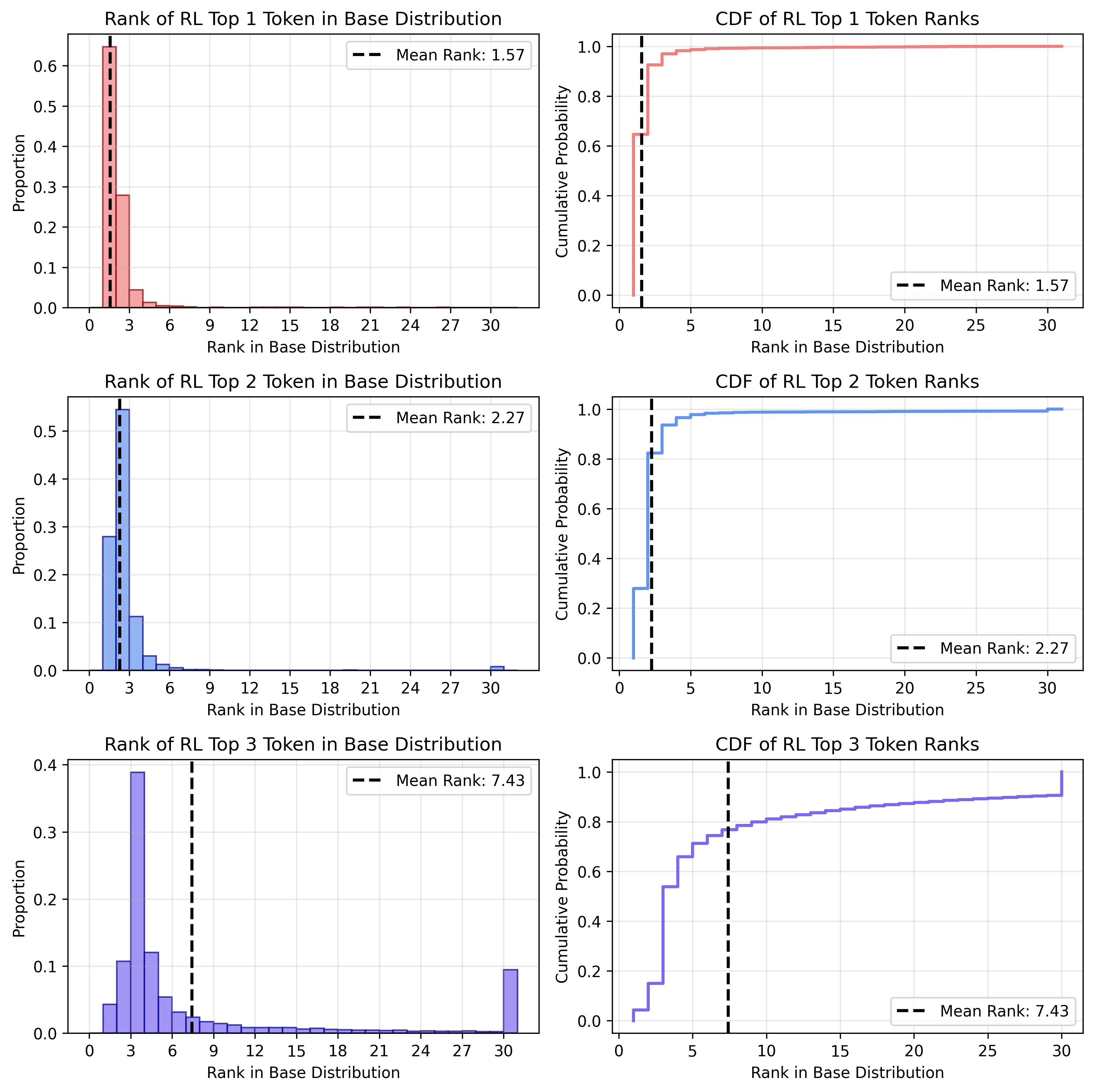}
        \caption{DAPO (0.2) AIME 2024}
    \end{subfigure}
    \hfill
    \begin{subfigure}{0.48\linewidth}
        \includegraphics[width=\linewidth]{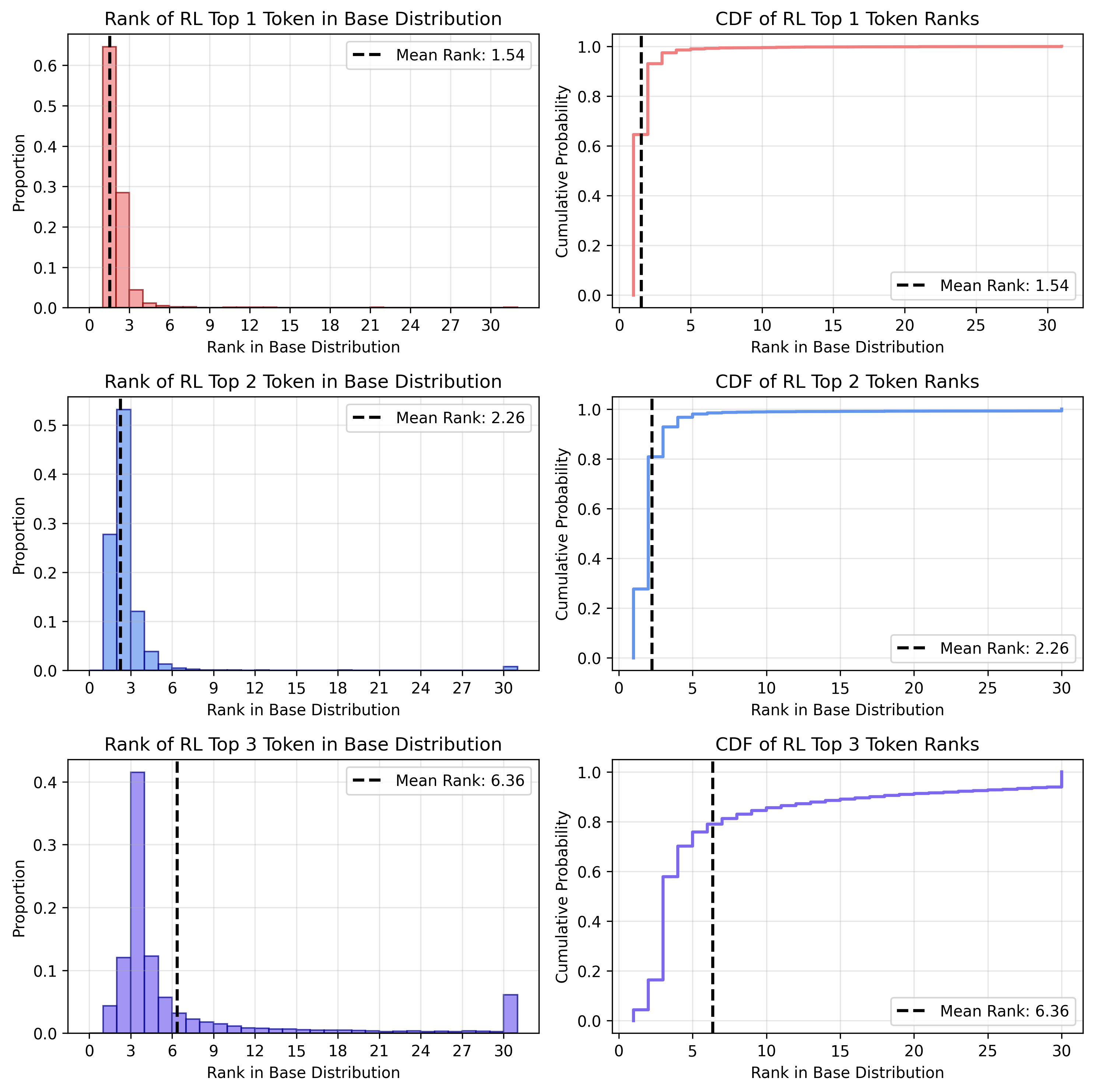}
        \caption{DAPO (0.2) AIME 2025}
    \end{subfigure}
    
    \caption{
        Distribution of base-model ranks for RL's top-3 tokens at high-divergence positions ($\js > 0.1$) for DAPO variants on AIME 2024 and AIME 2025.
    }
    \label{fig:ranks_dapo_variants}
\end{figure}

\FloatBarrier

\paragraph{Fine-tuning Data Results.}
We also analyze distributional shifts on the fine-tuning data to examine how models behave on data they were fine-tuned on. Figure~\ref{fig:js_dapo_variants_train} shows JS divergence distributions, while Figures~\ref{fig:positional_dapo_variants_train},~\ref{fig:topk_dapo_variants_train} show additional analyses.

\begin{figure}[!htbp]
    \centering
    \begin{subfigure}{0.40\textwidth}
        \includegraphics[width=\linewidth]{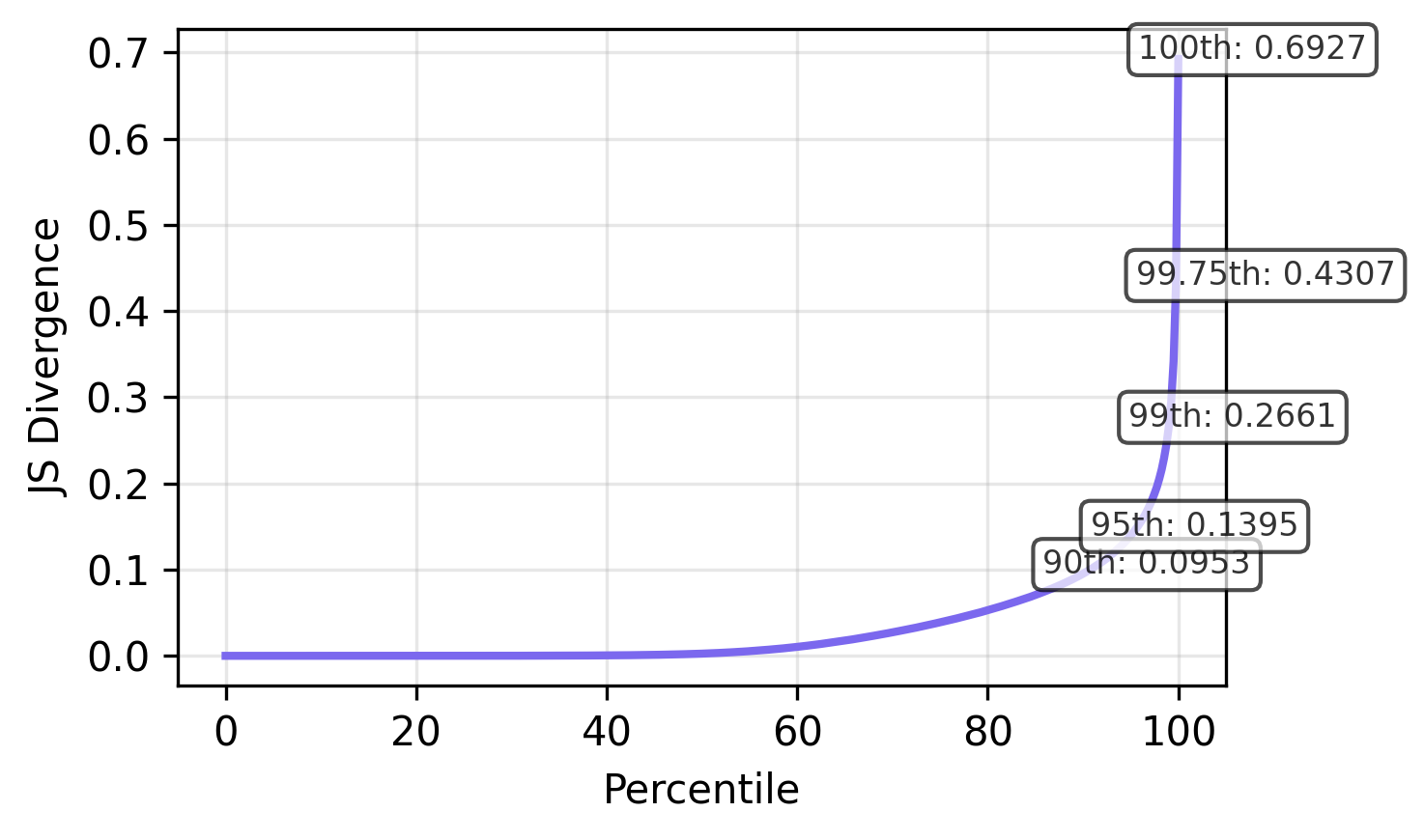}
        \caption{DAPO (clip-higher=0.28): Percentiles}
    \end{subfigure}
    \hspace{1cm}
    \begin{subfigure}{0.40\textwidth}
        \includegraphics[width=\linewidth]{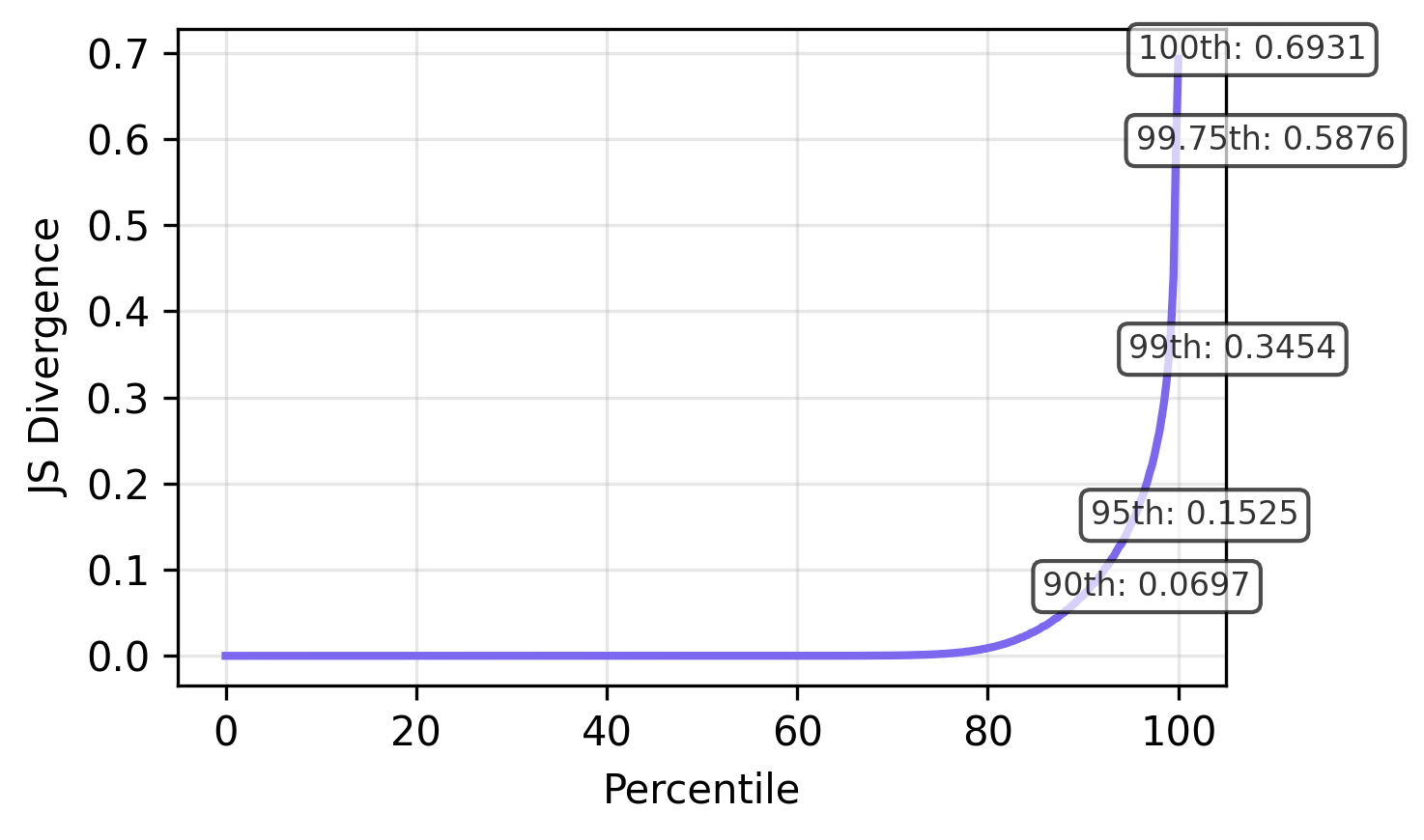}
        \caption{DAPO (clip-higher=0.2): Percentiles}
    \end{subfigure}
    
    \caption{
        JS divergence distributions for DAPO variants of Qwen2.5-Math-7B on fine-tuning data (computed using approximated full distributions instead of truncated ones)
    }
    \label{fig:js_dapo_variants_train}
\end{figure}

\begin{figure}[!htbp]
    \centering
    \begin{subfigure}{0.42\linewidth}
        \includegraphics[width=\linewidth]{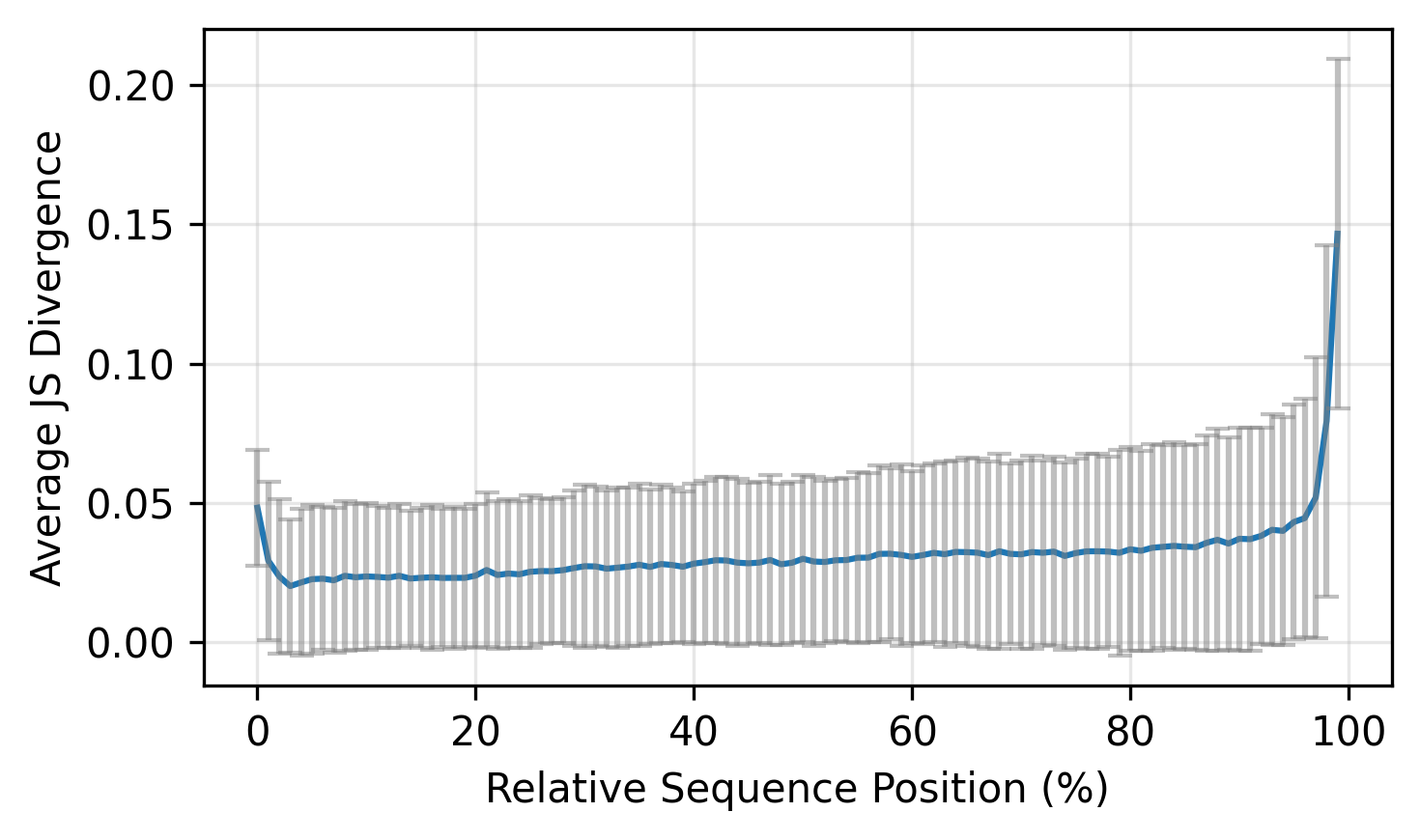}
        \caption{DAPO (clip-higher=0.28)}
    \end{subfigure}
    \hspace{1cm}
    \begin{subfigure}{0.42\linewidth}
        \includegraphics[width=\linewidth]{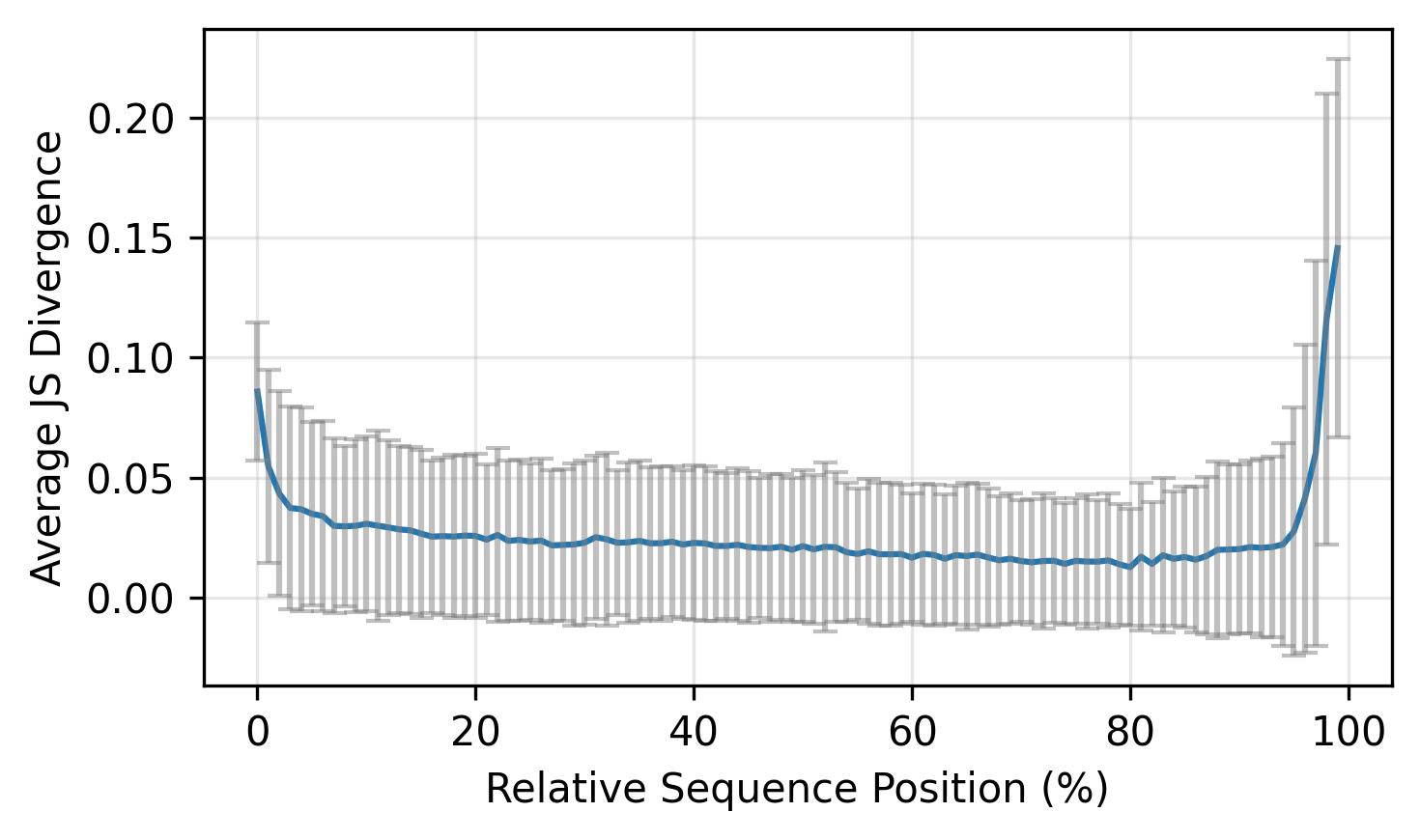}
        \caption{DAPO (clip-higher=0.2)}
    \end{subfigure}
    
    \caption{
        Mean JS divergence by normalized token position for DAPO variants on fine-tuning data.
    }
    \label{fig:positional_dapo_variants_train}
\end{figure}

\begin{figure}[!htbp]
    \centering
    \begin{subfigure}{0.42\linewidth}
        \includegraphics[width=\linewidth]{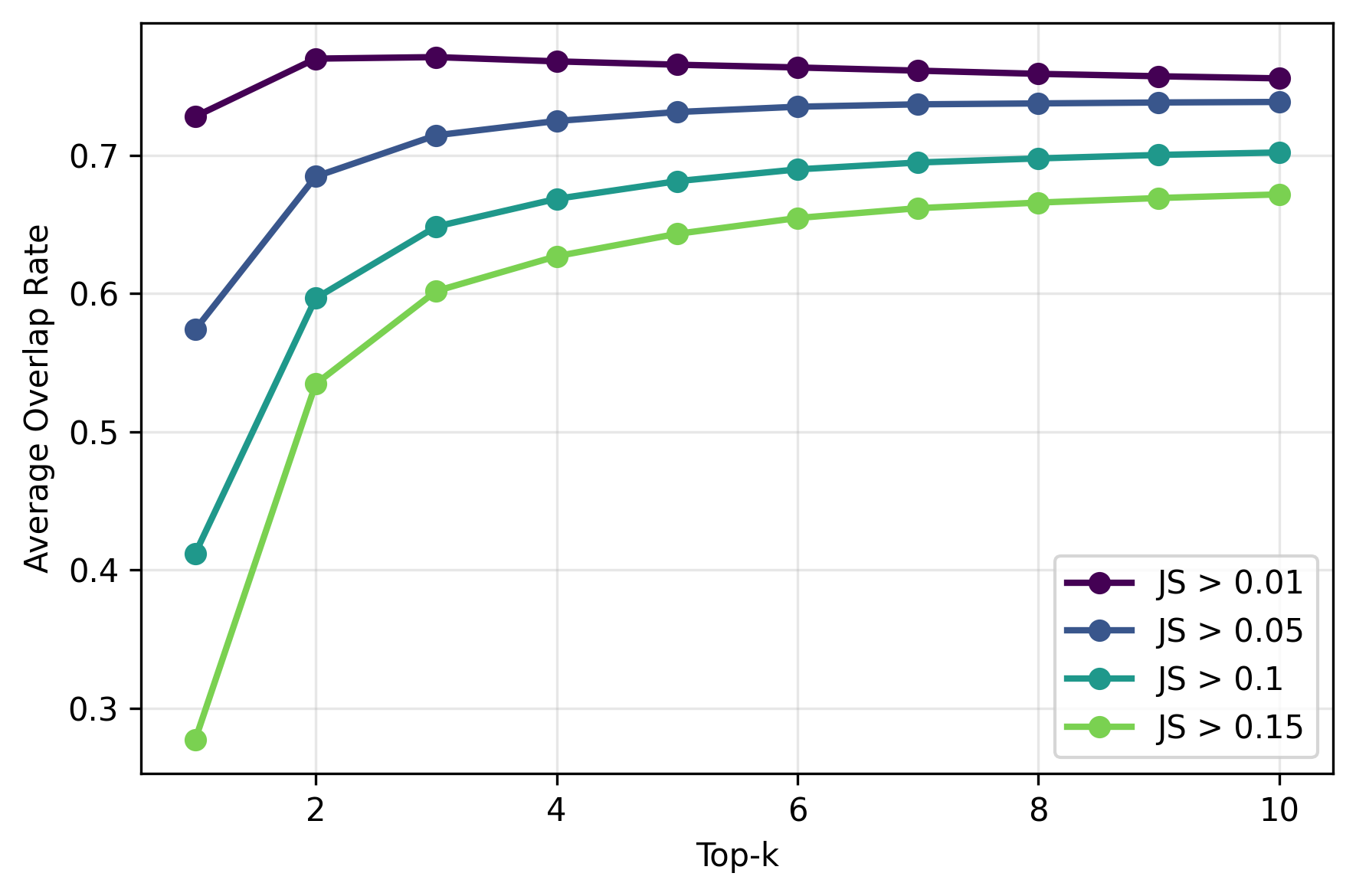}
        \caption{DAPO (clip-higher=0.28)}
    \end{subfigure}
    \hspace{1cm}
    \begin{subfigure}{0.42\linewidth}
        \includegraphics[width=\linewidth]{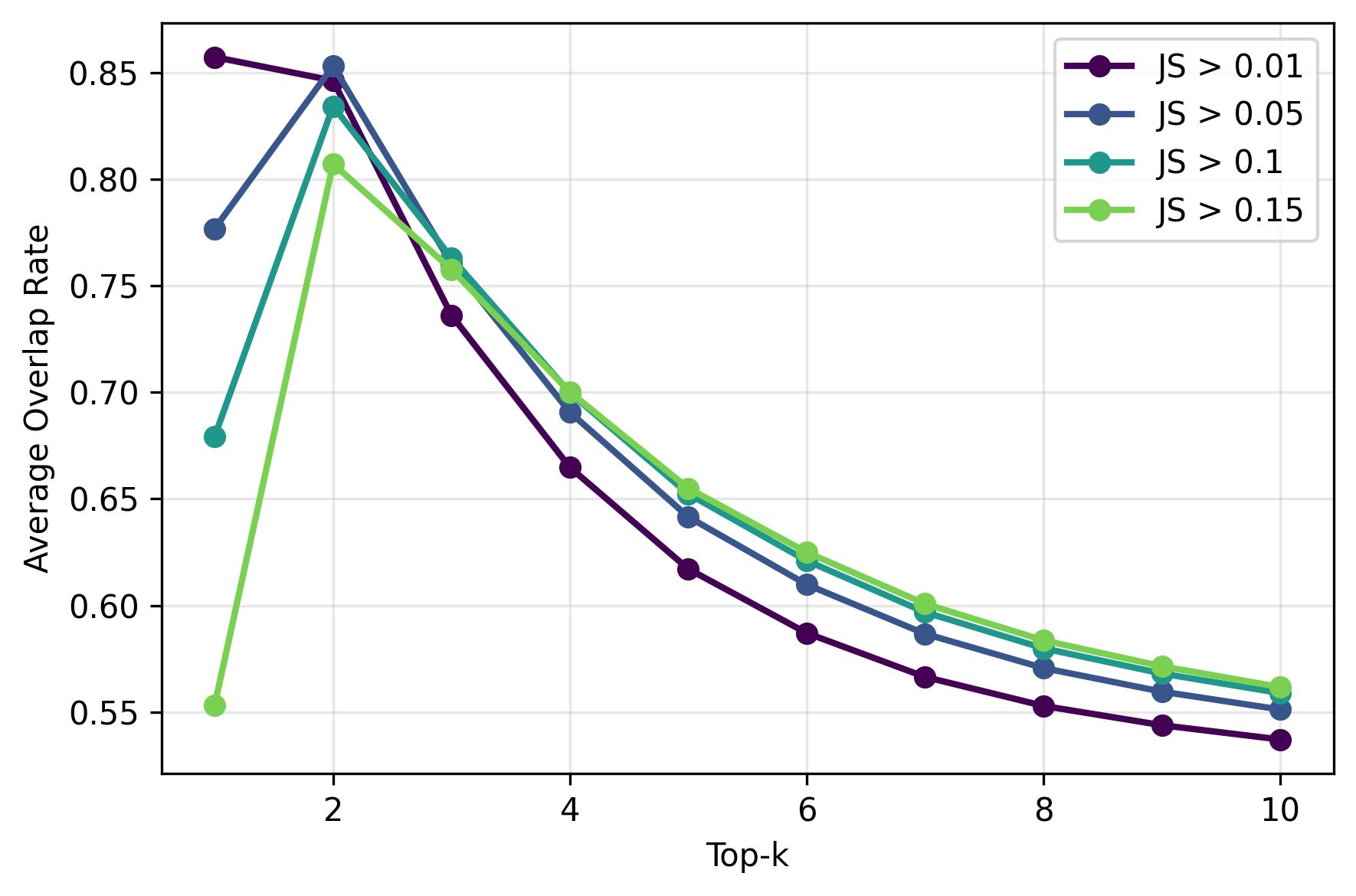}
        \caption{DAPO (clip-higher=0.2)}
    \end{subfigure}
    
    \caption{
        Top-$k$ token overlap between base and RL models at divergent positions ($\mathrm{JS}_t>0.1$) for DAPO variants on fine-tuning data.
    }
    \label{fig:topk_dapo_variants_train}
\end{figure}

\FloatBarrier
\subsubsection{Qwen3-8B with DAPO}

We analyze Qwen3-8B trained with DAPO on AIME 2024 and AIME 2025 to further demonstrate the robustness of our findings across model families and training configurations. Notably, this model was fine-tuned for nearly twice as many RLVR training steps as the Qwen2.5-Math-7B, providing a useful point of comparison for understanding the effect of extended RL training.

Figure~\ref{fig:js_qwen3_8b} shows JS divergence percentile curves, revealing similarly sparse distributional shifts despite the longer training horizon. Figure~\ref{fig:positional_qwen3_8b} shows positional concentration, Figure~\ref{fig:entropy_qwen3_8b} shows entropy distributions across divergence bins, and Figure~\ref{fig:tolerance_qwen3_8b} shows tail behavior analysis.

\begin{figure}[!htbp]
    \centering
    \begin{subfigure}{0.40\textwidth}
        \includegraphics[width=\linewidth]{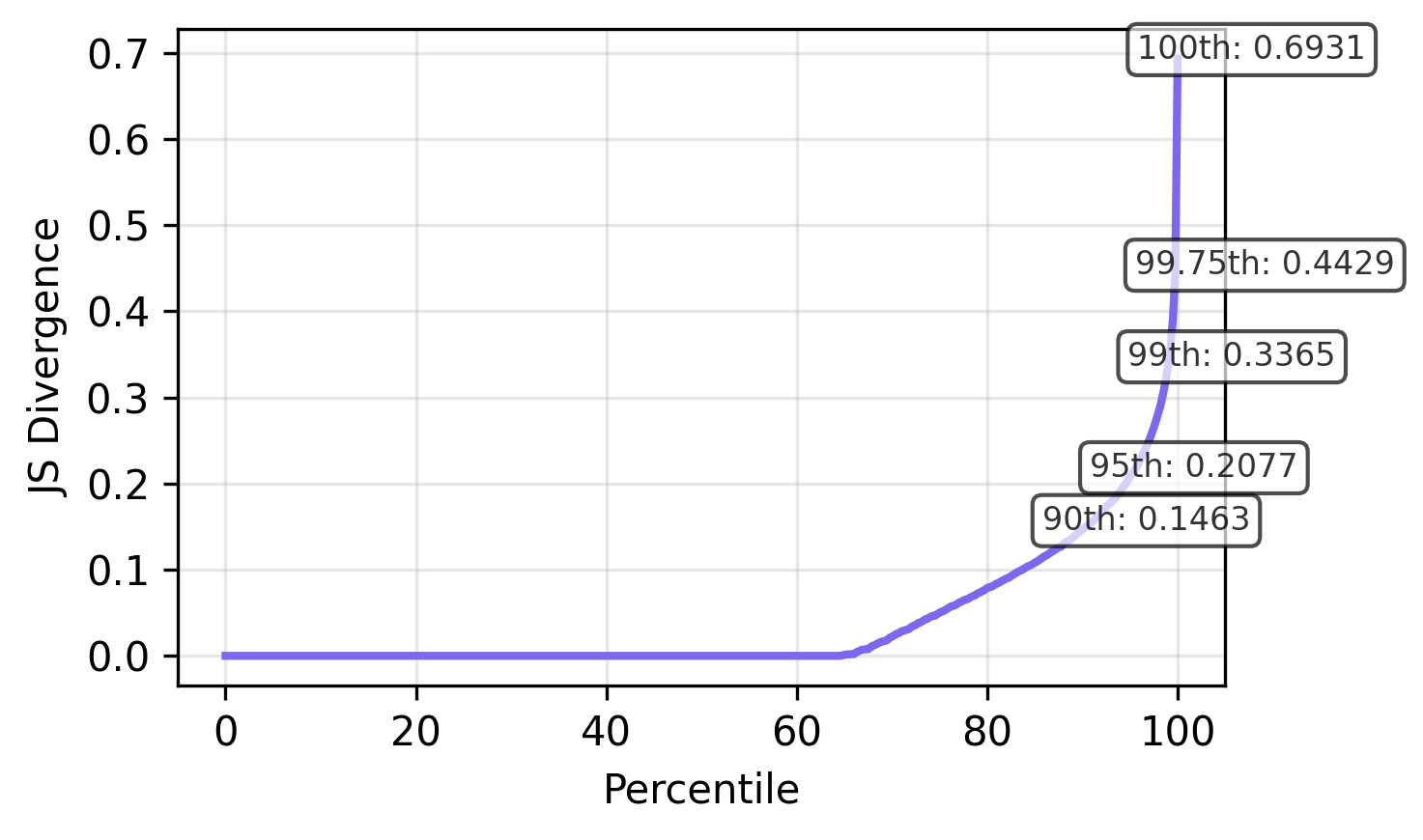}
        \caption{AIME 2024: Percentiles}
    \end{subfigure}
    \hspace{1cm}
    \begin{subfigure}{0.40\textwidth}
        \includegraphics[width=\linewidth]{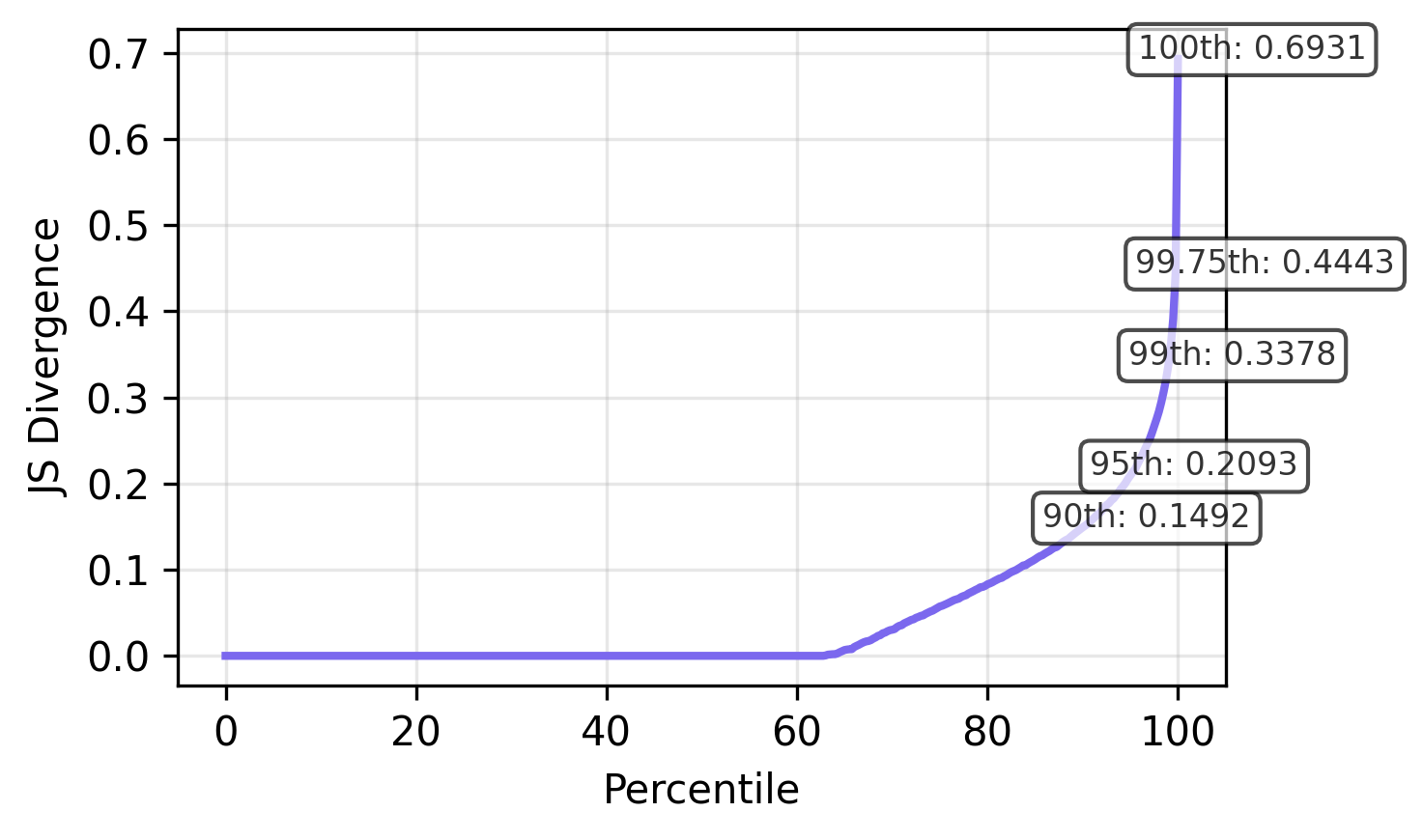}
        \caption{AIME 2025: Percentiles}
    \end{subfigure}
    
    \caption{
        JS divergence distributions for Qwen3-8B with DAPO on AIME 2024 and AIME 2025.
        Sparse distributional shifts persist even under extended training.
    }
    \label{fig:js_qwen3_8b}
\end{figure}

\begin{figure}[!htbp]
    \centering
    \begin{subfigure}{0.40\textwidth}
        \includegraphics[width=\linewidth]{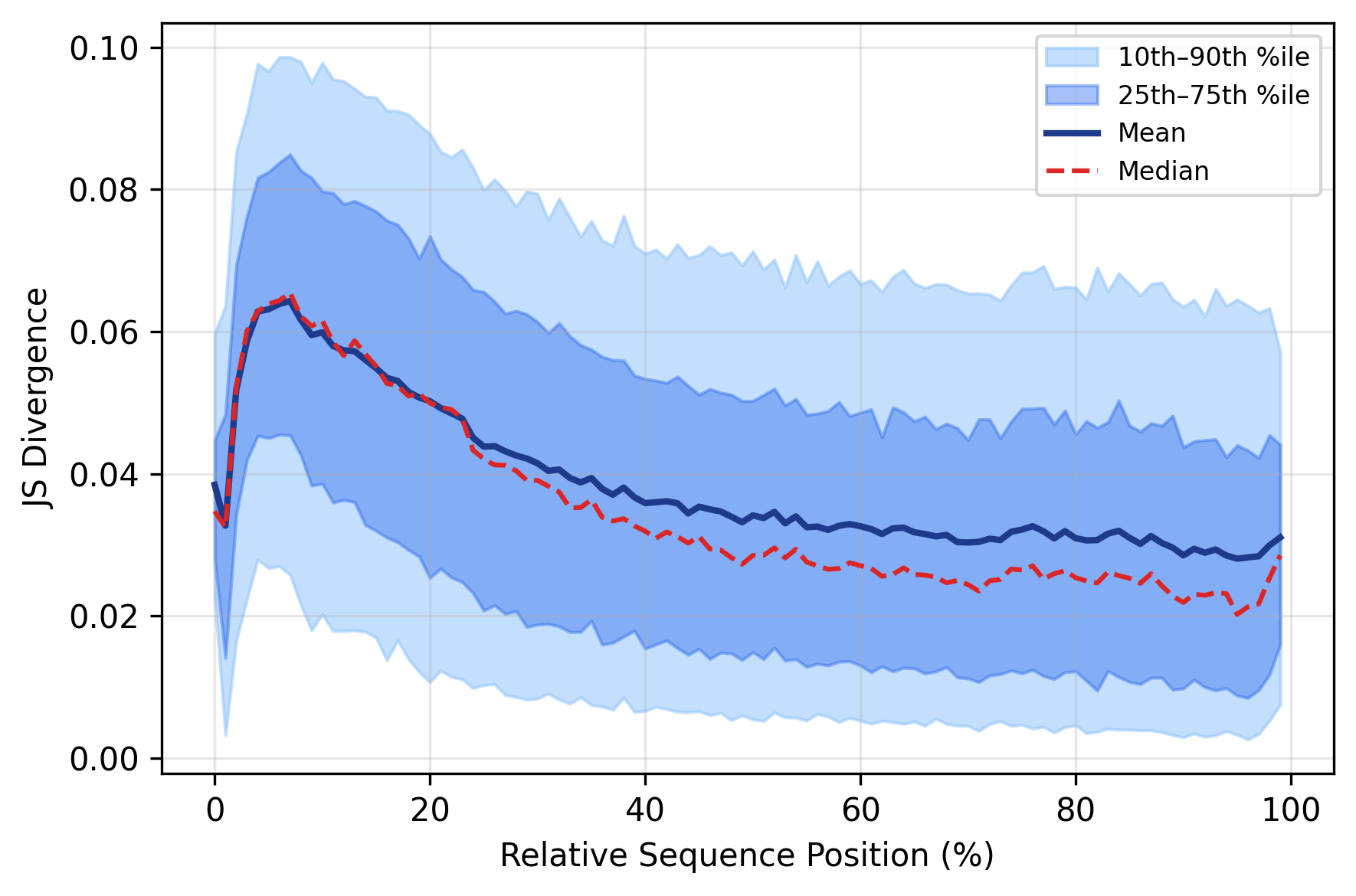}
        \caption{AIME 2024}
    \end{subfigure}
    \hspace{1cm}
    \begin{subfigure}{0.40\textwidth}
        \includegraphics[width=\linewidth]{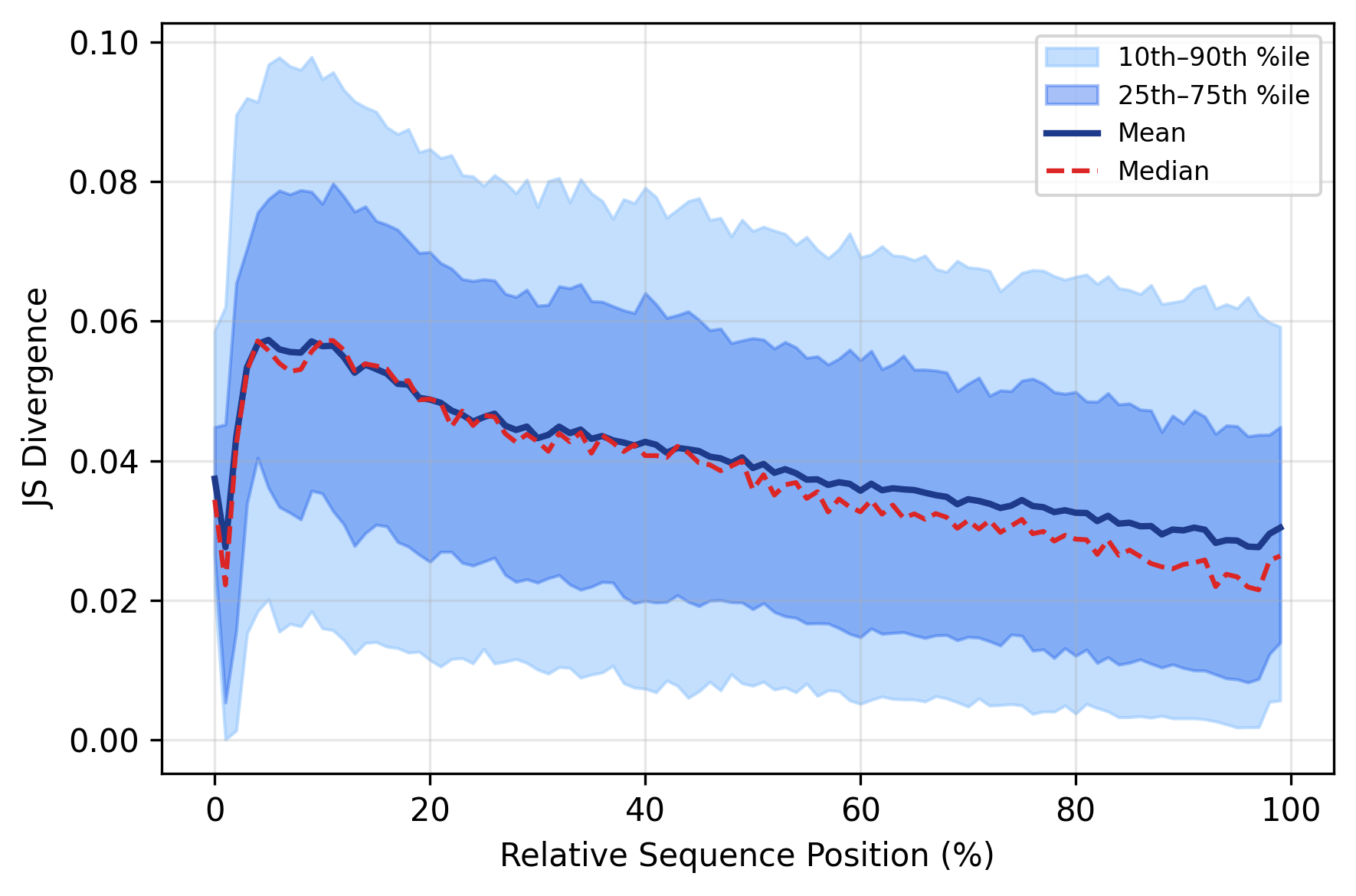}
        \caption{AIME 2025}
    \end{subfigure}
    
    \caption{
        Mean JS divergence by normalized token position for Qwen3-8B with DAPO on AIME 2024 and AIME 2025.
        Divergences remain concentrated at early and late positions, consistent with other models.
    }
    \label{fig:positional_qwen3_8b}
\end{figure}

\begin{figure}[!htbp]
    \centering
    \begin{subfigure}{0.42\linewidth}
        \includegraphics[width=\linewidth]{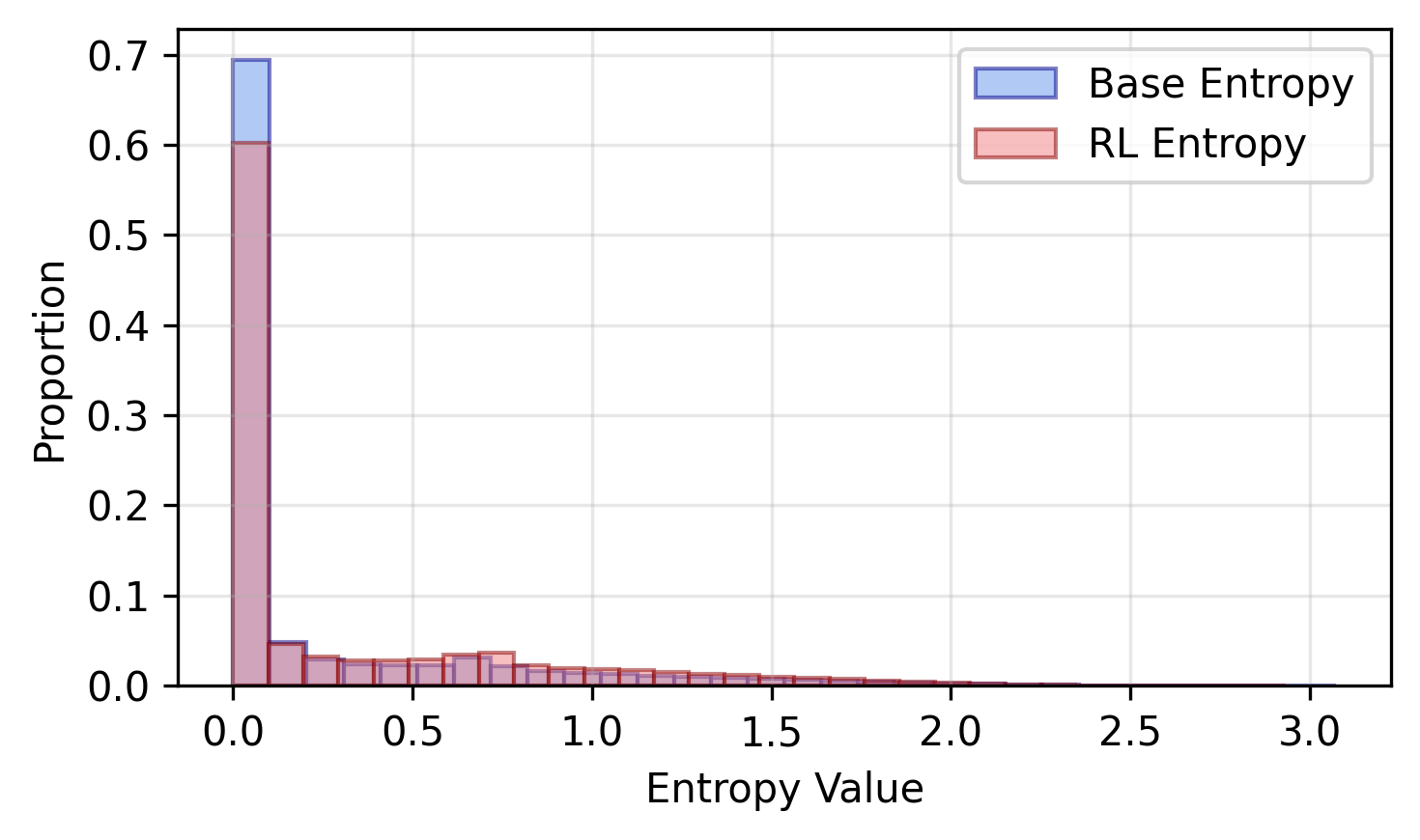}
        \caption{Qwen3-8B DAPO Low JS bin ($<0.1$)}
    \end{subfigure}
    \hspace{1cm}
    \begin{subfigure}{0.42\linewidth}
        \includegraphics[width=\linewidth]{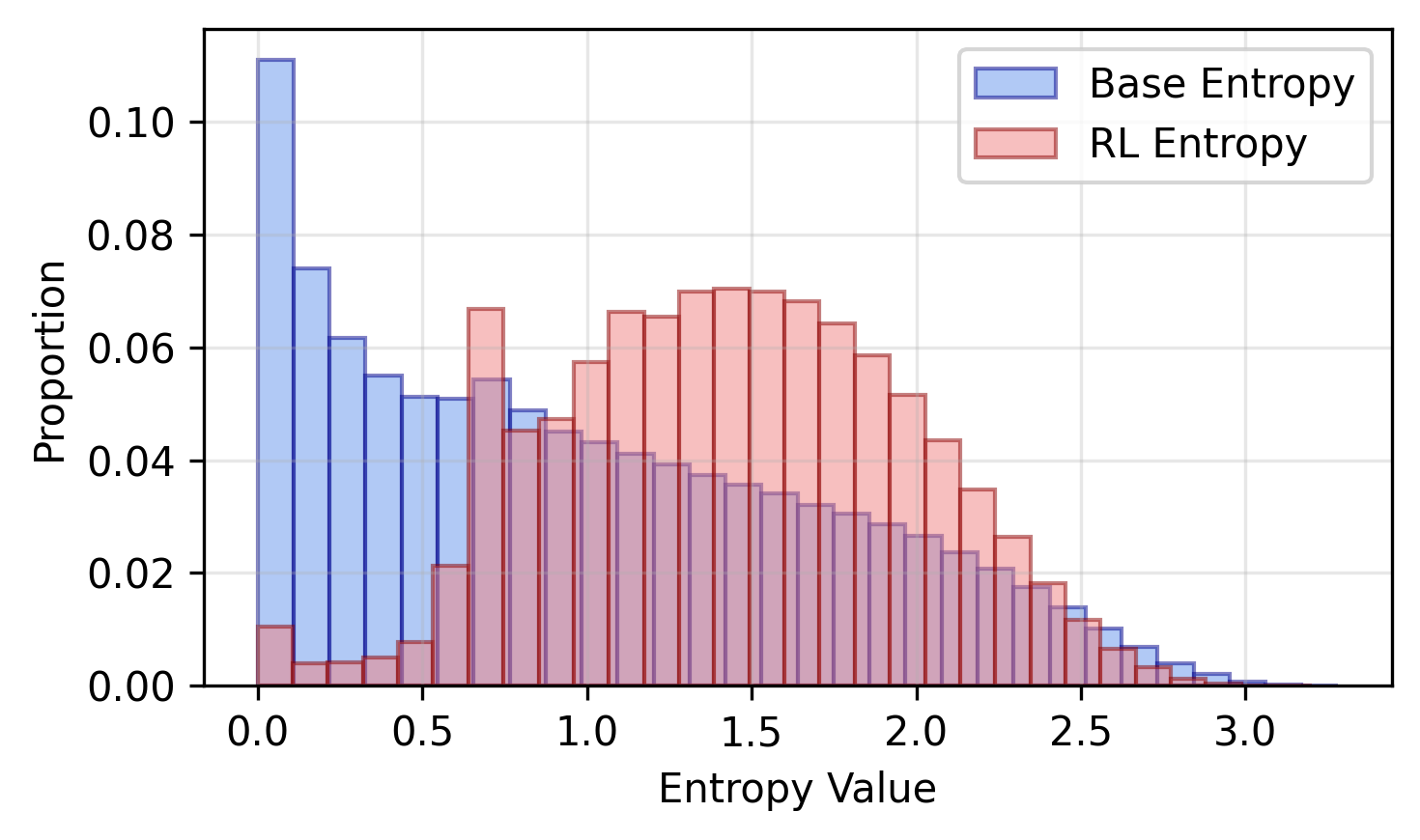}
        \caption{Qwen3-8B DAPO High JS bin ($>0.1$)}
    \end{subfigure}
    
    \caption{
        Entropy distributions across divergence bins for Qwen3-8B with DAPO on AIME 2024.
    }
    \label{fig:entropy_qwen3_8b}
\end{figure}

\begin{figure}[!htbp]
    \centering
    \begin{subfigure}{0.42\linewidth}
        \includegraphics[width=\linewidth]{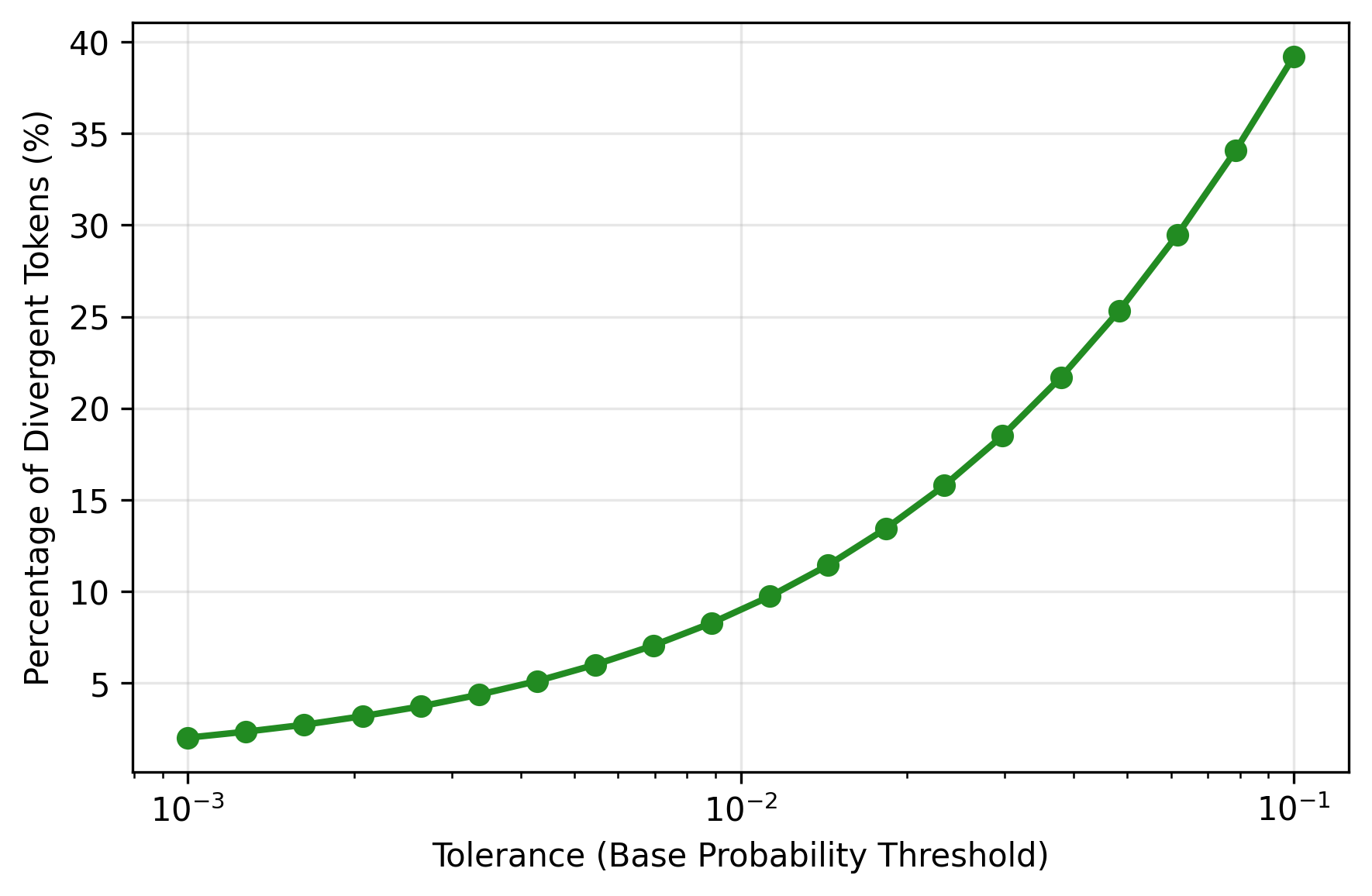}
        \caption{Qwen3-8B DAPO AIME 2024}
    \end{subfigure}
    \hspace{1cm}
    \begin{subfigure}{0.42\linewidth}
        \includegraphics[width=\linewidth]{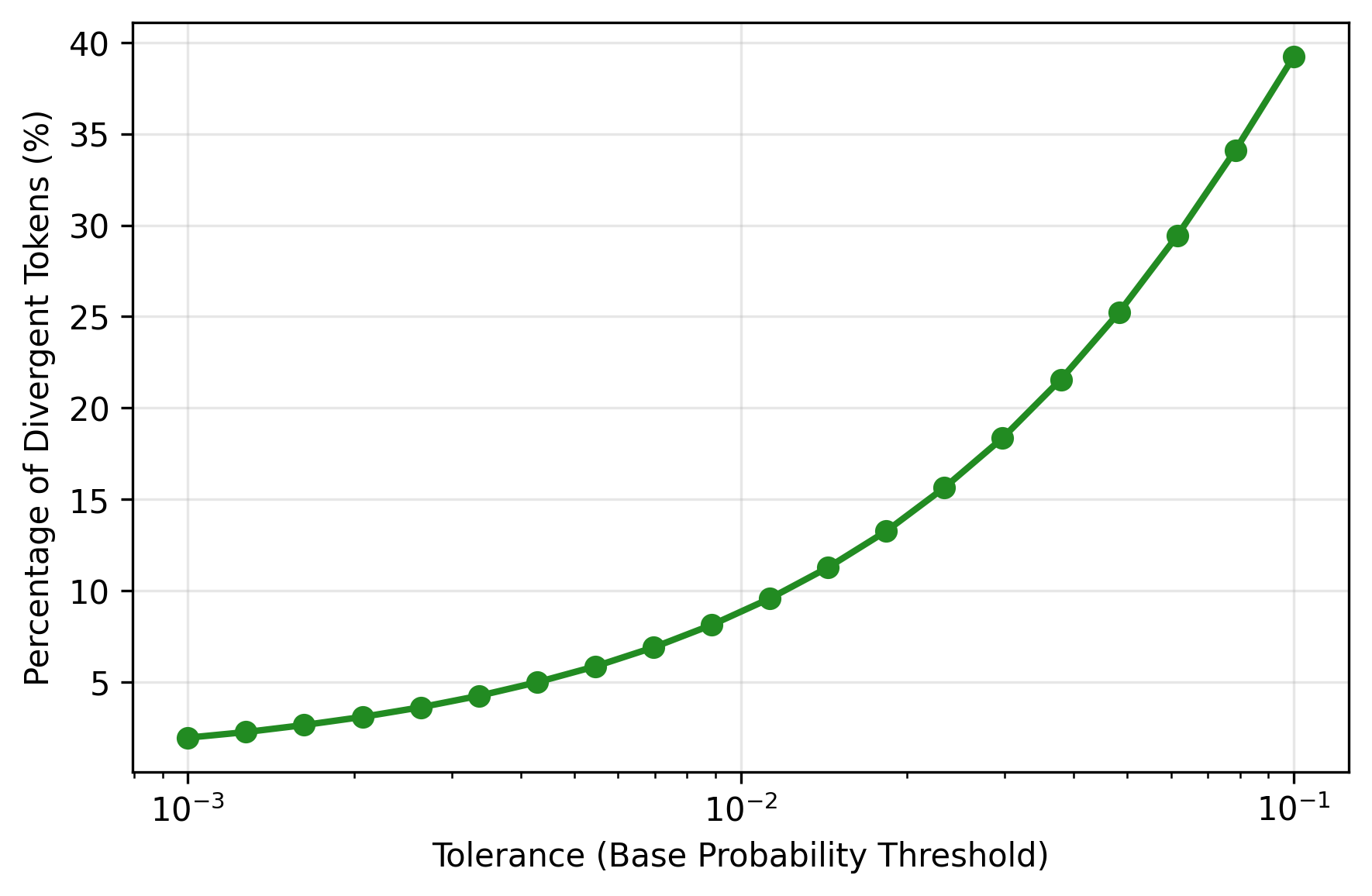}
        \caption{Qwen3-8B DAPO AIME 2025}
    \end{subfigure}
    
    \caption{
        Percentage of divergent tokens whose RL top-1 choice had base probability below a given threshold for Qwen3-8B with DAPO on AIME 2024 and AIME 2025.
    }
    \label{fig:tolerance_qwen3_8b}
\end{figure}
\begin{figure}[!htbp]
    \centering
    \begin{subfigure}{0.7\linewidth}
        \includegraphics[width=\linewidth]{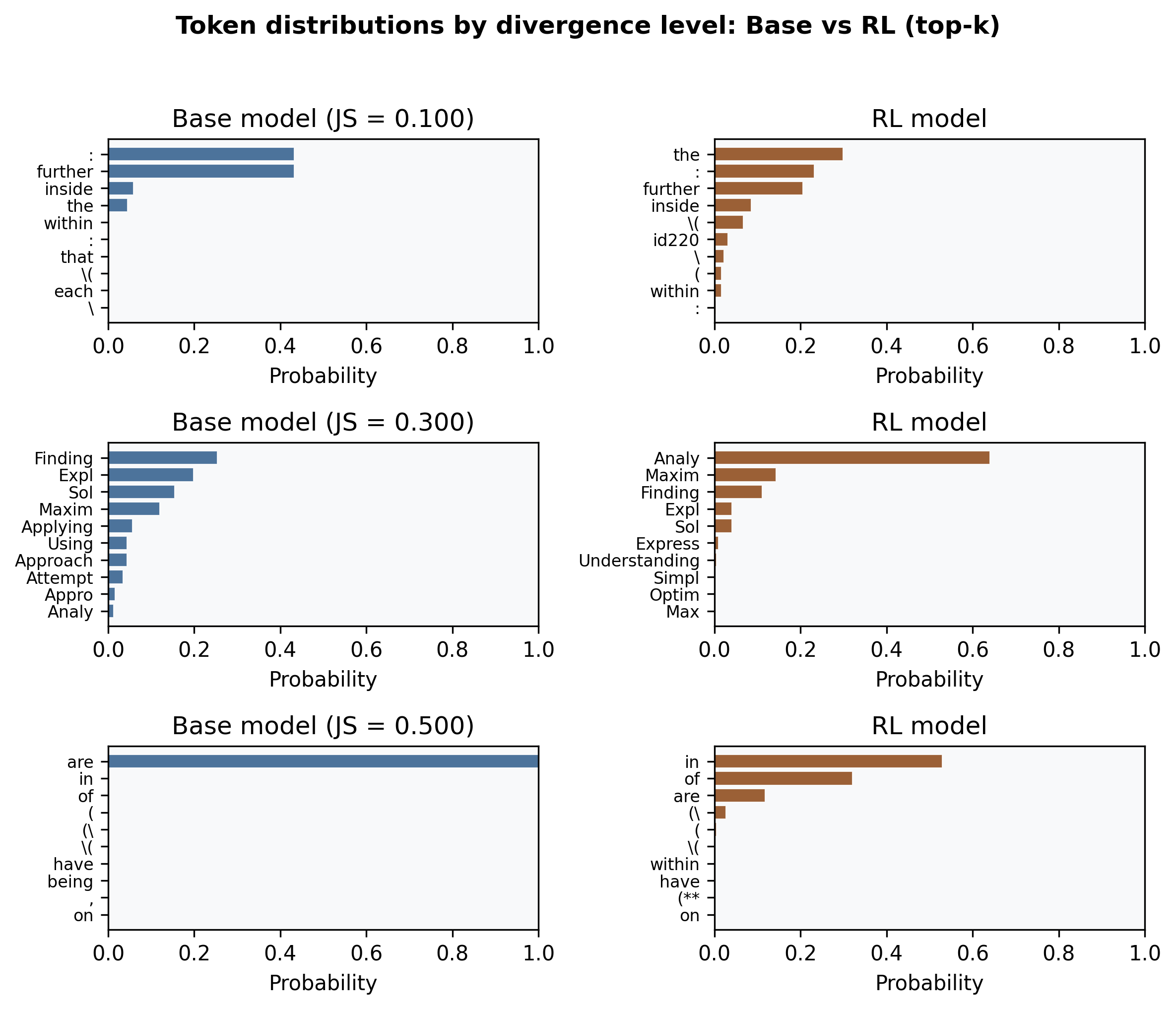}
        \caption{Qwen3-8B DAPO AIME 2024}
    \end{subfigure}
    
    \caption{
        Example of base vs RL distributions at differing divergence levels.
    }
    \label{fig:example_dist_diff_qwen3_8b}
\end{figure}

\FloatBarrier
\subsubsection{Mistral-Small-24B with SimpleRL}

We analyze Mistral-Small-24B trained with SimpleRL on AIME 2024 and AIME 2025 to demonstrate the generalizability of our findings across different model architectures. Figure~\ref{fig:js_mistral} shows JS divergence percentile curves, revealing consistent sparsity patterns. Figure~\ref{fig:positional_mistral} shows positional concentration, Figure~\ref{fig:entropy_mistral} shows entropy distributions across divergence bins, and Figure~\ref{fig:tolerance_mistral} shows tail behavior analysis.

\begin{figure}[!htbp]
    \centering
    \begin{subfigure}{0.40\textwidth}
        \includegraphics[width=\linewidth]{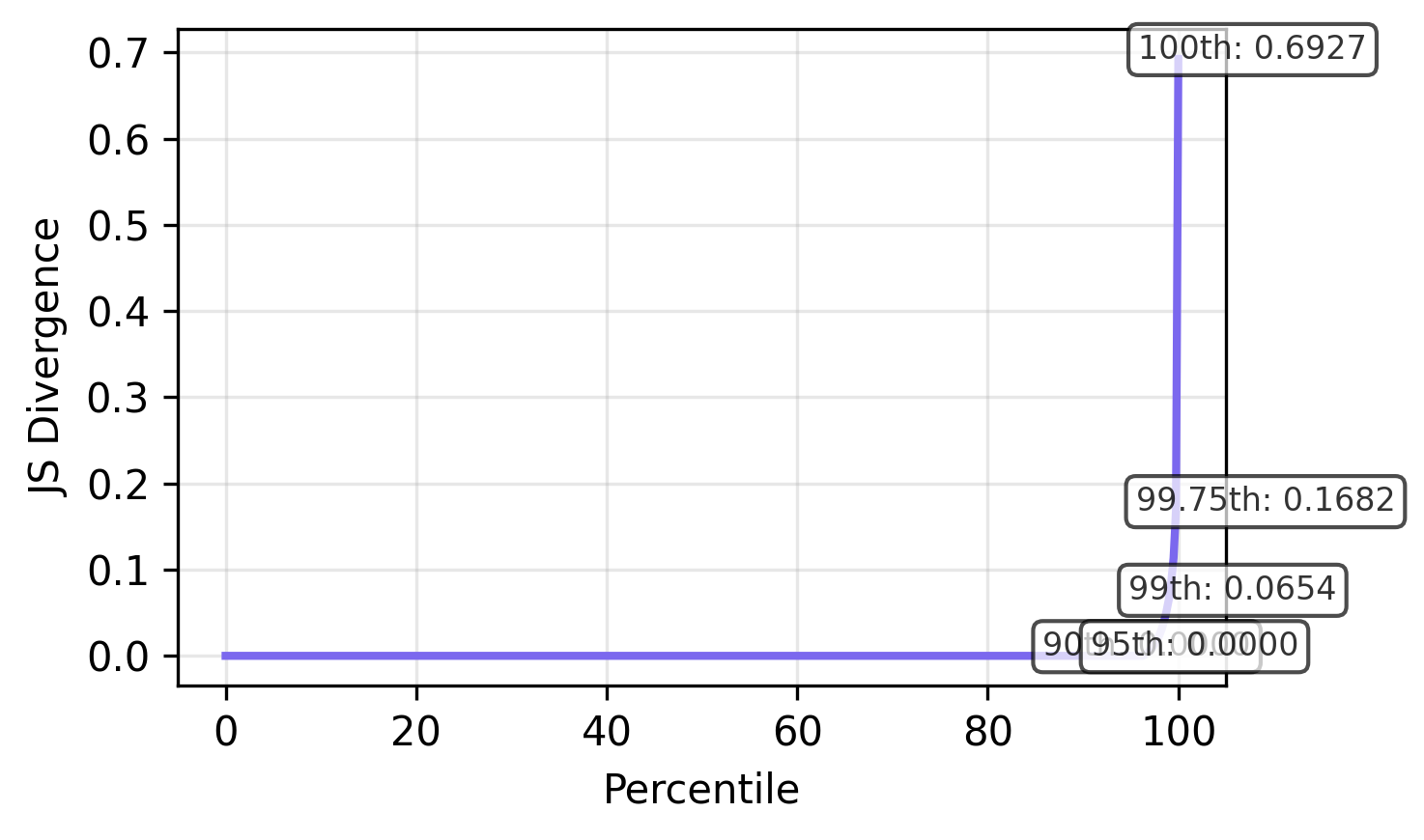}
        \caption{AIME 2024: Percentiles}
    \end{subfigure}
    \hspace{1cm}
    \begin{subfigure}{0.40\textwidth}
        \includegraphics[width=\linewidth]{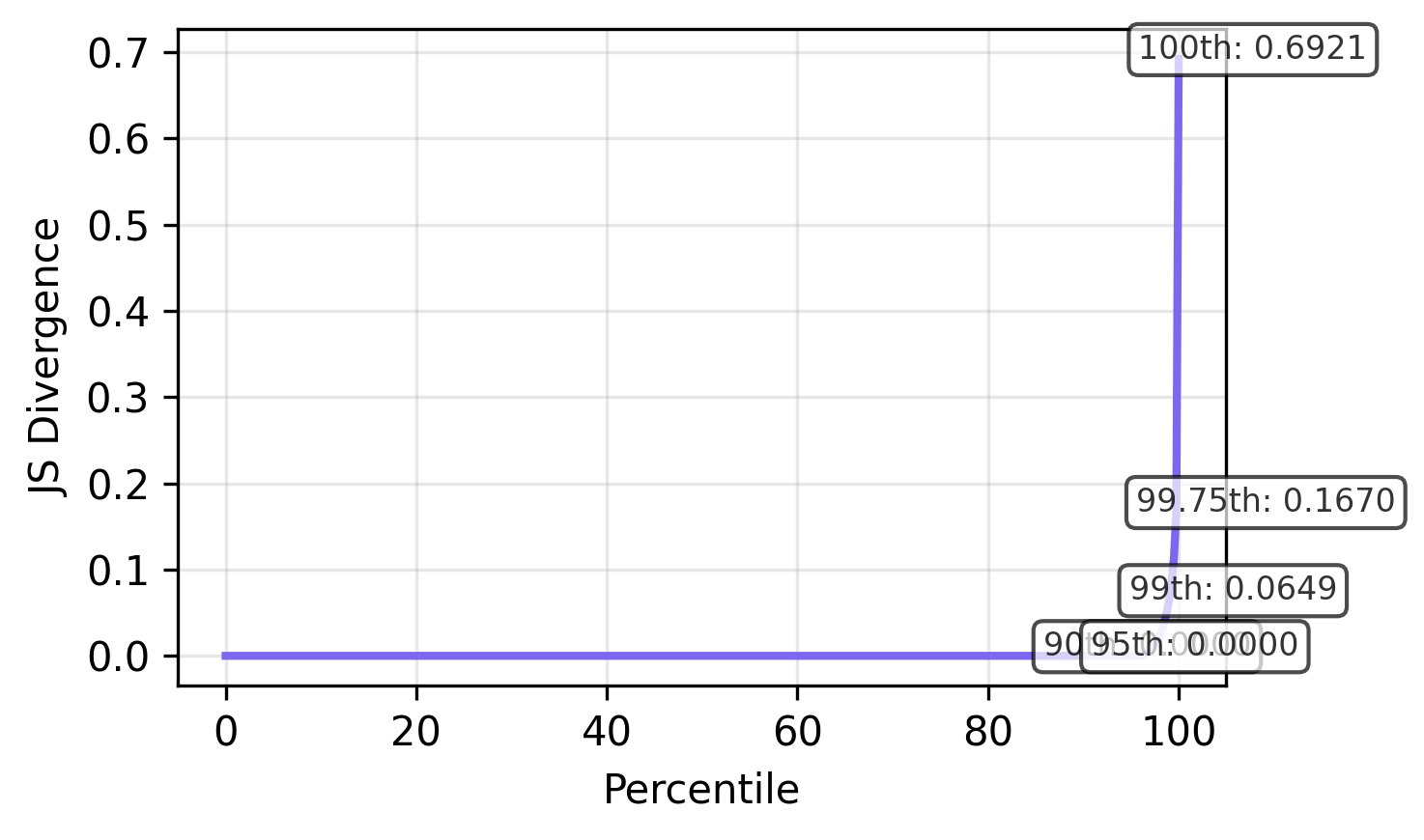}
        \caption{AIME 2025: Percentiles}
    \end{subfigure}
    
    \caption{
        JS divergence distributions for Mistral-Small-24B with SimpleRL on AIME 2024 and AIME 2025.
        Sparse distributional shifts are consistent with findings in the main text across both datasets.
    }
    \label{fig:js_mistral}
\end{figure}

\begin{figure}[!htbp]
    \centering
    \begin{subfigure}{0.40\textwidth}
        \includegraphics[width=\linewidth]{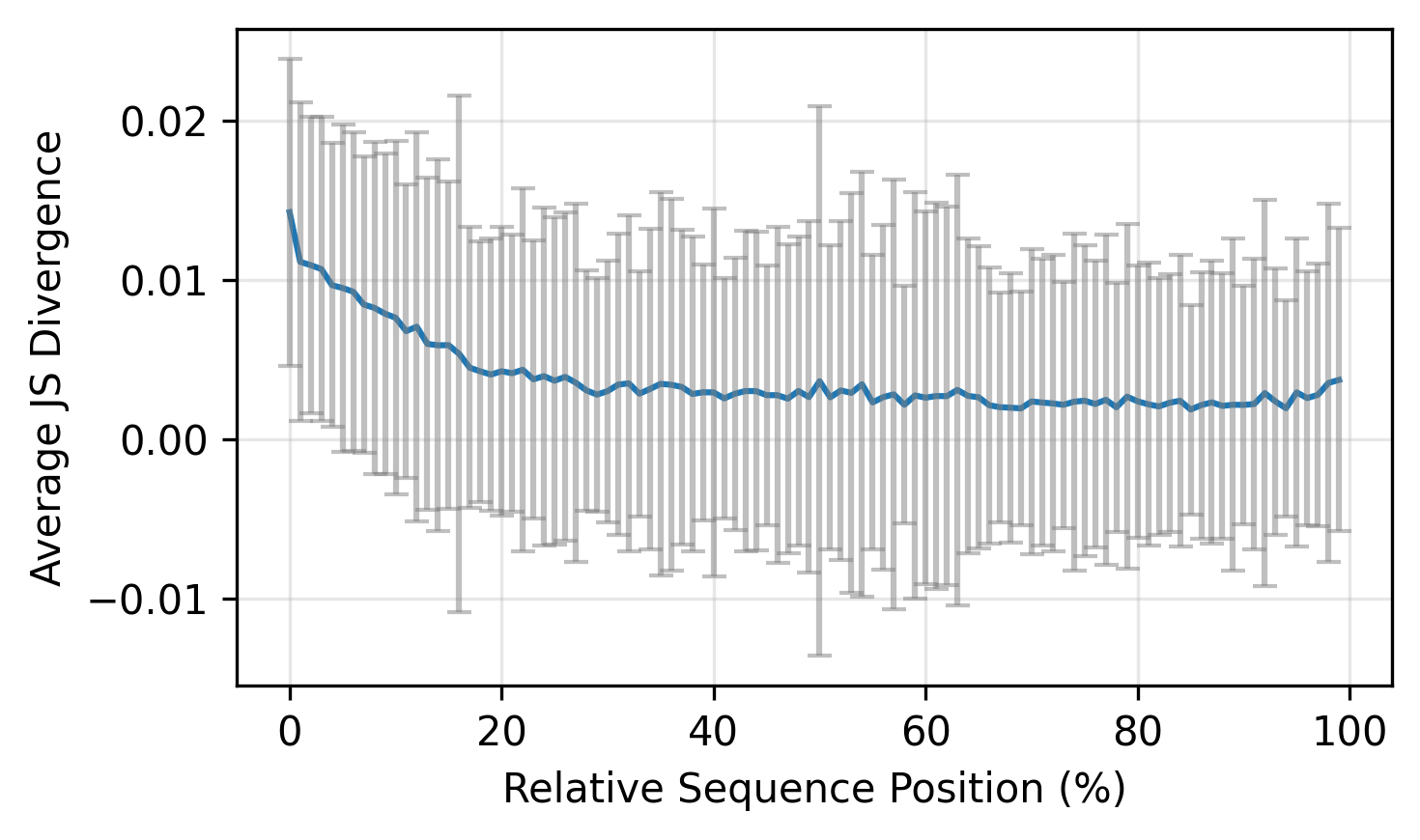}
        \caption{AIME 2024}
    \end{subfigure}
    \hspace{1cm}
    \begin{subfigure}{0.40\textwidth}
        \includegraphics[width=\linewidth]{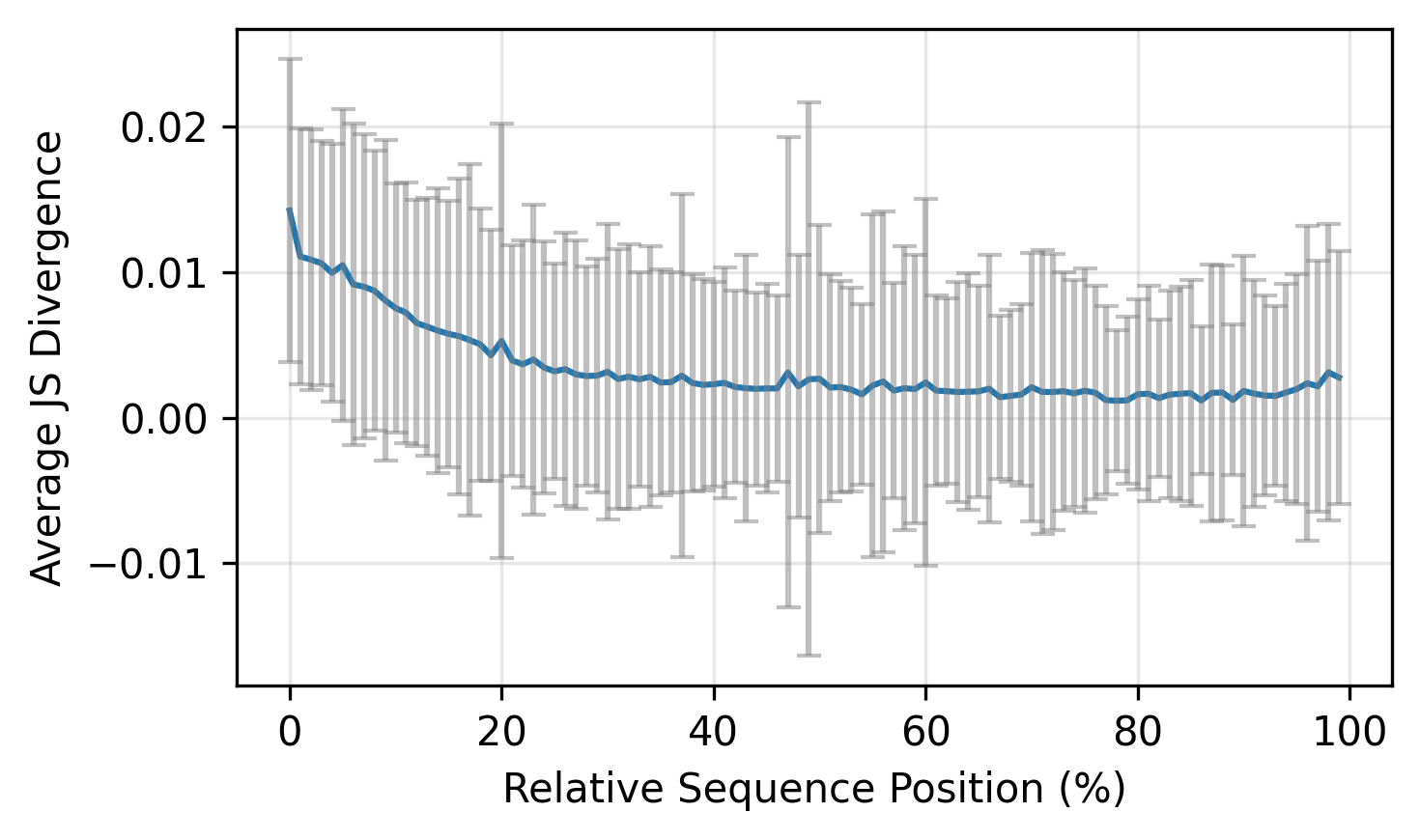}
        \caption{AIME 2025}
    \end{subfigure}
    
    \caption{
        Mean JS divergence by normalized token position for Mistral-Small-24B with SimpleRL on AIME 2024 and AIME 2025.
        Consistent with findings for other models, divergences are concentrated at the start and end of responses.
    }
    \label{fig:positional_mistral}
\end{figure}

\begin{figure}[!htbp]
    \centering
    \begin{subfigure}{0.42\linewidth}
        \includegraphics[width=\linewidth]{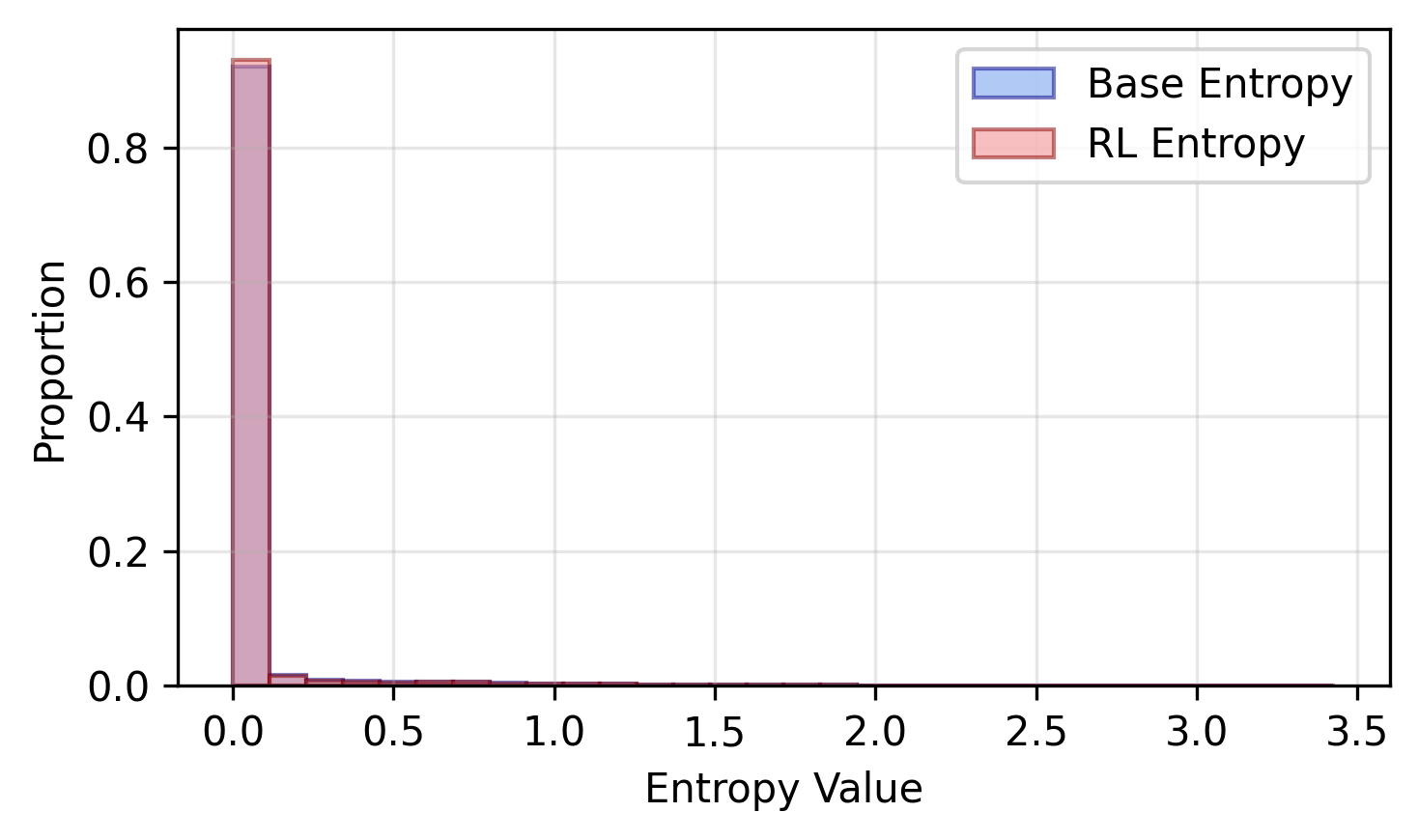}
        \caption{Mistral-24B SimpleRL Low JS bin ($<0.1$)}
    \end{subfigure}
    \hspace{1cm}
    \begin{subfigure}{0.42\linewidth}
        \includegraphics[width=\linewidth]{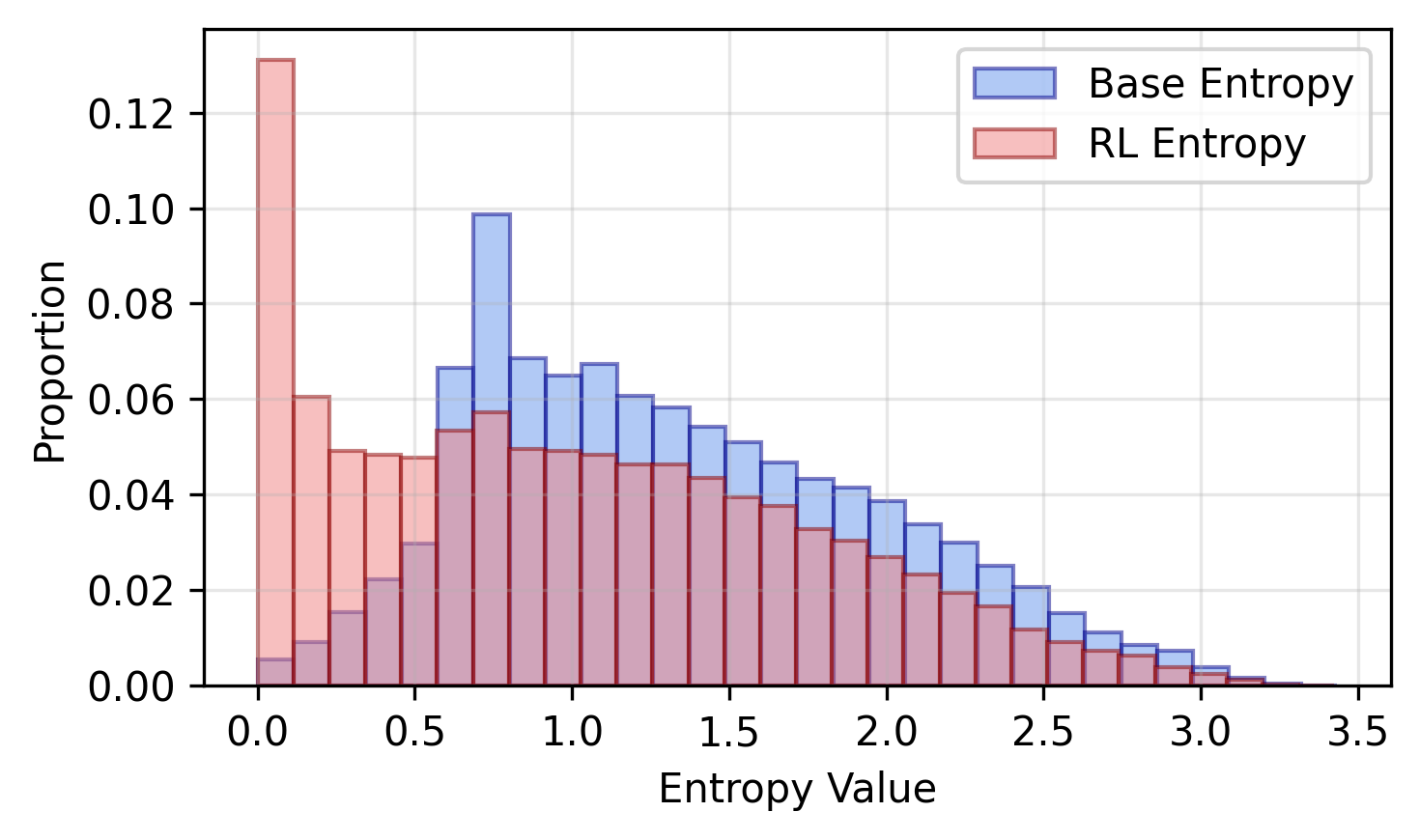}
        \caption{Mistral-24B SimpleRL High JS bin ($>0.1$)}
    \end{subfigure}
    
    \caption{
        Entropy distributions across divergence bins for Mistral-Small-24B with SimpleRL on AIME 2024.
    }
    \label{fig:entropy_mistral}
\end{figure}

\begin{figure}[!htbp]
    \centering
    \begin{subfigure}{0.42\linewidth}
        \includegraphics[width=\linewidth]{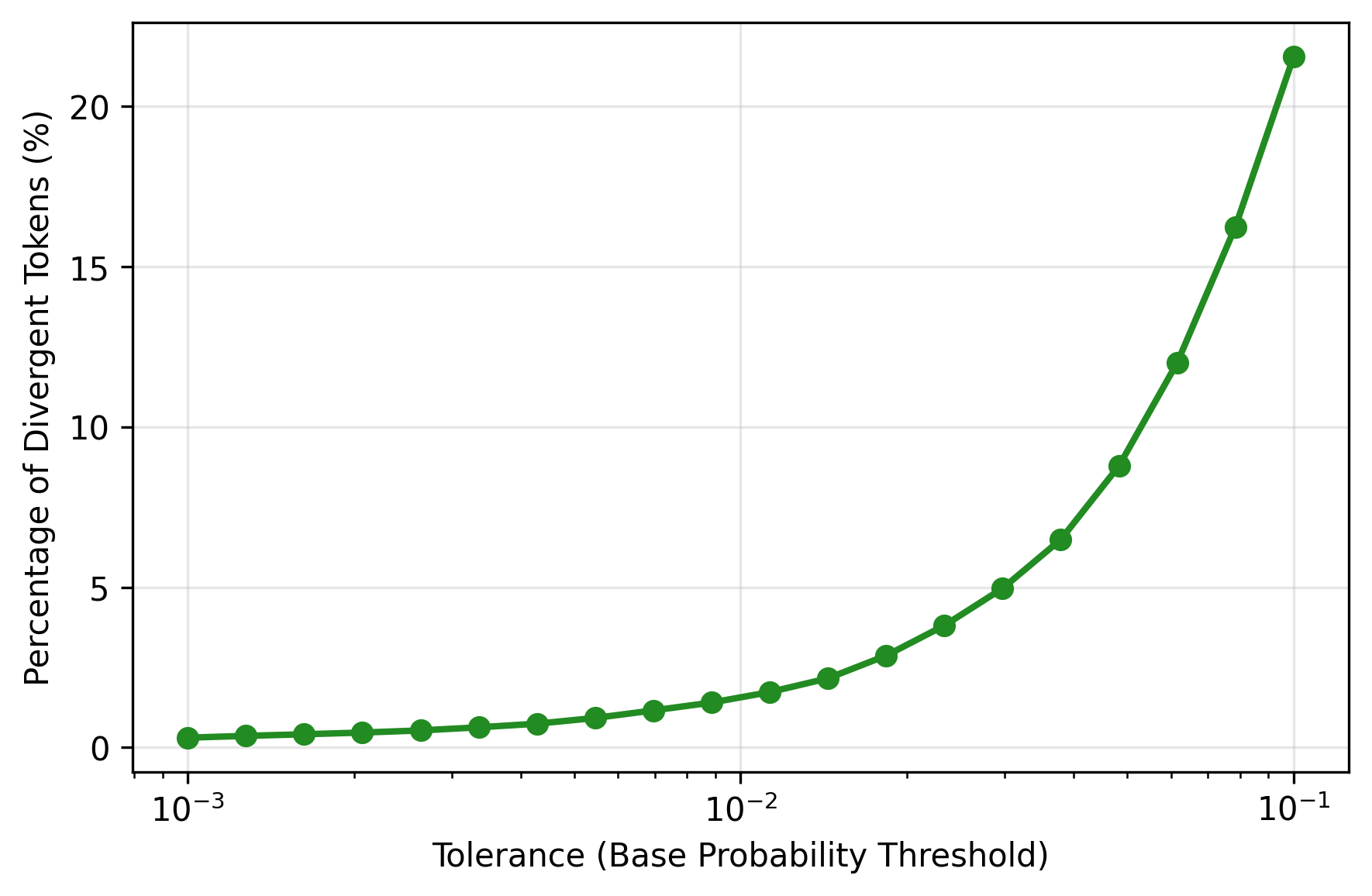}
        \caption{Mistral-24B SimpleRL AIME 2024}
    \end{subfigure}
    \hspace{1cm}
    \begin{subfigure}{0.42\linewidth}
        \includegraphics[width=\linewidth]{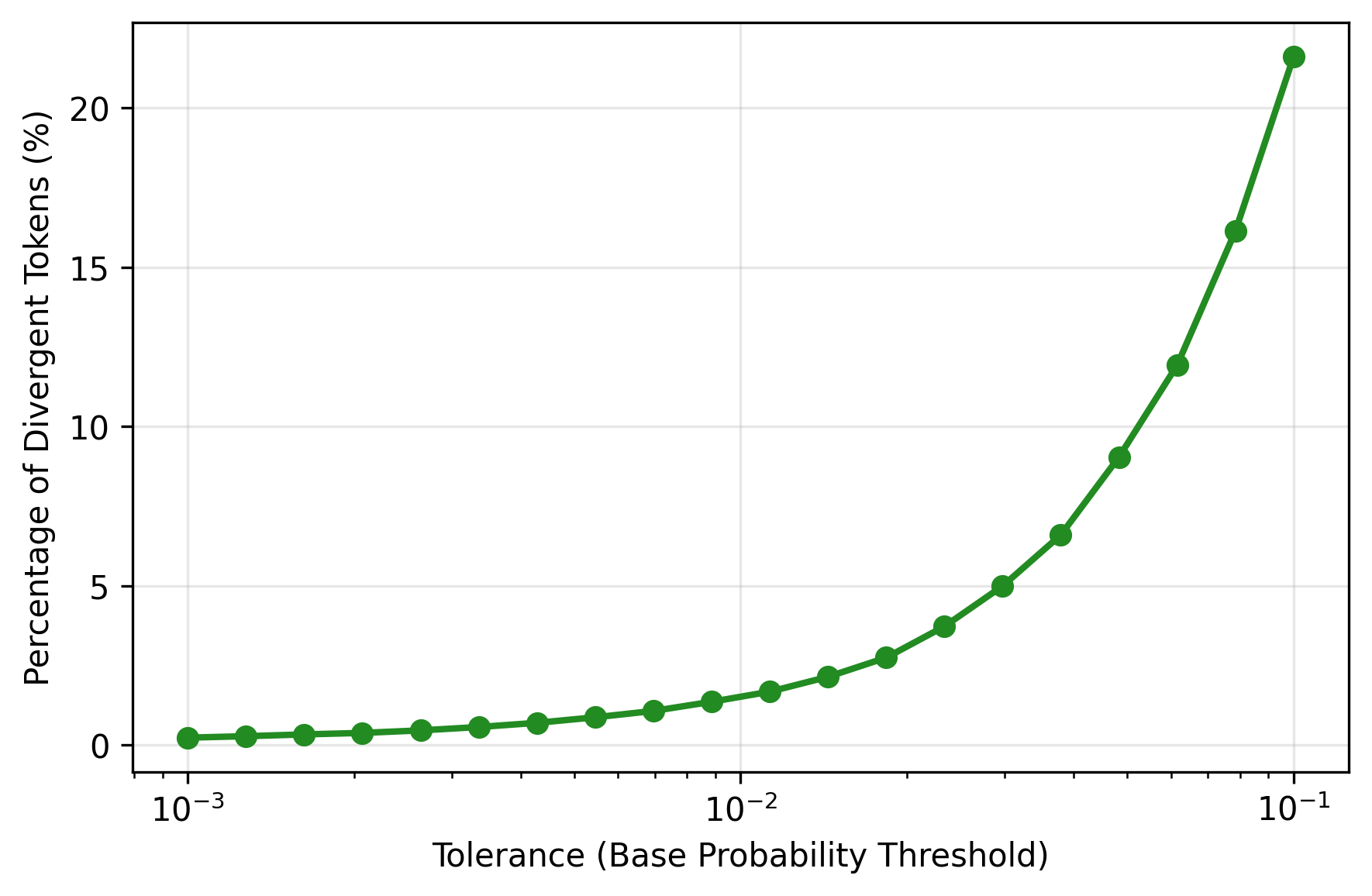}
        \caption{Mistral-24B SimpleRL AIME 2025}
    \end{subfigure}
    
    \caption{
        Percentage of divergent tokens whose RL top-1 choice had base probability below a given threshold for Mistral-Small-24B with SimpleRL on AIME 2024 and AIME 2025.
    }
    \label{fig:tolerance_mistral}
\end{figure}

\FloatBarrier
\newpage
\subsection{Additional Cross-Sampling Results}
This section provides supplementary cross-sampling results and a general description of the algorithm used for the cross-sampling experiments. 

Algorithm~\ref{alg:cross_sampling} describes the general procedure for a single prompt (in practice, we batch prompts for efficiency), which generates a response primarily under a \emph{primary policy} and selectively intervenes using an \emph{intervention policy} at positions where the
token-level divergence exceeds a fixed threshold. Note that to obtain the cross-sampling plots, for every $k$ cross-sampling interventions, we complete the response with $\pi_{\mathrm{prim}}$ and evaluate, then continue with cross-sampling using the prefix from prior to completing it with the primary model.

\label{subsec:additional_cross_sampling}
\begin{algorithm}[htbp]
\caption{Cross-Sampling for a Single Prompt}
\label{alg:cross_sampling}
\begin{algorithmic}[1]
\Require Prompt prefix $x_{<1}$, primary policy $\pi_{\mathrm{prim}}$, intervention policy $\pi_{\mathrm{int}}$, divergence threshold $\varepsilon$, maximum steps $T$
\Ensure Generated sequence $x_{1:t}$, number of intervention steps $k$

\State $k \gets 0$
\State Initialize prefix $x_{<1}$

\For{$t = 1, \dots, T$}
    \State Compute divergence 
    \(
        \mathrm{JS}_t = D_{\mathrm{JS}}\!\left(
        \pi_{\mathrm{prim}}(\cdot \mid x_{<t})\parallel 
        \pi_{\mathrm{int}}(\cdot \mid x_{<t})
        \right)
    \)
    \If{$\mathrm{JS}_t > \varepsilon$}
        \State Sample $x_t \sim \pi_{\mathrm{int}}(\cdot \mid x_{<t})$
        \State $k \gets k + 1$
    \Else
        \State Sample $x_t \sim \pi_{\mathrm{prim}}(\cdot \mid x_{<t})$
    \EndIf

    \State Append $x_t$ to prefix $x_{<t+1}$
    \If{$x_t = \mathrm{EOS}$}
        \State \textbf{break}
    \EndIf
\EndFor

\State \Return generated sequence $x_{1:t}$ and intervention count $k$
\end{algorithmic}
\end{algorithm}

\begin{table}[htbp]
\centering
\small
\caption{Summary of cross-sampled tokens required to reach approximate RL-level performance (forward) or base-level performance (reverse) for Qwen2.5-32B on AIME 2024 and AIME 2025 with a token budget of 8000.
Effective token counts/percentages exclude identity swaps during cross-sampling. Token percentages are computed at the sequence level.}
\label{tab:token_stats_app}
\begin{tabular}{l@{\hspace{0.5em}}l@{\hspace{0.2em}}c@{\hspace{0.2em}}c@{\hspace{0.2em}}c@{\hspace{0.2em}}c@{\hspace{0.2em}}|c@{\hspace{0.2em}}|c}
\toprule
\textbf{Dataset} & \textbf{Method} & \textbf{Eff. \%} & \textbf{\%} & \textbf{Eff. \#} & \textbf{\#} & \textbf{Initial} & \textbf{Final} \\
 &  & \textbf{Tokens} & \textbf{Tokens} & \textbf{Tokens} & \textbf{Tokens} & \textbf{Acc.} & \textbf{Acc.} \\
\midrule
\multirow{4}{*}{AIME24}
 & SimpleRL   & $3.86$\% & $7.58$\% & $38$ & $75$ & $8.23$ & $>25$ \\
 & SimpleRL Rev.  & $5$\% & $8.3$\% & $29$ & $51$ & $25.52$ & $<8.3$ \\
 & DAPO       & $7.8$\% & $11.9$\% & $280$ & $410$ & $8.23$ & $>44$ \\
 & DAPO Rev.       & $10.1$\% & $14.9$\% & $173$ & $258$ & $44.8$ & $<8.5$ \\
\midrule
\multirow{4}{*}{AIME25}
 & SimpleRL   & $1.53$\% & $2.97$\% & $13$ & $26$ & $5.3$ & $>14$ \\
 & SimpleRL Rev.  & $4.73$\% & $7.87$\% & $31$ & $53$ & $12.71$ & $<4$ \\
 & DAPO       & $6.47$\% & $9.18$\% & $230$ & $326$ & $4.8$ & $>33$ \\
 & DAPO Rev.       & $9.89$\% & $14.19$\% & $181$ & $261$ & $32$ & $<4.5$ \\
\bottomrule
\end{tabular}
\end{table}

\begin{figure}[!htbp]
    \centering
    \begin{subfigure}{0.48\textwidth}
        \includegraphics[width=\linewidth]{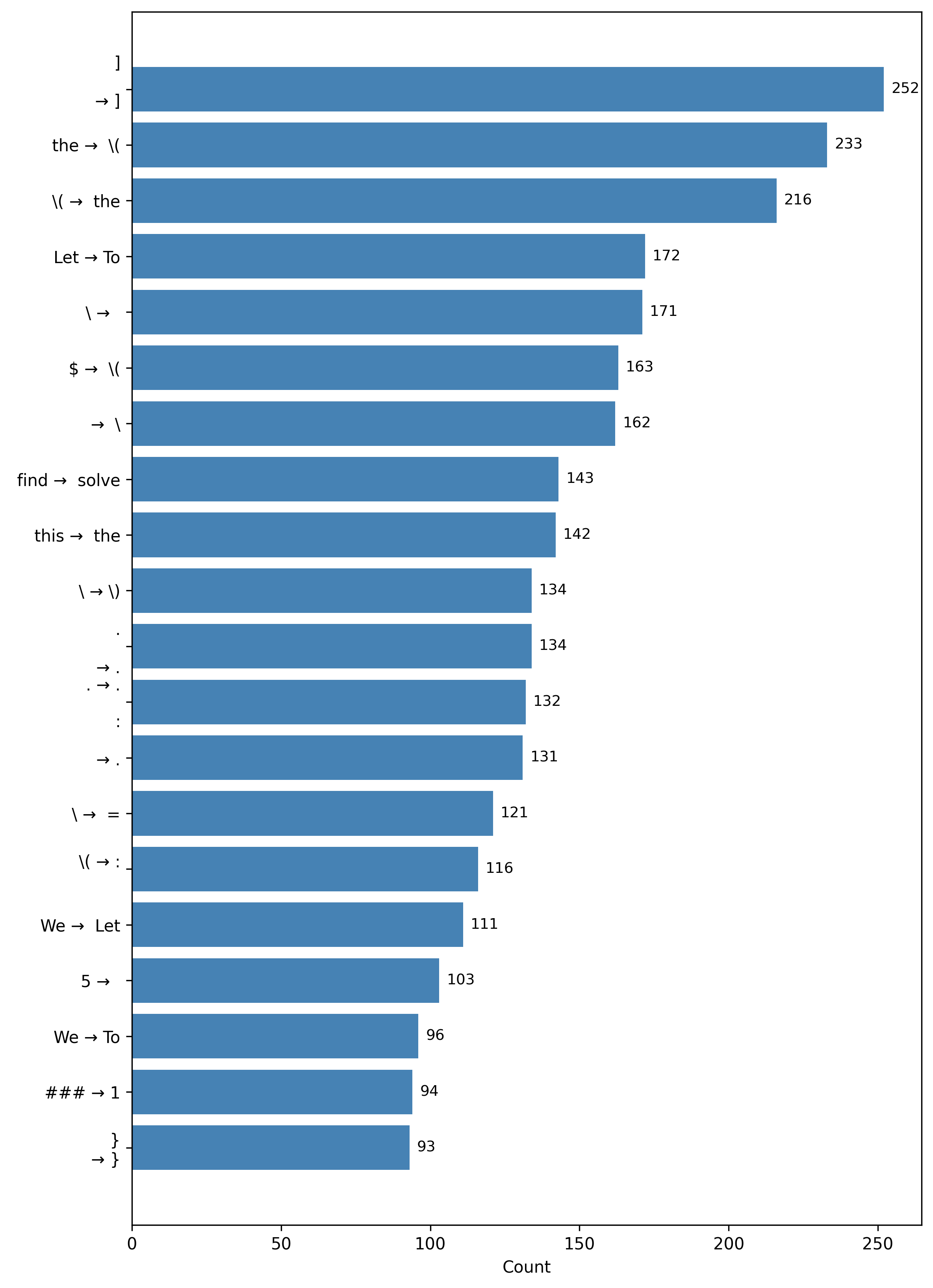}
        \caption{Forward cross-sampling Qwen2.5-32B SimpleRL.}
        
    \end{subfigure}
    \hfill
    \begin{subfigure}{0.48\textwidth}
        \includegraphics[width=\linewidth]{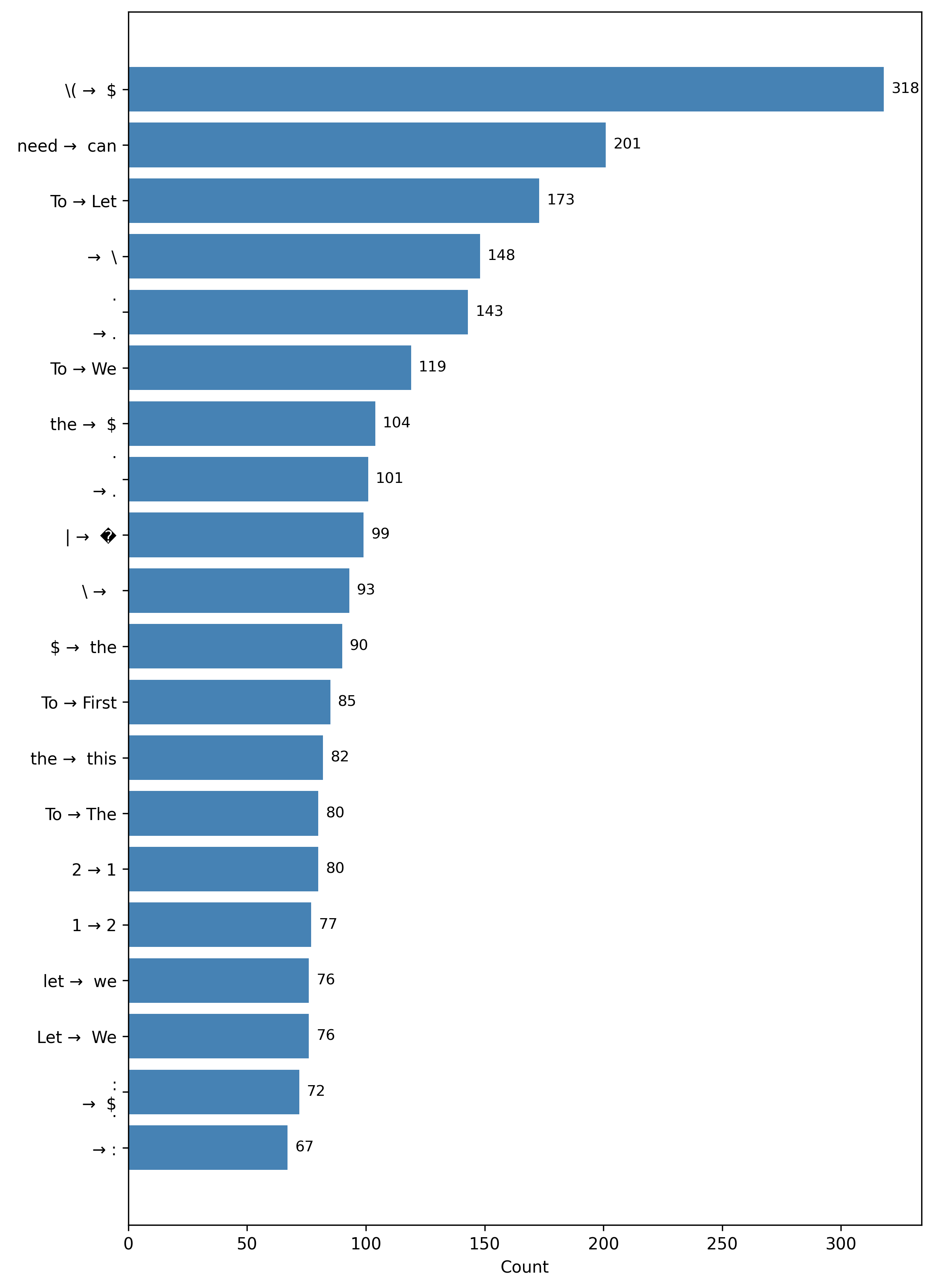}
        \caption{Reverse cross-sampling Qwen2.5-32B SimpleRL}
        
    \end{subfigure}
    \begin{subfigure}{0.48\textwidth}
        \includegraphics[width=\linewidth]{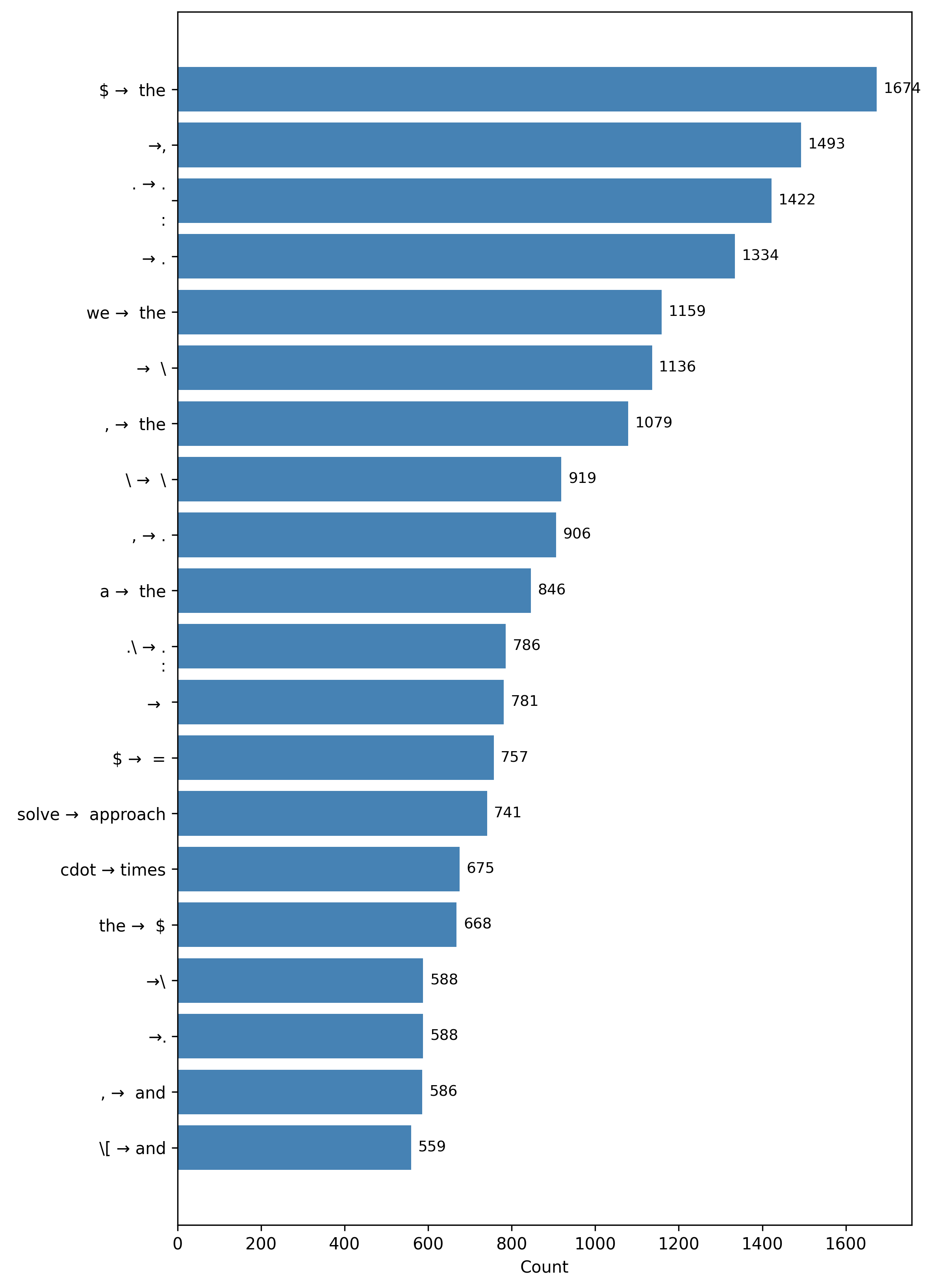}
        \caption{Forward cross-sampling Qwen2.5-32B DAPO}
        
    \end{subfigure}
    \hfill
    \begin{subfigure}{0.48\textwidth}
        \includegraphics[width=\linewidth]{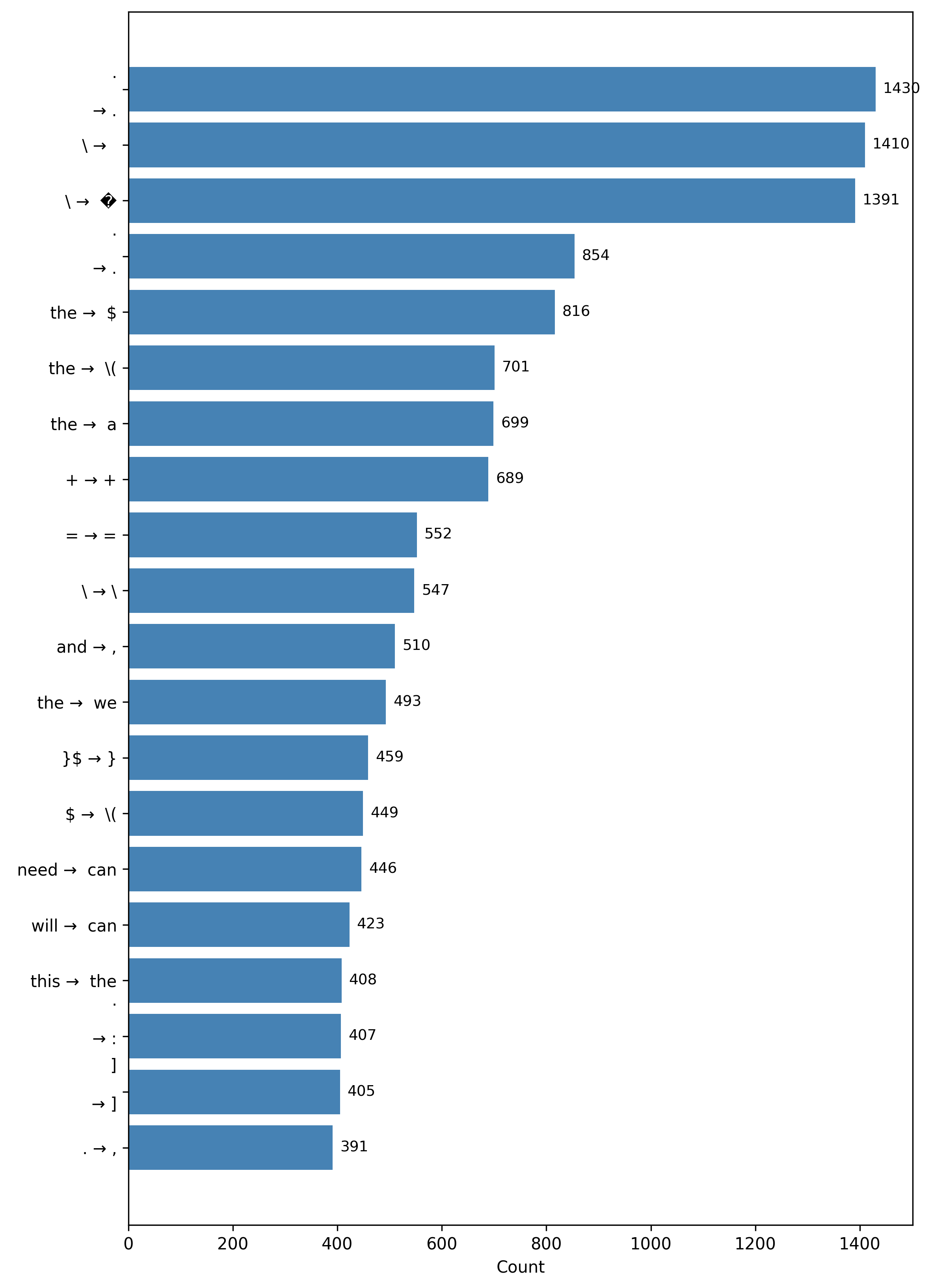}
        \caption{Reverse cross-sampling Qwen2.5-32B DAPO}
        
    \end{subfigure}
    \caption{Cross-sampling token pair histograms of the form (primary sampled token) -> (intervention sampled token) at cross-sampling intervention points (excluding identity token swaps).}
    \label{fig:resample_token_pairs}
\end{figure}

\begin{figure}[!htbp]
    \centering
    \begin{subfigure}{0.48\textwidth}
        \includegraphics[width=\linewidth]{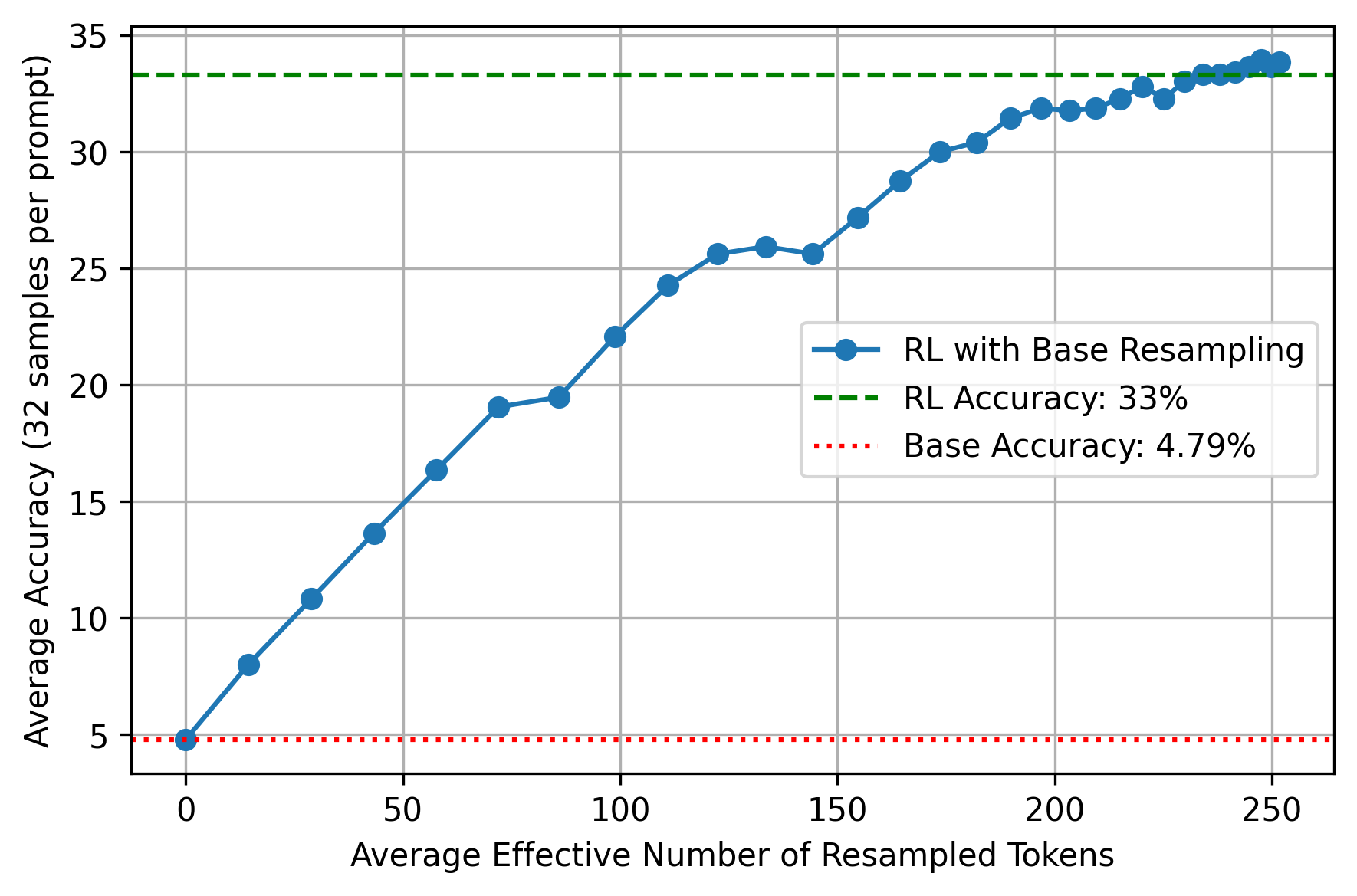}
        \caption{Forward cross-sampling}
        \label{fig:forward_resample_dapo_aime25}
    \end{subfigure}
    \hfill
    \begin{subfigure}{0.48\textwidth}
        \includegraphics[width=\linewidth]{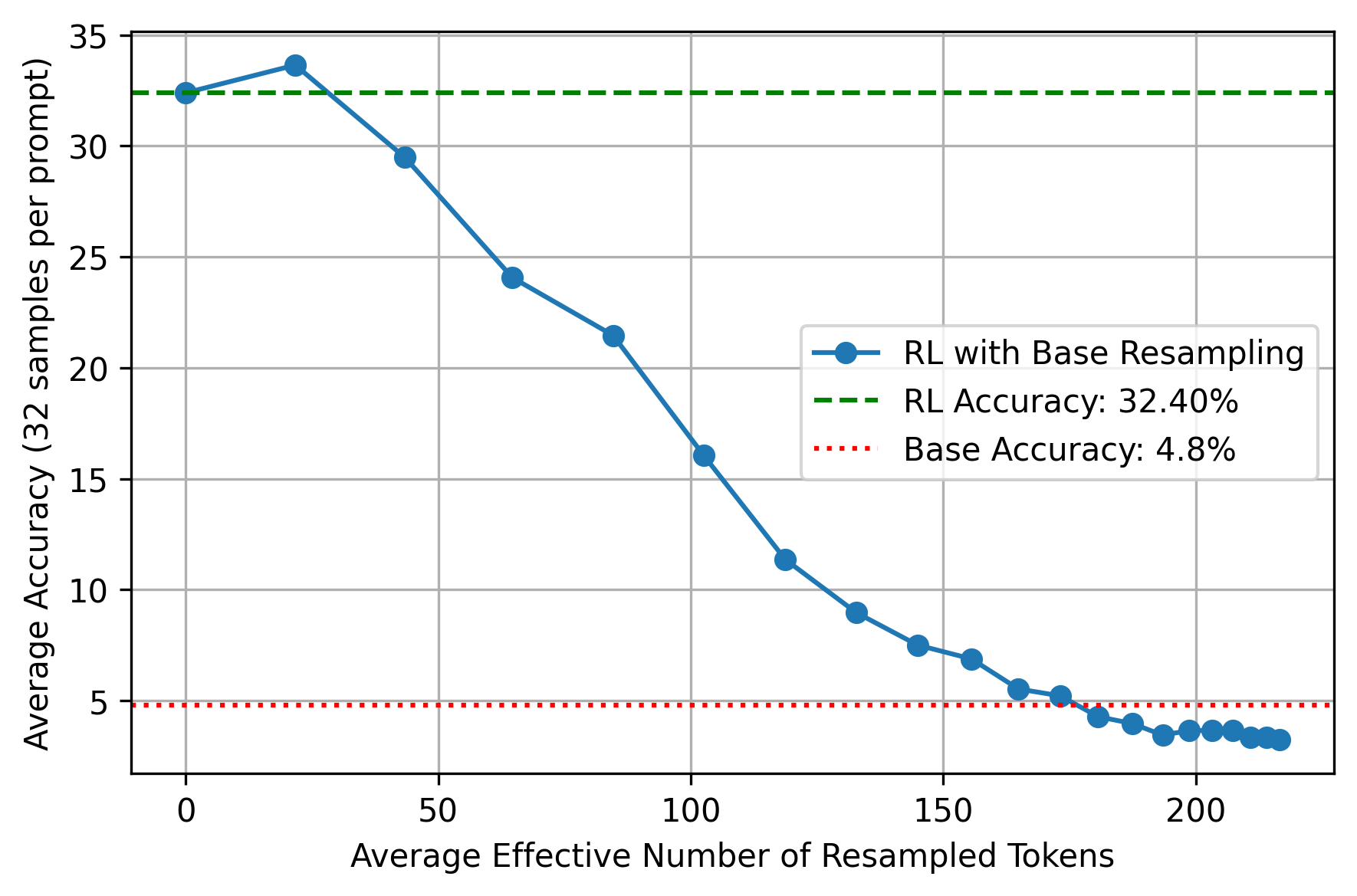}
        \caption{Reverse cross-sampling}
        \label{fig:reverse_resample_dapo_aime25}
    \end{subfigure}
    \caption{Cross-sampling results (DAPO on AIME 2025): injecting RL tokens into base generations progressively recovers RL accuracy, while reverting RL tokens with base tokens causes near-monotonic degradation toward base performance.}
    \label{fig:resample_dapo_aime25}
\end{figure}

\begin{figure}[!htbp]
    \centering
    \begin{subfigure}{0.48\textwidth}
        \includegraphics[width=\linewidth]{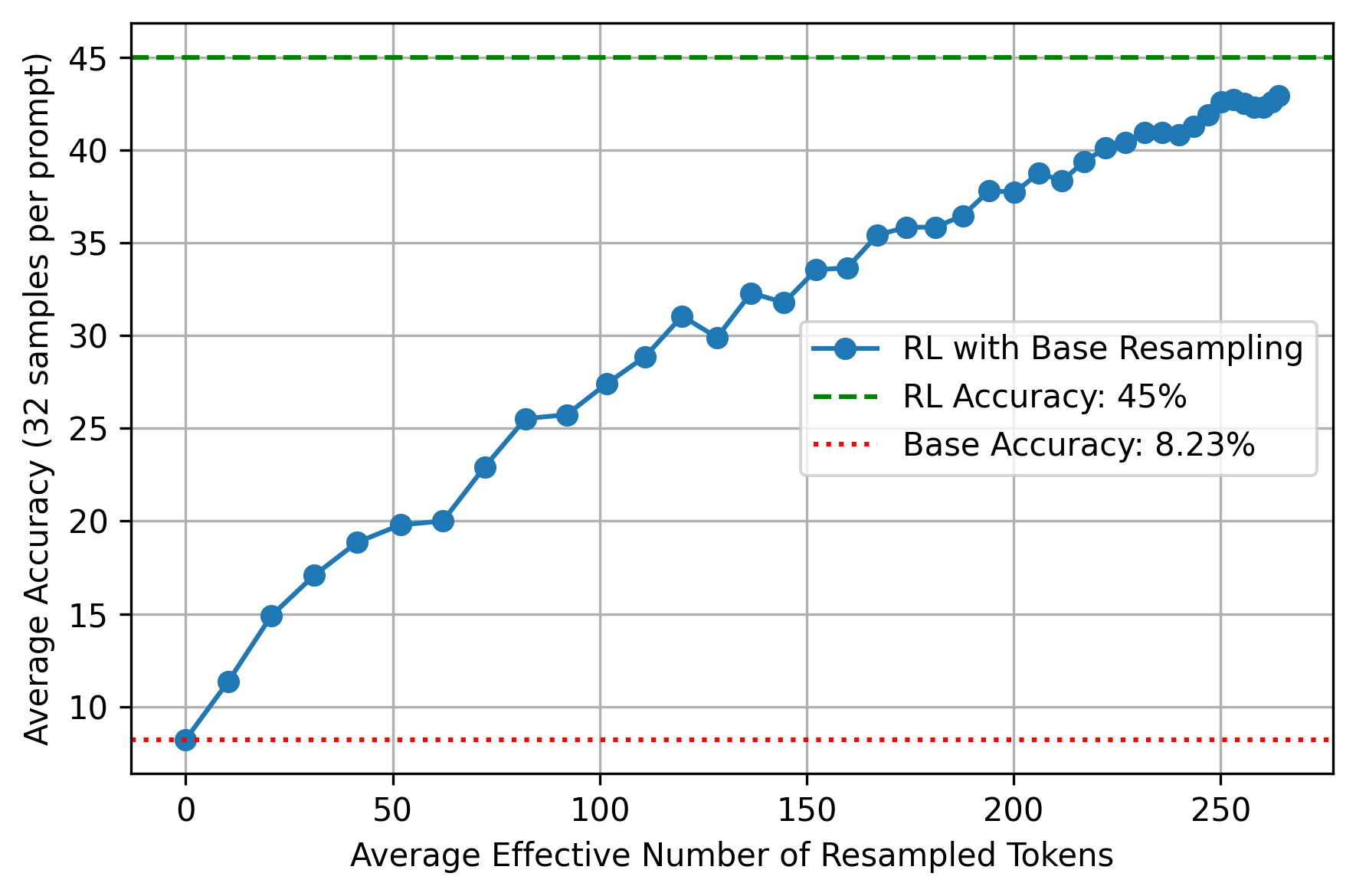}
        \caption{Forward cross-sampling}
        \label{fig:forward_resample_dapo}
    \end{subfigure}
    \hfill
    \begin{subfigure}{0.48\textwidth}
        \includegraphics[width=\linewidth]{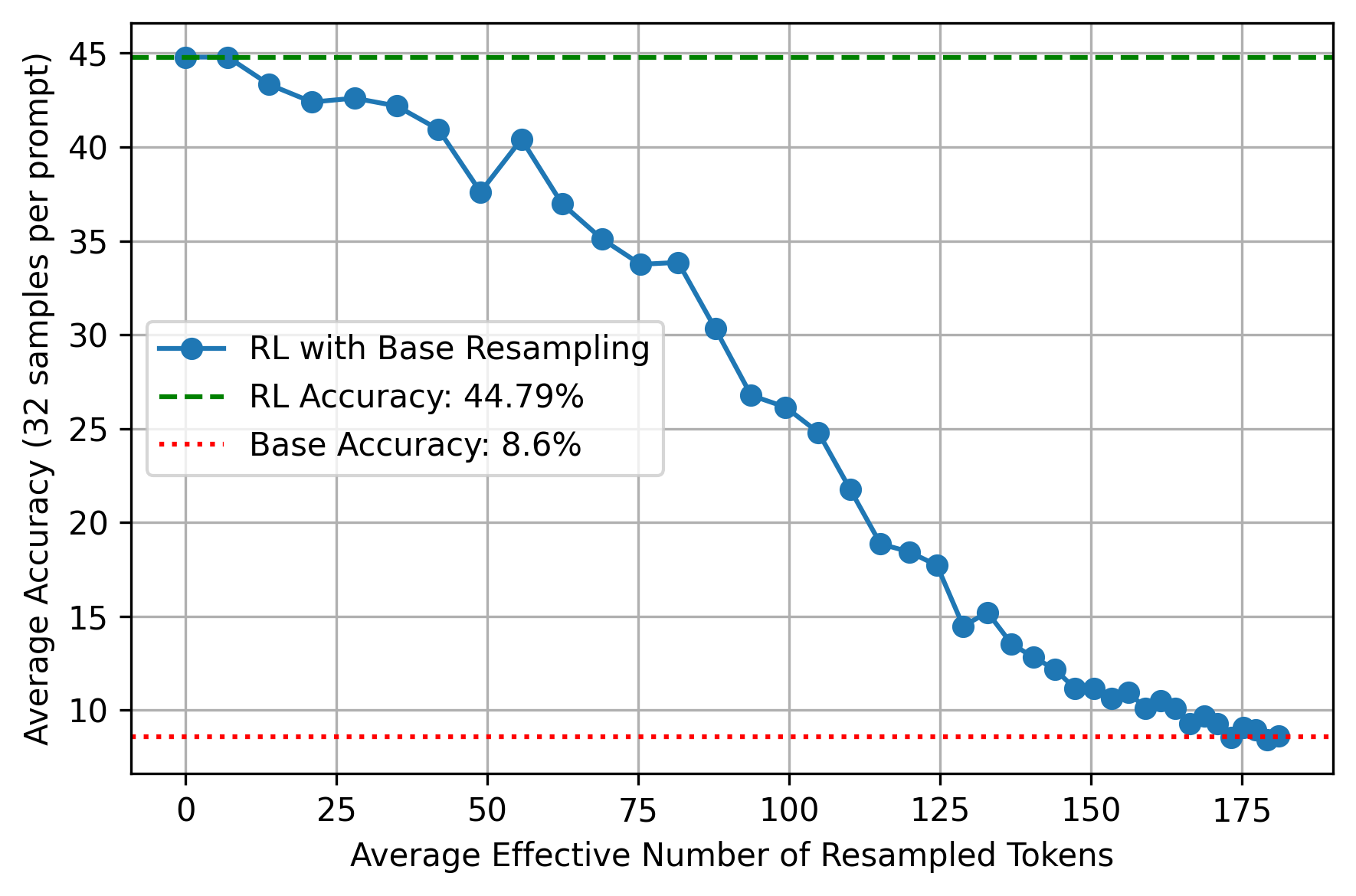}
        \caption{Reverse cross-sampling}
        \label{fig:reverse_resample_dapo}
    \end{subfigure}
    \caption{Cross-sampling results (DAPO on AIME 2024): injecting RL tokens into base generations progressively recovers RL accuracy, while reverting RL tokens with base tokens causes near-monotonic degradation toward base performance.}
    \label{fig:resample_dapo}
\end{figure}

\begin{figure}[!htbp]
    \centering
    \begin{subfigure}{0.48\textwidth}
        \includegraphics[width=\linewidth]{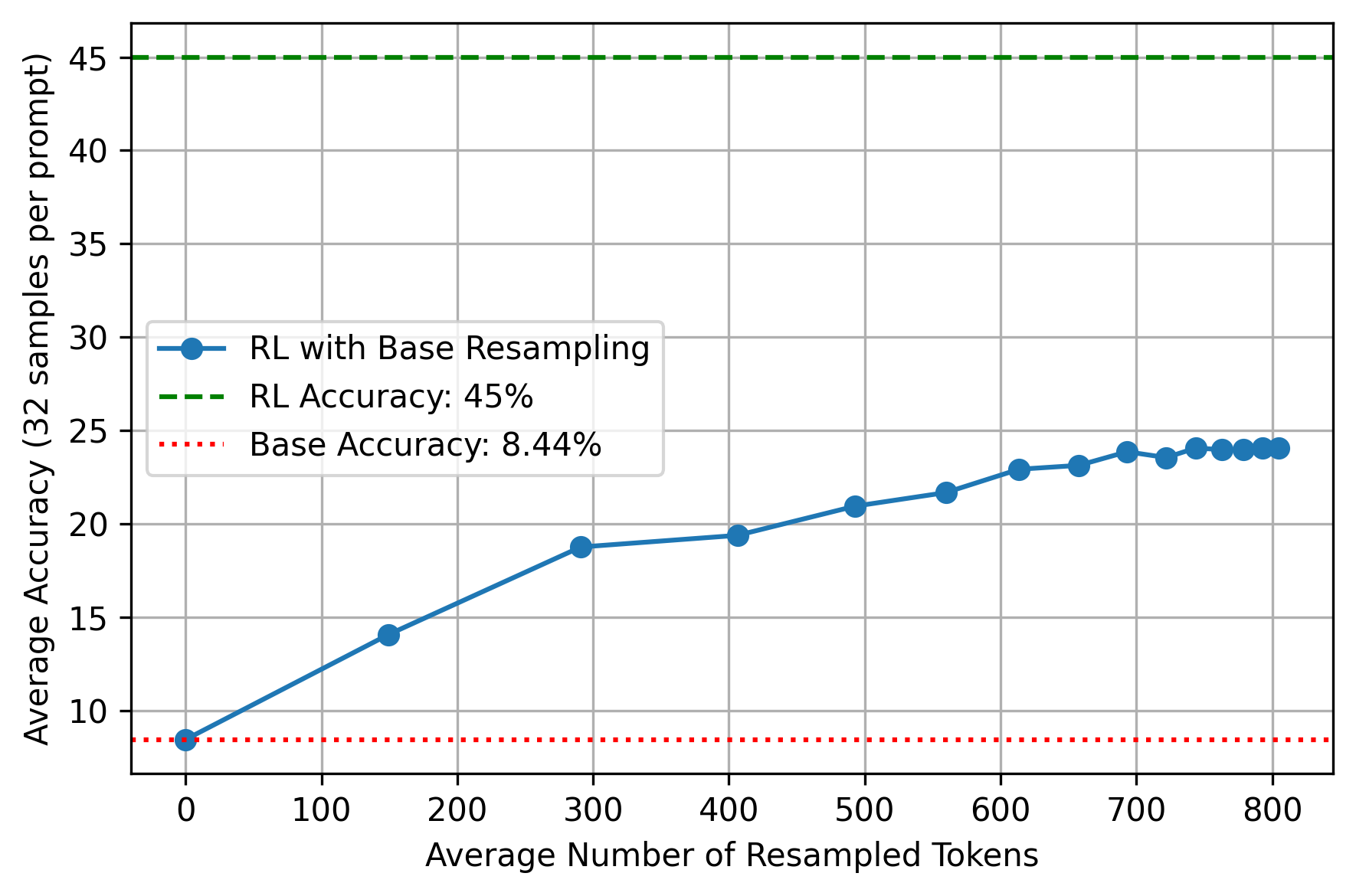}
        \caption{Random baseline}
        \label{fig:random_baseline_aime24}
    \end{subfigure}
    \hfill
    \begin{subfigure}{0.48\textwidth}
        \includegraphics[width=\linewidth]{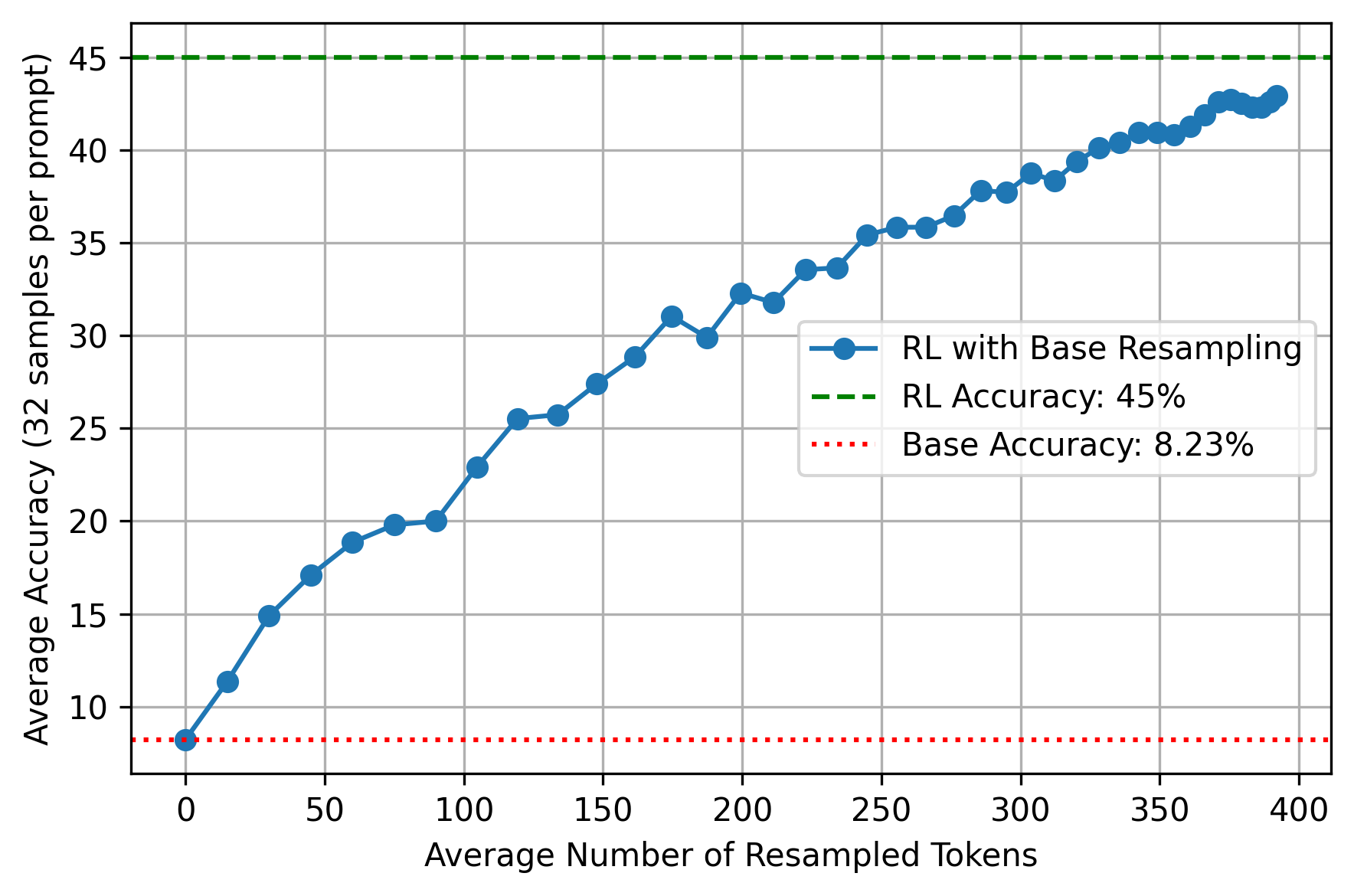}
        \caption{DAPO}
        \label{fig:dapo_avg_token_aime24}
    \end{subfigure}
    \caption{Comparison of random baseline and DAPO cross-sampling on AIME 2024: average number of tokens (including identity swaps) replaced versus accuracy. The random baseline shows slow performance improvement, demonstrating that targeted RL token selection is critical for performance gains. Performing random replacement may skip critical positions for the reasoning trajectories.}
    \label{fig:random_baseline_dapo_aime24}
\end{figure}

\begin{figure}[!htbp]
    \centering
    \begin{subfigure}{0.48\textwidth}
        \includegraphics[width=\linewidth]{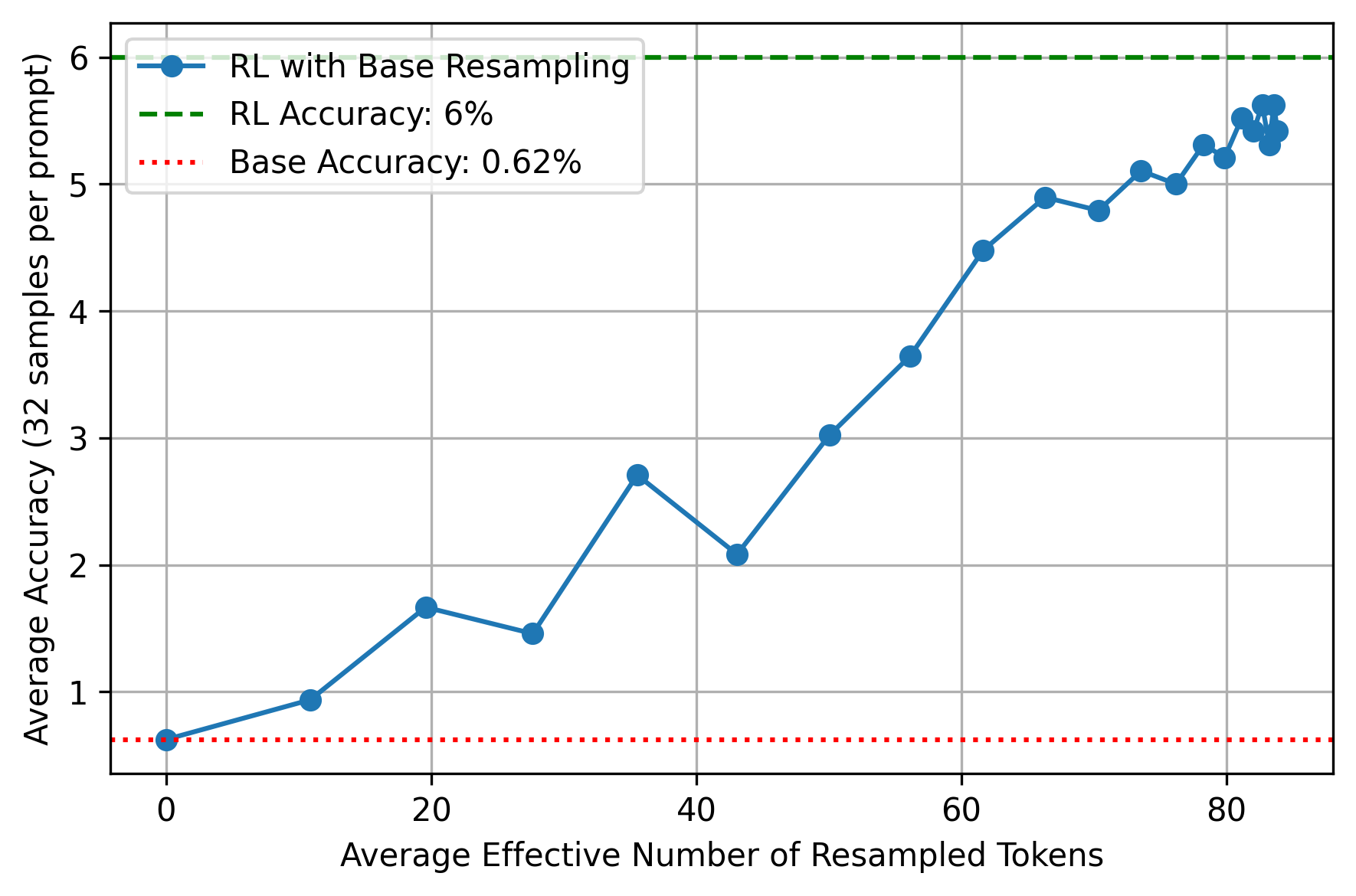}
        \caption{Forward cross-sampling}
        \label{fig:forward_resample_mistral_aime24}
    \end{subfigure}
    \hfill
    \begin{subfigure}{0.48\textwidth}
        \includegraphics[width=\linewidth]{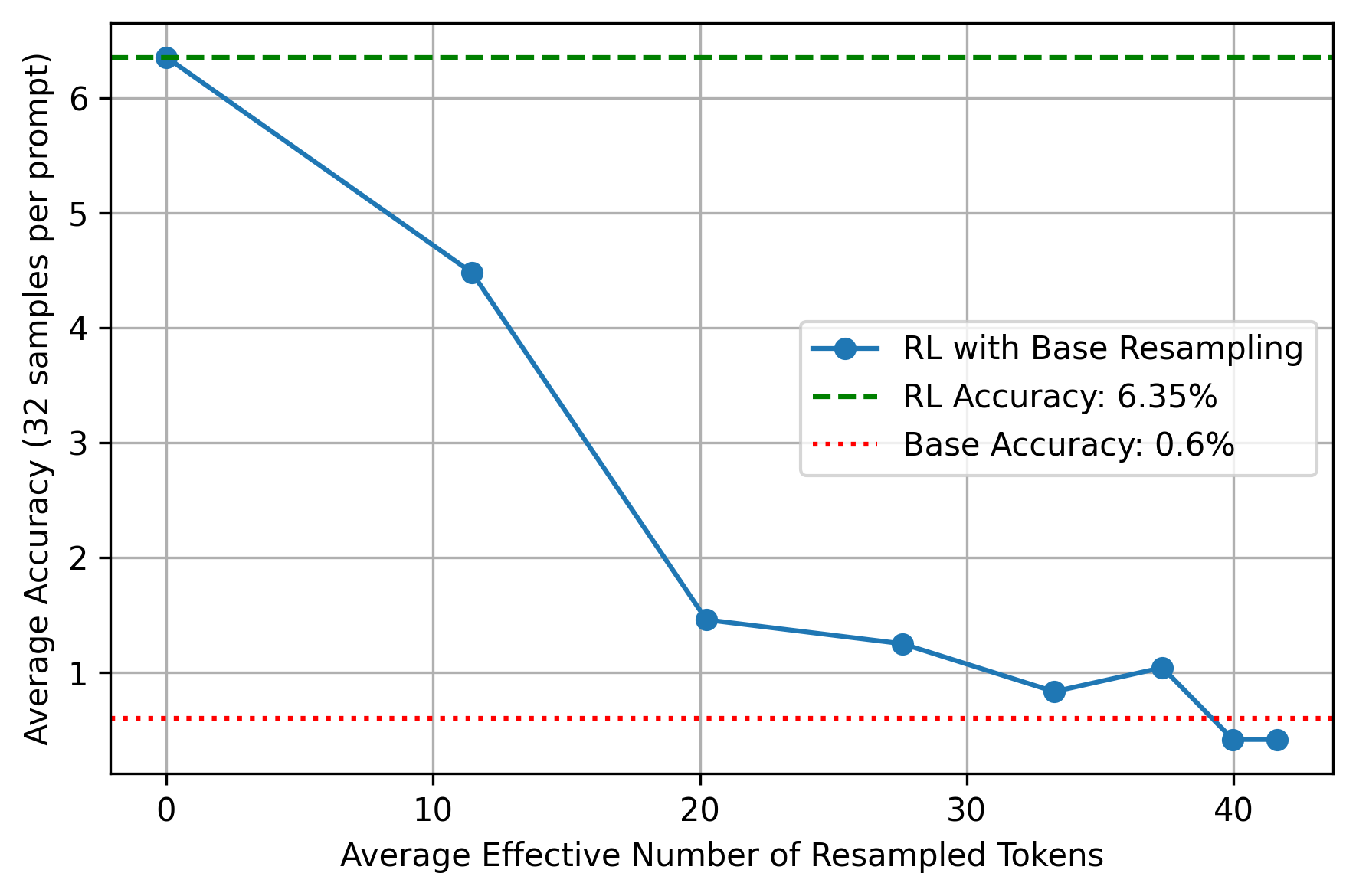}
        \caption{Reverse cross-sampling}
        \label{fig:reverse_resample_mistral_aime24}
    \end{subfigure}
    \caption{Cross-sampling results (Mistral-Small-24B SimpleRL on AIME 2024): injecting RL tokens into base generations progressively recovers RL accuracy, while reverting RL tokens with base tokens causes near-monotonic degradation toward base performance.}
    \label{fig:resample_mistral_aime24}
\end{figure}
\FloatBarrier

\subsection{Additional Results}

For completeness, Figure~\ref{fig:qwen3_32b_dapo_run} shows a training run on Qwen3-32B-Base \citep{yang2025qwen3technicalreport} using the standard DAPO recipe and dataset. We observe that AIME24 performance appears to plateau between approximately steps 80 and 180 (around mean@32 of 57–60\%), before improving steadily beyond step 180 to exceed 70\%. This delayed improvement suggests that additional performance gains may manifest only after extended training.

Given the substantial computational cost of such runs (exceeding 140{,}000 GPU hours to reach roughly step 500 for a single training run), it is plausible that shorter training horizons could lead to prematurely terminated runs and consequently undertrained baselines in prior work. Due to these resource constraints, we do not investigate this phenomenon further here, but we highlight it as a potentially important consideration for future work.

\begin{figure}[htbp]
    \centering
    \includegraphics[width=0.7\linewidth]{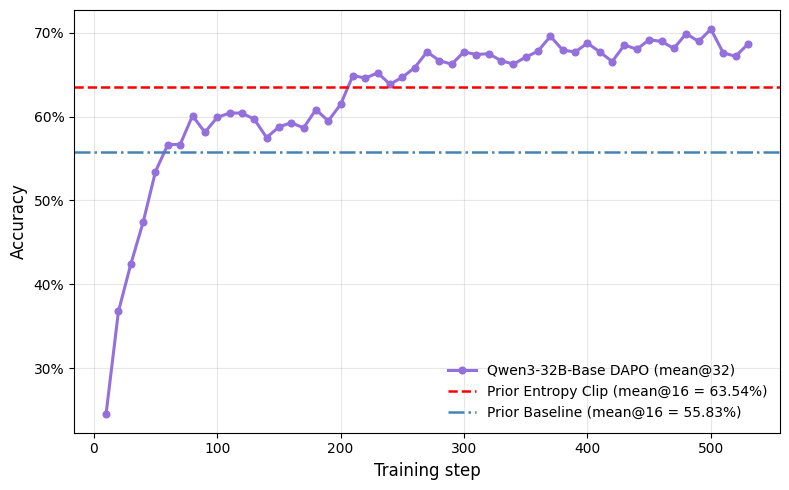}
    \caption{AIME24 performance of Qwen3-32B-Base trained with DAPO.}
    \label{fig:qwen3_32b_dapo_run}
\end{figure}

\end{document}